%% file: main.tex
\documentclass[10pt]{article}

\usepackage[round, authoryear]{natbib}
\usepackage[utf8]{inputenc}
\usepackage[T1]{fontenc}
\usepackage{lmodern}
\input{commands}

\input{defs.tex}

\newcommand{\colppat}{\textcolor{violet}{PPAT}}
\newcommand{\ase}{\textcolor{NavyBlue}{ASE}}
\newcommand{\random}{\textcolor{gray}{Random}}
\newcommand{\colure}{\textcolor{orange}{LURE}}

\DeclareRobustCommand{\solidline}[1]{%
  \tikz[baseline=-0.6ex]{\draw[#1, line width=1.2pt] (0,0) -- (0.55,0);}%
}
\DeclareRobustCommand{\dashedline}[1]{%
  \tikz[baseline=-0.6ex]{\draw[#1, line width=1.2pt, dashed] (0,0) -- (0.55,0);}%
}

\title{Prediction--Powered Active Testing}
\author[1]{Kianoosh Ashouritaklimi}
\author[1]{Valentin Kilian}
\author[2]{Daolang Huang}
\author[1]{Tom Rainforth}
\author[1]{François Caron}
\affiliation[1]{Department of Statistics, University of Oxford}
\affiliation[2]{Department of Computer Science, Aalto University.\\
Correspondence to: \texttt{ashouritaklimi@stats.ox.ac.uk}}

\begin{document}

\begin{abstract}
\looseness=-1
Active testing provides a label--efficient approach to risk estimation by adaptively selecting which test points should be labelled. However, existing estimators fail to exploit the informative predictions of powerful black--box models, even though such predictions are increasingly available in settings where labels remain expensive.
To address this, we propose \textbf{Prediction--Powered Active Testing (PPAT)}, a novel label--efficient risk estimation framework that combines the unbiased LURE estimator \citep{farquhar2021statistical} with a prediction--powered control variate. Rather than using proxy predictions as biased pseudo--labels, PPAT uses them to residualise the loss, preserving unbiasedness while reducing variance. 
Beyond the estimator itself, PPAT also changes which points should be acquired: we derive oracle and practical surrogate--based acquisition rules tailored to reducing the variance of our estimator. 
Moreover, we establish asymptotic normality for PPAT, yielding asymptotically valid confidence intervals and thus a principled estimate of the uncertainty around our estimates. Across tabular regression and image--classification tasks, PPAT outperforms existing methods in risk estimation, while its confidence intervals attain the target coverage with substantially fewer labels and smaller widths.

\end{abstract}

\maketitle

\section{Introduction}
\label{sec:intro}
\input{introduction}

\section{Background}
\label{sec:background}
\input{background/background}

\newpage
\section{Prediction--Powered Active Testing}
\label{sec:pp-at}
\input{method/ppat}

\section{Asymptotic Properties and Approximate Confidence Intervals}
\label{sec:asymp_props_approx_cis}
\input{theory/asymptotics.tex}

\section{Experiments}
\label{sec:experiments}
\input{experiments/experiments}

\section{Related Work}
\label{sec:related_work_main}
\input{related_work_main}

\section{Discussion}
\label{sec:discussion}
\input{future_work}

\newpage
\section*{Acknowledgments}
KA is supported by the EPSRC Centre for
Doctoral Training in Modern Statistics and Statistical Machine Learning (EP/S023151/1). VK is supported by the Clarendon Funds Scholarship. DH is supported by the Research Council of Finland (Flagship programme: Finnish Center for Artificial Intelligence FCAI, 359207). TR is supported by the EPSRC grant EP/Y037200/1. The authors are grateful to Stefano Cortinovis for helpful discussions.

\clearpage
\bibliographystyle{plainnat} 
\bibliography{references}

\clearpage
\appendix
\startcontents[appendix]

\begingroup
\hypersetup{linkcolor=LabPrimary}
\section*{Appendix Contents}
\printcontents[appendix]{}{1}{\setcounter{tocdepth}{2}}
\endgroup
\input{appendix/appendix}

\end{document}

%% file: commands.tex
\PassOptionsToPackage{dvipsnames,table}{xcolor}

\usepackage{metatemp}
\usepackage{microtype}
\usepackage{parskip}
\usepackage{titletoc}

\usepackage{booktabs}
\usepackage{caption}
\usepackage{subcaption}
\usepackage{multirow}
\usepackage{wrapfig}
\usepackage{makecell}
\usepackage{adjustbox}
\usepackage{fullpage}
\usepackage{ragged2e}
\usepackage{tabularx}
\usepackage{array}
\usepackage{xurl}
\usepackage{graphicx}
\usepackage{nicefrac}
\usepackage{tikz}

\usepackage{amsmath}
\usepackage{amssymb}
\usepackage{amsfonts}
\usepackage{mathtools}
\usepackage{amsthm}
\usepackage{thmtools}
\usepackage{thm-restate}
\usepackage{cancel}
\usepackage[inline]{enumitem}

\setlist[enumerate,1]{leftmargin=2em, itemsep=1em, parsep=0pt}
\setlist[itemize,1]{leftmargin=2em}

\usepackage[capitalize,noabbrev]{cleveref}
\crefformat{section}{Section~#2#1#3}
\crefformat{subsection}{Section~#2#1#3}
\crefformat{subsubsection}{Section~#2#1#3}
\crefformat{equation}{(#2#1#3)}
\crefrangeformat{equation}{(#3#1#4) to (#5#2#6)}
\crefmultiformat{equation}{(#2#1#3)}{ and (#2#1#3)}{, (#2#1#3)}{ and (#2#1#3)}
\crefname{appendix}{Appendix}{Appendices}
\Crefname{appendix}{Appendix}{Appendices}
\crefformat{appendix}{Appendix~#2#1#3}
\AddToHook{cmd/appendix/before}{%
  \crefalias{section}{appendix}%
  \crefalias{subsection}{appendix}%
  \crefalias{subsubsection}{appendix}%
}
\crefname{table}{Table}{Tables}
\crefname{figure}{Figure}{Figures}

\definecolor{equationgold}{HTML}{B8860B}
\definecolor{d62728}{RGB}{214,39,40}
\definecolor{1f77b4}{RGB}{31,119,180}
\definecolor{9467bd}{RGB}{148,103,189}
\definecolor{FigNY}{HTML}{1a1a3e}
\definecolor{FigSD}{HTML}{b0a8d0}
\usepackage{colortbl}
\definecolor{pinegreen}{rgb}{0.0, 0.47, 0.44}
\definecolor{cornellred}{rgb}{0.7, 0.11, 0.11}
\definecolor{cadmiumgreen}{rgb}{0.0, 0.42, 0.24}
\definecolor{spirodiscoball}{rgb}{0.06, 0.75, 0.99}
\definecolor{blizzardblue}{rgb}{0.73, 0.96, 0.99}
\definecolor{aliceblue}{rgb}{0.91, 0.94, 0.97}
\definecolor{darkblue}{RGB}{232, 224, 240}
\definecolor{Red7}{rgb}{0.941, 0.243, 0.243}
\definecolor{Green7}{RGB}{55, 178, 77}

\makeatletter
\def\thm@space@setup{%
  \thm@preskip=0.8em plus 0.2em minus 0.2em
  \thm@postskip=0.8em plus 0.2em minus 0.2em
}
\makeatother

\newcolumntype{C}[1]{>{\centering\arraybackslash}p{#1}}
\newcolumntype{L}[1]{>{\raggedright\arraybackslash}p{#1}}

\providecommand{\answerTODO}[1][]{\textcolor{red}{\bf [TODO]}}
\providecommand{\justificationTODO}[1][]{\textcolor{red}{\bf [TODO]}}

%% file: defs.tex
\usepackage[inline]{enumitem}
\newlist{enuminline}{enumerate*}{1}
\setlist[enuminline]{label=(\roman*)}

\newcommand{\E}{\mathbb E}

\newcommand{\bbR}{\mathbb R}
\newcommand{\bbE}{\mathbb E}

\newcommand{\calL}{\mathcal L}
\newcommand{\calD}{\mathcal D}
\newcommand{\calX}{\mathcal X}
\newcommand{\calY}{\mathcal Y}

\newcommand{\bx}{\mathbf x}

\newcommand{\Rlure}{\widehat{R}_{\mathrm{LURE}}}
\newcommand{\Rppat}{\widehat{R}_\mathrm{PPAT}}

\newcommand{\lure}{\mathrm{LURE}}
\newcommand{\ppat}{\mathrm{PPAT}}

\newcommand{\Clure}{\widehat{C}_{\mathrm{LURE}}}

\theoremstyle{plain}
\newtheorem{thm}{Theorem}[section]

\newtheorem{prop}[thm]{Proposition}

\newtheorem{lem}[thm]{Lemma}

\newtheorem{cor}[thm]{Corollary}

\theoremstyle{definition}

\newtheorem{definition}[thm]{Definition}

\newtheorem{rmrk}[thm]{Remark}

\newcommand{\ppatboxenvironment}[1]{%
  \tcolorboxenvironment{#1}{
    metatemp-common,
    colback=LabPrimary!4,
    left=12pt,
    right=12pt,
    top=8pt,
    bottom=8pt
  }%
}
\ppatboxenvironment{thm}
\ppatboxenvironment{theorem}
\ppatboxenvironment{prop}
\ppatboxenvironment{proposition}
\ppatboxenvironment{lem}
\ppatboxenvironment{lemma}
\ppatboxenvironment{cor}
\ppatboxenvironment{corollary}
\ppatboxenvironment{defn}
\ppatboxenvironment{definition}
\ppatboxenvironment{assump}
\ppatboxenvironment{assumption}
\ppatboxenvironment{rmrk}
\ppatboxenvironment{remark}
\ppatboxenvironment{restatable}

\crefname{thm}{Theorem}{Theorems}
\Crefname{thm}{Theorem}{Theorems}
\crefname{prop}{Proposition}{Propositions}
\Crefname{prop}{Proposition}{Propositions}
\crefname{lem}{Lemma}{Lemmas}
\Crefname{lem}{Lemma}{Lemmas}
\crefname{cor}{Corollary}{Corollaries}
\Crefname{cor}{Corollary}{Corollaries}
\crefname{defn}{Definition}{Definitions}
\Crefname{defn}{Definition}{Definitions}
\crefname{assump}{Assumption}{Assumptions}
\Crefname{assump}{Assumption}{Assumptions}
\crefname{rmrk}{Remark}{Remarks}
\Crefname{rmrk}{Remark}{Remarks}
\crefname{theorem}{Theorem}{Theorems}
\crefname{proposition}{Proposition}{Propositions}
\Crefname{proposition}{Proposition}{Propositions}
\crefname{lemma}{Lemma}{Lemmas}
\crefname{corollary}{Corollary}{Corollaries}
\crefname{definition}{Definition}{Definitions}
\crefname{assumption}{Assumption}{Assumptions}
\crefname{remark}{Remark}{Remarks}

\newcommand{\cX}{\mathcal{X}}

\newcommand{\cF}{\mathcal{F}}
\newcommand{\Pbb}{\mathbb{P}}

\definecolor{darkgreen}{RGB}{0, 100, 0} %

\newcommand{\Var}{\mathrm{Var}}

\newcommand{\Cov}{\mathrm{Cov}}

\usepackage{graphicx}

%% file: introduction.tex
\looseness=-1
Active learning \citep{settles2009active, al_model_data_survey, selective_sampling_networks, queries_concept_learn} reduces the cost of training, but in many applications the cost of \emph{evaluation} is just as important. In many domains, large unlabelled test pools are often easy to obtain, while ground--truth labels remain expensive and require expert annotation. Active testing \citep{active_risk_est, active_est_f_measures, sample_efficient_model_evaluation, active_evaluation_classifier_large_datasets, yu2023actively, kossen2021active, kossen2022active, berrada2025scaling} addresses this bottleneck by adaptively sampling which test points to label while correcting for the bias introduced by adaptive sampling. Crucially, adaptive sampling can achieve lower variance than random sampling under suitable acquisition strategies \citep{farquhar2021statistical}, while remaining unbiased.

\looseness=-1
In many settings, such as image classification \citep{radford2021learning, li2023blip, jia2021scaling, zhai2023sigmoid}, language modelling \citep{brown2020language, wang2021want, touvron2023llama, team2023gemini}, and biomolecular structure prediction \citep{Jumper2021HighlyAP, Abramson2024AccurateSP, corso2022diffdock, prat2025sigmadock}, we also have access to a powerful black--box predictor that can produce cheap predictions on the entire test pool. A natural question here is: \emph{can we leverage these predictions to produce more accurate, label--efficient estimates of our risk?} One approach would be to replace the true, unknown labels with the predictions of the black--box model, similar to the work of \cite{kossen2022active}. This, however, generally results in biased estimates of the risk. While it can yield accurate estimates of the risk when the labelling budget is very small and the black--box predictor is already highly accurate, the resulting estimator is generally biased, with the bias becoming increasingly problematic as the labelling budget grows.

\looseness=-1
To address this, we propose \textbf{Prediction--Powered Active Testing (PPAT)}, which combines the unbiased importance--weighted LURE estimator \citep{farquhar2021statistical} with a prediction--powered control variate built from cheap black--box predictions \citep{angelopoulos2023prediction}. Rather than using these predictions as biased pseudo--labels, PPAT uses them to \emph{residualise} the loss: it subtracts, from each true loss, the loss induced by the black--box prediction after centring it by its average over the test pool, thereby preserving the unbiasedness of active testing. Our approach introduces a tuning parameter $\lambda$ that controls the strength of this correction and recovers LURE as the special case $\lambda=0$, making PPAT a principled extension of the LURE estimator. We show that this control--variate formulation can reduce variance relative to LURE and discuss practical choices of $\lambda$.

\looseness=-1
Beyond the estimator itself, PPAT also changes what should be sampled. Rather than reusing active-testing proposals designed for the raw loss, we derive proposals tailored to the residualised objective induced by our estimator. This leads to an oracle proposal and a practical surrogate acquisition rule that samples points expected to be most informative for reducing the variance of the resulting risk estimate. We also establish asymptotic normality for the PPAT estimator, which allows us to construct asymptotically valid confidence intervals. Empirically, we evaluate PPAT extensively on tabular regression and image-classification tasks. We find that PPAT outperforms existing methods in risk estimation, while also attaining the target coverage level with substantially fewer labels and narrower confidence intervals.

%% file: background/background.tex
\looseness=-1
Throughout this paper, we reserve capital letters for random variables and lowercase letters for observations. For readability, we state the results in the main body using the simplified notation introduced below; their formal versions are given in \S\ref{app:mathematical_results}.

Let $\mathcal{D} = \{\mathbf{x}_1, \dots, \mathbf{x}_N\}$ be an unlabelled test pool of size $N$, where each input $\mathbf{x}_i\in \mathcal{X}$ has a true, unknown label $y_i \in \mathcal{Y}$. We treat both the inputs and labels as fixed quantities. Furthermore, let $f: \mathcal{X} \to \mathcal{Y}$ be a fixed function we wish to evaluate for a given loss function $\mathcal{L}: \mathcal{Y} \times \mathcal{Y} \to \mathbb{R}$. We define the loss for each test data point as $\ell_i := \mathcal{L}(f(\mathbf{x}_i), y_i)$ for $i = 1, \dots, N$, and the sequence of these losses as $\ell_{1:N} := (\ell_1, \dots, \ell_N)$. For a fixed labelling budget $M\ll N$, our objective is to estimate the test pool risk
\begin{equation}
R=\frac{1}{N} \sum_{i=1}^N \ell_i
\label{eq:testrisk}
\end{equation}
based on the inputs $\calD$ and on $M$ queried labels $y_{I_1},\ldots,y_{I_M}$, with indices $I_1,\ldots,I_M$. A naive approach is to sample the indices randomly without replacement from $[N]:=\{1,\ldots,N\}$ and use the unweighted estimator
\begin{equation}
\label{eqn:unweighted_estimator}
\widehat R_{\text{unw}}=\frac{1}{M}\sum_{m=1}^M \ell_{I_m}.  
\end{equation}
This estimator is unbiased but typically suffers from high variance \citep{farquhar2021statistical, kossen2021active}.

\subsection{Active Testing and the LURE Estimator}
\label{sub:active_testing_lure_estimator}
\looseness=-1
Active testing (AT) \citep{kossen2021active} aims to estimate the test pool risk $R$ more efficiently by sequentially selecting the most useful points to label. More formally, for each acquisition round $m$, let $S_m \subseteq \{1,\dots,N\}$ denote the remaining indices to be sampled and $\mathcal{F}_{m-1}$ the $\sigma$-algebra generated by all the information available before round $m$ (e.g.\ $D_{\mathrm{}}$, past queried indices and labels). AT selects the $m$th index to query by sampling from a proposal $Q_m$:
\begin{equation}
Q_m(\cdot) := \Pbb(I_m=\cdot\mid \mathcal{F}_{m-1})
\quad\text{supported on }S_m.    
\end{equation}

To estimate the risk, AT uses the LURE (\emph{Levelled Unbiased Risk Estimator}) estimator of \cite{farquhar2021statistical}:
\begin{equation}
\label{eq:lureestimator}
\widehat R_\mathrm{LURE}=\frac{1}{M}\sum_{m=1}^M V_m \ell_{I_m},
\end{equation}
where $I_m \mid I_{1:m-1}\sim Q_m$ and 
\begin{align}
\label{eq:lure-weight}
V_m &= 1 + \frac{N-M}{N-m}\left(\frac{1}{(N-m+1)\,Q_m(I_m)} - 1\right).
\end{align}

This is an \emph{unbiased} estimate of $R$ and, when $\mathcal{L}$ is nonnegative, has strictly \emph{lower variance} than random sampling without replacement under the oracle proposal $Q_m^\star(i) \propto \ell_i$ \citep[Thm.~5]{farquhar2021statistical}. For general real-valued losses, the variance is instead minimised greedily at each step using the oracle \emph{myopic} proposal $Q^\star_m(i) \propto |\ell_i|$ (see \S\ref{sec:app:myopic_oracle_proposal} for further details). In practise, the true labels are not known, and so $Q^\star_m$ is approximated\footnote{Note that when the true loss is unknown, the oracle $Q^\star_m \propto |\ell_i|$ is better approximated by the root‑mean‑square score \eqref{eq:oracle-q-expected-loss} rather than by an expected absolute value; see \S\ref{app:surrogate_score_choice} for more details.} with 
\begin{equation}
\label{eq:oracle-q-expected-loss}
Q^{\mathrm{AT}}_m(i) \;\propto\; \sqrt{\E_{\pi_m(Y\mid \mathbf{x}_i)}[\,\mathcal{L}(f(\mathbf{x}_i),Y)^2\,]},
\end{equation}
where 
$\pi_m(Y\mid \mathbf{x})$ is a \emph{surrogate model} used to approximate $p(Y\mid \mathbf{x})$ at round $m$ \citep{kossen2021active}.


\input{background/active_testing_with_ase}


\subsection{Control--Variate Estimator}
\label{sec:controlvariate}
\looseness=-1
We conclude this background section with a short review of control variate estimators. Let $\widehat R $ be an unbiased estimator of $R$. Let $\widehat C$ be a zero--mean random variable, called the \emph{control variable} \citep{glasserman2003monte}. For $\lambda\in\bbR$, the control variate estimator, $\widehat R_{\mathrm{cv}}$, of $R$ is defined as 
\begin{equation}
\label{eqn:cv_estimator}
\widehat R_{\mathrm{cv}} = \widehat R - \lambda \widehat C.
\end{equation}
\looseness=-1
The centred variable $\widehat C$ acts as a control--variate correction to the estimator $\widehat R$. By design, $\widehat R_{\textrm{cv}}$ is unbiased if $\widehat R$ is, with variance
$\Var(\widehat R_{\textrm{cv}})=\Var(\widehat R)-2\lambda\Cov(\widehat R,\widehat C)+\lambda^2\Var(\widehat C).$ Moreover, $\widehat R_{\textrm{cv}}$ has lower variance than $\widehat R$ whenever $\lambda\in(\min(0,2\lambda^*),\max(0,2\lambda^*))$, where $\lambda^*=\Cov(\widehat R,\widehat C)/\Var(\widehat C)$ results in the largest variance reduction. At $\lambda^\star$, $\Var(\widehat R_{\textrm{cv}})=(1-\rho^2)\Var(\widehat R)$, where $\rho$ is the correlation between $\widehat R$ and $\widehat C$. Thus, a larger absolute correlation implies a larger variance reduction.

%% file: background/active_testing_with_ase.tex
\subsection{Active Surrogate Estimators}
\label{sub:ases_background}

\emph{Active Surrogate Estimators} (ASEs, \cite{kossen2022active}) provide an alternative approach to AT that replaces importance--weighted Monte Carlo estimation with an interpolation--based surrogate estimate of the loss. Rather than estimating the risk directly from the queried losses, ASEs learn a surrogate model for the unknown test--time label distribution and use it to impute losses across the full test pool. Concretely, after $M$ acquisitions, the estimator is
\begin{equation*}
\widehat{R}_{\mathrm{ASE}}
=
\frac{1}{N}\sum_{i=1}^{N}
\mathbb{E}_{Y \sim \pi_M(\cdot \mid \mathbf{x}_i)}
\!\left[ \mathcal{L}\!\left(f(\mathbf{x}_i), Y\right)\right].    
\end{equation*}

In contrast to LURE--based AT, the acquisition in ASE is designed to improve the surrogate itself rather than to produce an unbiased importance--weighted estimator. To this end, \cite{kossen2022active} propose the \textit{Expected Weighted Disagreement (XWED)} acquisition function, which prefers points with high epistemic uncertainty under $\pi_m$ that are also expected to contribute strongly to the final risk estimate (see \S\ref{app:extended_background_ase} for additional details). 

A key advantage of ASEs is that they can make more effective use of surrogate generalisation and allow for more flexible acquisition strategies, including deterministic ones. However, they do not retain the finite--sample unbiasedness of LURE: surrogate misspecification and finite--sample estimation error enter the final estimator directly. Consequently, their performance depends strongly on the quality of the learned surrogate and they suffer from weaker theoretical guarantees than LURE--based AT.

%% file: method/ppat.tex
\looseness=-1
We now describe our approach, \textbf{Prediction--Powered Active Testing (PPAT)}, which leverages a cheap proxy labeller $g$ to reduce the variance of the LURE estimator in a similar spirit to \cite{angelopoulos2023prediction, angelopoulos2023ppi++}. 

Assume we can evaluate a cheap proxy labeller $g:\cX\to{\mathcal{Y}}$ on all test inputs, yielding proxy labels
$\widetilde y_i = g(\mathbf{x}_i)$. 
Define the \emph{proxy loss} $\widetilde\ell_i:=\calL(f(\mathbf{x}_i), \widetilde y_i)$ and the associated proxy size--$N$ test risk $\widetilde R:=\frac{1}{N}\sum_{i=1}^N \widetilde\ell_i$.  
We first note that, for any $\lambda\in \bbR$, the test risk can be written alternatively as
\begin{equation*}
R=\frac{1}{N}\sum_{i=1}^N (\ell_i-\lambda(\widetilde\ell_i-\widetilde R)),     
\end{equation*}
that is, an average of the \emph{residualised losses} $(\ell_i-\lambda(\widetilde\ell_i-\widetilde R))_{i=1}^N$. 
We propose to apply the LURE weighting scheme to this residualised representation of $R$, yielding our PPAT estimator:
\begin{equation}
\Rppat(\lambda)=\frac{1}{M} \sum_{m=1}^M V_m(\ell_{I_m} - \lambda (\widetilde\ell_{I_m}-\widetilde R)).\label{eq:ppatestimator}
\end{equation}
\looseness=-1
Our estimator recovers the standard LURE estimator \eqref{eq:lureestimator} as a special case when $\lambda=0$. Moreover, noting that
\begin{equation*}
\Rppat(\lambda) = \Rlure - \lambda\Clure,    
\end{equation*}
\looseness=-1
the PPAT estimator can be interpreted as a control--variate correction of $\Rlure$, with the control variate being
%
%
\begin{equation*}
\Clure=\frac{1}{M}\sum_{m=1}^M V_m  (\widetilde\ell_{I_m}-\widetilde R),
\end{equation*}
\looseness=-1
where $\Clure$ is itself a zero--mean LURE-style estimator\footnote{Indeed, we have that $\bbE[\Clure]=\frac{1}{N}\sum_{i=1}^N (\tilde\ell_i - \tilde R)=0$.}. This control--variate view suggests that suitable choices of $\lambda$ should reduce variance relative to LURE. We make this more precise in the next section, where we study the finite--sample properties of the PPAT estimator.


\subsection{Finite--Sample Properties of PPAT}
\label{sub:finite_sample_properties}
\looseness=-1
The results below hold for any $1\leq M< N$, and any choice of proposals $Q_m$. We first note that our estimator naturally inherits the unbiasedness of LURE.
\begin{restatable}[Unbiasedness]{prop}{unbiasednessppat}
\label{prop:cv-unbiased}
For any fixed $\lambda\in\mathbb{R}$, we have $\E\!\left[\widehat R_{\mathrm{PPAT}}(\lambda)\right]= R$.
\end{restatable}
\emph{Proof.} See \S\ref{sec:app:lureunbiased}.    

Secondly, as $\widehat R_{\mathrm{PPAT}}(\lambda)$ is a control--variate estimator, it follows that it achieves lower variance than $\widehat R_{\mathrm{LURE}}$ for a range of $\lambda$ values.
\begin{restatable}[Variance of $\widehat R_{\mathrm{PPAT}}$]{prop}{varianceppat} 
Let $\lambda \in \bbR$. We have
\label{prop:var-quadratic}
\begin{equation}
\label{eq:var-quadratic}
\Var\!\left(\widehat R_{\mathrm{PPAT}}(\lambda)\right)
=
\Var(\Rlure)
-2\lambda\,\Cov(\Rlure,\Clure)
+\lambda^2\Var(\Clure).
\end{equation}
Thus, $\Var\!\left(\widehat R_{\mathrm{PPAT}}(\lambda)\right) < \Var(\widehat R_\mathrm{LURE})$ whenever $\lambda \in \left(\min \left\{0,2 \lambda^{\star}\right\}, \max \left\{0,2 \lambda^{\star}\right\}\right)$ and $\Var(\Clure) > 0$, where 
\begin{equation}
\label{eq:lambda-star}
\lambda^\star
= \frac{\Cov(\widehat R_\mathrm{LURE},\Clure)}{\Var(\Clure)}.
\end{equation}

In particular, $\lambda^\star$ minimises the variance of $\widehat R_{\mathrm{PPAT}}(\lambda)$.
\end{restatable}
\emph{Proof.} See \S\ref{sec:app:lureunbiased}.

Note that under $\lambda^*$, the variance of the PPAT estimator can also be written as $\Var(\widehat R_\mathrm{LURE})(1-{\rho}^2)$, where $\rho$ is the correlation between $\widehat{R}_\mathrm{LURE}$ and $\Clure$. Here, a larger correlation between these two terms results in a larger reduction in variance relative to LURE.

While Prop. \ref{prop:var-quadratic} shows that PPAT can achieve lower variance than LURE, this depends critically on our choice of $\lambda$ and holds only when both LURE and PPAT use the same proposals. Moreover, although it is possible to adopt the same LURE proposal, $Q_m^{AT}$, for our approach and tune only $\lambda$, this will typically be suboptimal as it selects points by their raw loss $\ell_i$, whereas the variance of PPAT is driven by the residual $\ell_i - \lambda(\widetilde{\ell}_i - \widetilde{R})$. In the following sections, we derive a proposal tailored to our PPAT estimator and discuss practical choices of $\lambda$.


\subsection{PPAT Proposal}
\label{sub:ppat_active_proposal}
While the previous results hold for any valid proposal $Q_m$, we
would ideally like to choose the proposal that minimises the variance of the PPAT
estimator. As the exact finite--horizon variance--minimising proposal is
intractable, we instead consider proposals that minimise the alternative, myopic proxy described below.

Firstly, note that since $\Rppat$ is itself a LURE estimator applied to the residualised losses $\ell_i - \lambda(\widetilde \ell_i - \widetilde R)$, its sampling variance decomposes across rounds in the same way as that of LURE (see \S\ref{sec:app:myopic_oracle_proposal} for further details). In
particular, conditionally on $I_{1:m-1}$, the $m$th contribution to this variance
is proportional to
\begin{equation}
\label{eqn:ppat_myopic_proposal_proxy}
    \mathrm{Var}\!\left(
        \frac{\ell_{I_m}-\lambda(\widetilde \ell_{I_m} - \widetilde R)}
             {N\, Q_m(I_m)}
        \;\middle|\;
        I_{1:m-1}
    \right).
\end{equation}
Thus, at round $m$, a natural myopic oracle is the proposal which minimises this
conditional variance over the remaining pool $S_m$. This oracle admits a simple closed form.
\begin{restatable}[Myopic oracle proposal]{prop}{myopicOracleProposal}
\label{prop:ppat_myopic_oracle}
Fix \(\lambda \in \mathbb R\). For each round \(m\), among all proposals supported
on \(S_m\), the myopic proxy \eqref{eqn:ppat_myopic_proposal_proxy} is minimised by
\begin{equation*}
    Q^\star_{m,\lambda}(i)
    =
    \frac{|\ell_i-\lambda (\widetilde \ell_i - \widetilde R) |}
         {\sum_{j\in S_m}|\ell_j-\lambda (\widetilde \ell_j - \widetilde R) |}.
\end{equation*}
\end{restatable}
\emph{Proof.} See \S\ref{sec:app:myopic_oracle_proposal}.

In practice, this oracle proposal cannot be used, as it depends on the unknown labels \(y_i\) through the true losses \(\ell_i\). 
As in \S\ref{sub:active_testing_lure_estimator}, we approximate the oracle using the surrogate predictive distribution $\pi_m(\cdot\mid \mathbf{x}_i)$, giving the surrogate score\footnote{As in \S\ref{sub:active_testing_lure_estimator}, the root‑mean‑square score \eqref{eqn:ppat_score} -- rather than an expected absolute value -- is the more appropriate surrogate of the oracle here.}
\begin{equation}
\label{eqn:ppat_score}
    a_{m,\lambda}(i)
    :=
    \sqrt{\mathbb E_{Y\sim \pi_m(\cdot\mid \bx_i)}
    \left[
        \left(
        \calL(f(\bx_i),Y)-\lambda (\widetilde \ell_i - \widetilde R)
        \right)^2
    \right]},
\end{equation}
and the PPAT proposal
\begin{equation}
\label{eqn:surrogate_proposal_2}
        Q^{\mathrm{PPAT}}_{m,\lambda}(i)
    =
    \frac{a_{m,\lambda}(i)}
    {\sum_{j\in S_m}a_{m,\lambda}(j)},
    \qquad i\in S_m.
\end{equation}
To ensure strict positivity, this proposal is mixed with a uniform
distribution over \(S_m\) like in \citep{kossen2021active}, replacing \(Q^{\mathrm{PPAT}}_{m,\lambda}(i)\) by
\((1-\varepsilon)\,Q^{\mathrm{PPAT}}_{m,\lambda}(i) + \varepsilon/|S_m|\) for a small
\(\varepsilon\in(0,1)\).



\subsection{Choosing $\lambda$}
\label{sub:lambda_estimation}
\looseness=-1
As mentioned previously, the variance reduction achieved by PPAT depends critically on the choice of $\lambda$. More specifically, Prop.~\ref{prop:var-quadratic} shows that PPAT achieves lower variance than LURE only for $\lambda \in (\min\{0, 2\lambda^\star\}, \max\{0, 2\lambda^\star\})$; a $\lambda$ with the wrong sign, or with too large a magnitude, lies outside this range and instead inflates the variance relative to LURE. Selecting $\lambda$ appropriately is therefore crucial, but is also further complicated by the fact that it affects both the PPAT estimator \emph{and} the PPAT proposal. Indeed, for given proposals $Q_m$, the control--variate coefficient $\lambda^\star$ in \eqref{eq:lambda-star} depends on the sampling law they induce, which cannot be estimated from a single active--testing run. Moreover, $\lambda$ also parameterises our proposal, $Q^{\mathrm{PPAT}}_{m,\lambda}$, and so shapes the sampling law on which $\lambda^\star$ itself depends. 

These considerations motivate two practical strategies for choosing $\lambda$: a \emph{heuristic} choice of fixed values, and an \emph{online} estimate updated from the labels acquired during testing.

\paragraph{Heuristic, fixed choices of $\lambda$.} 
In a similar spirit to \cite{angelopoulos2023prediction}, we propose the fixed values $\lambda = 1$ and $\lambda = \tfrac{1}{2}$, where we expect $\lambda = 1$ to be most effective when the proxy is strongly predictive of the true loss, and use $\lambda = \tfrac{1}{2}$ as a more conservative choice. To see why these are sensible, first note that $\lambda^\star$ can be written as $\lambda^\star = \rho\,\sigma_R/\sigma_C$, where $\rho$ is the correlation between {{$\widehat{R}_{\text{LURE}}$}} and {{$\widehat{C}_{\text{LURE}}$}} and $\sigma_R, \sigma_C$ are their standard deviations. 
When the proxy losses vary on a scale comparable to the true losses, we have $\sigma_R \approx \sigma_C$, so that $\lambda^\star \approx \rho$. Moreover, for an informative proxy whose loss tracks the true loss, we expect $\rho > 0$. Under these conditions, PPAT reduces variance over LURE when $\lambda \in (0, 2\rho)$ and, consequently, $\lambda = 1$ yields a reduction when $\rho > \tfrac{1}{2}$, while $\lambda = \tfrac{1}{2}$ does so over the wider range $\rho > \tfrac{1}{4}$. For a genuinely useful proxy, we expect it to be strongly correlated with the true loss, giving $\rho > \tfrac{1}{2}$, in which case both choices reduce variance; a less useful proxy, with only moderate correlation $\tfrac{1}{4} < \rho \le \tfrac{1}{2}$, falls outside the range for $\lambda = 1$ but is still handled by the more conservative $\lambda = \tfrac{1}{2}$.


\paragraph{Online estimation of $\lambda$.}
As $\lambda^\star$ depends on the sampling law induced by the proposals $Q_m$, it cannot be estimated from a single active--testing run. We therefore target an alternative coefficient $\lambda^\dagger$ that (i) does not depend on the proposal, and so can be estimated online from the acquired labels, and (ii) flattens the residualised losses $\ell_i - \lambda c_i$ across the pool, which minimises an upper bound on the estimator's variance.
Concretely, writing $c_i = \widetilde{\ell}_i - \widetilde{R}$, we define
\begin{equation*}
\lambda^\dagger
:=
\arg\min_{\lambda\in\mathbb R}
\frac1N\sum_{i=1}^N(\ell_i-\lambda c_i)^2
=
\frac{N^{-1}\sum_{i=1}^N \ell_i c_i}
{N^{-1}\sum_{i=1}^N c_i^2},    
\end{equation*}
which is well--defined whenever $N^{-1}\sum_{i=1}^{N} c_i^2 > 0$. Here, $\lambda^\dagger$ can be seen as choosing the amount of proxy correction to subtract so that the residualised losses are as flat as possible across the pool. Moreover, under mild assumptions about our proposal, it can be shown that this minimises an upper bound on the variance of our estimator:

\begin{restatable}[Variance bound for PPAT]{prop}{varianceBoundPPAT}
\label{prop:ppat-bound}
Suppose that our proposals satisfy a uniform overlap condition: there exists $\beta > 0$ such that
\begin{equation*}
Q_m(i) \ge \frac{\beta}{N},
\qquad m = 1,\dots,M,\ \ i \in S_m.    
\end{equation*}
Then, for every $\lambda \in \mathbb{R}$,
\begin{equation*}
\operatorname{Var}\!\bigl(\widehat R^{\mathrm{PPAT}}_M(\lambda)\bigr)
\;\le\;
\frac{1}{\beta M}\,\frac{N}{N-M+1}
\left(\frac1N\sum_{i=1}^N(\ell_i-\lambda c_{i})^2\right).    
\end{equation*}
Moreover, whenever $N^{-1}\sum_{i=1}^N c_i^2 > 0$, this bound is minimised over $\lambda$ by
\begin{equation*}
\lambda^{\dagger}
= \frac{N^{-1} \sum_{i=1}^N \ell_i c_i}{N^{-1} \sum_{i=1}^N c_i^2}.
\end{equation*}
\end{restatable}
\emph{Proof.} See \S\ref{app:sec:var_ppat_bound}.

In particular, we expect this bound to become tight when the labelling budget is small relative to the pool size and the proposal is not overly concentrated on a small subset of points (see \S\ref{app:sec:var_ppat_bound} for further discussion). Additionally, we empirically verify the suitability of $\lambda^\dagger$ as an alternative to $\lambda^\star$ in \S\ref{app:verifying_choice_of_lamba}.

Crucially, $\lambda^\dagger$ can be estimated online during active testing. As the denominator is known on the full pool, it suffices to estimate the numerator, which we do via its LURE estimate, giving the plug--in estimate
\begin{equation}
\label{eqn:plugin_lambda}    
\widehat{\lambda}_M = \frac{M^{-1}\sum_{m=1}^M V_m \ell_{I_m} c_{I_m}}{N^{-1}\sum_{i=1}^N c_i^2}.
\end{equation}
While using this plug--in estimate of $\lambda$ yields a biased estimator of the risk, we show in \S\ref{sec:app:pluginPPAT} that the resulting estimator remains consistent, so that its bias vanishes asymptotically. In practice, updating $\widehat{\lambda}_M$ after every label can be unstable, since each update changes the residualised loss used for sampling; we therefore update only after every $k$ newly collected labels, keeping $\lambda$ fixed in between.


%% file: theory/asymptotics.tex
\looseness=-1
While the previous sections have focused on point estimation of the risk, in practice it is typically beneficial to establish the estimator's asymptotic behavior and provide a measure of uncertainty around this estimate. After confirming the consistency of our estimator, we therefore turn to the construction of asymptotically valid confidence intervals (CIs) for PPAT -- an aspect that has been largely neglected in prior work on active testing.\footnote{A notable exception is \citet{berrada2025scaling}, who bootstrap the reweighted losses to obtain confidence intervals for LURE. However, the reweighted losses are not i.i.d.\ -- they arise from active sampling without replacement -- and their bootstrap resampling ignores this dependence, so the resulting intervals lack theoretical guarantees.}

\looseness=-1
All expectations and probabilities for the following result are over the active sampling process, conditionally on the finite test pool. For each budget $M$, let $N_M$ denote the pool size and $R_M=N_M^{-1}\sum_{i=1}^{N_M}\ell_i$ the test pool risk. We write $\widehat R_{M,\ppat}(\lambda)$ for the PPAT estimator at budget $M$ and tuning parameter $\lambda\in\mathbb R$, with $\widehat R_{M,\lure}=\widehat R_{M,\ppat}(0)$ the corresponding LURE estimator. We state the main result informally below; its full statement and proof are deferred to \S\ref{app:mathematical_results}.

\vspace{0.1cm}
\begin{thm}[Consistency and asymptotic normality of PPAT, informal]
\label{thm:lure_ppat_unified}
Fix $\lambda\in\mathbb R$. Suppose that our proposals satisfy a uniform overlap condition (i.e. for each $1\leq m \leq M$, there exists $\beta > 0$ s.t. $Q_m(i) \geq \beta/N_M$ for all $i \in S_m$) and that the pool ratio $N_M/M$ converges to some $\alpha>1$.  Then: 
\begin{enumerate}[leftmargin=*,topsep=2pt,itemsep=1pt]
    \item \emph{(Consistency.)} Assume the finite--pool averages $N_M^{-1}\sum_{i=1}^{N_M}\ell_i^2$ and $N_M^{-1}\sum_{i=1}^{N_M}\widetilde\ell_i^{\,2}$ are both bounded. Then we have $\widehat R_{M,\ppat}(\lambda) - R_M \to_p 0$.

    \item \emph{(Asymptotic normality.)} Assume the finite--pool averages $N_M^{-1}\sum_{i=1}^{N_M}\ell_i^4$ and $N_M^{-1}\sum_{i=1}^{N_M}\widetilde\ell_i^{\,4}$ are both bounded. Write $\zeta_i(\lambda)=\ell_i-\lambda(\widetilde\ell_i-\widetilde R_M)$ for the residualised loss, with $\widetilde R_M=N_M^{-1}\sum_{i=1}^{N_M}\widetilde\ell_i$, and let $\gamma_{M,m}=\tfrac{N_M(N_M-M)}{(N_M-m)(N_M-m+1)}$. Assume the predictable quadratic variation
    \begin{equation}
    \label{eq:sigmaM_PPAT}
    \sigma_M^2(\lambda)=\frac1M\sum_{m=1}^M\gamma_{M,m}^2\,s_m^2(\lambda),
    \qquad s_m^2(\lambda)=\frac{1}{N_M^2}\left[\sum_{i\in S_m}\frac{\zeta_i(\lambda)^2}{Q_m(i)}-\Bigl(\sum_{i\in S_m}\zeta_i(\lambda)\Bigr)^{2}\right],
    \end{equation}
    converges in probability to some $\sigma_{\ppat}^2(\lambda)\in(0,\infty)$. Then
    \begin{equation*}
    \sqrt M\bigl\{\widehat R_{M,\ppat}(\lambda)-R_M\bigr\}\;\Rightarrow\;\mathcal N\!\bigl(0,\sigma_{\ppat}^2(\lambda)\bigr),
    \end{equation*}
    and $\sigma_{\ppat}^2(\lambda)$ admits the plug--in estimator
    \begin{equation*}
    \widehat\sigma_M^2(\lambda)=\frac1M\sum_{m=1}^M\gamma_{M,m}^2\bigl(A_m(\lambda)-\widehat R_{M,\ppat}(\lambda)\bigr)^2,
    \qquad
    A_m(\lambda)=\frac{1}{N_M}\Bigl(\frac{\zeta_{I_m}(\lambda)}{Q_m(I_m)}+\sum_{t=1}^{m-1}\zeta_{I_t}(\lambda)\Bigr),
    \end{equation*}
    which is computable from the queried labels and satisfies $\widehat\sigma_M^2(\lambda)\to_p\sigma_{\ppat}^2(\lambda)$.
    
\end{enumerate}
Moreover, analogous consistency and asymptotic normality results hold when $\lambda$ is replaced by the plug--in estimate $\widehat\lambda_M$ from \eqref{eqn:plugin_lambda}.
\end{thm}

\emph{Proof.} See \S\ref{sec:app:fixedlambdaPPAT}-\ref{sec:app:pluginPPAT}.

\looseness=-1
Crucially, Thm. \ref{thm:lure_ppat_unified} immediately yields the asymptotic $(1-\delta)$--CI
\begin{equation}
\label{eqn:asymptotic_valid_cis_main}
\widehat R_{M,\ppat}(\lambda) \pm z_{1-\delta/2}\,\widehat\sigma_M(\lambda)/\sqrt M,
\end{equation}
where $z_{1-\delta/2}$ denotes the $(1-\delta/2)$--quantile of the standard normal distribution. As LURE is the special case $\lambda=0$, this also yields, as a by--product, an asymptotic CI for LURE. 

Note that the uniform overlap condition is the usual positivity requirement for importance--weighted active testing \citep{farquhar2021statistical} and can be enforced by mixing with a uniform proposal \citep{kossen2021active}, while the moment conditions are bounded finite--pool averages requiring no distributional assumption, as we work conditionally on the test pool. Moreover, the condition $N_M/M\to\alpha>1$ is the usual finite--population asymptotic regime \citep{farquhar2021statistical}. The remaining assumption -- that the predictable quadratic variation $\sigma^2_M(\lambda)$ converges to a finite, strictly positive limit -- is the standard condition for the martingale central limit theorem and is needed only to show asymptotic normality; we discuss this assumption in detail in \S\ref{sec:app:assumption-discussion}. We provide empirical support for this assumption indirectly through the coverage experiments in \S\ref{sec:experiments}, where we show that PPAT can attain nominal coverage substantially faster than competing strategies, while also producing substantially narrower intervals.

%% file: experiments/experiments.tex
\vspace{-0.2cm}
We evaluate our approach on real--world regression and classification datasets.  We provide a description of our experimental setup and results below with full details in \S\ref{app:further_details_experimental_setup}.

\vspace{-0.3cm}
\input{experiments/experimental_setup}
\input{experiments/main_results}
\input{experiments/ablations}

%% file: experiments/experimental_setup.tex
\subsection{Experimental Setup}
\label{sub:experimental_setup}
\paragraph{Datasets.}
For regression, we use the following UCI datasets \citep{uci}: \texttt{Keggundirected}, \texttt{Keggdirected}, \texttt{Sml}, \texttt{Bike}.
For each dataset, we randomly sample a fixed training set of 250 points to train the model \(f\) whose risk we wish to estimate, and use the remaining points as the test pool. For classification, we use \texttt{CIFAR--10}, \texttt{CIFAR-100} \citep{krizhevsky2009learning}, and \texttt{Tiny--ImageNet} \citep{Le2015TinyIV}. We respect the original train/test splits, using the training set to train \(f\) and treating the test set as the test pool. Additional dataset details are given in \S\ref{app:further_details_experimental_setup}.

\paragraph{Models and proxies.}
\looseness=-1
For regression, \(f\) is a Gaussian process regression model with an RBF kernel whose hyperparameters are optimised using the marginal likelihood, and the active--testing surrogate is a Bayesian linear regression model with the default \texttt{Scikit-Learn} hyperparameters \citep{scikit-learn}. For classification, we follow the standard pretrained--encoder pipeline: inputs are mapped to CLIP embeddings \citep{radford2021learning}, and \(f\) is a linear classifier trained on these frozen embeddings. The surrogate is a Laplace-approximated Bayesian neural network \citep{bayesinterpolation,daxberger2021laplace}: an MLP on CLIP embeddings with three hidden layers of width 128 for \texttt{CIFAR-10} and \texttt{CIFAR-100}, and five hidden layers of width 128 for \texttt{Tiny-ImageNet}.

For our proxy, we use strong off-the-shelf predictors suited to each task.   For regression, we use the \texttt{TabPFN--2.5} foundation model \citep{grinsztajn2025tabpfn} and for classification we use the zero-shot predictions of the \texttt{ViT-L-14} CLIP model from the \texttt{sentence--transformers} library \citep{reimers-2019-sentence-bert}. Further implementation details and ablations over surrogate and proxy choices are deferred to \S\ref{app:further_details_experimental_setup} and~\ref{app:additional_results} respectively.

\paragraph{Baselines.}
\looseness=-1
We compare \colppat{} with \colure{} \citep{kossen2021active}, \ase{} \citep{kossen2022active}, and \random{}, which samples randomly without replacement and uses the unweighted estimator in \eqref{eqn:unweighted_estimator}. Within each experimental setting, all active methods are run with the same surrogate model. For \colppat{}, we report the fixed choices \(\lambda=1\) and \(\lambda=0.5\), as well as the plug--in estimate \(\widehat{\lambda}_M\), which we initialise at $0.5$ and update after every 100 acquired points. For \ase{}, we use 100 Monte Carlo samples to compute the acquisition scores; further details about the baselines are given in \S\ref{app:related_work} and \S\ref{app:further_details_experimental_setup}.

\paragraph{Active-testing setup and metrics.}
In all experiments, we use a labelling budget of \(M=500\) and run 1000 independent trials. For regression, the loss is squared error and the surrogate is updated after every newly acquired label. For classification, the loss is cross--entropy and, following prior work \citep{kossen2021active,kossen2022active}, we keep the surrogate fixed during acquisition. Following \cite{kossen2021active, kossen2022active}, we report the median squared error, where the error is the difference between the estimated risk and the true risk on the full test pool. Mean error, confidence--interval widths, and further ablations are provided in \S\ref{app:additional_results}.

%% file: experiments/main_results.tex
\subsection{PPAT Improves Risk Estimation}
\label{sub:main_results}
\paragraph{Regression.}
Fig.~\ref{fig:main_uci} shows that, across all datasets, \colppat{} achieves lower median squared error than \random{}, \colure{}, and \ase{} for all three choices of \(\lambda\). The only exception occurs at very small labelling budgets, where \ase{} can sometimes attain lower error, for example on \texttt{Keggundirected}. However, this comes at the cost of using a \textbf{biased} risk estimate.

\begin{figure}[!h]
    \centering
    \includegraphics[width=.75\textwidth]{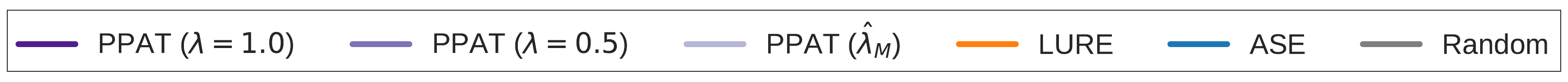}
    \vspace{0.75em}
    
    \begin{subfigure}[t]{0.22\textwidth}
        \centering
        \includegraphics[width=\linewidth]{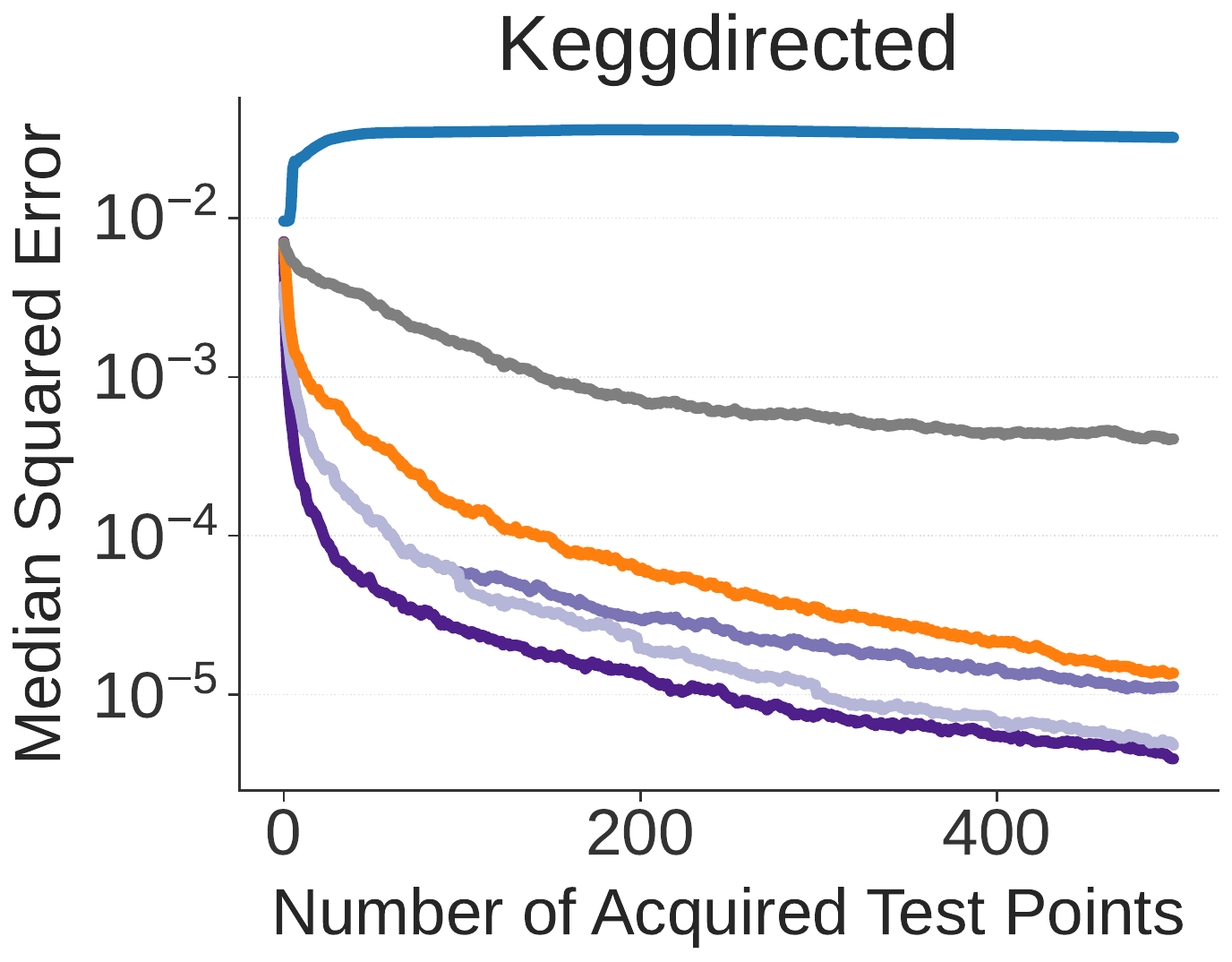}
        \label{fig:main_keggdir}
    \end{subfigure}
    \hspace{0.25cm}
    \begin{subfigure}[t]{0.22\textwidth}
        \centering
        \includegraphics[width=\linewidth]{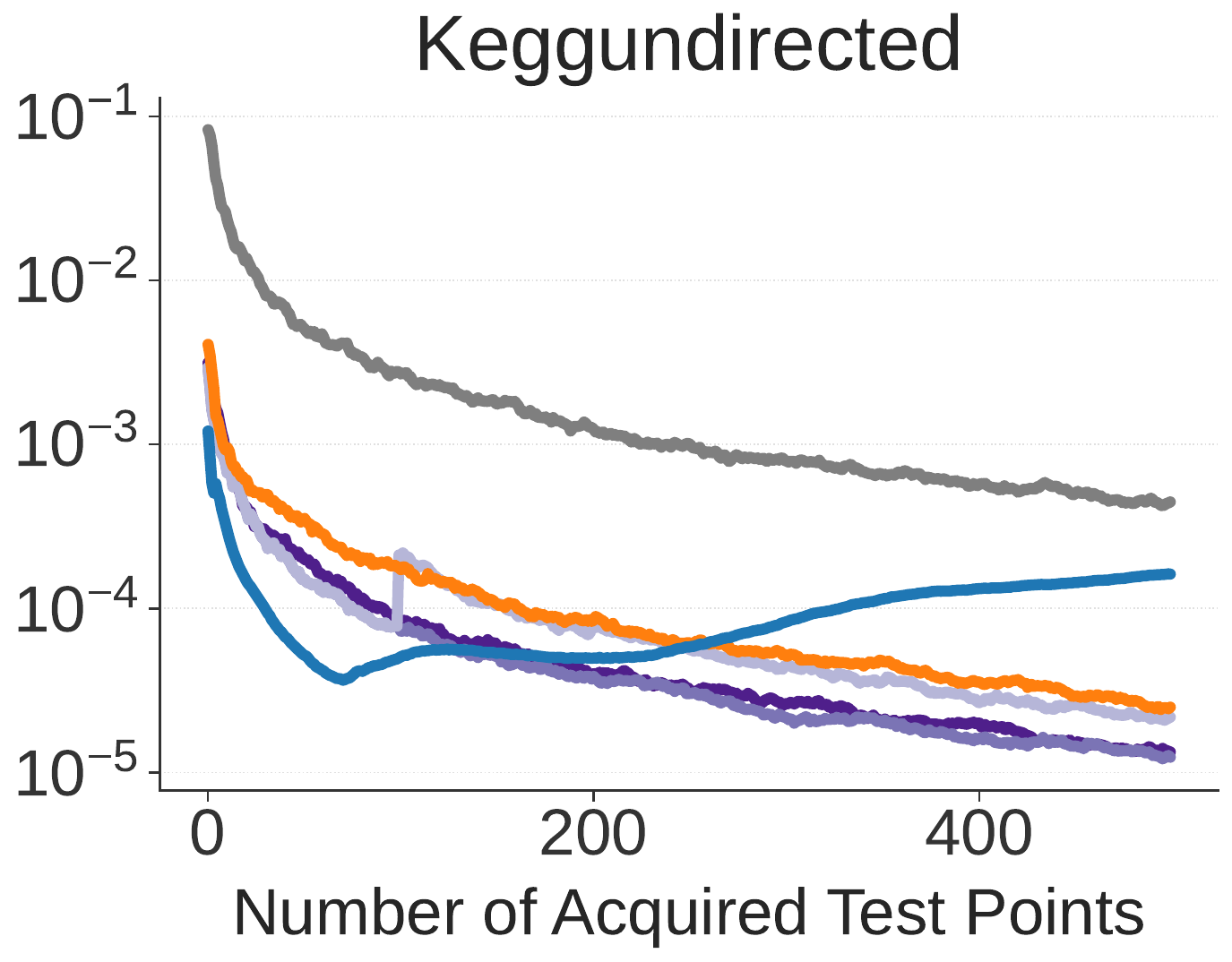}
        \label{fig:main_keggundir}
    \end{subfigure}
    \hspace{0.25cm}
    \begin{subfigure}[t]{0.22\textwidth}
        \centering
        \includegraphics[width=\linewidth]{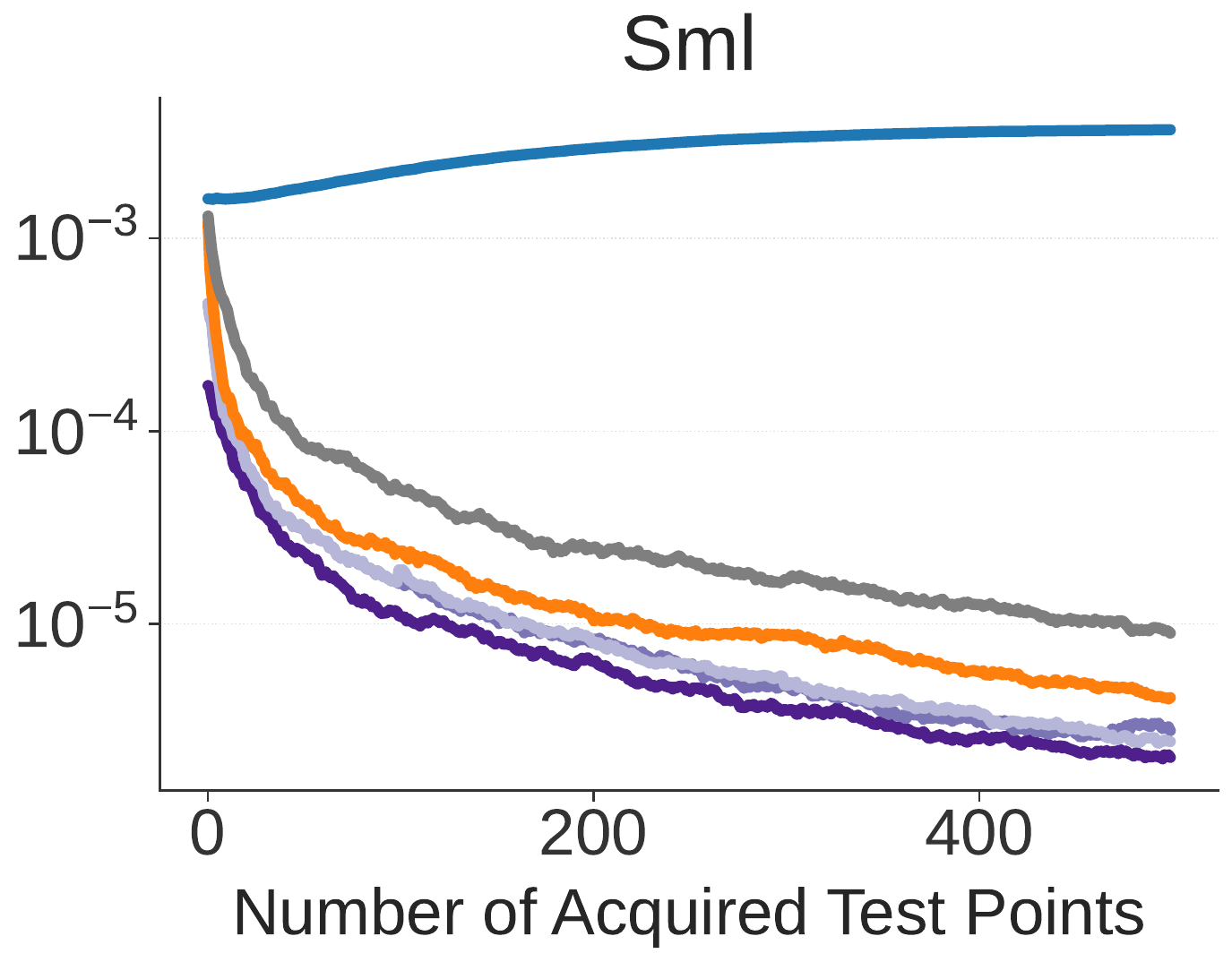}
        \label{fig:main_sml}
    \end{subfigure}
    \hspace{0.25cm}
    \begin{subfigure}[t]{0.22\textwidth}
        \centering
        \includegraphics[width=\linewidth]{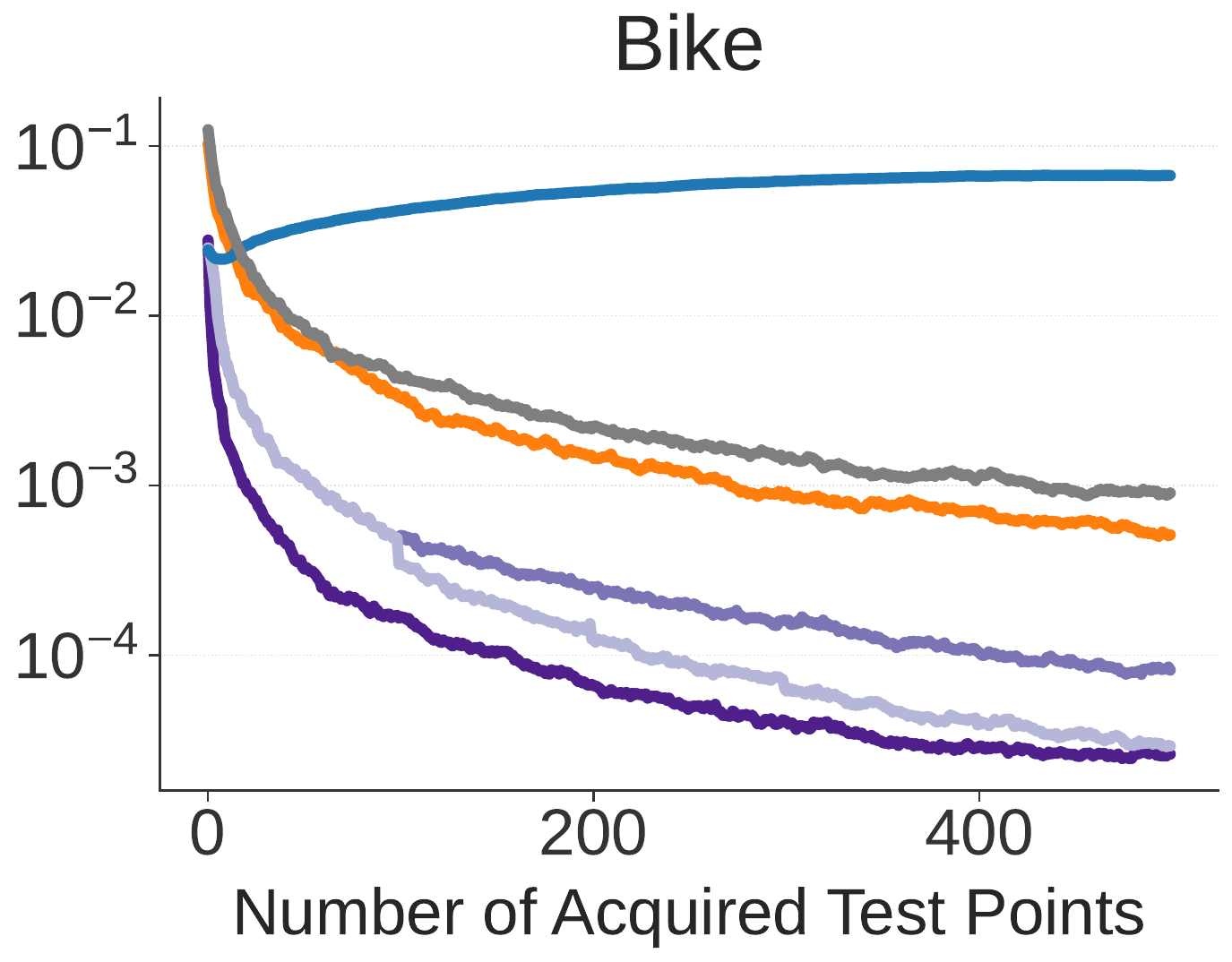}
        \label{fig:main_bike}
    \end{subfigure}
    \vspace{-0.5cm}
    \caption{    
    Regression experiments comparing \textcolor{violet}{\colppat{}}, \textcolor{orange}{\colure{}}, \ase{}, and \textcolor{gray}{\random{}}. Plots show the median squared error across 1000 trials.
    }
    \label{fig:main_uci}
\end{figure}

For the plug--in choice \(\widehat{\lambda}_M\), we occasionally observe a small transient increase in error at early labelling budgets (e.g. on \texttt{Keggundirected}). 
This is expected, since \(\widehat{\lambda}_M\) is updated from the actively collected labels during acquisition, which introduces finite--sample bias relative to fixed-\(\lambda\) \colppat{} and also changes the proposal used to sample future points. As the number of acquired labels grows, however, \(\widehat{\lambda}_M\) stabilises and this bias diminishes (see ~\S \ref{app:mean_errors}). Importantly, we also see that \colppat{} with $\widehat{\lambda}_M$ generally converges to \colppat{} with the best performing fixed $\lambda$, suggesting that targeting $\lambda^\dagger$ is an effective choice (see also \S\ref{app:verifying_choice_of_lamba}).

\paragraph{Classification.}
Fig.~\ref{fig:classification_datasets} shows the same pattern on the classification datasets: \colppat{} consistently achieves lower median squared error than \random{}, \colure{}, and \ase{} for all choices of \(\lambda\), with the exception of \texttt{Tiny-ImageNet}, where \colppat{} with \(\lambda = 1\) performs comparably to \random{}. We further note that \random{} outperforms \colure{} on these datasets, likely a result of the surrogate poorly tracking the losses. Indeed, since \colure{} builds its proposal, and hence its importance weights, from the surrogate's expected losses, a surrogate that poorly tracks the true losses can result in importance weights that inflate the variance of \colure{}. In contrast, \colppat{} is less affected by this issue: it importance--weights the residualised loss rather than the raw loss, and since the proxy explains part of the loss variation, the residuals have smaller spread, making the variance less sensitive to errors in the surrogate.

\begin{figure}[!h]
    \centering
    \includegraphics[width=.75\textwidth]{figures/uci/uci_main_legend.pdf}
    \vspace{0.75em}    

    \begin{subfigure}[t]{0.22\textwidth}
        \centering
        \includegraphics[width=\linewidth]{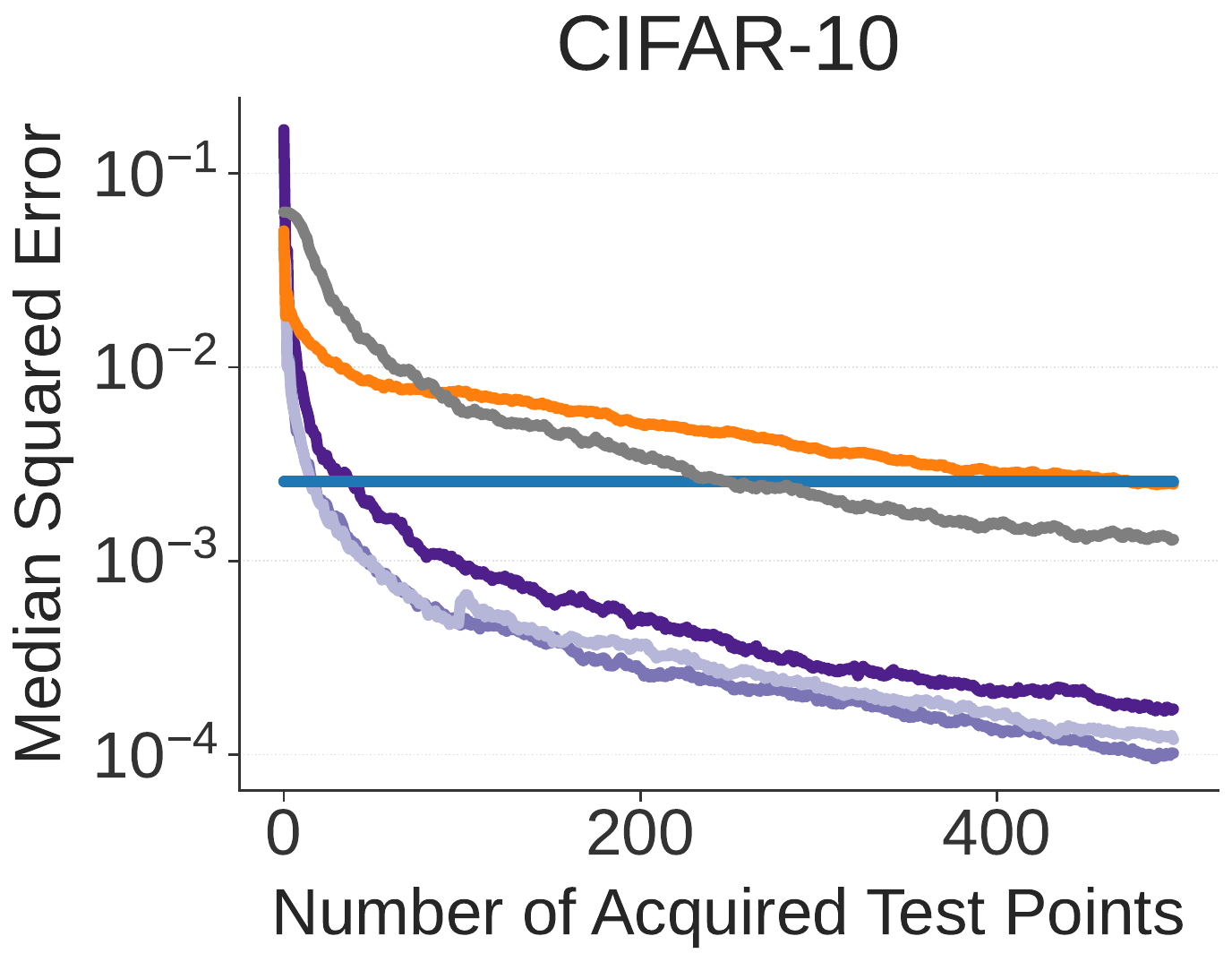}
        \label{fig:main_cifar10}
    \end{subfigure}
    \hspace{0.1cm}
    \begin{subfigure}[t]{0.22\textwidth}
        \centering
        \includegraphics[width=\linewidth]{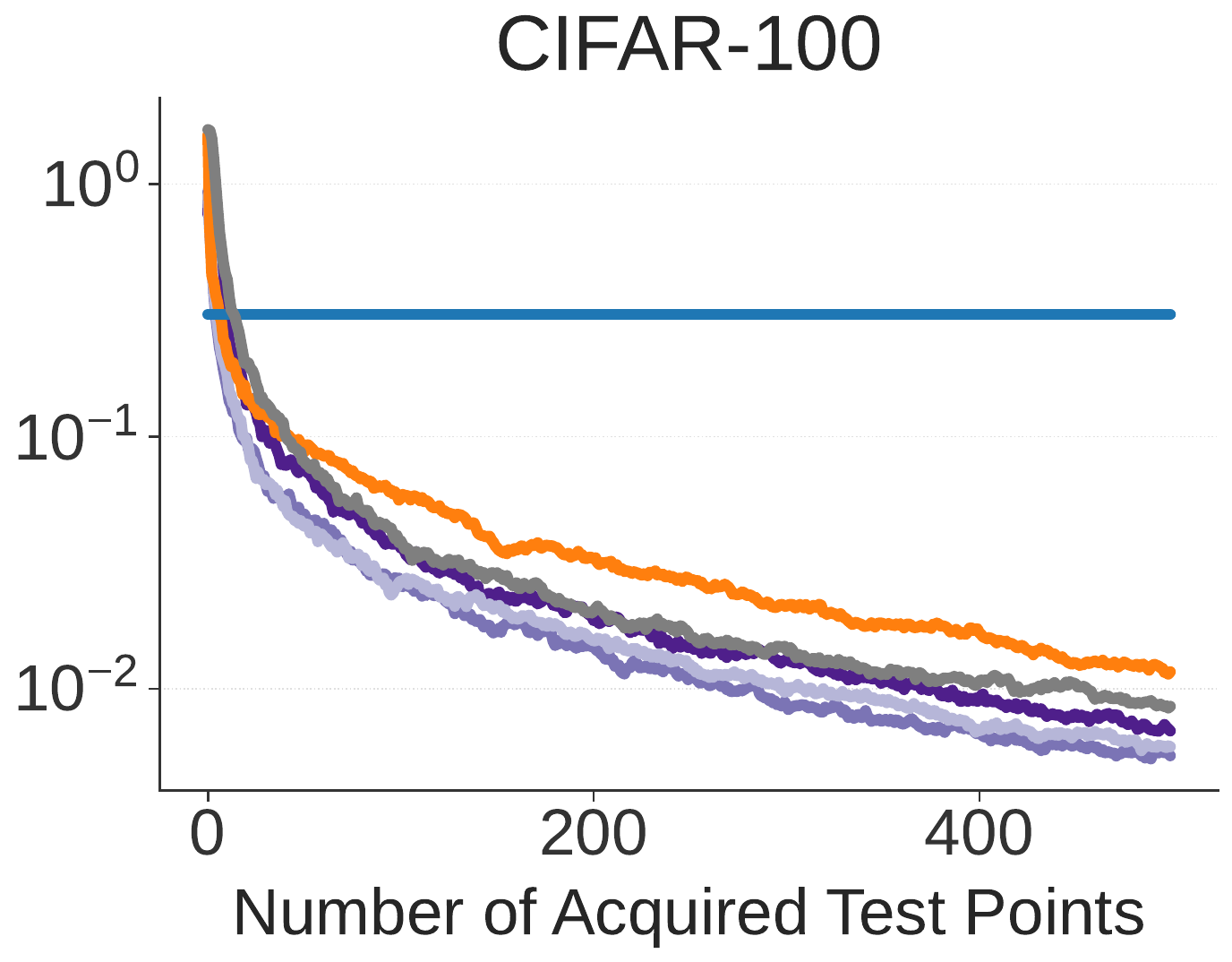}
        \label{fig:main_cifar100}
    \end{subfigure}
    \hspace{0.1cm}
    \begin{subfigure}[t]{0.22\textwidth}
        \centering
        \includegraphics[width=\linewidth]{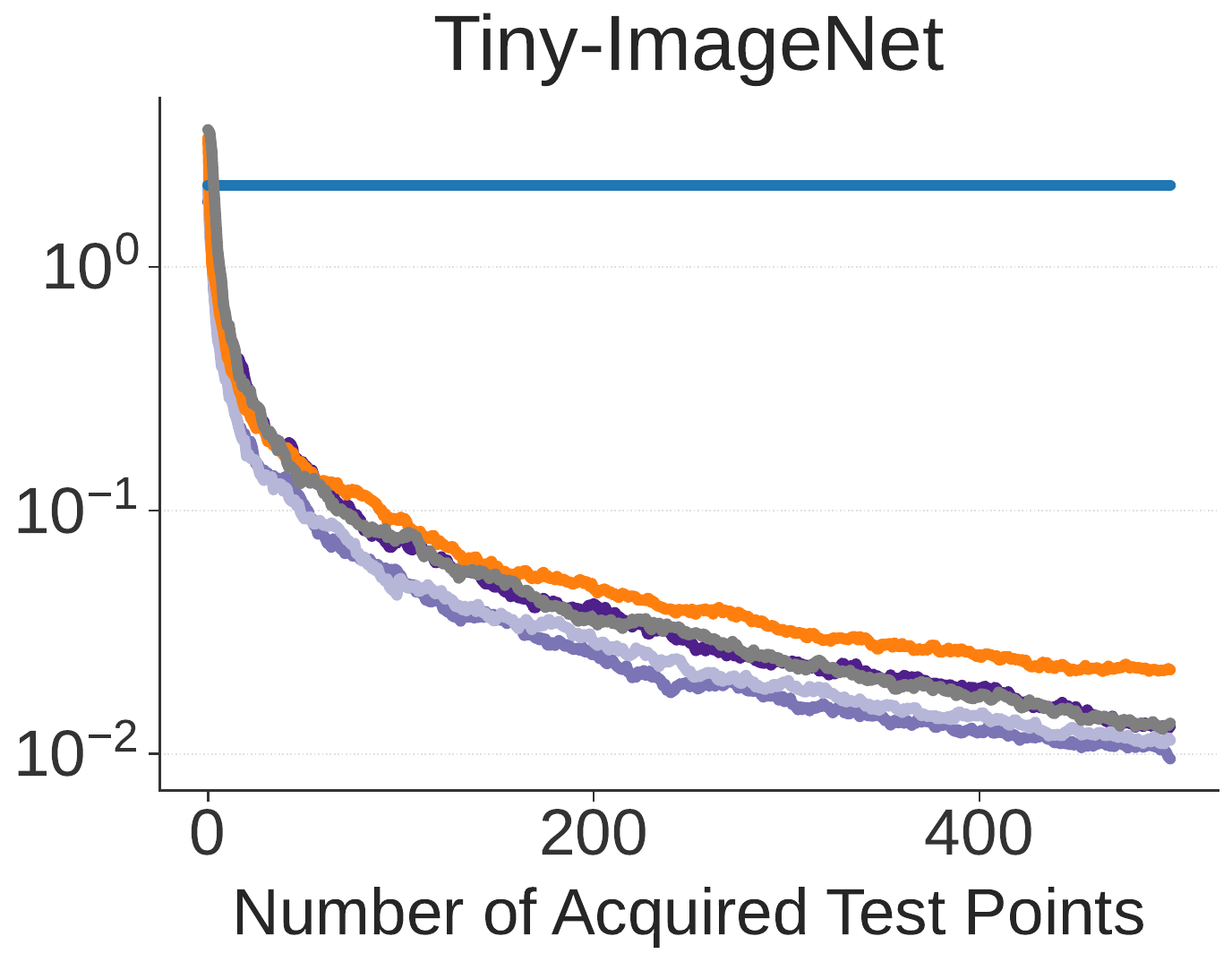}
        \label{fig:main_tiny_imagenet}
    \end{subfigure}    
    \vspace{-0.5cm}
    \caption{Classification experiments comparing \textcolor{violet}{\colppat{}}, \textcolor{orange}{\colure{}}, \textcolor{NavyBlue}{\ase{}}, and \textcolor{gray}{\random{}}. Plots show the median squared error across 1000 trials.}
\label{fig:classification_datasets}
\end{figure}

Overall, these results indicate that \colppat{} improves risk estimation not only for tabular regression, but also for more challenging, higher--dimensional image-classification tasks.

%% file: experiments/ablations.tex
\begin{figure}[!h]
    \centering    \includegraphics[width=.5\textwidth]{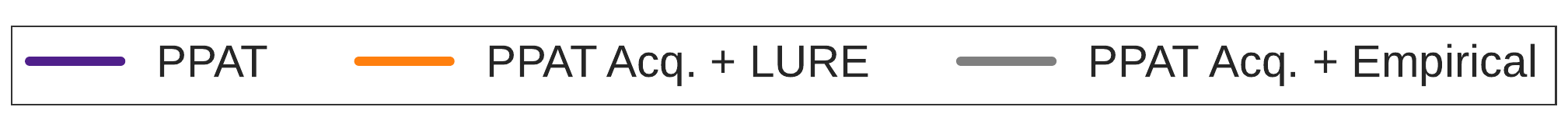}
    \vspace{0.75em}    
    
    \begin{subfigure}[t]{0.22\textwidth}
        \centering
        \includegraphics[width=\linewidth]{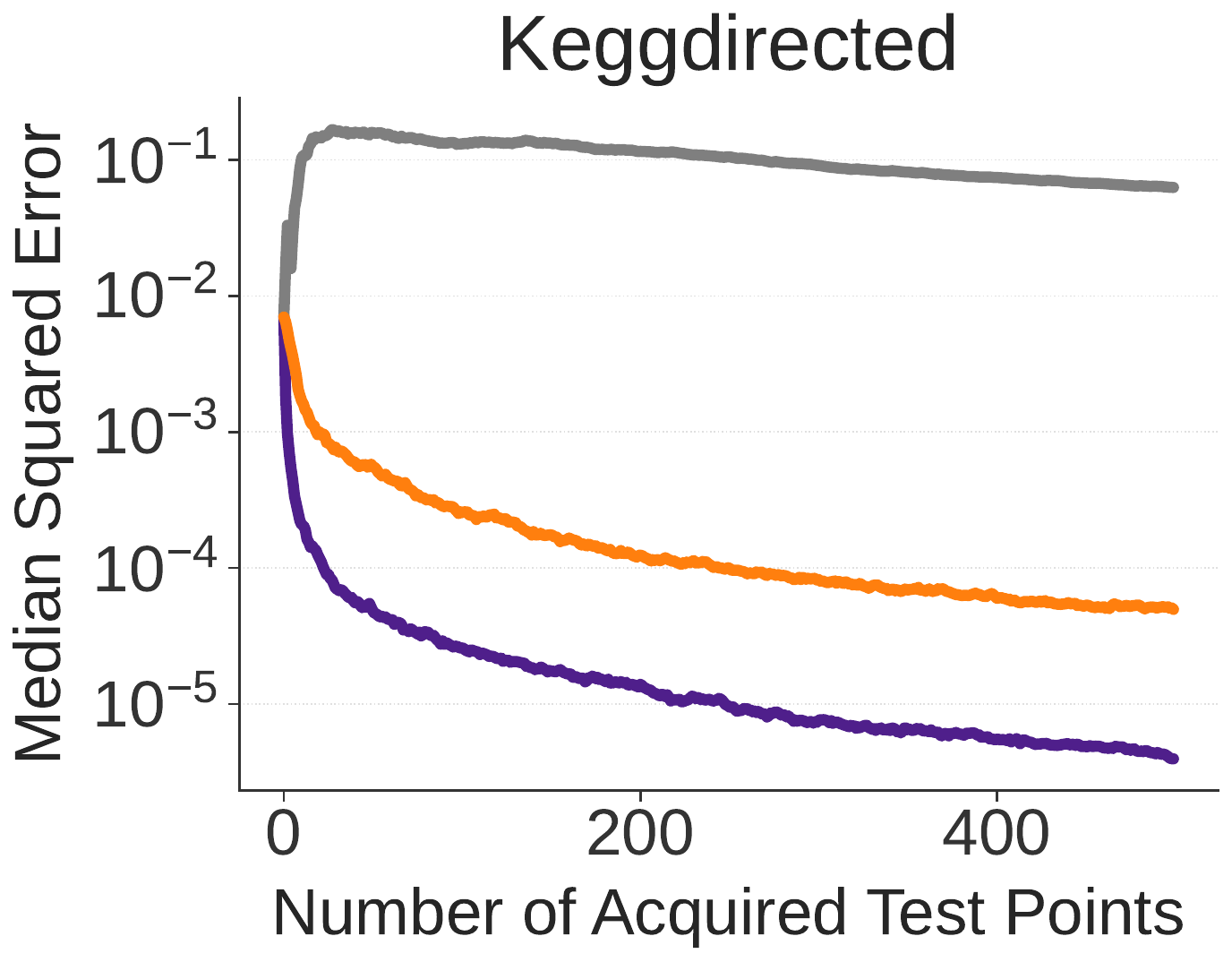}
        \label{fig:est_ablations_keggdir}
    \end{subfigure}
    \hspace{0.25cm}
    \begin{subfigure}[t]{0.22\textwidth}
        \centering
        \includegraphics[width=\linewidth]{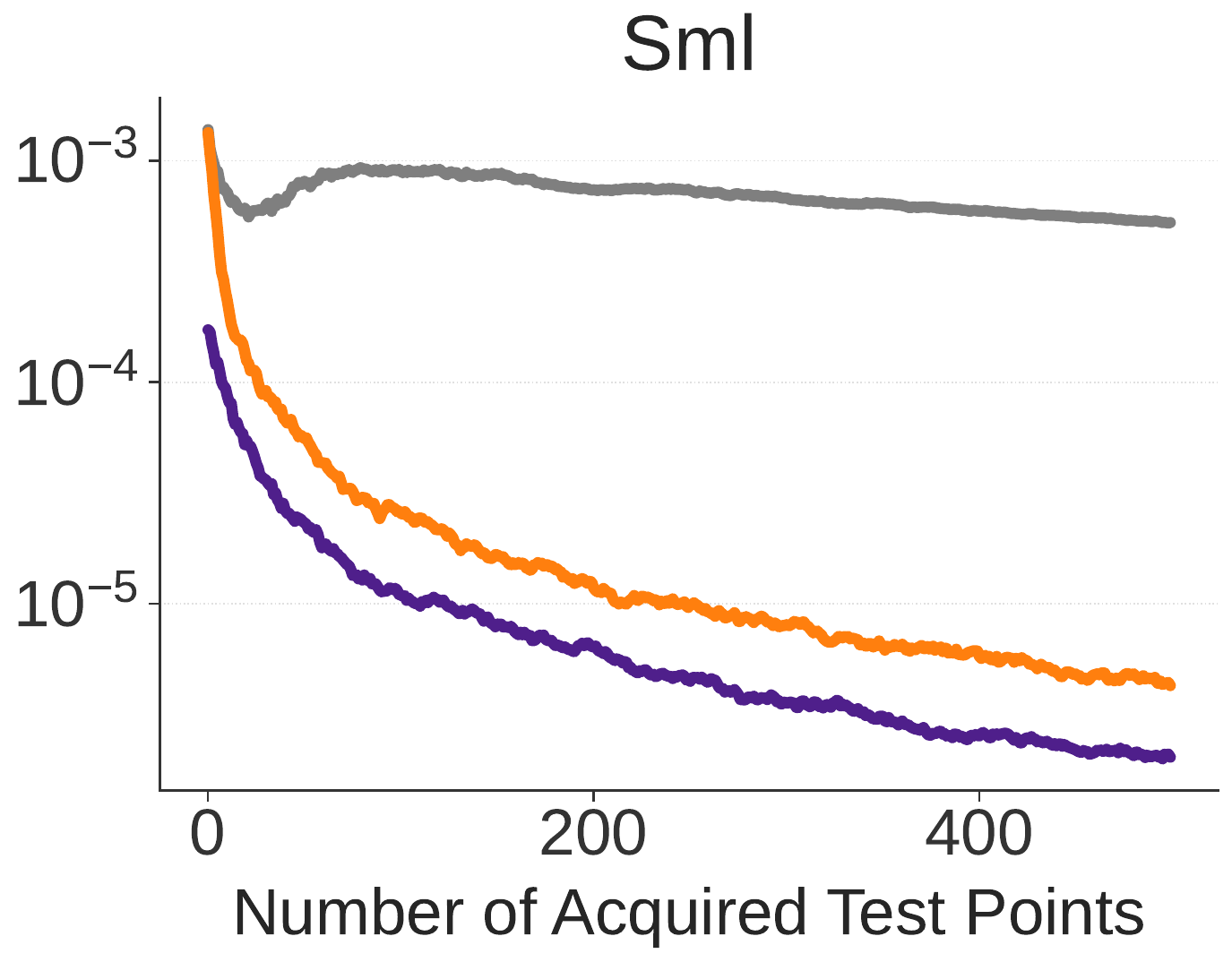}
        \label{fig:est_ablations_gas}
    \end{subfigure}
    \hspace{0.25cm}
    \begin{subfigure}[t]{0.22\textwidth}
        \centering        \includegraphics[width=\linewidth]{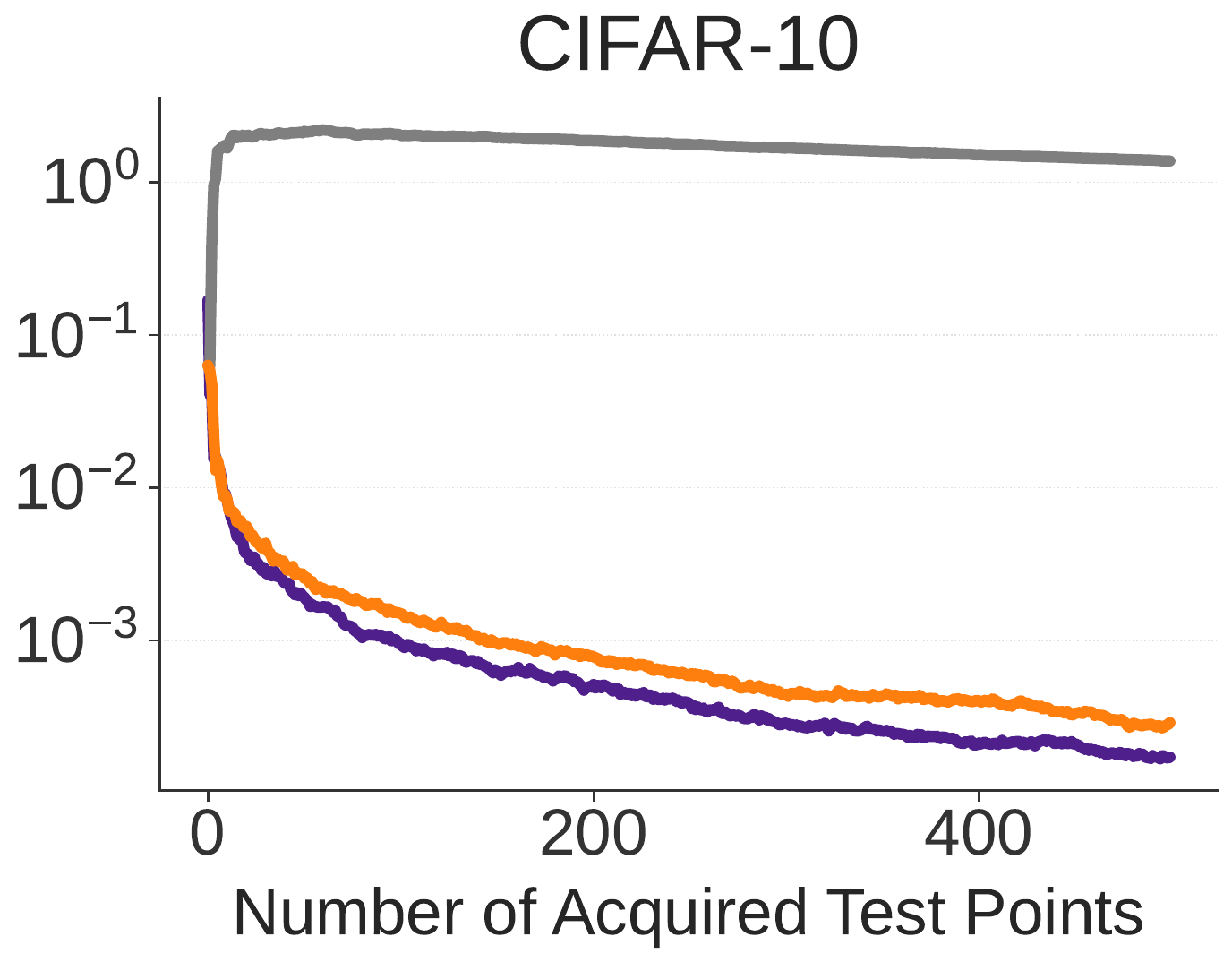}
        \label{fig:est_ablations_cifar10}
    \end{subfigure}
    \hspace{0.25cm}
    \begin{subfigure}[t]{0.22\textwidth}
        \centering
        \includegraphics[width=\linewidth]{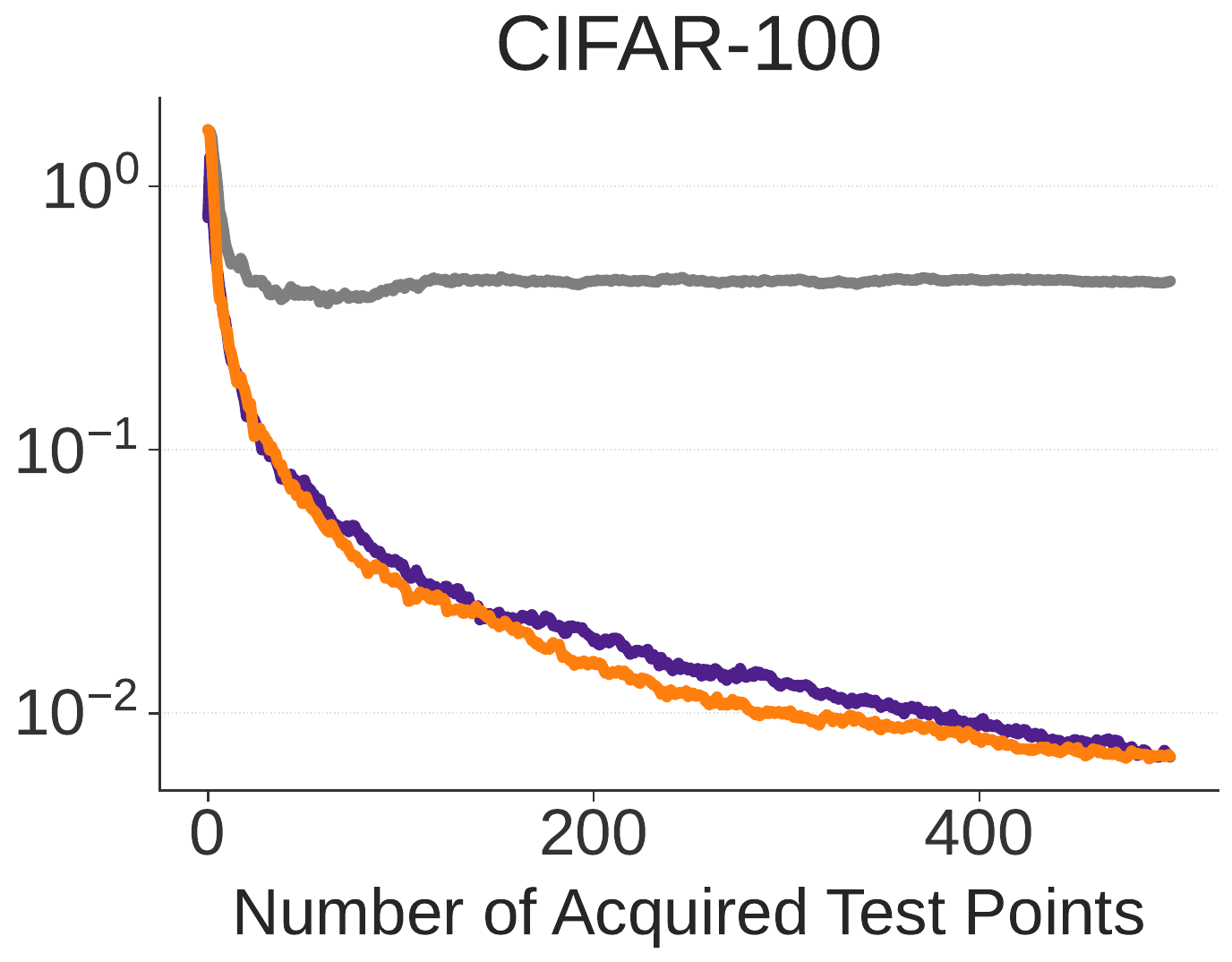}
        \label{fig:fig:main_bike}
    \end{subfigure}
    \vspace{-0.5cm}
    \caption{Experiments on UCI and classification datasets studying the influence of our estimator. We compare \textcolor{violet}{PPAT} with $\lambda=1$ to using the PPAT acquisition strategy, $Q^\mathrm{PPAT}$, but with the LURE estimate of the risk (\textcolor{orange}{PPAT Acq. + LURE}) and the empirical estimate of the risk (\textcolor{gray}{PPAT Acq. + Empirical}). Plots show the median squared error across 1000 trials.}
    \label{fig:effect_of_estimator}
\end{figure}

\subsection{Disentangling the Components of PPAT}
\label{sec:ablations}
Beyond the choice of $\lambda$, which we have discussed in \S\ref{sub:main_results} and \S\ref{app:verifying_choice_of_lamba}, there are two main components that can affect the behaviour of \colppat{}: the choice of acquisition strategy and the choice of the PPI--like estimator. To disentangle their contributions, we study the effect of using different estimators in Fig.~\ref{fig:effect_of_estimator} and different acquisition strategies in Fig.~\ref{fig:effect_of_aq} on four representative datasets.

The notable improvements observed over both  different acquisition strategies (Fig.~\ref{fig:effect_of_aq}) and different estimators (Fig.~\ref{fig:effect_of_estimator}) show that the gains of \colppat{} do not come merely from using a better, PPI--like estimator. Instead, they come from combining our PPI--like estimator with an acquisition strategy tailored to the corresponding residualised objective. 

\begin{figure}[!h]
    \centering    \includegraphics[width=.5\textwidth]{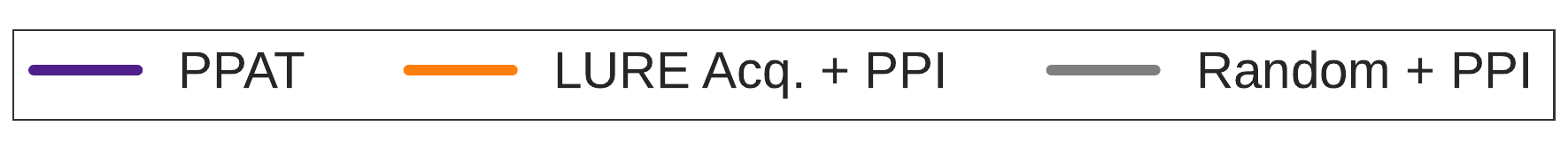}
    \vspace{0.75em}    
    
    \begin{subfigure}[t]{0.22\textwidth}
        \centering
        \includegraphics[width=\linewidth]{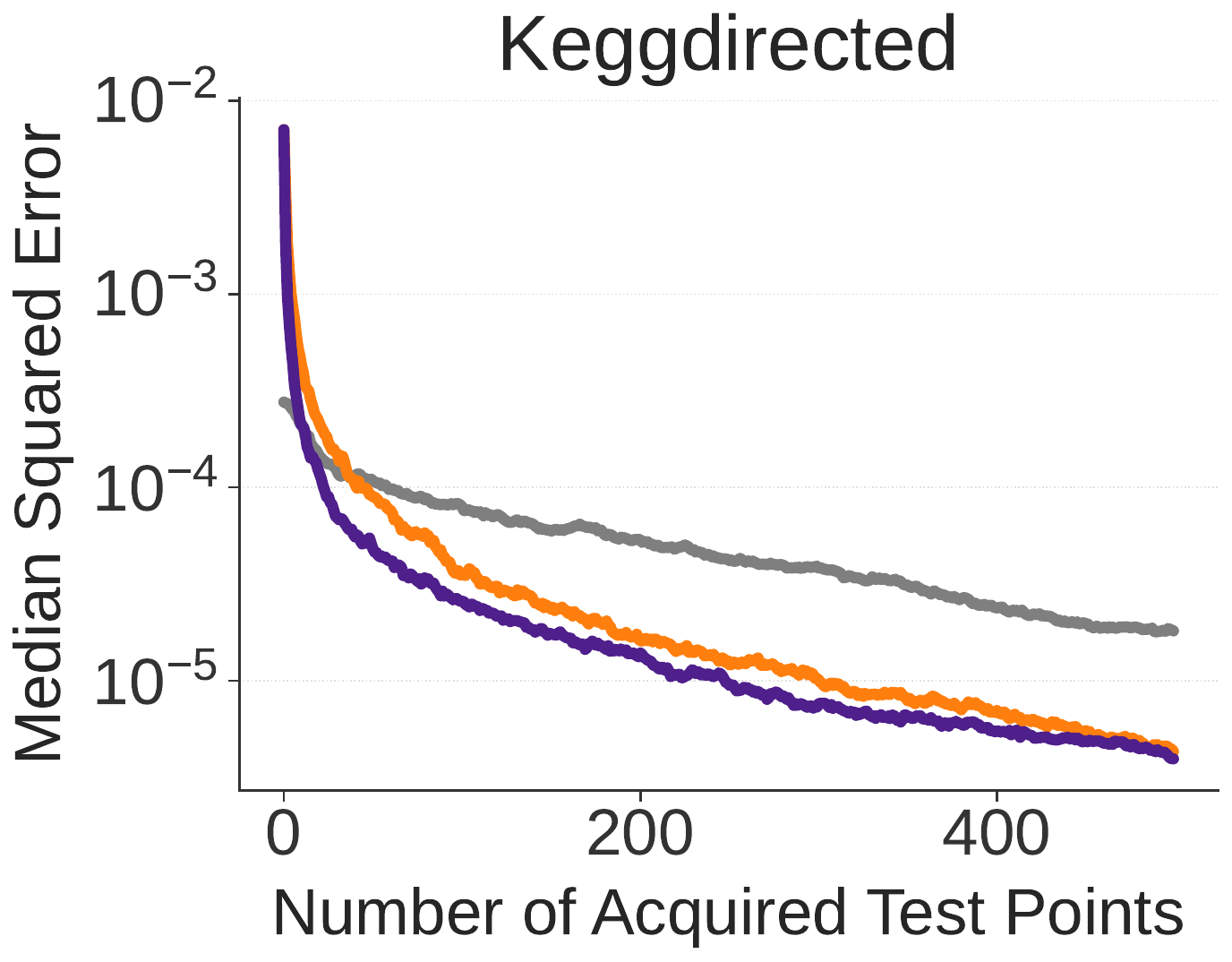}
        \label{fig:acq_ablations_keggdir}
    \end{subfigure}
    \hspace{0.25cm}
    \begin{subfigure}[t]{0.22\textwidth}
        \centering
        \includegraphics[width=\linewidth]{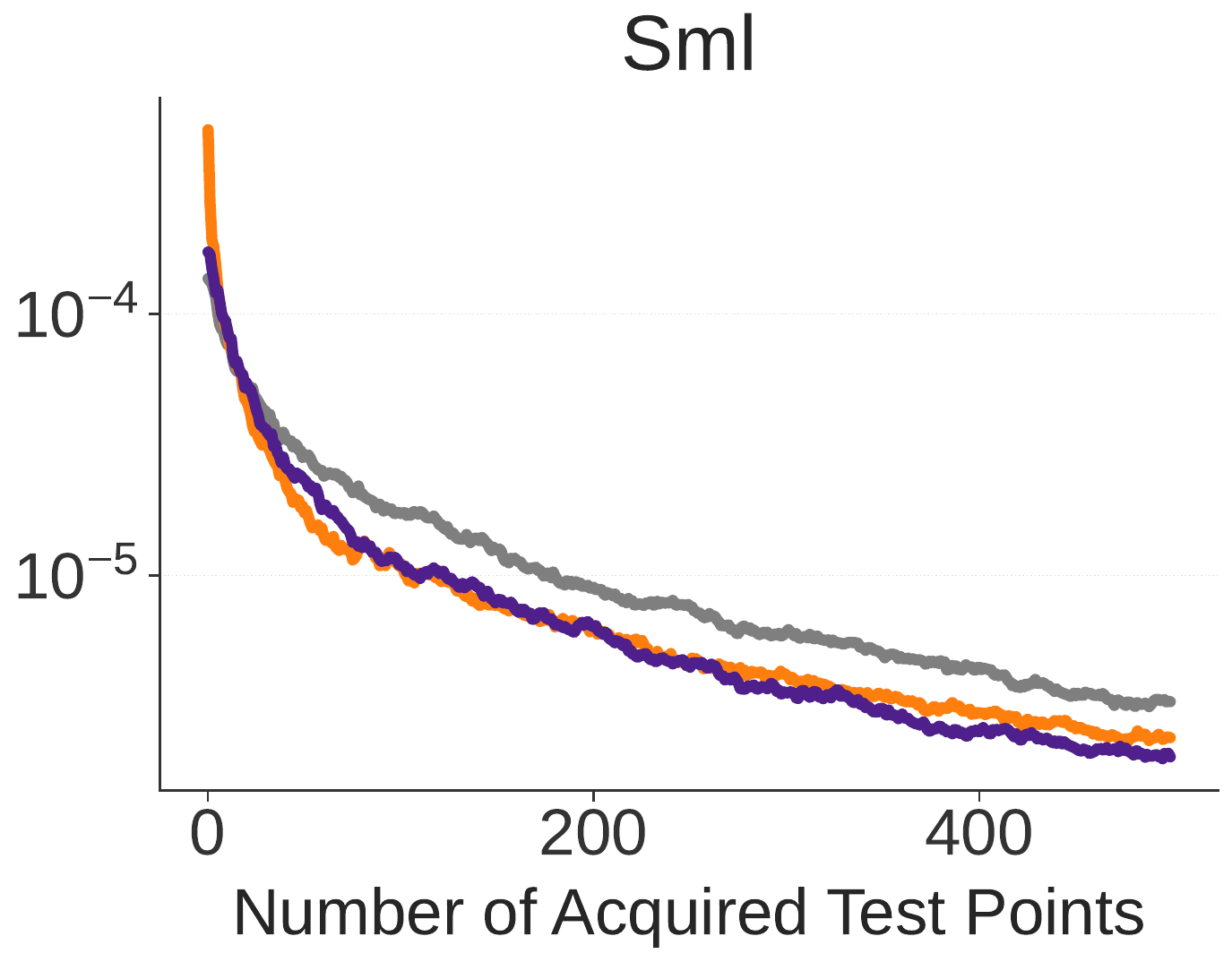}
        \label{fig:acq_ablations_gas}
    \end{subfigure}
    \hspace{0.25cm}
    \begin{subfigure}[t]{0.22\textwidth}
        \centering        \includegraphics[width=\linewidth]{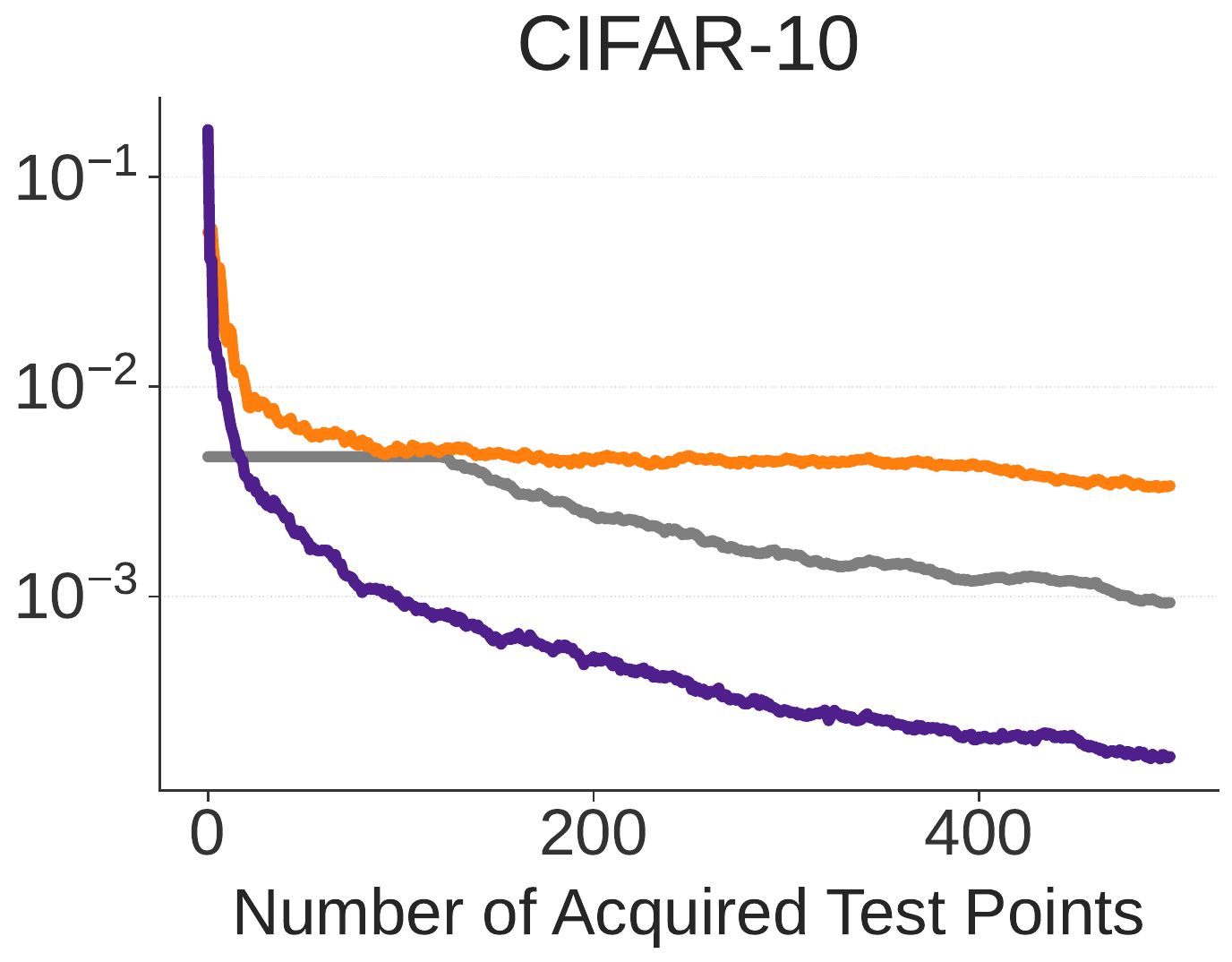}
        \label{fig:acq_ablations_cifar10}
    \end{subfigure}
    \hspace{0.25cm}
    \begin{subfigure}[t]{0.22\textwidth}
        \centering
        \includegraphics[width=\linewidth]{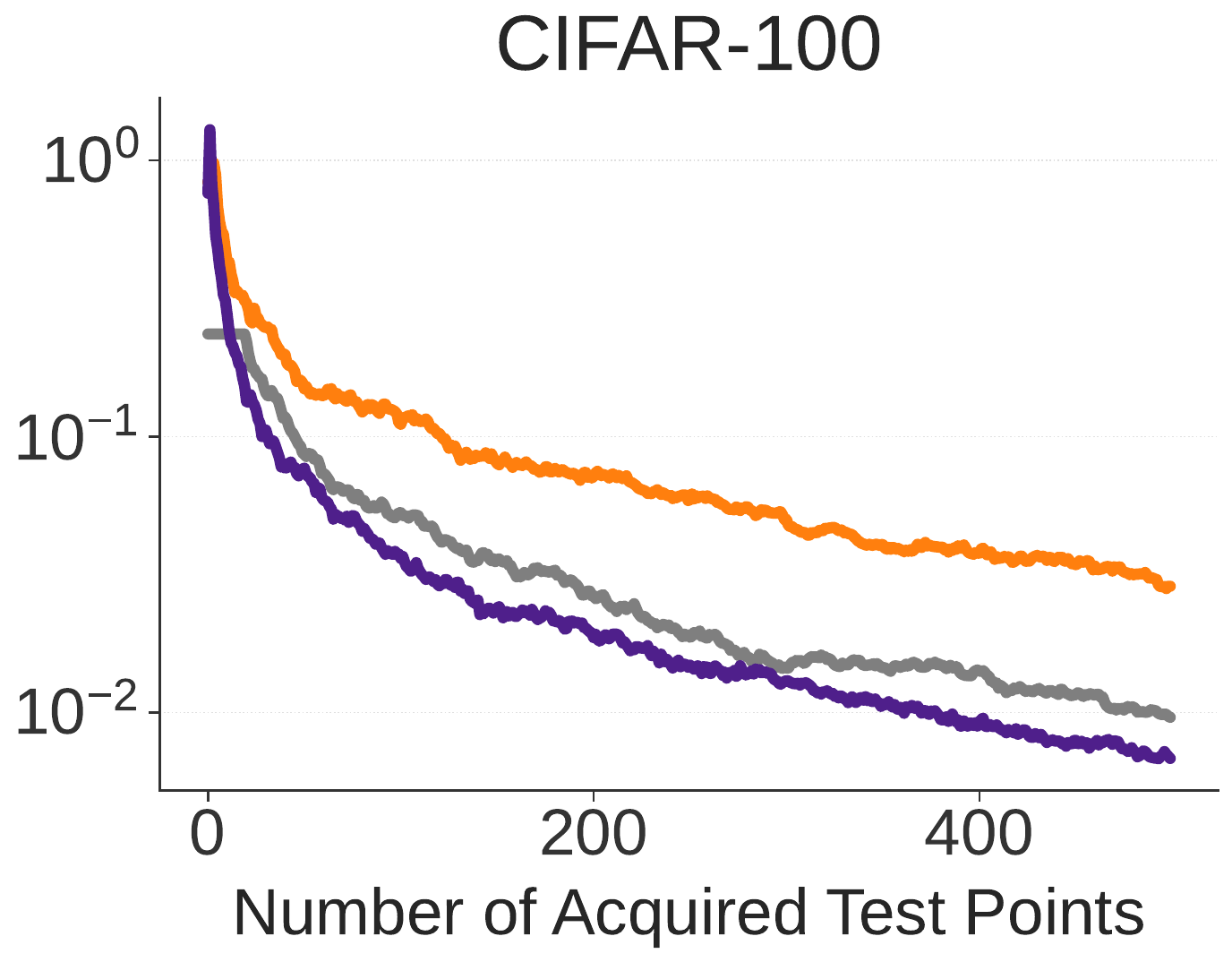}
        \label{fig:acq_ablations_cifar100}
    \end{subfigure}
    \vspace{-0.5cm}
    \caption{Experiments on UCI and classification datasets studying the influence of our acquisition strategy. We compare \textcolor{violet}{PPAT} with $\lambda=1$ to using the PPI estimator with the LURE acquisition strategy (\textcolor{orange}{LURE Acq. + PPI}) and random sampling (\textcolor{gray}{Random + PPI}). Plots show the median squared error across 1000 trials.}
    \label{fig:effect_of_aq}
\end{figure}

\subsection{Coverage}
\label{sub:coverage}
Beyond reducing estimation error, we also want our CIs to be well-calibrated: across repeated active-testing runs, a nominal $(1-\delta$)-CI should contain the true test-pool risk with frequency close to $(1-\delta)$. Fig. \ref{fig:coverage} shows the coverage for the CIs presented in \S\ref{sec:asymp_props_approx_cis} for \colppat{}, \colure{}, and \random{} on four representative datasets. \ase{} is excluded as it \textbf{cannot} provide asymptotically valid CIs without additional assumptions about $f$ and $\pi_m(\cdot|\bx)$. We see that on all datasets, our approach achieves the desired coverage level for all choices of $\lambda$. More importantly, it achieves the desired coverage at either the same or a significantly faster rate than all other approaches. Additionally, our approach attains significantly smaller widths for the confidence intervals than the competing approaches (see  \S\ref{app:ci_widths}).

\begin{figure}[!h]
    \centering    \includegraphics[width=.75\textwidth]{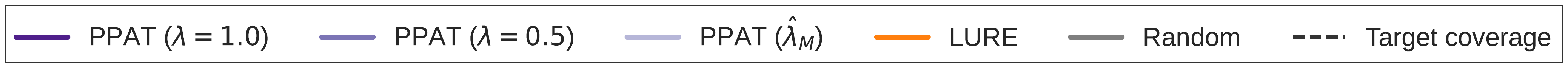
    }
    \vspace{0.75em}    
    
    \begin{subfigure}[t]{0.22\textwidth}
        \centering
        \includegraphics[width=\linewidth]{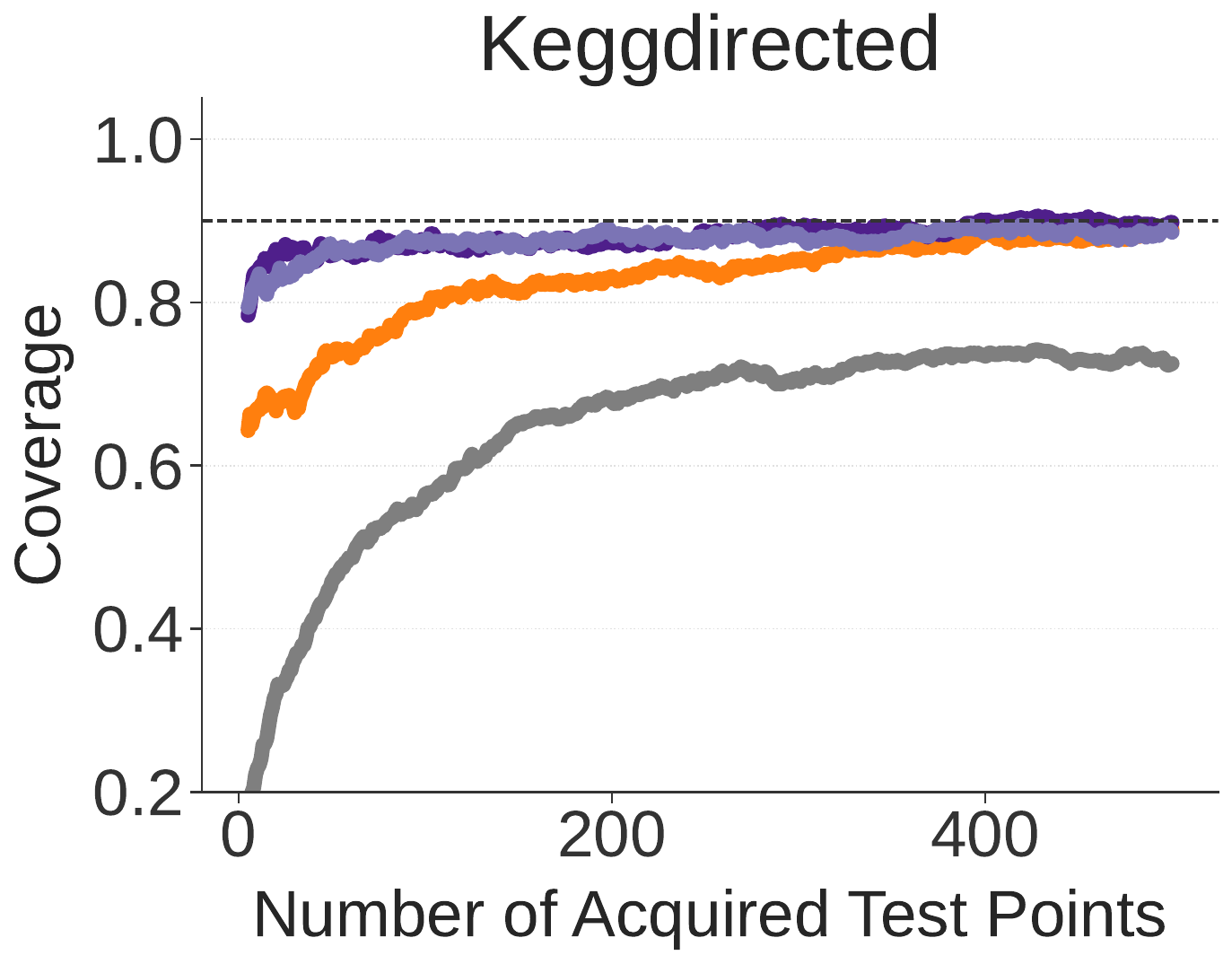}
        \label{fig:coverage_keggdir}
    \end{subfigure}
    \hspace{0.25cm}
    \begin{subfigure}[t]{0.22\textwidth}
        \centering
        \includegraphics[width=\linewidth]{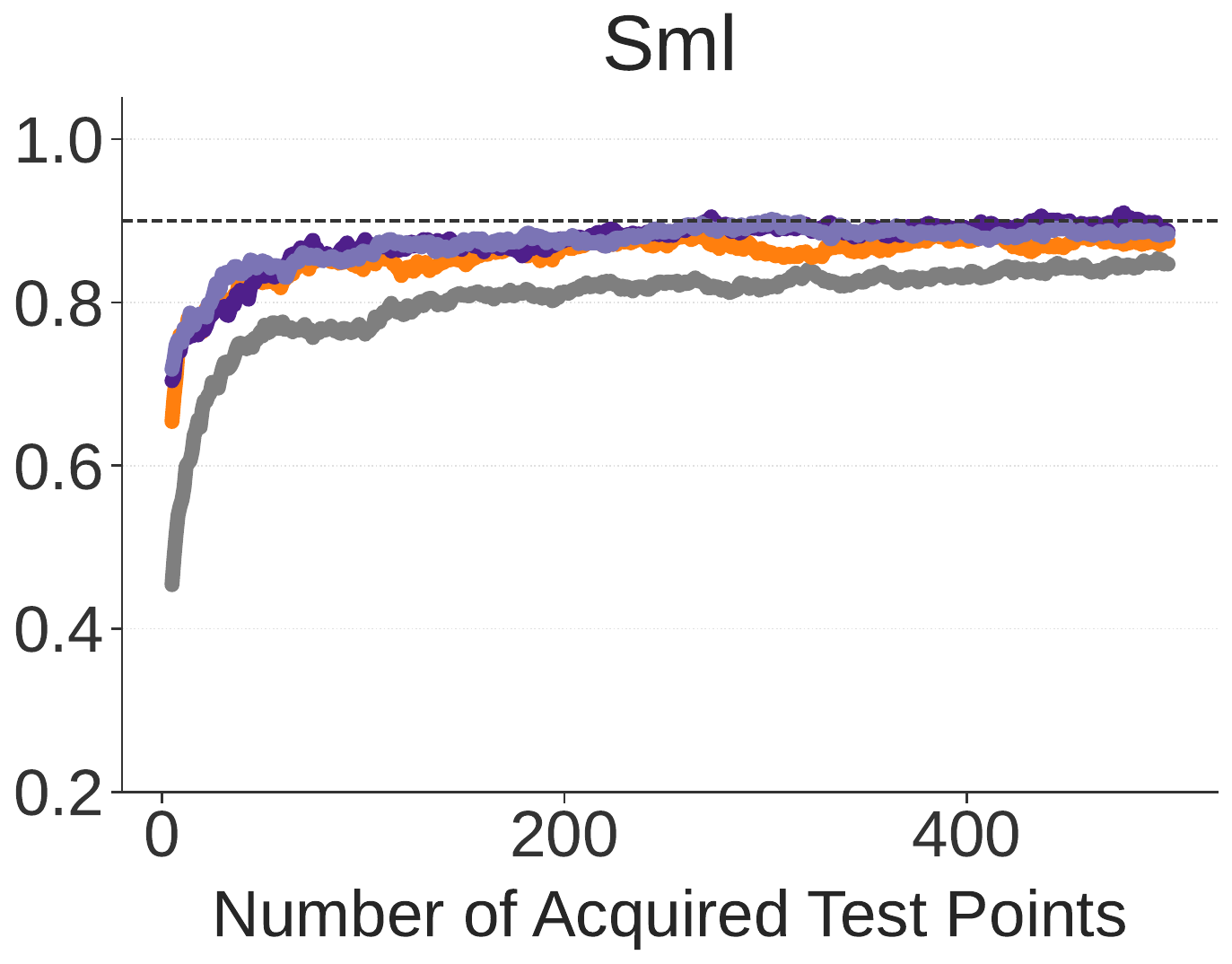}
        \label{fig:coverage_gas}
    \end{subfigure}
    \hspace{0.25cm}
    \begin{subfigure}[t]{0.22\textwidth}
        \centering        \includegraphics[width=\linewidth]{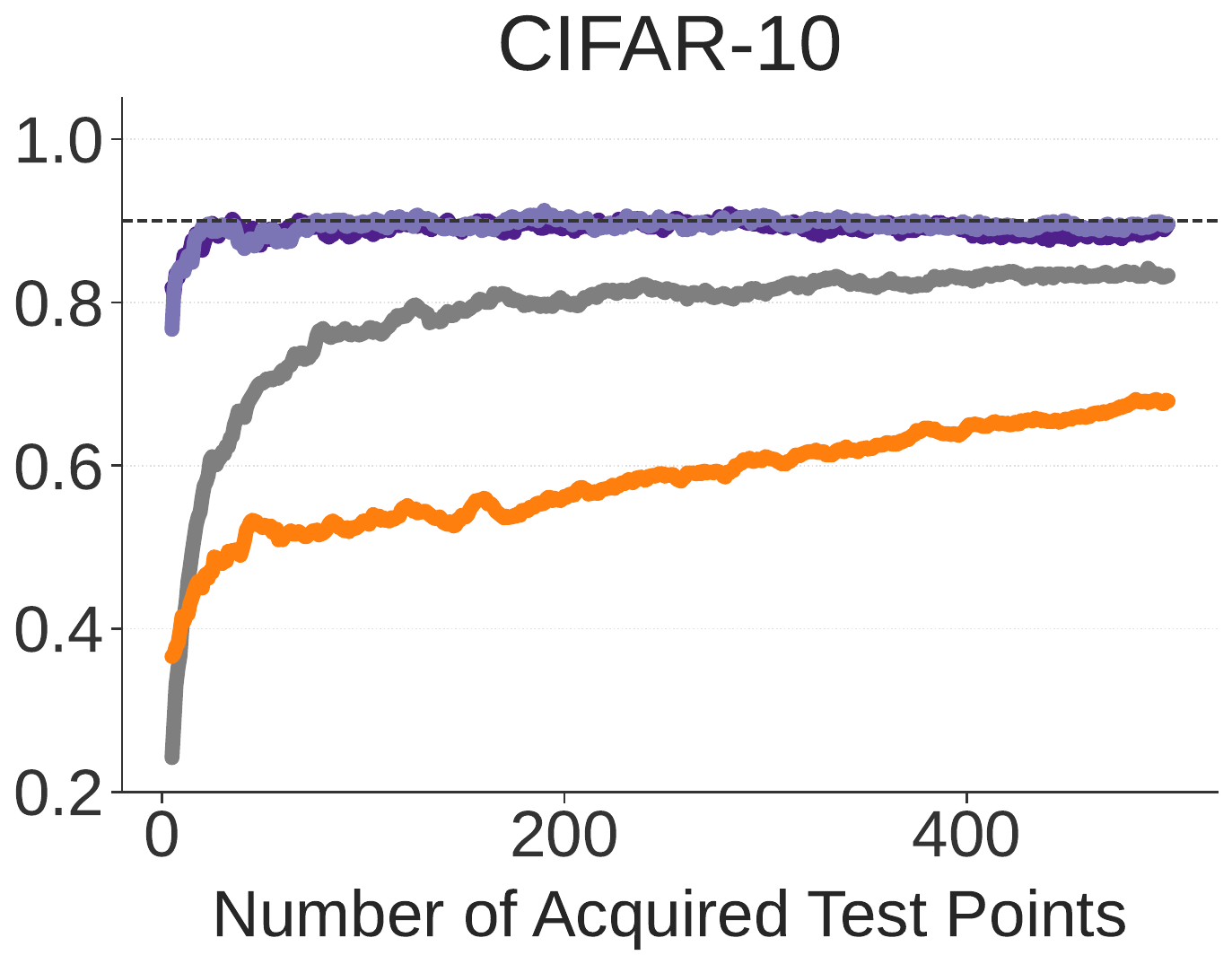}
        \label{fig:coverage_cifar10}
    \end{subfigure}
    \hspace{0.25cm}
    \begin{subfigure}[t]{0.22\textwidth}
        \centering
        \includegraphics[width=\linewidth]{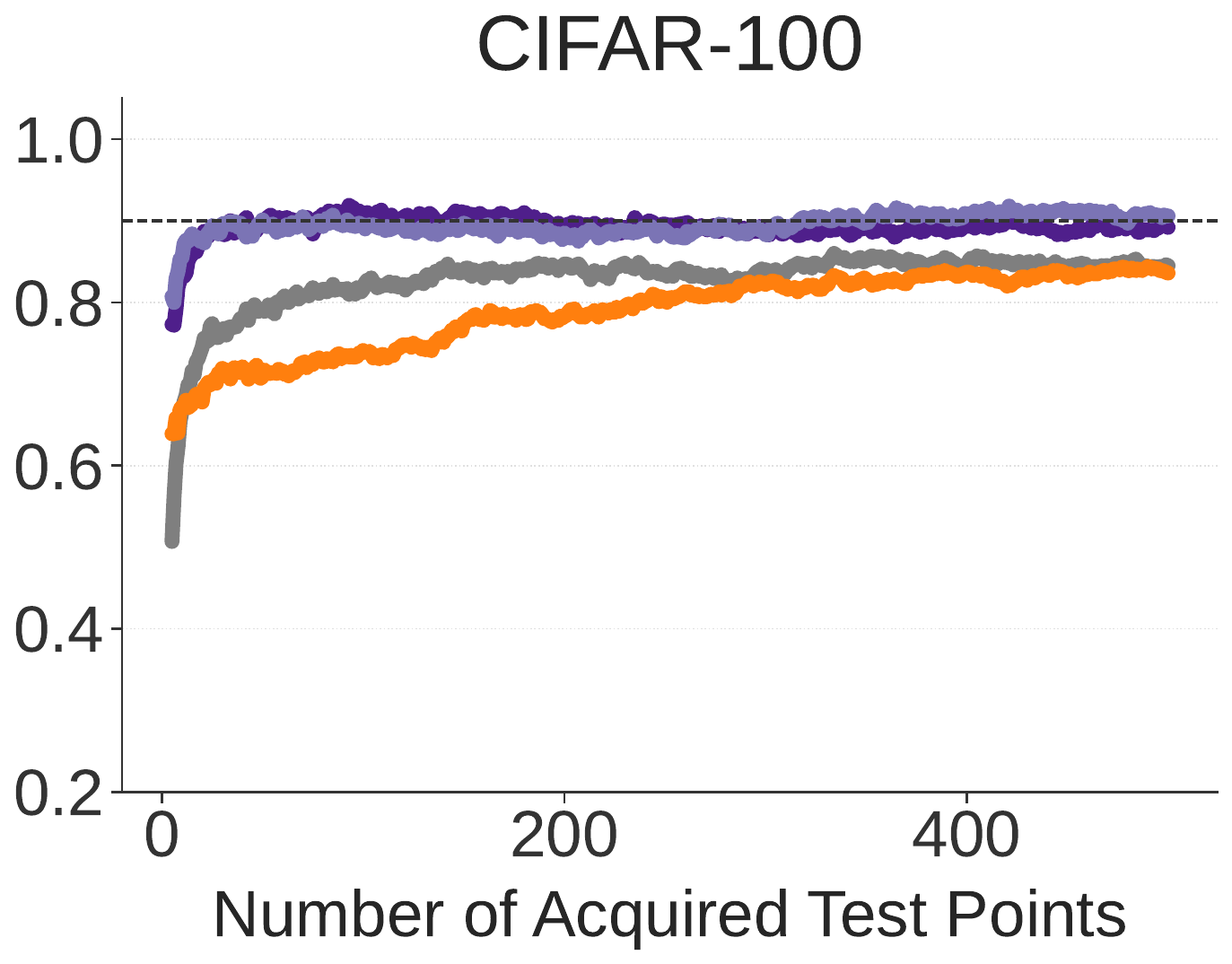}
        \label{fig:coverage_cifar100}
    \end{subfigure}
    \vspace{-0.5cm}
    \caption{Representative experiments on UCI and classification datasets for the coverage of \textcolor{violet}{PPAT}, \textcolor{gray}{Random}, \textcolor{orange}{LURE}. Plots show the coverage of asymptotic confidence intervals at a target error level of $\delta=0.1$ across 1000 trials. 
    The \textcolor{NavyBlue}{ASE} baseline is excluded as it does \textbf{not} provide asymptotically valid confidence intervals. 
    The width of the intervals are shown in Appendix \ref{app:ci_widths}.
    }
    \label{fig:coverage}
\end{figure}

%% file: related_work_main.tex
We review the most closely related work here, and defer a more extensive discussion to \S\ref{app:related_work}.

\paragraph{Active testing.}
PPAT builds directly on the active--testing line of work initiated by \cite{farquhar2021statistical}, who introduced the \emph{Levelled Unbiased Risk Estimator (LURE)}, and \cite{kossen2021active}, who developed it into a sample--efficient model--evaluation framework. A closely related approach is that of \emph{Active Surrogate Estimators (ASE)} \citep{kossen2022active}, which, like PPAT, exploit predictions over the entire pool. ASE, however, does so in a fundamentally different way: rather than correcting an importance--weighted estimator with observed labels, it learns a surrogate for the conditional loss and imputes losses across the pool via interpolation. While this can be successful at very low labelling budgets, ASE loses the theoretical guarantees of LURE, since surrogate misspecification or finite--sample estimation error enters the risk estimate directly. More recently, \cite{berrada2025scaling} scaled LURE--based active testing to large language model (LLM) evaluation using fixed, in--context--learning surrogates.
Our approach is complementary: it can incorporate the same in--context surrogates while additionally leveraging cheap black--box predictions to improve the estimator itself.

\paragraph{Earlier work on label--efficient evaluation.}
Earlier work considered more restricted sampling schemes. Stratification--based methods \citep{online_evaluating_classifiers, ji2021active, active_evaluation_classifier_large_datasets, kumar2018classifier} partitioned the test pool into strata, for example using a simple non–adaptive measure of model confidence, and sampled uniformly within each. PPAT is complementary here: once a stratification is fixed, it can be applied within each stratum to further improve label efficiency. A different line of work considered importance and Poisson sampling: \cite{active_risk_est} used importance sampling \emph{with} replacement, which is suboptimal relative to estimators that explicitly account for sampling without replacement (see Appendix~D of \cite{kossen2021active}); \cite{sample_efficient_model_evaluation} proposed a Poisson--sampling approach, but it is restricted to particular evaluation metrics, uses a non--adaptive acquisition rule, and relies on ratio--type estimates that are generally biased. \cite{nguyen2018active} studied a more specialised active--testing setting focused on human vetting of noisy labels.

%% file: future_work.tex
In this work, we introduced \textbf{Prediction--Powered Active Testing (PPAT)}, a label--efficient framework for risk estimation that combines the unbiased LURE estimator with a prediction--powered control variate. Rather than imputing labels, PPAT uses cheap proxy predictions to residualise the loss, preserving unbiasedness while reducing variance. 
We also established asymptotic normality for PPAT, yielding asymptotic confidence intervals that provide a principled quantification of uncertainty around the risk estimate. Our results show that PPAT improves risk estimation across tabular regression and image--classification tasks, while attaining target coverage with substantially fewer labels and producing narrower confidence intervals than competing approaches.

PPAT also suggests several natural directions for future work. One is to tighten our variance bound by identifying additional, practically reasonable assumptions on the proposals under which it more closely reflects the exact variance. Another is to complement our asymptotic confidence--interval theory with non--asymptotic guarantees that more directly characterise finite--budget behaviour. A further extension is to consider the setting where the proxy model is also updated as labels are acquired. Since PPAT only requires proxy losses over the pool and uses them through a centred control variate, our approach and theoretical results transfer naturally to this setting.


%% file: appendix/appendix.tex
\newpage

\section{Extended Related Work}
\label{app:related_work}
\input{related_work}

\newpage
\section{Further Details on Experimental Setup}
\label{app:further_details_experimental_setup}
\input{appendix/experimental_details}

\newpage
\input{appendix/additional_results/additional_results}

\newpage
\section{Additional Details}
\label{sec:additional_details}
\input{appendix/additional_details/explanation_of_general_surrogate_score}

\input{appendix/additional_details/additional_details}
\input{appendix/additional_details/ppi}
\input{appendix/additional_details/ase}

\newpage
\section{Mathematical Results}
\label{app:mathematical_results}
\input{appendix/mathematical_results/asymptoticresults}

%% file: related_work.tex
Here, we provide a more detailed discussion of related work.

\paragraph{Applications and extensions of active testing.}
Active testing has recently been extended to a range of settings. In the LLM evaluation setting, \cite{huang2026actracer} proposed \emph{AcTracer} which partitions the test pool via hidden--state representations of the target LLM and performs multi--stage sampling using online variance estimates and confidence traces. 
\cite{metaAT} addressed dense computer vision tasks such as segmentation and object detection. \cite{yu2023actively} incorporated active testing into model training, and \cite{ashury2024label, matsuura2023active} studied active testing for model selection.

\paragraph{Benchmark subsampling.}
A related but distinct line of work reduces evaluation cost by constructing smaller test sets from existing \emph{labelled} LLM benchmarks. \cite{maynez2023benchmarking} showed that uniform subsampling can preserve stable model rankings in some settings, and \cite{polo2024tinybenchmarks, saranathan2024dele, vivek2024anchor} developed more sophisticated dataset--reduction methods. These approaches are complementary to ours in their goal of reducing evaluation cost, but address a different setting: they assume an already labelled test set and aim to support comparisons across many models, whereas we study the problem of acquiring new labels to estimate the risk of a given model on a fixed unlabelled test pool.

\paragraph{Active measurement.}
Also related to our work is \emph{Active Measurement} (AM, \citep{hamilton2025active}), which also builds on LURE and combines model predictions with importance sampling without replacement, but targets scientific measurements (e.g., total bird counts) rather than a model's test risk. Our framework uses predictions in fundamentally different ways: AM uses AI predictions to define the proposal itself and iteratively refines the predictor with acquired labels, whereas PPAT uses a fixed proxy as a \emph{control variate} that residualises the loss, with the proposal constructed separately from a surrogate tailored to this residualised objective. 
Moreover, we note that AM assumes their proxy model can be updated as more labels are acquired, which is less general than our setup\footnote{Indeed, our approach and theoretical results naturally transfer to the setting where our proxy model can also be updated as more labels are acquired.}. In principle, AM's iterative proposal refinement and weighting schemes could also be combined with PPAT's control--variate estimator.

\paragraph{Active statistical inference.}
\emph{Active statistical inference} (ASI, \cite{zrnic2024active}) considers adaptive label sampling for statistical inference, proposing batch and sequential methods that yield valid confidence intervals for $M$--estimation problems. Recent work revisits the sequential mean-estimation setting and provides non--asymptotic guarantees for online policies that choose the probability of querying each ground--truth label \citep{sfyraki2026revisiting}. While risk estimation is a particular instance of $M$--estimation, their framework is not framed around model evaluation, and their sequential formulation is closer to online selective labelling than to pool--based acquisition: labels are acquired point--by--point from a stream rather than selected from the full remaining pool. This differs from our setting, where we focus on pool--based risk estimation and use black--box predictions as control variates within an unbiased active--testing estimator. We compare with ASI in \S\ref{app:comparison_with_asi} and find that we can outperform their approach in their ``active'' setting while achieving similar performance in the ``batch'' setting.

\paragraph{Prediction--powered inference and classical augmented estimators.}
\emph{Prediction--powered inference} (PPI, \cite{angelopoulos2023prediction}) was introduced to construct tighter confidence intervals in the semi--supervised setting by leveraging machine learning predictions, and has since been extended in several directions: PPI++ \citep{angelopoulos2023ppi++} provides a computationally efficient, loss--based formulation with a power--tuning parameter $\lambda$; Stratified PPI \citep{fisch2024stratified} improves efficiency via data stratification; and Cross--PPI \citep{zrnic2024cross} trains the predictor within the same prediction--powered pipeline.  PPI is also related to classical augmented estimators: in survey sampling, regression estimators use auxiliary variables to improve the efficiency of the estimator \citep{cochran1977sampling}; in causal inference, augmented inverse probability weighted estimators combine inverse--probability weighting with an outcome--model correction and are doubly robust when either the propensity model or the outcome model is correctly specified \citep{robins1994estimation,bang2005doubly}. 

Our work is related to these works in spirit but differs in both objective and setting: rather than semi--supervised inference with i.i.d.\ labelled and unlabelled data, we study active risk estimation over a finite test pool, using black--box predictions as a control variate inside an importance--weighted active--testing estimator.

%% file: appendix/experimental_details.tex
Here, we provide further details for the experiments in \S\ref{sec:experiments}.

\subsection{Section \ref{sub:main_results}}
\label{app:regression_exp_details}
\label{app:main_results_details}

Here, we provide details for the experiments in \S\ref{sub:main_results}.

\subsubsection{Regression Experiments}
\label{app:regression_exp-app_details}

\paragraph{Datasets.} Below, we provide further details on the UCI datasets used. We use $n$ to denote the full dataset size and $d$ to denote the number of features.
\begin{itemize}
    \item \texttt{Keggundirected} ($n=63{,}608$, $d=27$): A biology dataset derived from KEGG metabolic pathways represented as undirected reaction networks. The inputs are graph-level descriptors of pathway topology.

    \item \texttt{Keggdirected} ($n=48{,}827$, $d=20$): A related KEGG dataset where metabolic pathways are represented as directed relation networks. The features summarise structural properties of these directed pathway graphs.

    \item \texttt{Bike} ($n=17{,}379$, $d=17$): Hourly bike--rental data from the Capital Bikeshare system in 2011--2012. The covariates include weather, seasonal, and calendar information.

    \item \texttt{Sml} ($n=4{,}137$, $d=26$): A smart--home monitoring dataset collected over approximately 40 days in a domotic house. The inputs include indoor, outdoor, weather, humidity, lighting, and actuator measurements.
\end{itemize}

We load each of the above datasets from the \href{https://github.com/treforevans/uci_datasets}{\texttt{uci\_datasets}} repository. For each dataset, we randomly sample a fixed training set of 250 points which we use to train the model that we wish to evaluate. The remainder of the dataset is used as the test pool.

\paragraph{Models.} We choose $f$ to be a Gaussian process regression model with a Radial Basis Function (RBF) kernel. We optimise its hyperparameters using the marginal likelihood, with the lengthscale initialised at $1.0$. We use the \texttt{Scikit-Learn} implementation \citep{scikit-learn}.

For our surrogate model, we use a Bayesian linear regression model \citep{bayesridge} with the default hyperparameters from \texttt{Scikit-Learn} \citep{scikit-learn} and ablate with different choices of surrogate models in \S\ref{app:diff_surrogates}. We again use the \texttt{Scikit-Learn} implementation.

For the choice of $g$ for \colppat{}, we use the \texttt{TabPFN–2.5} \citep{grinsztajn2025tabpfn} foundation model and ablate with different choices of $g$ in \S\ref{app:diff_proxies}. To make predictions on our test pool, we use in--context learning \citep{grinsztajn2025tabpfn}, where we use our training data as the in--context data.

\paragraph{Baselines.}  We compare \colppat{} with \colure{} \citep{kossen2021active}, \ase{} \cite{kossen2022active}, and \random{} (sampling points randomly without replacement and using the unweighted estimator in \eqref{eqn:unweighted_estimator}). All approaches use the same surrogate model. For \colppat, we compare with the three different choices of $\lambda$ described in \S\ref{sub:lambda_estimation}.
To compute the acquisition scores for \ase{}, we use 100 Monte Carlo samples for each point: 10 samples from the posterior on the model parameters, and 10 samples from the model likelihood. See \S\ref{app:extended_background_ase} for more details about the \ase{}  acquisition function.

For \colppat{} and \colure{}, we compute the acquisition scores using the simplified formulas discussed in \S\ref{sub:simplification_proposal_mse_loss}, which can be readily computed using the mean prediction and predictive variance output by the surrogate model. For the \colppat{} plug--in estimate of $\lambda$, $\widehat{\lambda}_M$, we initialise it at $0.5$ and update it after every 100 acquired points using \eqref{eqn:plugin_lambda}.

\paragraph{Active testing setup and metrics.} We take our loss function to be the squared error . We use a budget of $M=500$ points and update our surrogate model after each newly acquired point. Each experiment is run for 1000 trials. Following \cite{kossen2021active, kossen2022active}, we report the median squared error. The mean error is presented in  \S\ref{app:mean_errors}. Here, the error is defined as the difference between our estimate of the risk and the true risk on the test pool.


\subsubsection{Classification Experiments}
\label{app:classification_exp_details}

\paragraph{Datasets.} Below, we provide further details about the classification dataset used. We provide the size of the train and test sets used as well as the resolution of the original images. Note that for \texttt{Tiny-ImageNet} we use the validation set as the test set, this is because the original test set is \emph{unlabelled}.
\begin{itemize}
    \item \texttt{CIFAR--10} (\cite{krizhevsky2009learning}, $50{,}000$ train, $10{,}000$ test; $32 \times 32 \times 3$): A natural--image classification dataset with $10$ object classes. Each class contains $6{,}000$ images.
    \item \texttt{CIFAR--100} (\cite{krizhevsky2009learning}, $50{,}000$ train, $10{,}000$ test; $32 \times 32 \times 3$): A more fine--grained version of \texttt{CIFAR--10} with $100$ object classes. Each class contains $600$ images, making the classification task harder.
    \item \texttt{Tiny--ImageNet} (\cite{Le2015TinyIV}, $100{,}000$ train, $10{,}000$ test; $64 \times 64 \times 3$): A reduced version of ImageNet with $200$ classes. It contains higher--resolution images and substantially more classes than \texttt{CIFAR-10} and \texttt{CIFAR-100}.
\end{itemize}
We load each dataset from the \href{https://github.com/pytorch/vision}{\texttt{torchvision}} \citep{torchvision2016} library.
For each dataset, we respect the original splits and use the provided training set and test set, where the test set is treated as the test pool and the training set is used to train $f$.
\paragraph{Models.} 
In line with standard practice in modern image--classification pipelines \citep{chen2020simple, grill2020bootstrap, radford2021learning, oquab2023dinov2, plested2026deep}, we use
a fixed pretrained image encoder to map each input \(\bx\) to an embedding
$e(\bx)$, and train lightweight prediction models on top of these embeddings. Here, the encoder $e$ is a map $e:\mathbb{R}^{d_1} \to \mathbb{R}^{d_2}$, where $d_1$ is the size of the input space, $d_2$ is the size of the embeddings and $d_2 \ll d_1$. In all of our
classification experiments, we use embeddings from a CLIP image encoder \citep{radford2021learning}. Specifically, we use the \texttt{ViT-B-32} encoder with model card \texttt{openai/clip-vit-base-patch32} from the \href{https://github.com/openai/CLIP/tree/main}{\texttt{CLIP}} library.

The model \(f\) whose risk we evaluate is a linear classifier trained on the CLIP
embeddings. This corresponds to the standard linear probing setup, where a linear
prediction head is trained on top of a frozen representation to evaluate the
usefulness of the learned features
\citep{alain2016understanding,kornblith2019better,chen2020simple}. For this, we use the logistic regression model class from \texttt{Scikit-Learn} \citep{scikit-learn}, where we set the maximum number of iterations to 300 and keep the remaining hyperparameters at their default values.

For the active--testing surrogate, it is important to use a model with well--calibrated
predictive uncertainties, since our acquisition rules depend on expectations
under the surrogate predictive distribution \(\pi_m(\cdot\mid \bx)\). We therefore use
a Laplace--approximated Bayesian neural network
\citep{bayesinterpolation, daxberger2021laplace}, implemented using the \href{https://github.com/runame/laplace-redux}{\texttt{laplace-redux}} library. The surrogate is an MLP trained on the CLIP embeddings, with three hidden layers
of width 128 for \texttt{CIFAR-10} and \texttt{CIFAR-100}, and five hidden layers
of width 128 for \texttt{Tiny-ImageNet}. We train the network for 200 epochs
using Adam \citep{kingma2014adam} with batch size 64, learning rate \(10^{-3}\), and weight decay
\(10^{-4}\). We then fit a last-layer Laplace approximation with a Kronecker--factored
structure \citep{ritter2018a}, optimising the prior precision by maximising the
marginal likelihood. For prediction, we use the linearised predictive distribution induced by the Laplace posterior, which gives an approximate Gaussian distribution over the latent network
outputs around the trained weights \citep{immer2021improving}. We then obtain class probabilities by Monte Carlo through the softmax link function using 100 samples. We again ablate with different surrogates in \S\ref{app:diff_surrogates}.

For the choice of $g$, we use the zero--shot predictions of a pretrained CLIP \citep{radford2021learning} model, specifically the \texttt{ViT-L-14} variant from the \texttt{sentence--transformers} library \citep{reimers-2019-sentence-bert} with model card \texttt{ViT-L-14::datacomp\_xl\_s13b\_b90k}. We again ablate with different choices of $g$ in \S\ref{app:diff_proxies}.
\paragraph{Baselines.} All our baselines use the same surrogate model and follow the same setup as for the regression datasets. To compute the acquisition scores for \ase{}, we use 100 samples from the posterior on the model parameters. Note that for classification problems we do not need to approximate the expectation over $Y$ as $y$ is discrete. 

\paragraph{Active testing setup and metrics.} We take our loss to be cross--entropy. The rest of our setup is the same as for the regression problems, with the exception that we do not update our surrogate model, in line with the image--classification experiments in \citep{kossen2021active, kossen2022active}.


\subsection{Section \ref{sec:ablations}}
\label{app:ablations_detials}

Here, we provide details for the experiments in Figs. \ref{fig:effect_of_estimator} and \ref{fig:effect_of_aq}.

\paragraph{The choice of the acquisition strategy.} The regression and classification experiments in Fig. \ref{fig:effect_of_aq} follow the same setup as detailed in \S\ref{app:regression_exp-app_details} and \ref{app:classification_exp_details}. The \textcolor{orange}{PPAT Acq. + LURE} approach uses the PPAT proposal \eqref{eqn:surrogate_proposal_2} and the LURE estimator \eqref{eq:lureestimator}; the \textcolor{gray}{PPAT Acq. + Empirical} approach again uses the same PPAT proposal but with the unweighted estimator \eqref{eqn:unweighted_estimator}; \colppat{} uses its standard estimator and acquisition rule. We use $\lambda=1$ for all the approaches.

\paragraph{The choice of the estimator.} The regression and classification experiments in Fig. \ref{fig:effect_of_estimator} follow the same setup as detailed in \S\ref{app:regression_exp-app_details} and \ref{app:classification_exp_details}. The \textcolor{orange}{LURE Acq. + PPI} approach uses the LURE proposal and the PPI estimator \eqref{eq:ppatestimator}\footnote{Note that we slightly abuse terminology here as \eqref{eq:ppatestimator} is not the PPI estimator described in \citep{angelopoulos2023prediction}, but instead the unbiased, importance weighted  version of it. For random sampling, the two coincide.}; the \textcolor{gray}{Random + PPI} approach uses random sampling in combination with the PPI estimator; \colppat{} uses its standard estimator and acquisition rule. We use $\lambda=1$ for all the approaches.


\subsection{Section \ref{sub:coverage}}
\label{app:coverage_details}

Here, we provide details for the coverage experiments in \S\ref{sub:coverage} (Fig. \ref{fig:coverage}). The regression and classification experiments follow the same setup as detailed in \S\ref{app:regression_exp-app_details} and \ref{app:classification_exp_details}. We use an error level of $\delta=0.1$ and compute the confidence intervals using \eqref{eqn:asymptotic_valid_cis}, noting that \random{} is a special case of \colure{} with $V_m=1, ~\forall m$. Coverage is computed with respect to the true test risk on the full pool for each labelling budget. We report the fraction of the 1000 trials for which the confidence interval contains this full--pool test risk. 


\subsection{Computational Resources and Licenses}
\label{app:resources}
All our experiments were run on a single NVIDIA H100 80GB GPU. Below, we provide the licenses for the packages and datasets used in this work.

\begin{table}[!h]
\centering
\footnotesize
\setlength{\tabcolsep}{3pt}
\renewcommand{\arraystretch}{1.12}
\caption{Software packages used in our experiments.}
\vspace{0.1cm}
\label{tab:software_licenses}

\begin{tabularx}{\linewidth}{p{0.18\linewidth}|p{0.32\linewidth}|p{0.15\linewidth}|X}
\hline
\textbf{Project} & \textbf{Citation} & \textbf{License} & \textbf{URL} \\
\hline
\texttt{NumPy} & \cite{harris2020array} & BSD & \url{https://numpy.org} \\
\hline
\texttt{SciPy} & \cite{virtanen2020scipy} & BSD & \url{https://scipy.org} \\
\hline
\texttt{scikit-learn} & \cite{scikit-learn} & BSD & \url{https://scikit-learn.org} \\
\hline
\texttt{Matplotlib} & \cite{hunter2007matplotlib} & PSF--based & \url{https://matplotlib.org} \\
\hline
\texttt{pandas} & \cite{pandas2020zenodo}; \cite{mckinney2010data} & BSD & \url{https://pandas.pydata.org} \\
\hline
\texttt{uci\_datasets} & \cite{evans_uci_datasets} & MIT & \url{https://github.com/treforevans/uci_datasets} \\
\hline
\texttt{CLIP} & \cite{radford2021learning} & MIT & \url{https://github.com/openai/CLIP} \\
\hline
\texttt{sentence-\allowbreak transformers} & \cite{reimers-2019-sentence-bert} & Apache-2.0 & \url{https://www.sbert.net} \\
\hline
\texttt{torchvision} & \cite{torchvision2016} & BSD & \url{https://pytorch.org/vision} \\
\hline
\texttt{laplace-redux} & \cite{daxberger2021laplace} & MIT & \url{https://github.com/runame/laplace-redux} \\
\hline
\texttt{wandb} & \cite{biewald2020experiment} & MIT & \url{https://wandb.ai} \\
\hline
\texttt{transformers} & \cite{wolf2020transformers} & Apache-2.0 & \url{https://huggingface.co/docs/transformers} \\
\hline
\texttt{timm} & \cite{wightman2019timm} & Apache-2.0 & \url{https://huggingface.co/docs/timm} \\
\end{tabularx}
\end{table}

%% file: appendix/additional_results/additional_results.tex
\begin{table}[!h]
\centering
\caption{Datasets used in our experiments.}
\label{tab:dataset_licenses}
\vspace{0.15cm}

\small
\setlength{\tabcolsep}{4pt}
\renewcommand{\arraystretch}{1.15}
\begin{tabular}{p{0.27\linewidth}|p{0.17\linewidth}|p{0.18\linewidth}|p{0.30\linewidth}}
\hline
\textbf{Dataset} & \textbf{Citation} & \textbf{License} & \textbf{URL} \\
\hline
\texttt{CIFAR-10}
&
\cite{krizhevsky2009learning}
&
Not specified
&
\url{https://www.cs.toronto.edu/~kriz/cifar.html}
\\
\hline
\texttt{CIFAR-100}
&
\cite{krizhevsky2009learning}
&
Not specified
&
\url{https://www.cs.toronto.edu/~kriz/cifar.html}
\\
\hline
\texttt{Tiny--ImageNet}
&
\cite{Le2015TinyIV}
&
Not specified
&
\url{http://cs231n.stanford.edu/tiny-imagenet-200.zip}
\\
\hline
\texttt{keggdirected}
&
\cite{naeem2011keggdirected}
&
CC BY 4.0
&
\url{https://doi.org/10.24432/C5CK52}
\\
\hline
\texttt{keggundirected}
&
\cite{naeem2011keggundirected}
&
CC BY 4.0
&
\url{https://doi.org/10.24432/C5G609}
\\
\hline
\texttt{bike}
&
\cite{fanaee2013bike}
&
CC BY 4.0
&
\url{https://doi.org/10.24432/C5W894}
\\
\hline
\texttt{sml}
&
\cite{romeu2014sml2010}
&
CC BY 4.0
&
\url{https://doi.org/10.24432/C5RS3S}
\\
\end{tabular}
\end{table}

\newpage
\section{Additional Results}
\label{app:additional_results}

\subsection{Empirical Verification of $\lambda^\dagger$}
\label{app:verifying_choice_of_lamba}

Here, we empirically validate the suitability of our target $\lambda^\dagger$. To do this, we compare \colppat{} run with the fixed value of $\lambda^\dagger$, where $\lambda^\dagger$ is computed on the full pool, with \colppat{} run with $\lambda \in \{0.1, 0.25, 0.5, 0.75, 1.0, 1.5, 2.5, 5.0\}$. For the experiments, we use one regression dataset (\texttt{Keggdirected}) and one classification dataset (\texttt{CIFAR-10}), and follow the same setup in \S\ref{app:regression_exp-app_details} and \S\ref{app:classification_exp_details}.

\begin{figure}[!h]
    \centering
    \includegraphics[width=.32\textwidth]{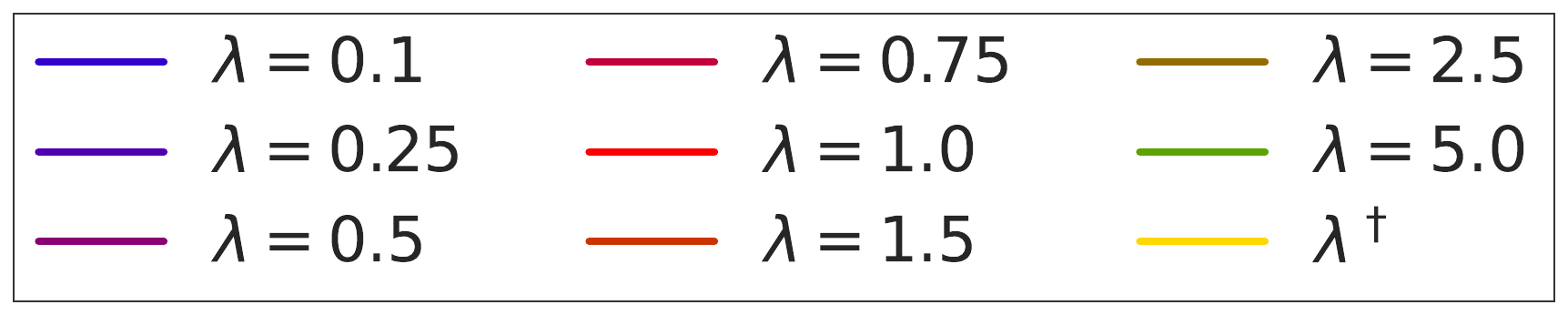}
    \vspace{0.75em} 
    
    \begin{subfigure}[t]{0.25\textwidth}
        \centering
        \includegraphics[width=\linewidth]{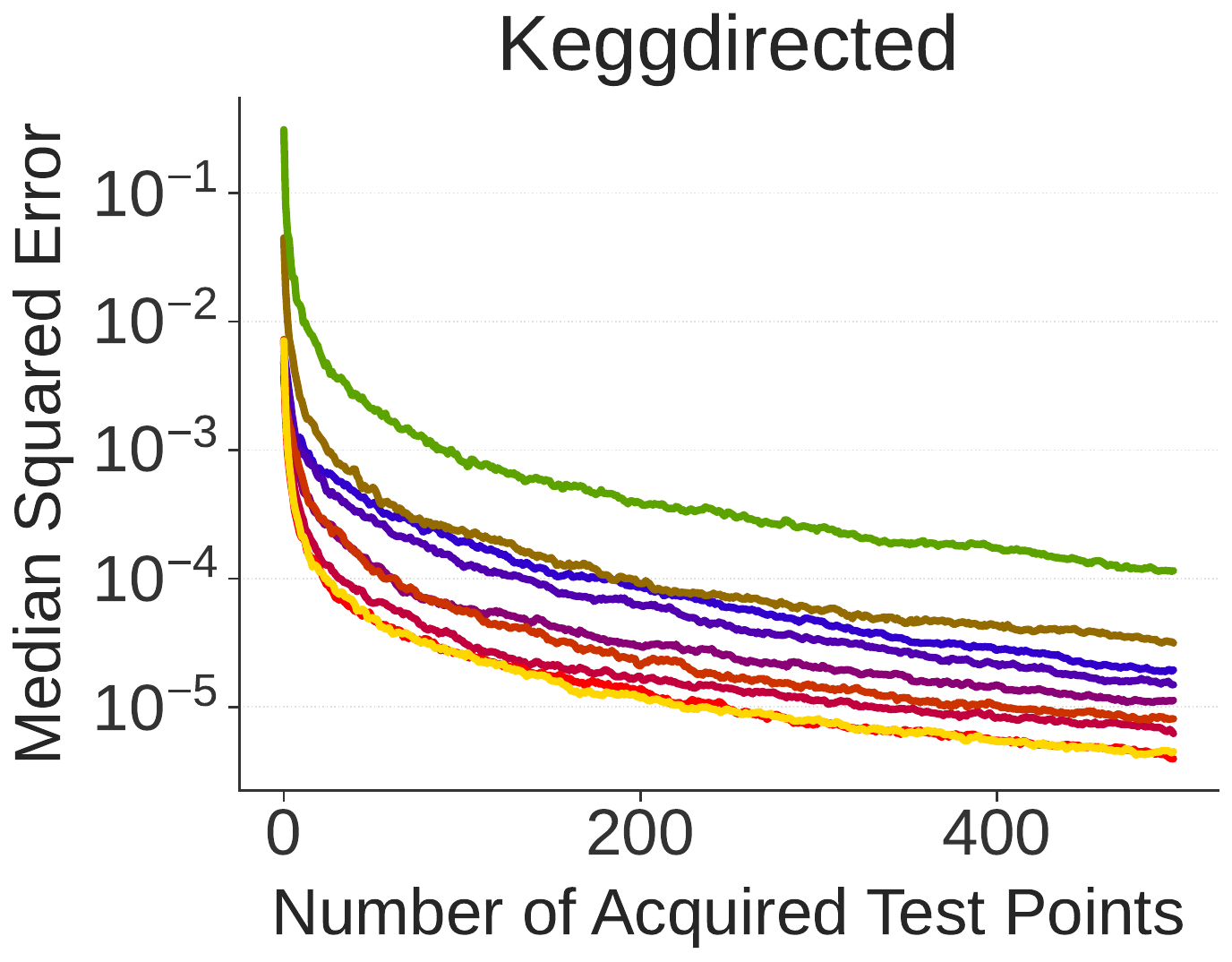}
        \label{fig:kegg_justifying_lambda_dagger}
    \end{subfigure}
    \hspace{0.5cm}
    \begin{subfigure}[t]{0.25\textwidth}
        \centering
        \includegraphics[width=\linewidth]{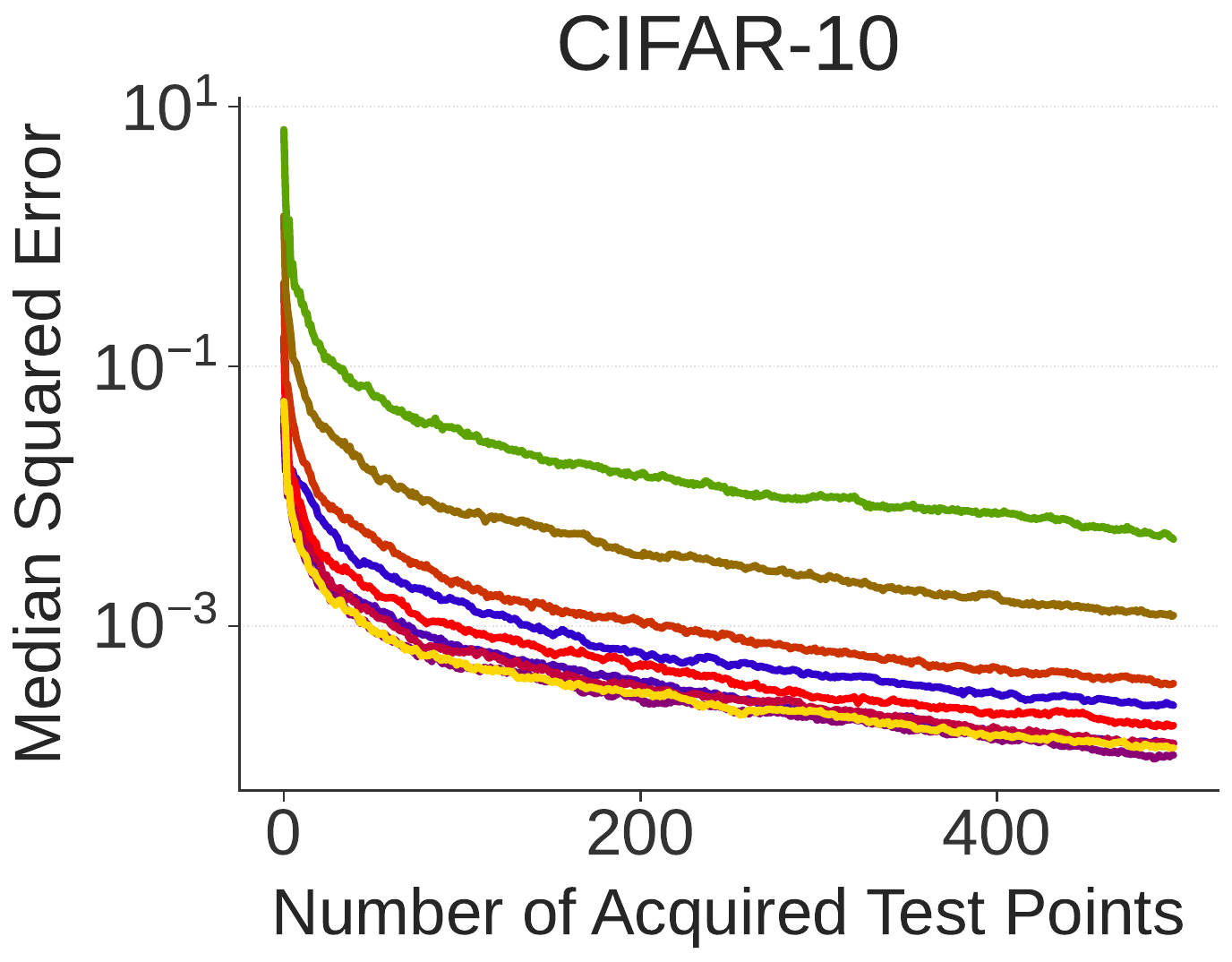}
        \label{fig:cifar_justfying_lambda_dagger}
    \end{subfigure}
    \caption{Plot of \colppat{} for a range of fixed $\lambda$ values and $\lambda^\dagger$ for the datasets \texttt{Keggdirected} and \texttt{CIFAR-10}. $\lambda^\dagger$ is computed on the full pool. The plots report the median squared error over 1000 trials.}
    \label{fig:justfying_lambda_dagger}
\end{figure}

Fig. \ref{fig:justfying_lambda_dagger} plots the median squared error across 1000 trials for all the different values of $\lambda$ described previously as well as $\lambda^\dagger$. As we can see, $\lambda^\dagger$ achieves the same median squared error as the best $\lambda$ for both datasets. Besides the discussion in \S\ref{sub:lambda_estimation}, this further supports the use of $\lambda^\dagger$ as a suitable $\lambda$ to target instead of the intractable $\lambda^\star$.


\subsection{Additional Ablations}
\label{app:further_ablations}

Here, we ablate with different proxy models and different surrogate models.


\subsubsection{Different Proxies}
\label{app:diff_proxies}

To ablate with different proxy models, we use one regression dataset (\texttt{Keggdirected}) and one classification dataset (\texttt{CIFAR-10}), and follow the same setup as in \S\ref{app:regression_exp-app_details} and \ref{app:classification_exp_details}. For the regression dataset, we compare with the \texttt{TabPFN--2.0} foundation model, where we make predictions on the pool using in--context learning with the context set being our training set in the same way as before. For the classification case, we compare with the \texttt{ViT-B-32} and \texttt{ViT-B-16} variants of the CLIP model from the \texttt{sentence-transformers} library\footnote{Their corresponding model cards are \texttt{ViT-B-32::datacomp\_xl\_s13b\_b90k} and \texttt{ViT-B-16::datacomp\_xl\_s13b\_b90k}} \citep{reimers-2019-sentence-bert}. 

From Fig. \ref{fig:diff_proxies}, we see that across both datasets, \colppat{} remains robust to different choices of the proxy model. Indeed, for all proxy choices, \colppat{} still retains the lowest median squared error for all choices of $\lambda$.

\begin{figure}[!h]
    \centering
    \includegraphics[width=.75\textwidth]{figures/uci/uci_main_legend.pdf}
    \vspace{0.75em} 
    
    \begin{subfigure}[t]{0.25\textwidth}
        \centering
        \includegraphics[width=\linewidth]{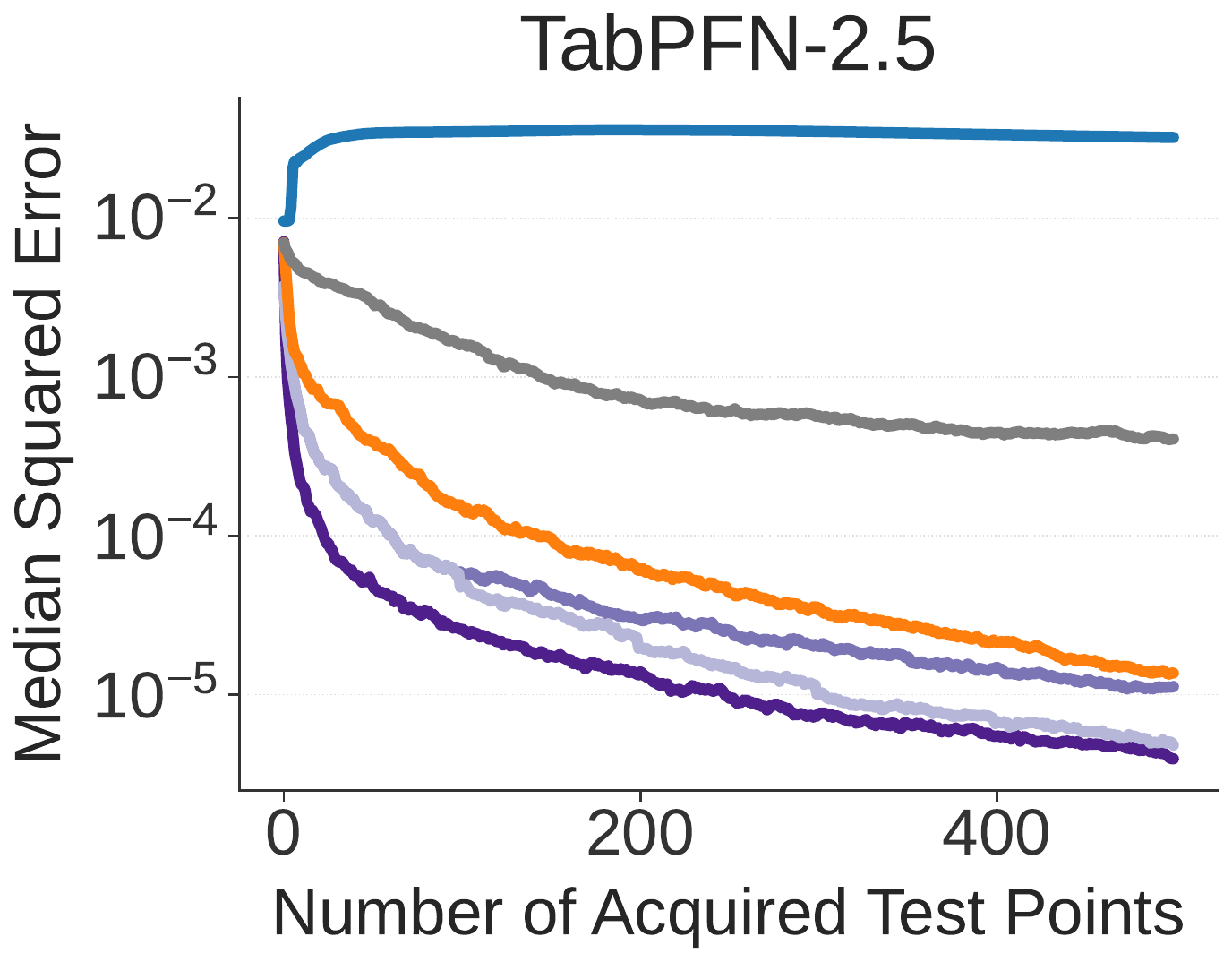}
        \label{fig:tabpfn2.5}
    \end{subfigure}
    \hspace{0.5cm}
    \begin{subfigure}[t]{0.25\textwidth}
        \centering
        \includegraphics[width=\linewidth]{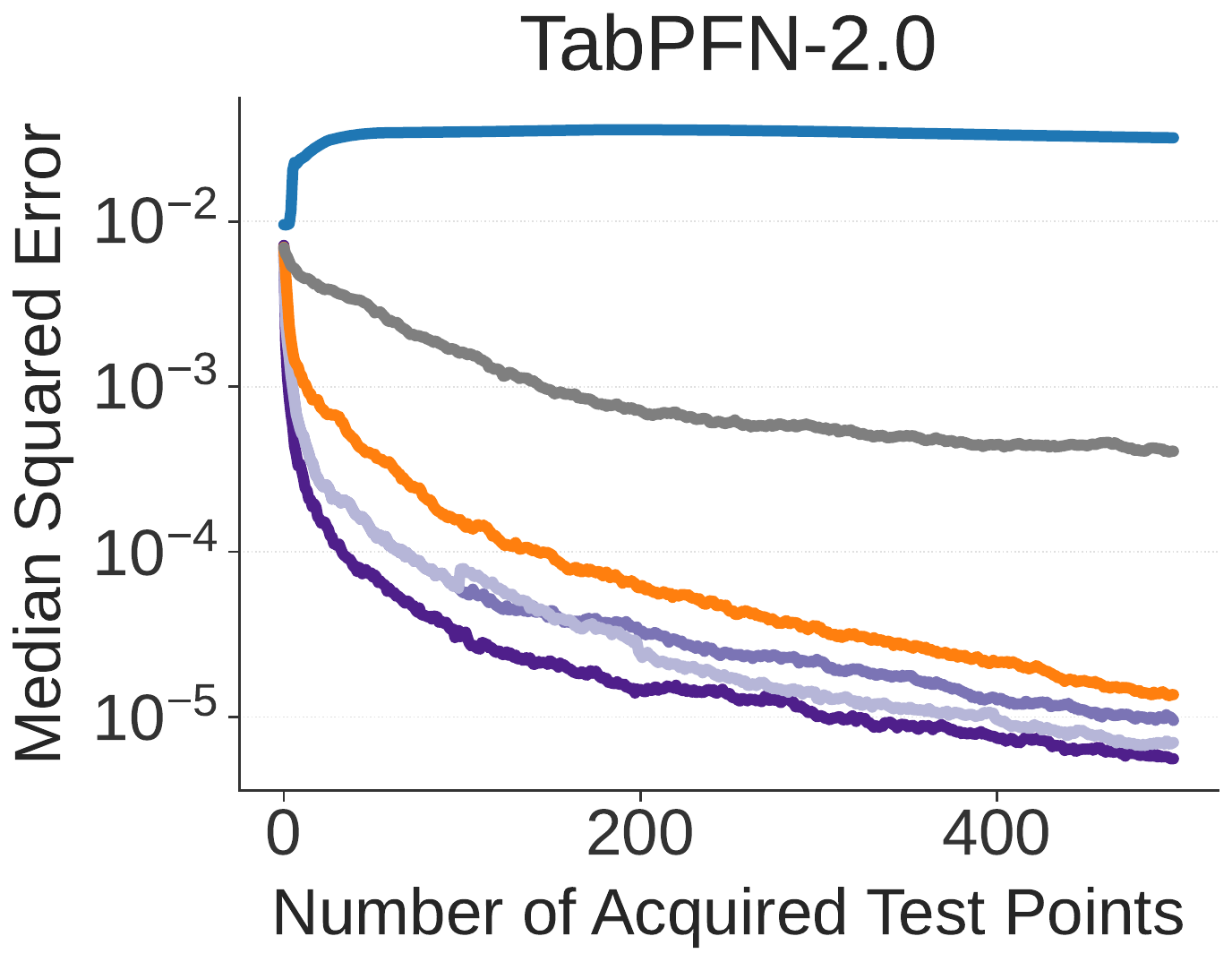}
        \label{fig:tabpfn2.0}
    \end{subfigure}
    \vspace{0.5em} 
    \\
    \begin{subfigure}[t]{0.25\textwidth}
        \centering
        \includegraphics[width=\linewidth]{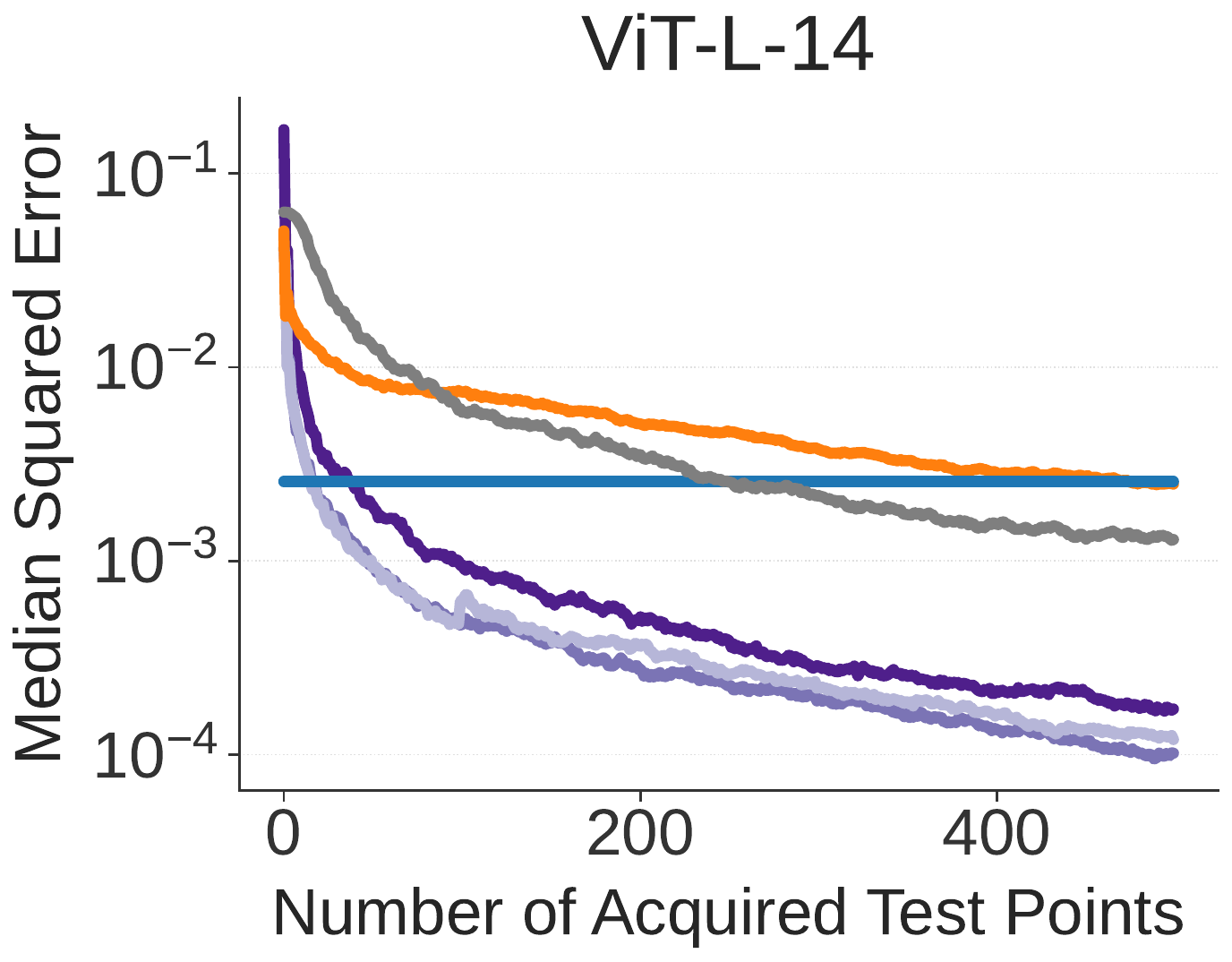}
        \label{fig:vit-l-14}
    \end{subfigure}
    \hspace{0.1cm}
    \begin{subfigure}[t]{0.25\textwidth}
        \centering
        \includegraphics[width=\linewidth]{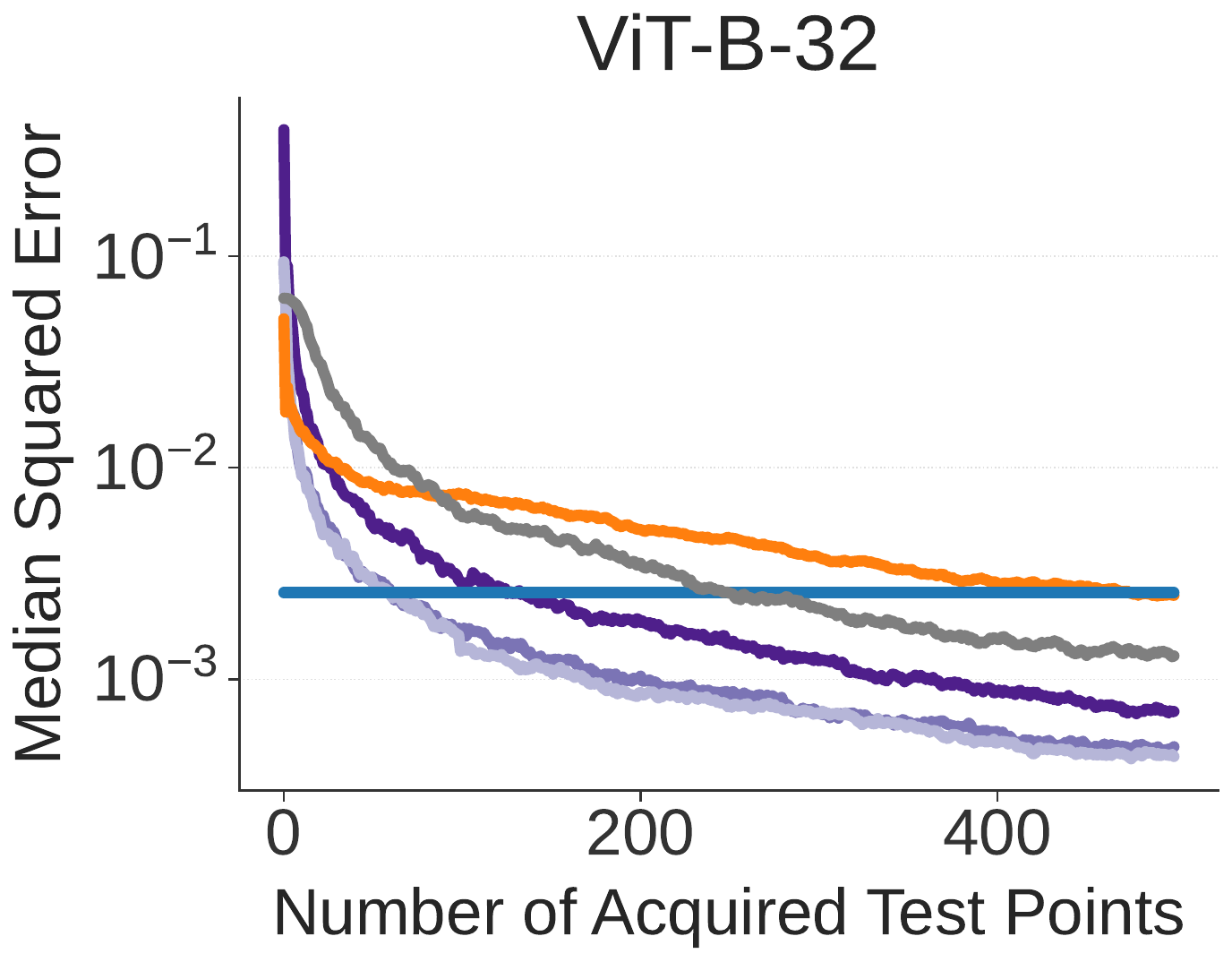}
        \label{fig:vit-b-32}
    \end{subfigure}
    \hspace{0.1cm}
    \begin{subfigure}[t]{0.25\textwidth}
        \centering
        \includegraphics[width=\linewidth]{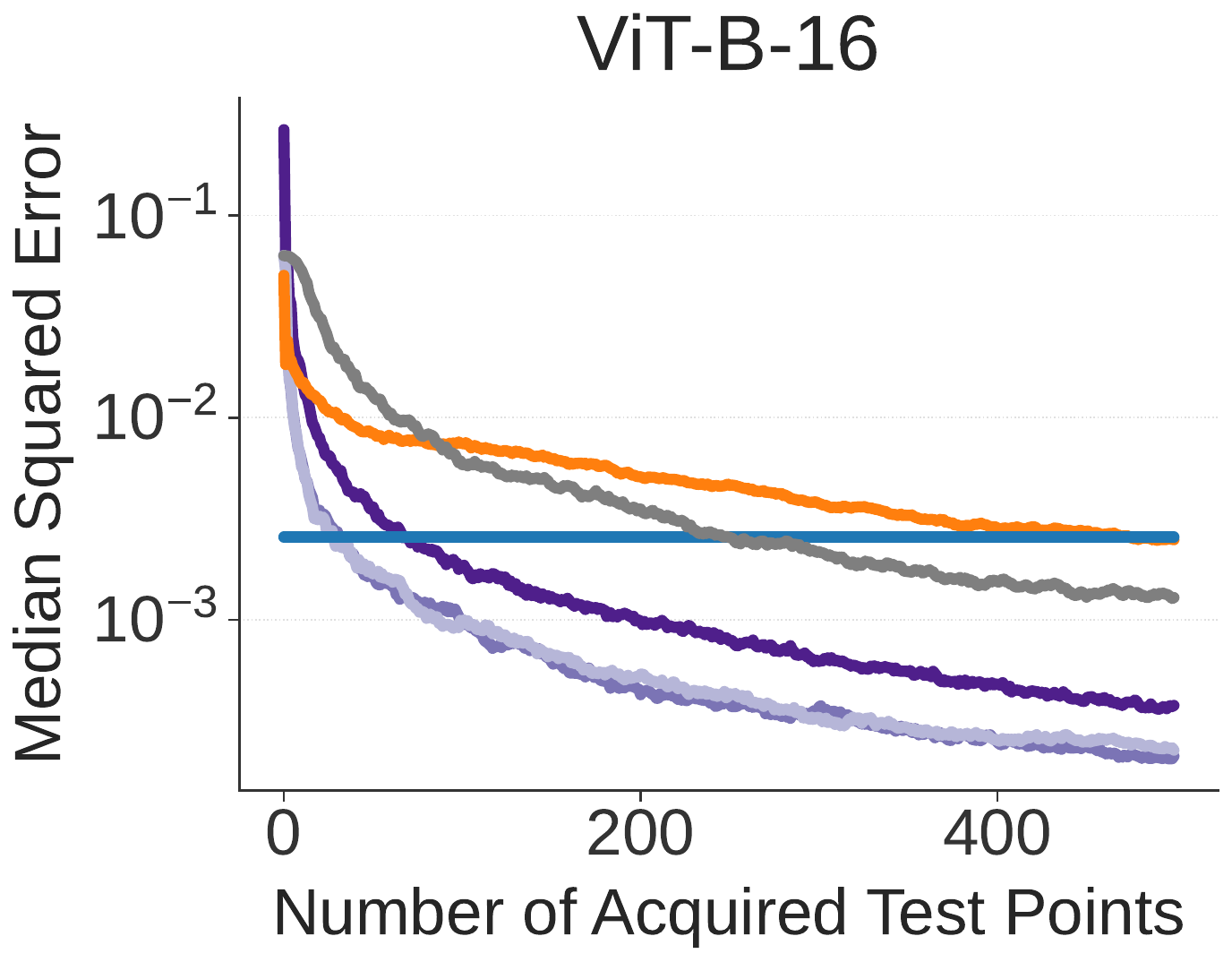}
        \label{fig:vit-b-16}
    \end{subfigure}    

    \caption{Ablation with different proxy models comparing \textcolor{violet}{PPAT} with \textcolor{gray}{Random}, \textcolor{NavyBlue}{ASE}, \textcolor{orange}{LURE}. \\ \textbf{Top row}: results on the \texttt{Keggdirected} dataset with \texttt{TabPFN--2.5} and \texttt{TabPFN--2.0} as the proxies. \\ \textbf{Bottom row}: results on the \texttt{CIFAR-10} dataset with the \texttt{Vit-L-14}, \texttt{Vit-B-32}, \texttt{Vit-B-16} CLIP models as the proxies.
    Plots show the median squared error across 1000 trials. 
    }
    \label{fig:diff_proxies}
\end{figure}

\subsubsection{Different Surrogate Models}
\label{app:diff_surrogates}

To ablate with different surrogate models, we use one regression dataset (\texttt{Keggdirected}) and one classification dataset (\texttt{CIFAR-10}), and follow the same setup as in \S\ref{app:regression_exp-app_details} and \ref{app:classification_exp_details}. For the regression dataset, we use a random forest regression model from \texttt{Scikit-Learn} \citep{scikit-learn} with 300 trees, with the remainder of the hyperparameters at their default values. Similarly, for the classification dataset, we use a random forest classification model from \texttt{Scikit-Learn} \citep{scikit-learn} with 500 trees, with the remainder of the hyperparameters at their default values.

From Fig. \ref{fig:diff_surr}, we see that across both datasets \colppat{} remains robust to different choices of surrogate models. Indeed, for all surrogate choices, \colppat{} still retains the lowest median squared error for all choices of $\lambda$, except for the random forest model on \texttt{CIFAR-10} where $\lambda=1$ performs similarly to \colure{}.

Moreover, on the \texttt{Keggdirected} dataset, \ase{} outperforms \random{} when using the random forest surrogate and improves as more labels are acquired. This suggests that the poor behaviour of \ase{} with the Bayesian linear model surrogate is not only due to active sampling itself, but also due to surrogate misspecification: because \ase{} estimates the risk through the surrogate predictive model, errors in this model can translate directly into bias in the risk estimate. The random forest surrogate appears to approximate the losses on this tabular regression task more accurately, reducing this bias enough for the additional acquired labels to improve the estimate. On \texttt{CIFAR--10}, however, \ase{} performs substantially worse than \random{} when using the random forest surrogate. This suggests that the random forest is not sufficiently expressive for this CLIP--embedding classification task, so the bias introduced by the \ase{} estimate dominates.

\begin{figure}[!h]
    \centering
    \includegraphics[width=.75\textwidth]{figures/uci/uci_main_legend.pdf}
    \vspace{0.75em} 
    
    \begin{subfigure}[t]{0.25\textwidth}
        \centering
        \includegraphics[width=\linewidth]{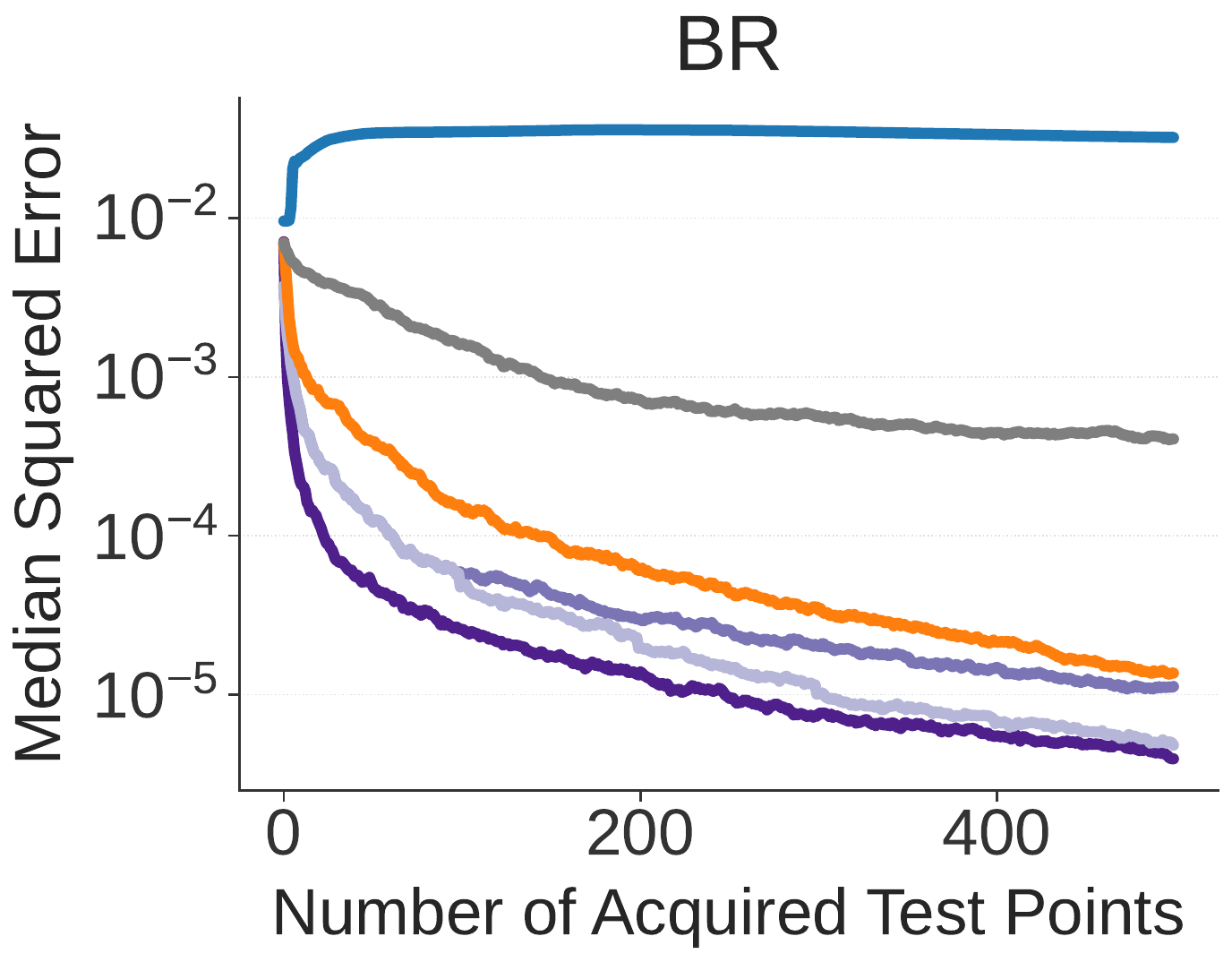}
        \label{fig:br_surr_kegg}
    \end{subfigure}
    \hspace{0.5cm}
    \begin{subfigure}[t]{0.25\textwidth}
        \centering
        \includegraphics[width=\linewidth]{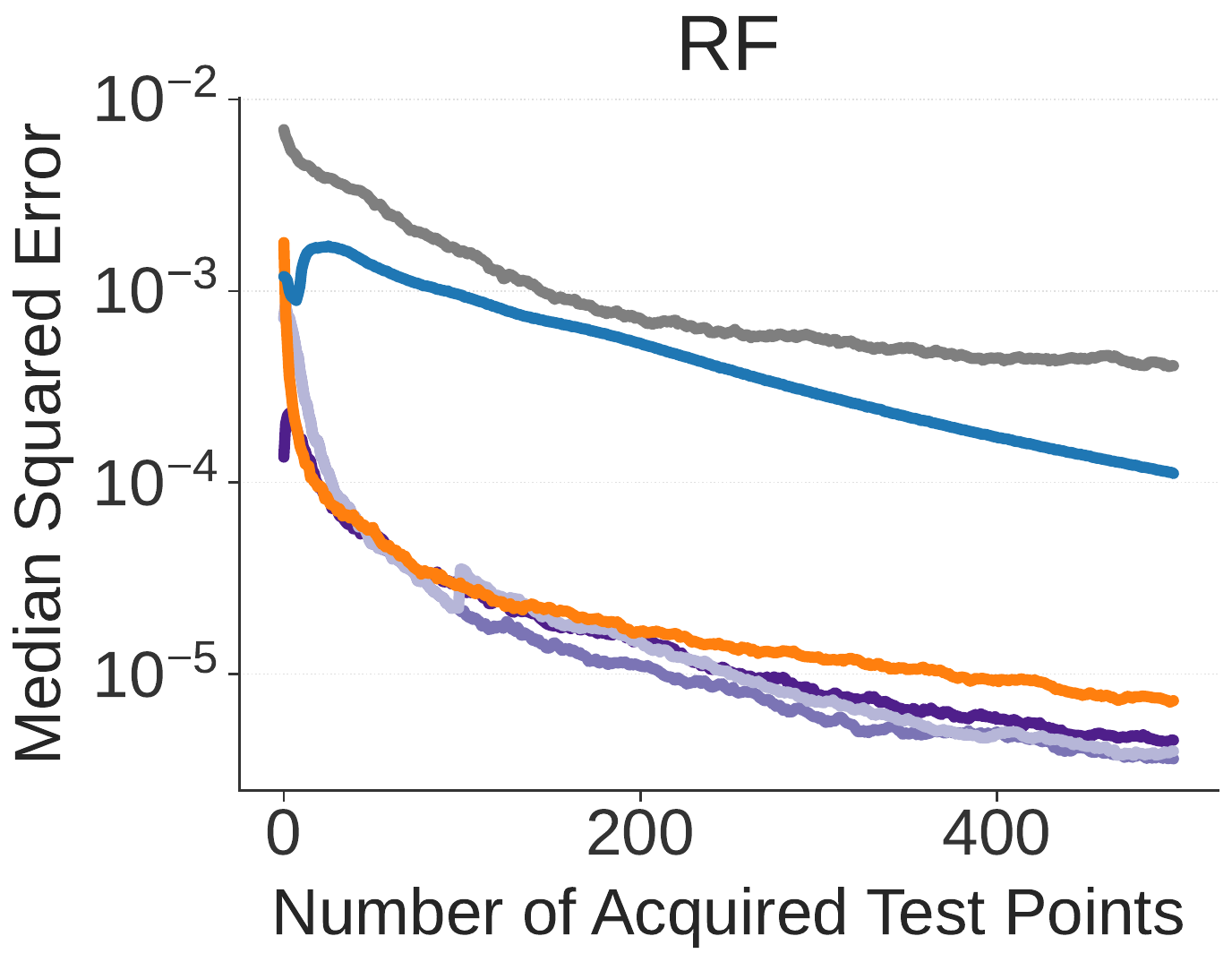}
        \label{fig:rf_surr_kegg}
    \end{subfigure}
    \vspace{0.5em} 
    \\
    \begin{subfigure}[t]{0.25\textwidth}
        \centering
        \includegraphics[width=\linewidth]{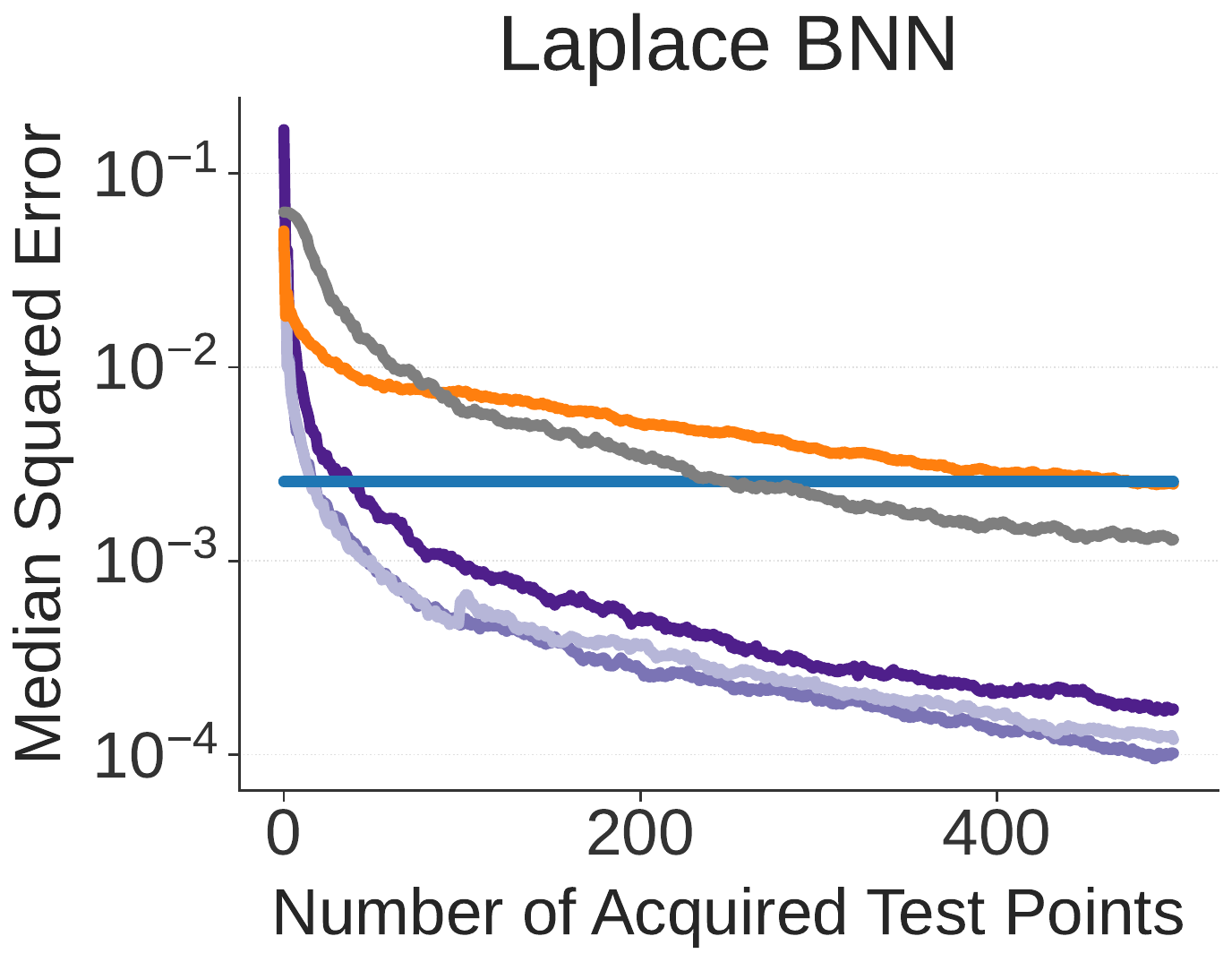}
        \label{fig:laplace_surr}
    \end{subfigure}
    \hspace{0.1cm}
    \begin{subfigure}[t]{0.25\textwidth}
        \centering
        \includegraphics[width=\linewidth]{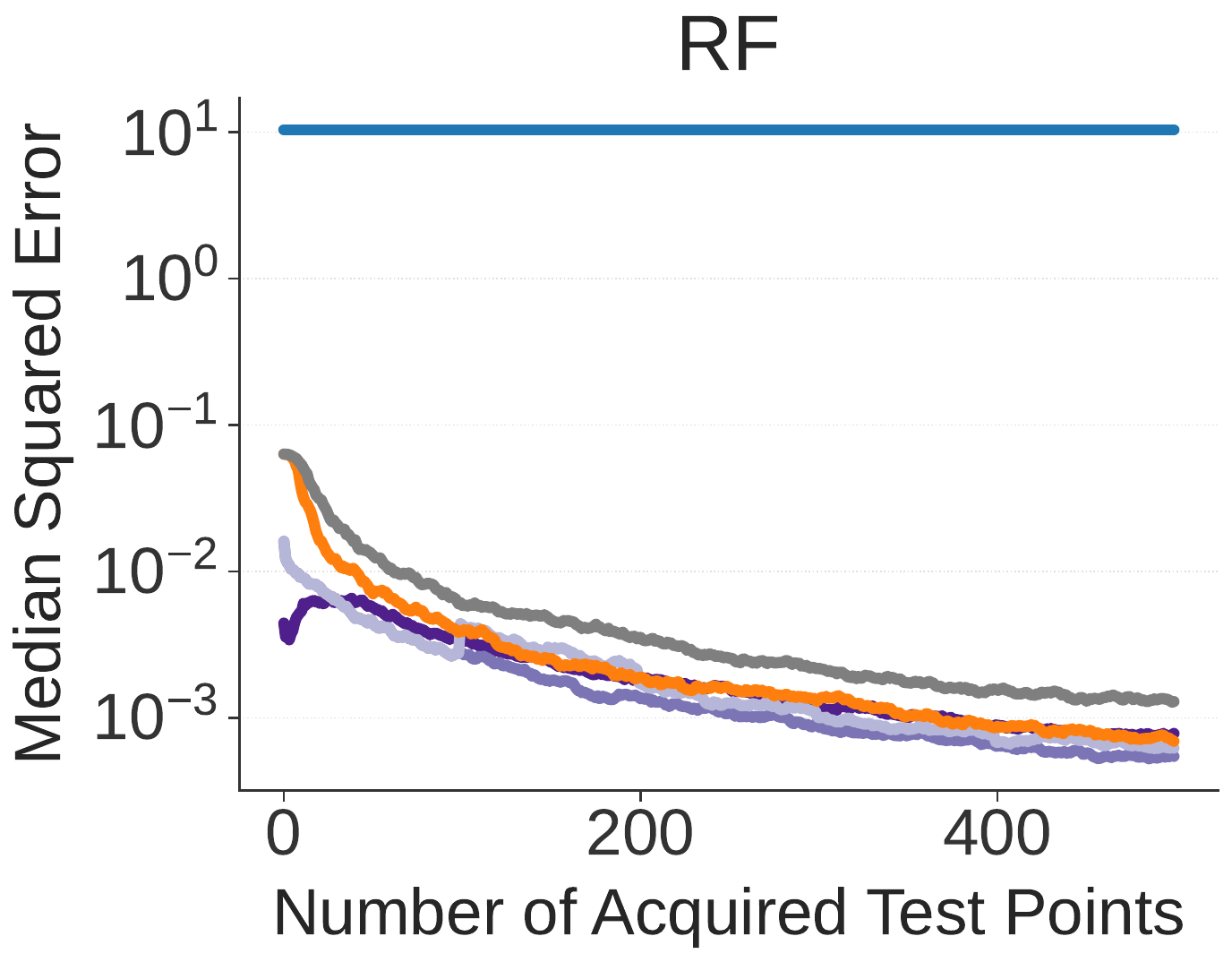}
        \label{fig:rf_surr_cifar}
    \end{subfigure}   

    \caption{Ablation with different surrogate models comparing \textcolor{violet}{PPAT} with \textcolor{gray}{Random}, \textcolor{NavyBlue}{ASE}, \textcolor{orange}{LURE}. \textbf{Top row}: results on the \texttt{Keggdirected} dataset with a  Bayesian linear regression (BR) and random forest (RF) model as the surrogate. 
    \textbf{Bottom row}: results on the \texttt{CIFAR-10} dataset with a Laplace BNN and random forest model as the surrogate.
    Plots show the median squared error across 1000 trials. Note that \textcolor{NavyBlue}{ASE} results in \textbf{biased} estimates of the risk.}
    \label{fig:diff_surr}
\end{figure}

\newpage
\subsection{Additional Metrics}
\label{app:additional_matrics}

Here, we present additional metrics for the experiments in \S\ref{sec:experiments}.

\subsubsection{Mean Error}
\label{app:mean_errors}

Figs. \ref{fig:main_uci_mean_err}, \ref{fig:classification_datasets_err}, \ref{fig:effect_of_estimator_err}, \ref{fig:effect_of_aq_err} show the mean error (bias) for the experiments in Figs.  \ref{fig:main_uci}, \ref{fig:classification_datasets}, \ref{fig:effect_of_estimator}, and \ref{fig:effect_of_aq} respectively. The mean error here is the mean difference between the estimate of the risk and the true test risk on the full pool.

\begin{figure}[!h]
    \centering    \includegraphics[width=.75\textwidth]{figures/uci/uci_main_legend.pdf}
    \vspace{0.75em}    
    
    \begin{subfigure}[t]{0.22\textwidth}
        \centering
        \includegraphics[width=\linewidth]{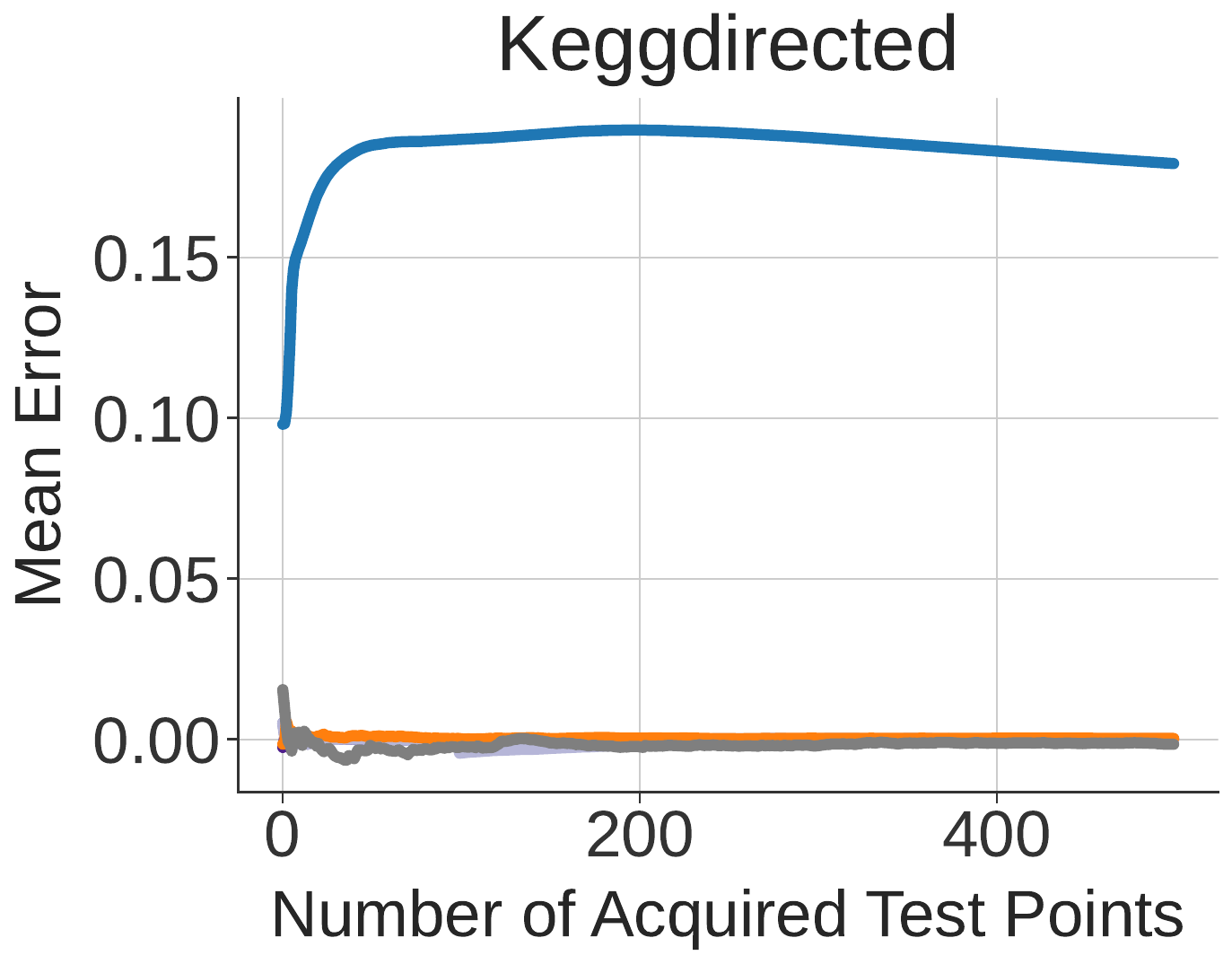}
        \label{fig:main_keggdir_err}
    \end{subfigure}
    \hspace{0.25cm}
    \begin{subfigure}[t]{0.22\textwidth}
        \centering
        \includegraphics[width=\linewidth]{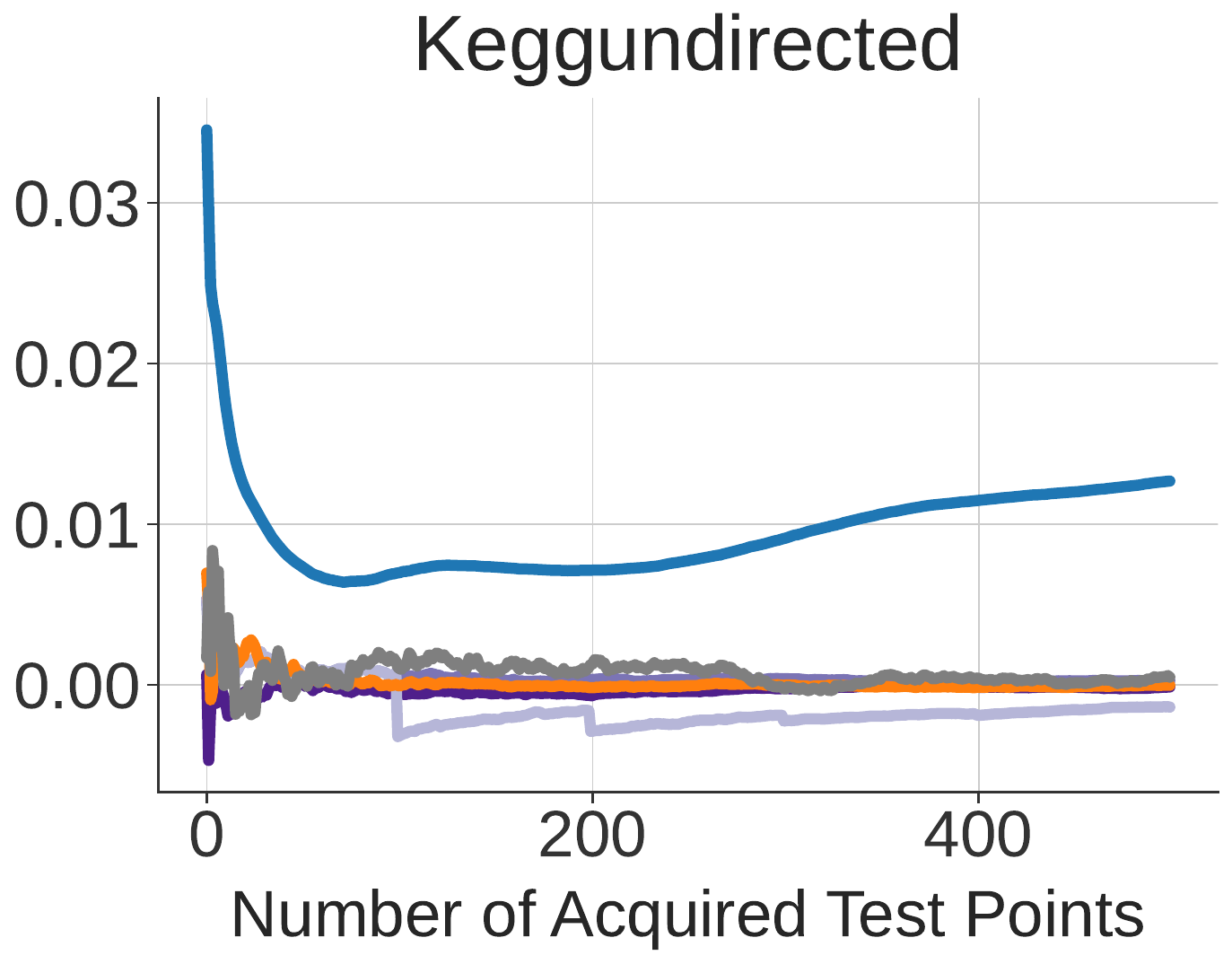}
        \label{fig:main_keggundir_err}
    \end{subfigure}
    \hspace{0.25cm}
    \begin{subfigure}[t]{0.22\textwidth}
        \centering        \includegraphics[width=\linewidth]{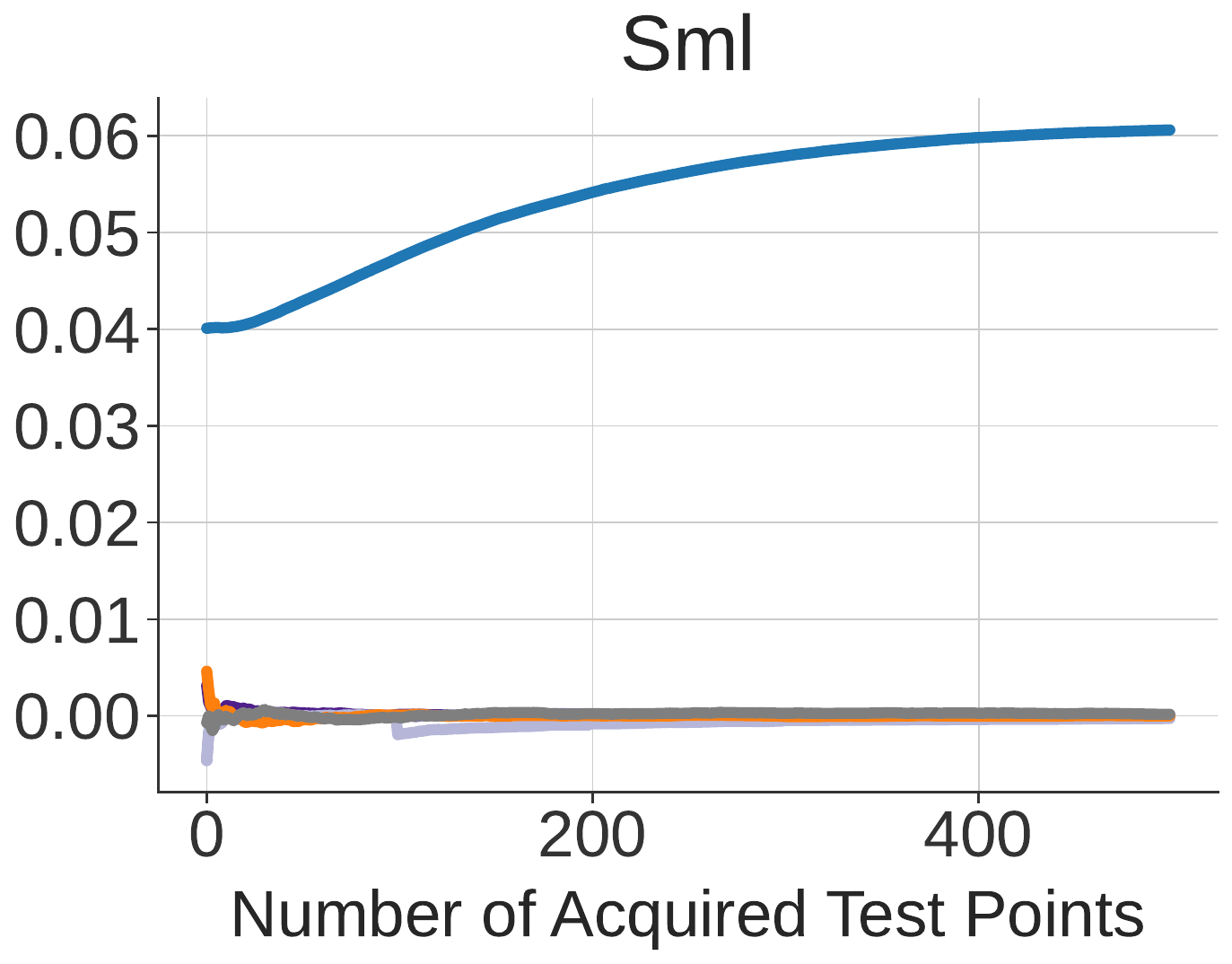}
        \label{fig:main_gas_err}
    \end{subfigure}
    \hspace{0.25cm}
    \begin{subfigure}[t]{0.22\textwidth}
        \centering
        \includegraphics[width=\linewidth]{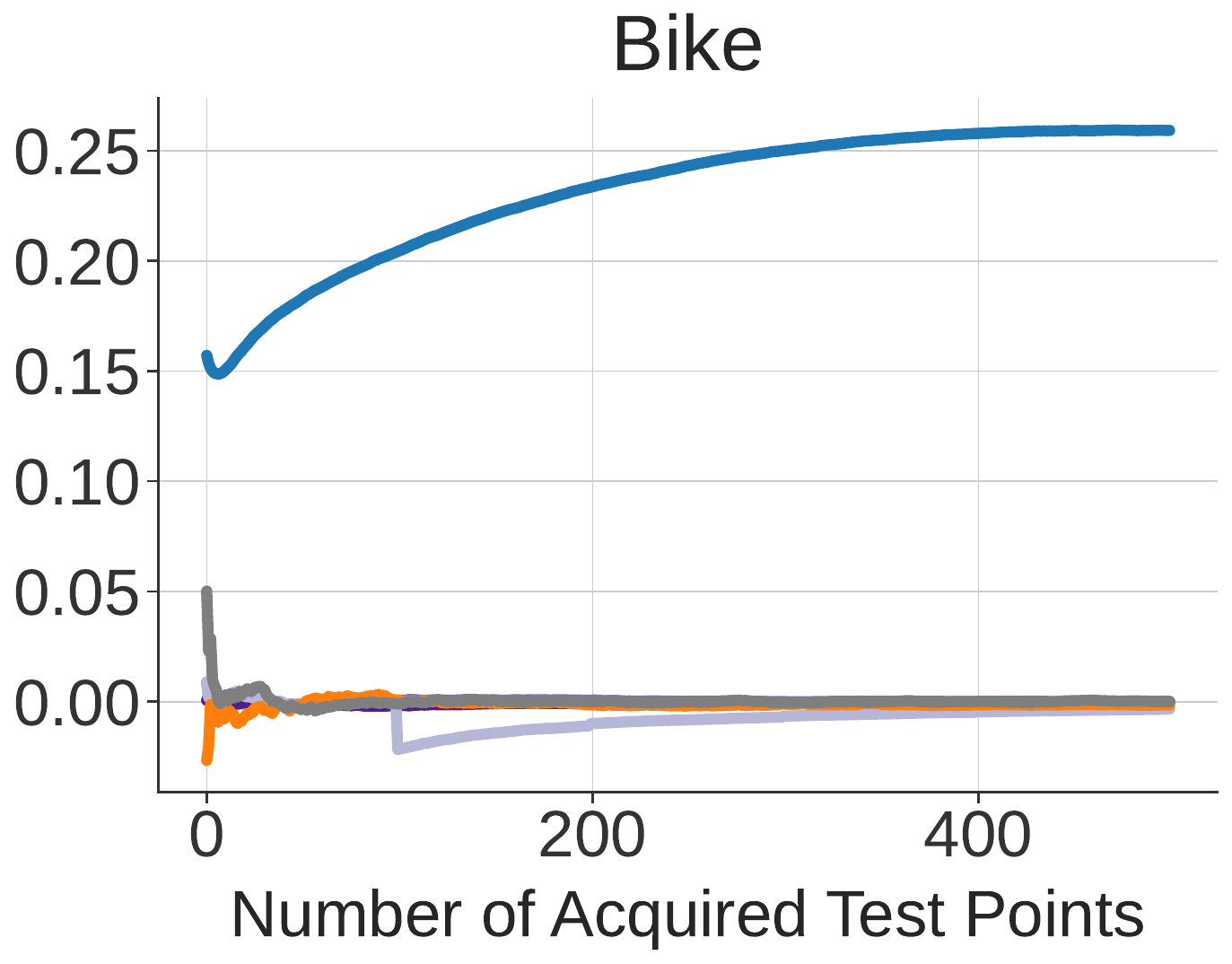}
        \label{fig:main_bike_err}
    \end{subfigure}

    \caption{Experiments on UCI datasets comparing \textcolor{violet}{PPAT} with \textcolor{gray}{Random}, \textcolor{NavyBlue}{ASE}, \textcolor{orange}{LURE}. Plots show the mean error (bias) across 1000 trials. Note that \textcolor{NavyBlue}{ASE} results in \textbf{biased} estimates of the risk.}
    \label{fig:main_uci_mean_err}
\end{figure}

\begin{figure}[!h]
    \centering    
    \includegraphics[width=.75\textwidth]{figures/uci/uci_main_legend.pdf}
    \vspace{0.75em}    
    
    \begin{subfigure}[t]{0.23\textwidth}
        \centering
        \includegraphics[width=\linewidth]{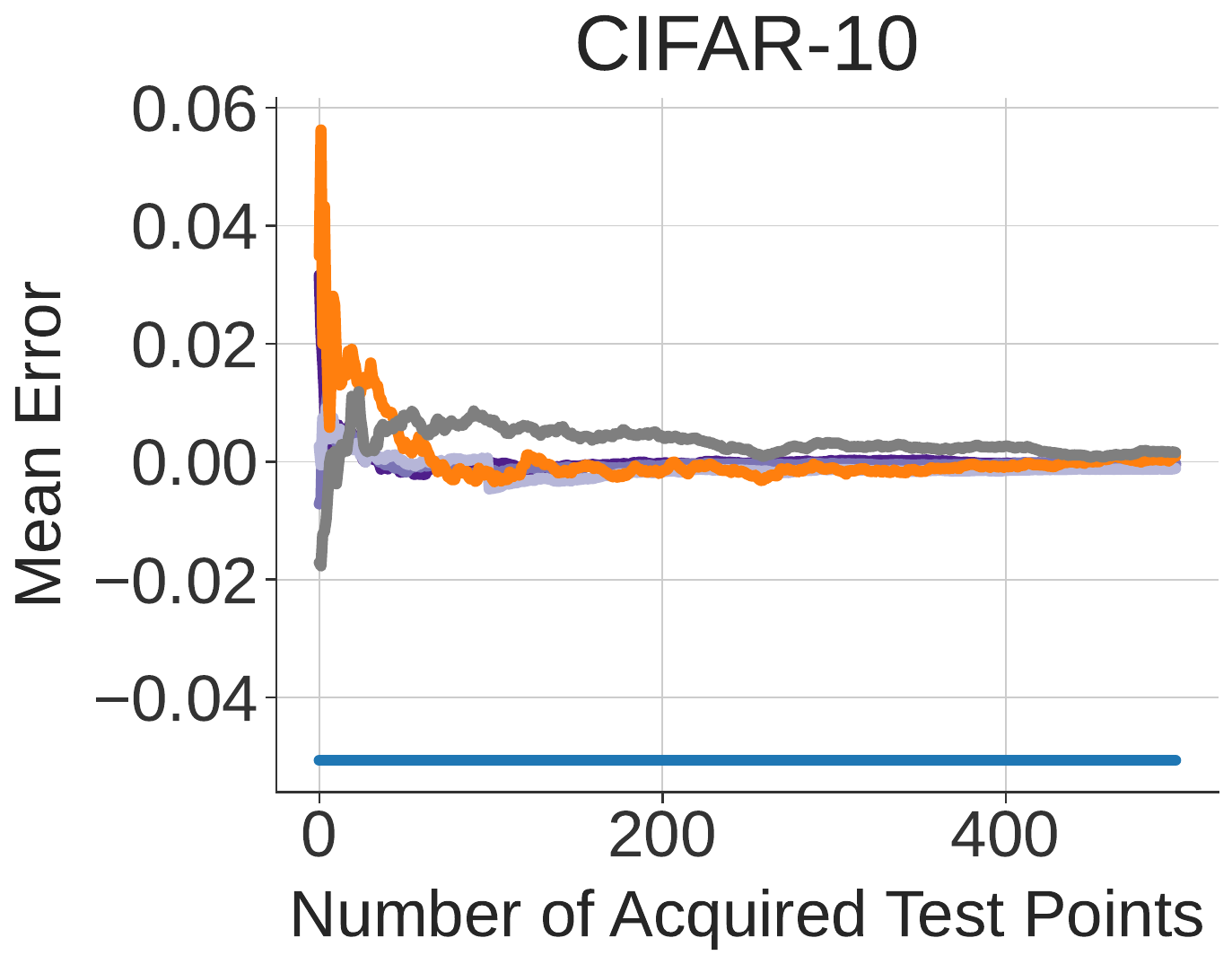}
        \label{fig:main_cifar10_err}
    \end{subfigure}
    \hspace{0.1cm}
    \begin{subfigure}[t]{0.23\textwidth}
        \centering
        \includegraphics[width=\linewidth]{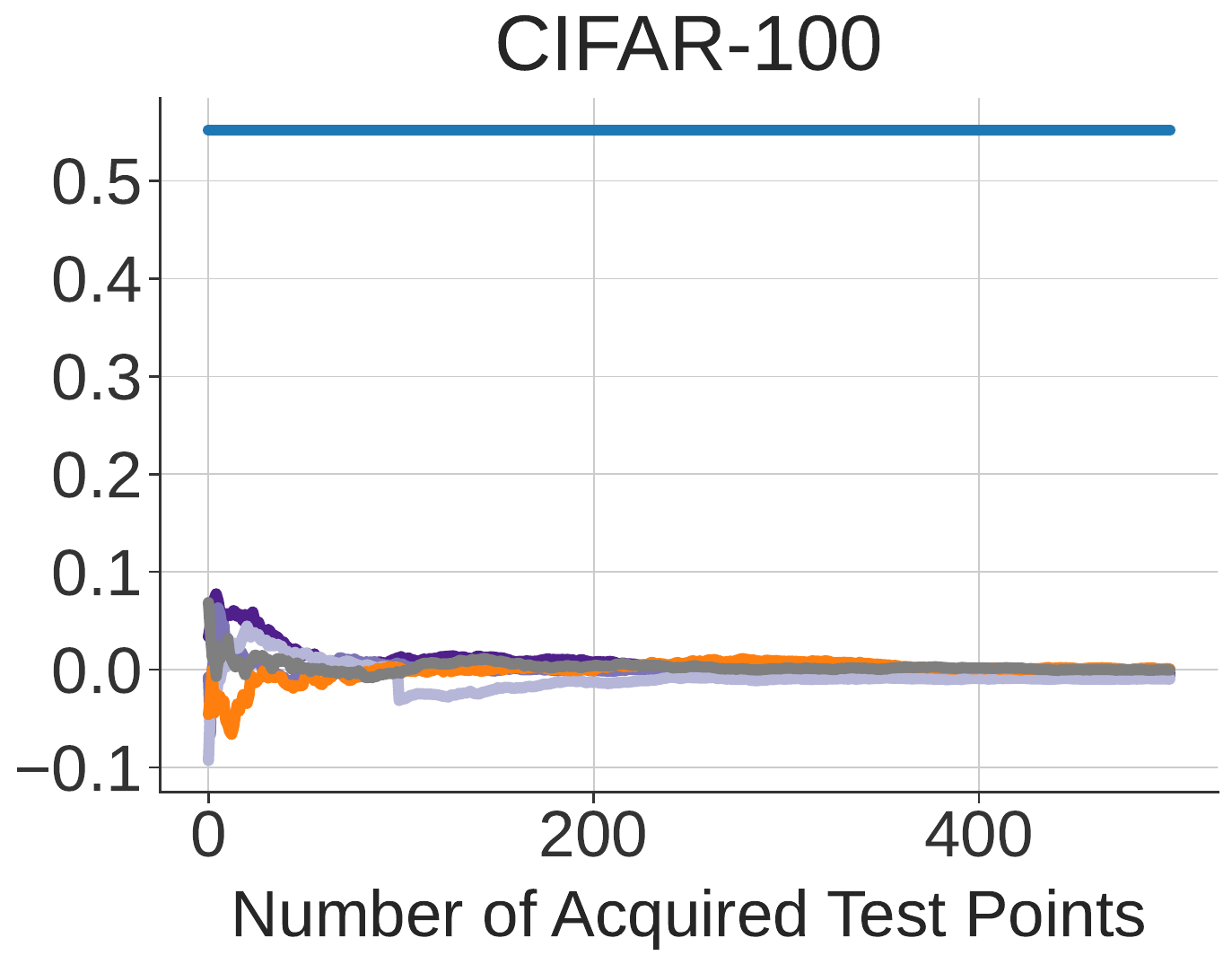}
        \label{fig:main_cifar100_err}
    \end{subfigure}
    \hspace{0.1cm}
    \begin{subfigure}[t]{0.23\textwidth}
        \centering
        \includegraphics[width=\linewidth]{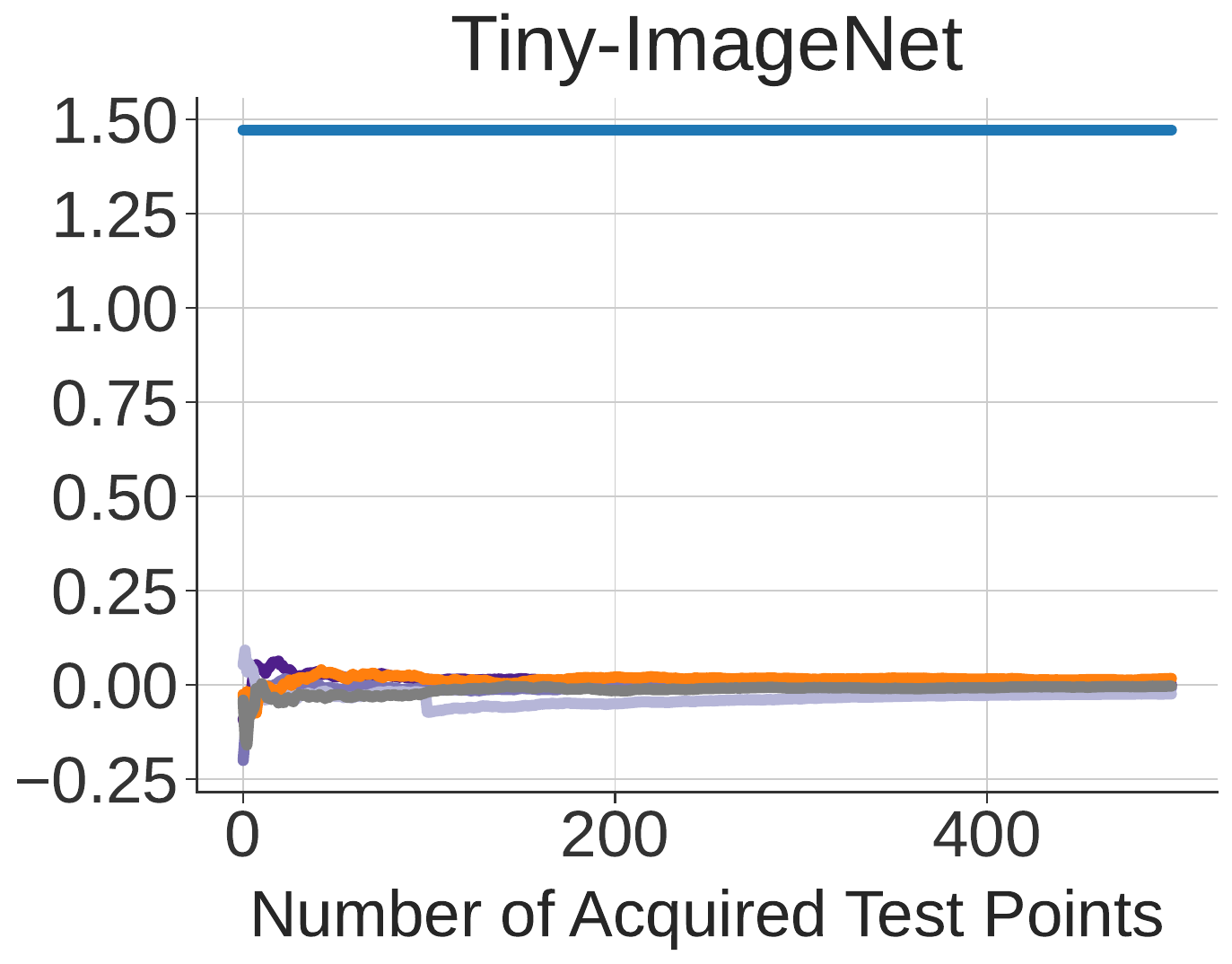}
        \label{fig:main_tiny_ignt_err}
    \end{subfigure}    

    \caption{Experiments on classification datasets comparing \textcolor{violet}{PPAT} with \textcolor{gray}{Random}, \textcolor{NavyBlue}{ASE}, \textcolor{orange}{LURE}. Plots show the mean error (bias) across 1000 trials. Note that \textcolor{NavyBlue}{ASE} results in \textbf{biased} estimates of the risk.}
    \label{fig:classification_datasets_err}
\end{figure}

Across Figs. \ref{fig:main_uci_mean_err} and \ref{fig:classification_datasets_err}, \ase{} exhibits persistent non-zero mean error, as expected: \ase{} estimates the risk through the surrogate predictive model, and any surrogate misspecification can therefore translate directly into bias in the resulting risk estimate. The plug-in version of \colppat{} can also introduce finite-sample bias because $\lambda$ is estimated from the actively acquired labels. However, this bias decreases as more labels are acquired, as expected from our theoretical results. 

\begin{figure}[!h]
    \centering    \includegraphics[width=.45\textwidth]{figures/ablations/diff_estimators_legend.pdf}
    \vspace{0.75em}    
    
    \begin{subfigure}[t]{0.22\textwidth}
        \centering
        \includegraphics[width=\linewidth]{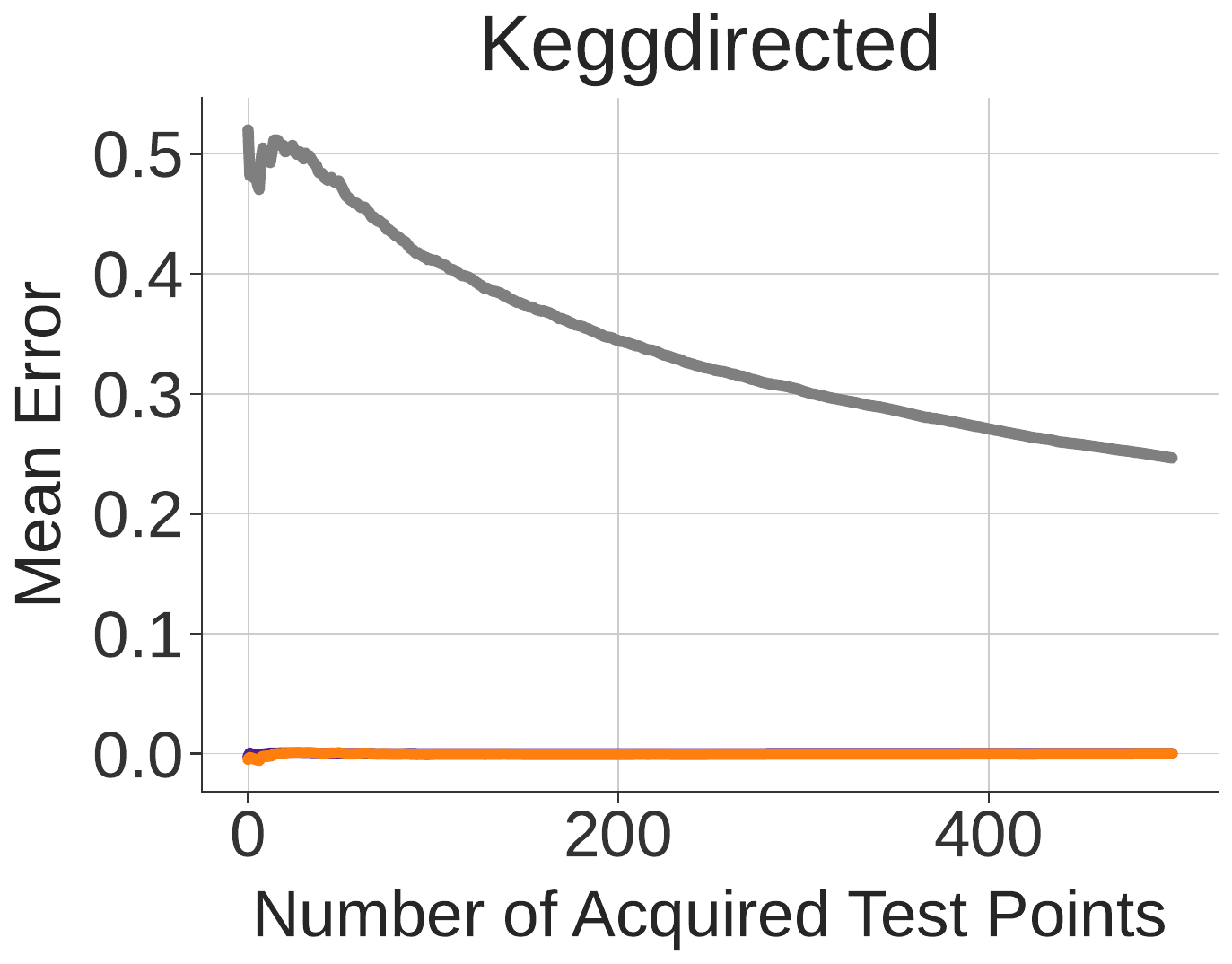}
        \label{fig:est_ablations_keggdir_err}
    \end{subfigure}
    \hspace{0.25cm}
    \begin{subfigure}[t]{0.22\textwidth}
        \centering
        \includegraphics[width=\linewidth]{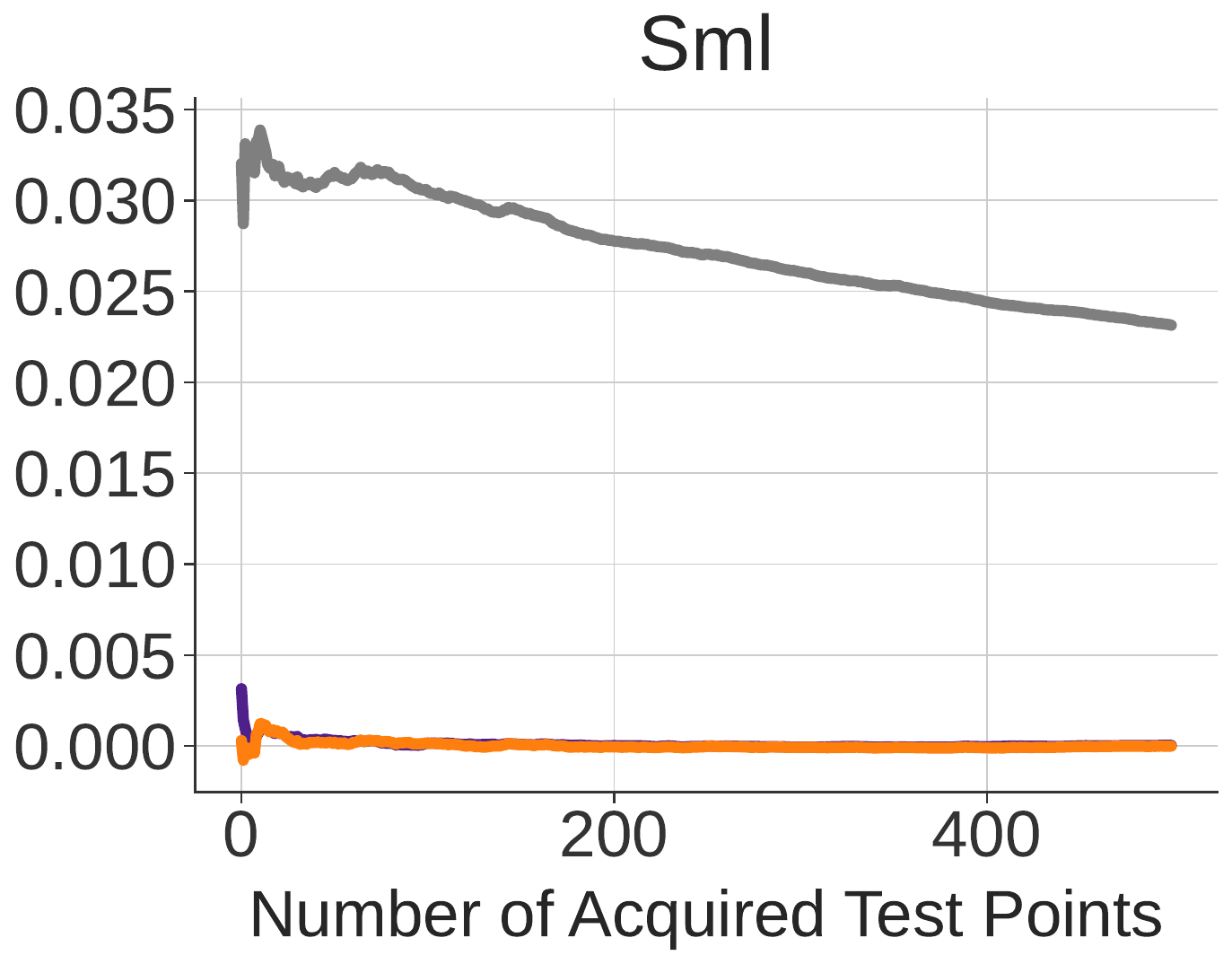}
        \label{fig:est_ablations_gas_err}
    \end{subfigure}
    \hspace{0.25cm}
    \begin{subfigure}[t]{0.22\textwidth}
        \centering        \includegraphics[width=\linewidth]{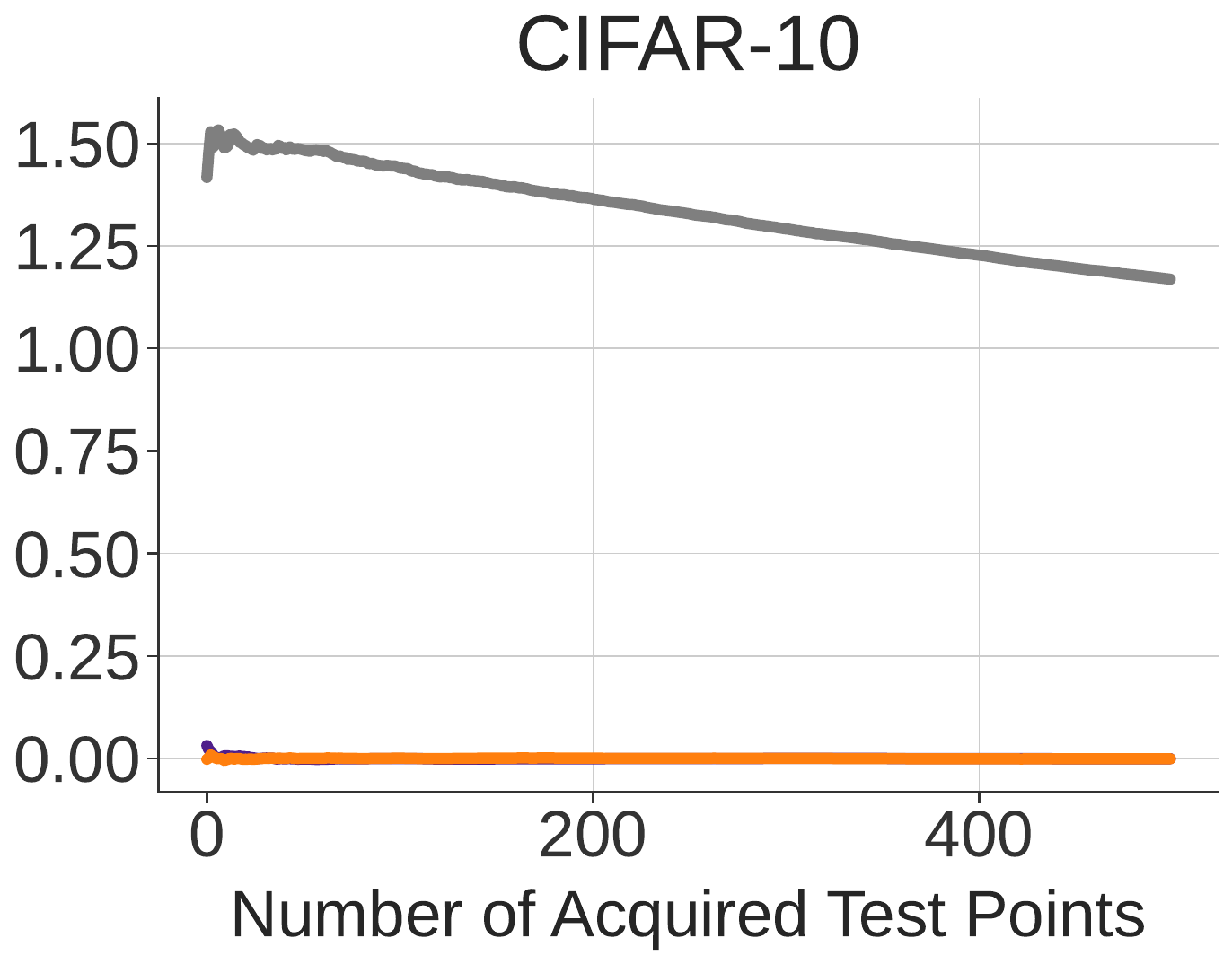}
        \label{fig:est_ablations_cifar10_err}
    \end{subfigure}
    \hspace{0.25cm}
    \begin{subfigure}[t]{0.22\textwidth}
        \centering
        \includegraphics[width=\linewidth]{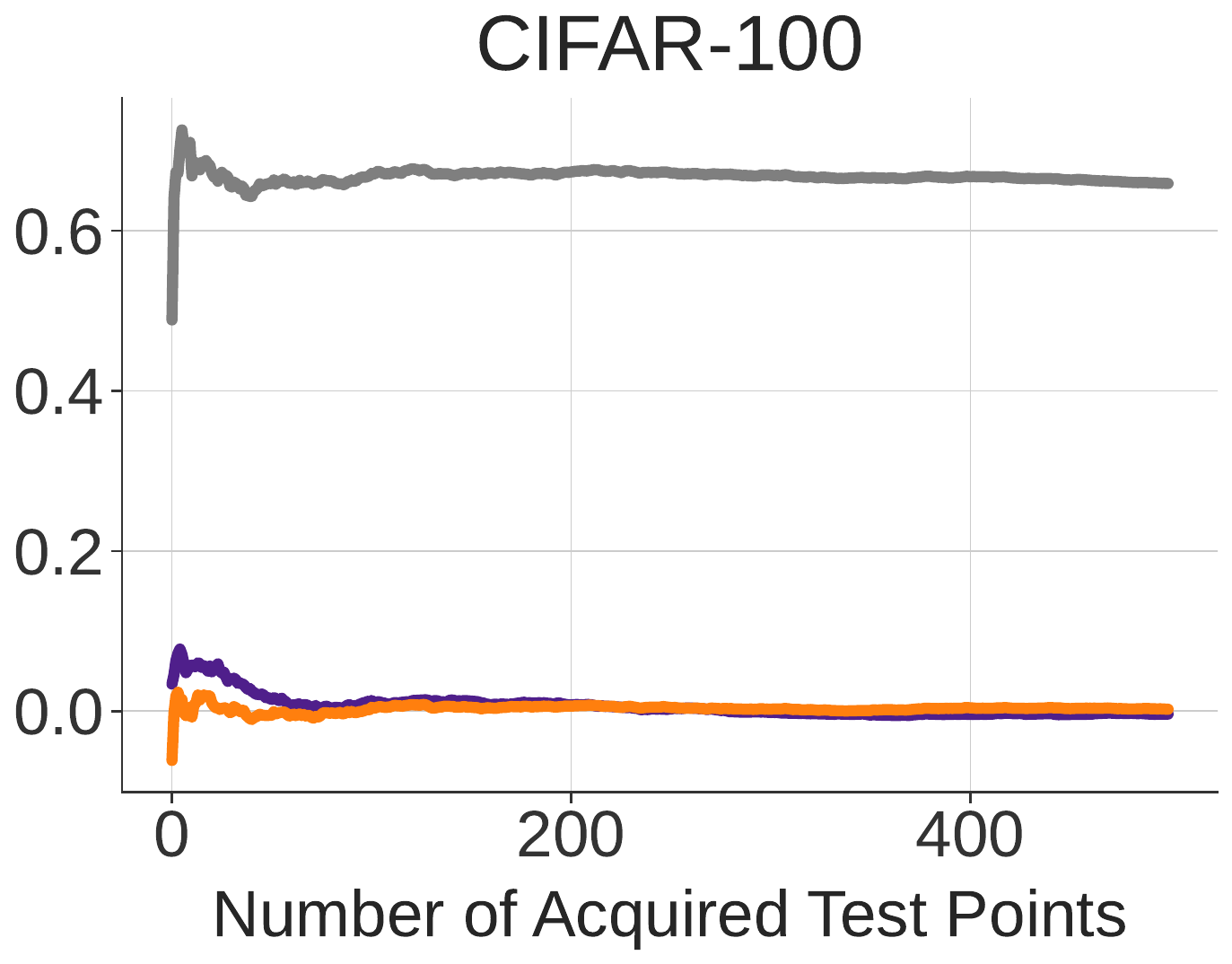}
        \label{fig:main_bike_err}
    \end{subfigure}

    \caption{Experiments on UCI and classification datasets studying the influence of our estimator. We compare \textcolor{violet}{PPAT} with $\lambda=1$ to using the PPAT acquisition strategy, $Q^\mathrm{PPAT}$, but with the LURE estimate of the risk (\textcolor{orange}{PPAT Acq. + LURE}) and the empirical estimate of the risk (\textcolor{gray}{PPAT Acq. + Empirical}). Plots show the mean error (bias) across 1000 trials. Note that \textcolor{gray}{PPAT Acq. + Empirical} results in \textbf{biased} estimates of the risk as the data is not randomly sampled.}
    \label{fig:effect_of_estimator_err}
\end{figure}

In Fig. \ref{fig:effect_of_estimator_err}, \textcolor{gray}{PPAT Acq. + Empirical} is biased for the expected reason: the data are selected using a non--uniform active proposal, but the estimator does not apply the LURE importance--weighting correction. Moreover, the ablations \textcolor{orange}{PPAT Acq. + LURE} in Fig. \ref{fig:effect_of_estimator_err} and \textcolor{orange}{LURE Acq. + PPI} in Fig. \ref{fig:effect_of_aq_err}, show a bias of approximately zero, which is expected as the estimators are still using the LURE importance--weighting correction. 

\begin{figure}[!h]
    \centering    \includegraphics[width=.45\textwidth]{figures/ablations/diff_acq_ablation.pdf}
    \vspace{0.75em}    
    
    \begin{subfigure}[t]{0.22\textwidth}
        \centering
        \includegraphics[width=\linewidth]{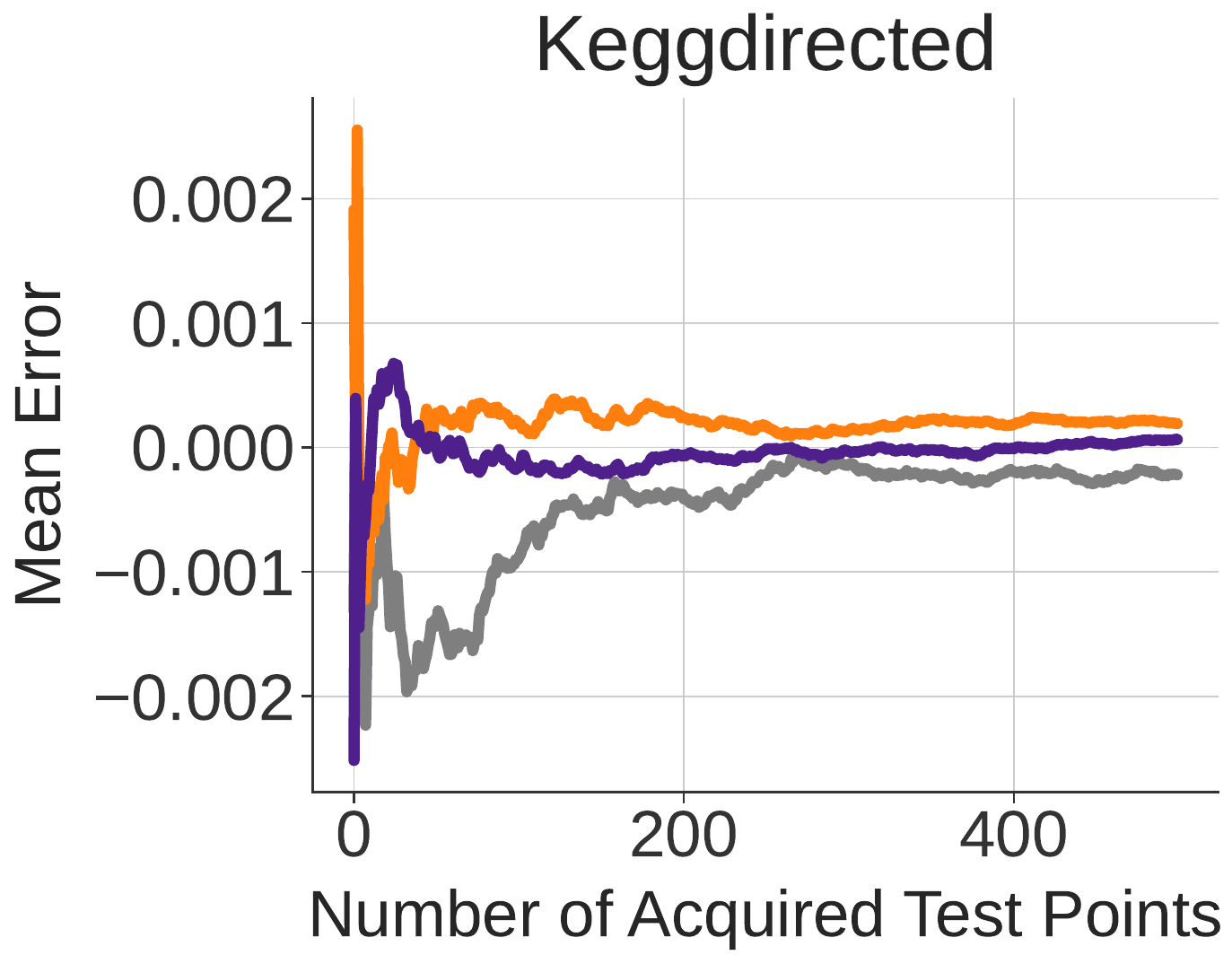}
        \label{fig:acq_ablations_keggdir_err}
    \end{subfigure}
    \hspace{0.25cm}
    \begin{subfigure}[t]{0.22\textwidth}
        \centering
        \includegraphics[width=\linewidth]{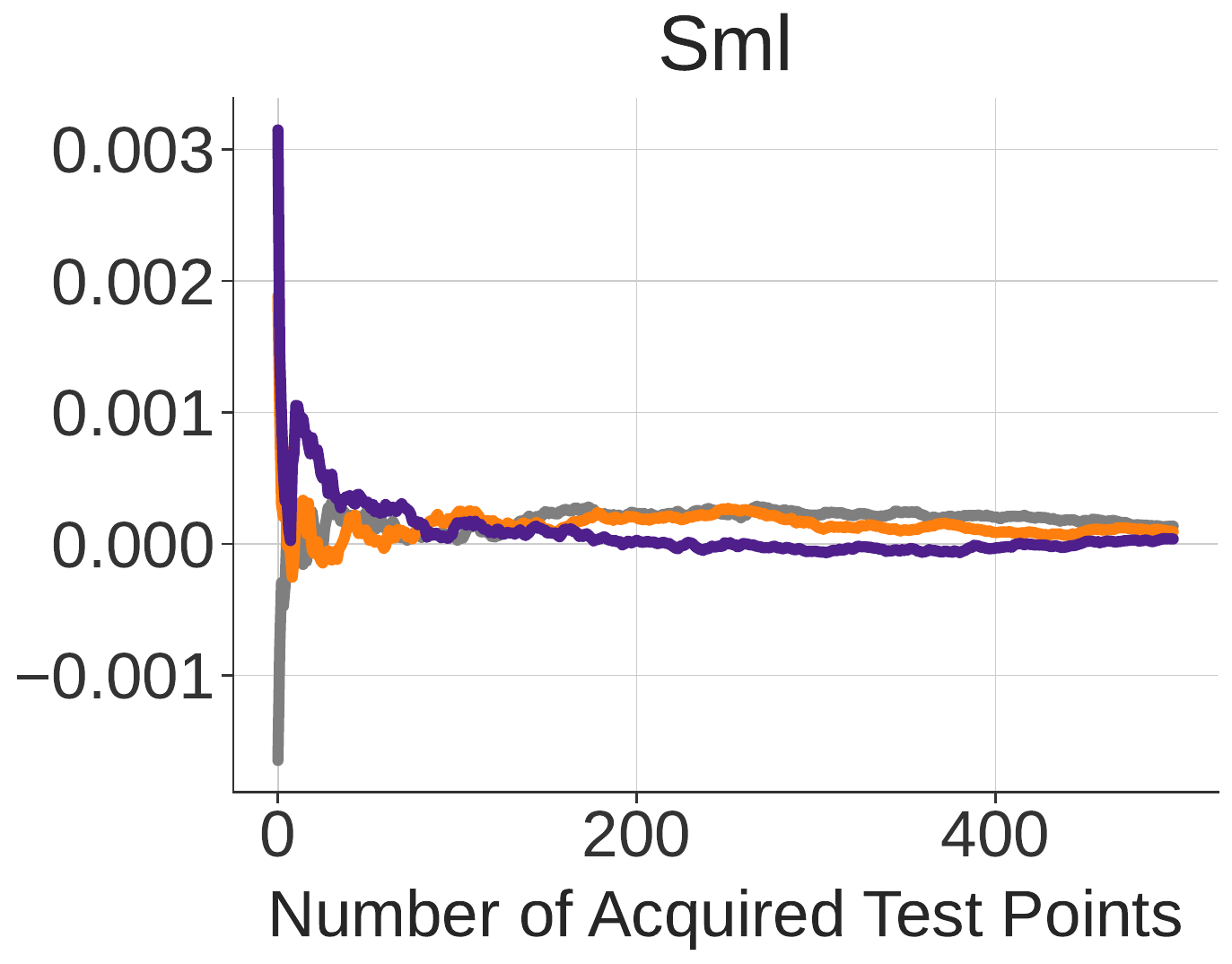}
        \label{fig:acq_ablations_gas_err}
    \end{subfigure}
    \hspace{0.25cm}
    \begin{subfigure}[t]{0.22\textwidth}
        \centering        \includegraphics[width=\linewidth]{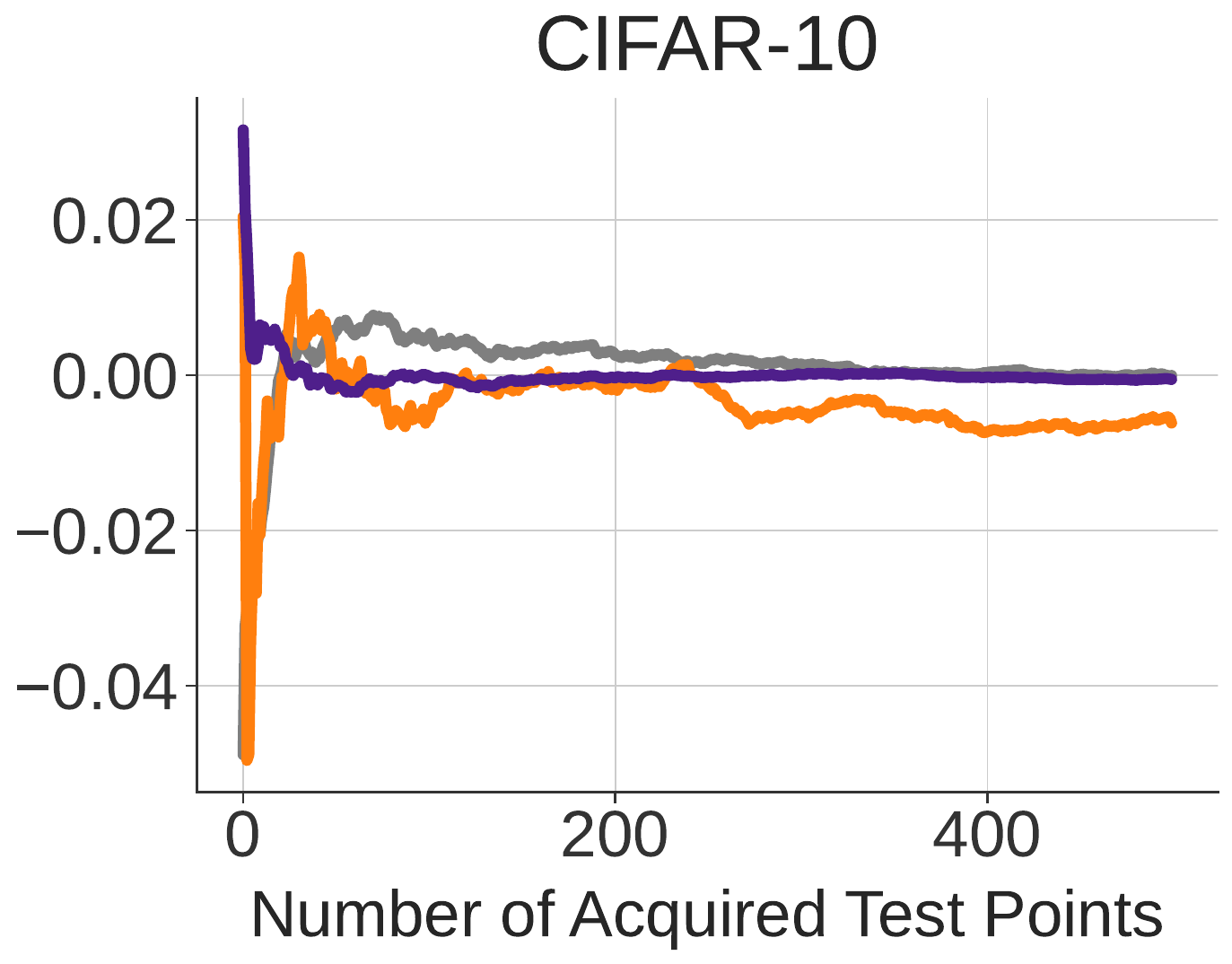}
        \label{fig:acq_ablations_cifar10_err}
    \end{subfigure}
    \hspace{0.25cm}
    \begin{subfigure}[t]{0.22\textwidth}
        \centering
        \includegraphics[width=\linewidth]{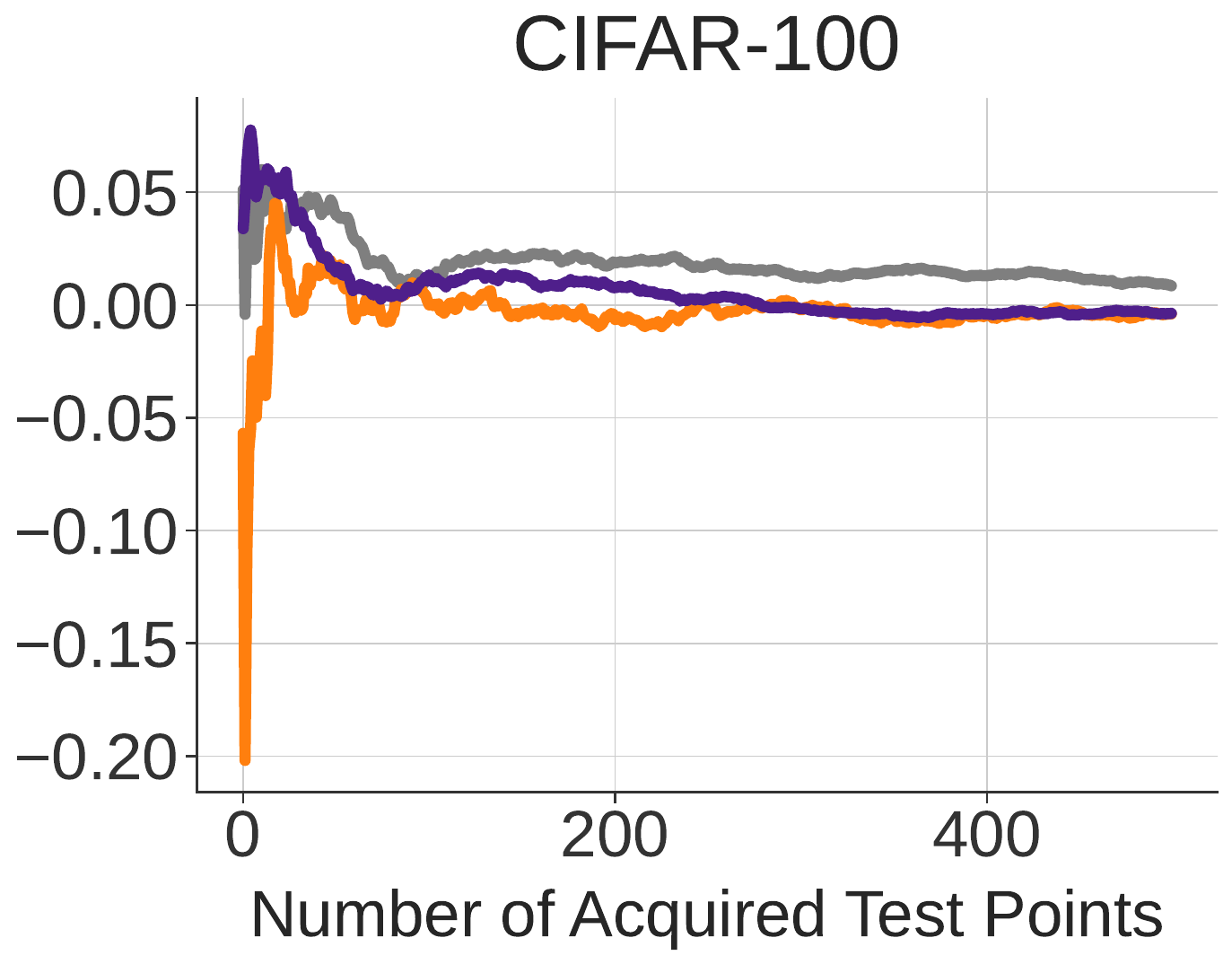}
        \label{fig:acq_ablations_cifar100_err}
    \end{subfigure}

    \caption{Experiments on UCI and classification datasets studying the influence of our acquisition strategy. We compare \textcolor{violet}{PPAT} with $\lambda=1$ to using the PPI estimator with the LURE acquisition strategy (\textcolor{orange}{LURE Acq. + PPI}) and random sampling (\textcolor{gray}{Random + PPI}). Plots show the mean error (bias) across 1000 trials.}
    \label{fig:effect_of_aq_err}
\end{figure}

\newpage
\subsubsection{Confidence Interval Widths}
\label{app:ci_widths}

Fig. \ref{fig:ci_widths} shows that the faster attainment of the target coverage level in Fig. \ref{fig:coverage} is not simply due to overly conservative confidence intervals. Across representative datasets, we see that \colppat{} achieves lower mean widths than competing approaches, while reaching the desired coverage level at a similar or substantially faster rate. The main exception is \texttt{CIFAR--100}, where \colppat{} with $\lambda=0.5$ performs similarly to \colure{} in terms of mean width. Overall, this suggests that the prediction--powered control variate improves not only point estimation, but also (frequentist) uncertainty quantification.

\begin{figure}[!h]
    \centering    \includegraphics[width=.8\textwidth]{figures/uci/uci_main_legend_coverage.pdf
    }
    \vspace{0.75em}    
    
    \begin{subfigure}[t]{0.22\textwidth}
        \centering
        \includegraphics[width=\linewidth]{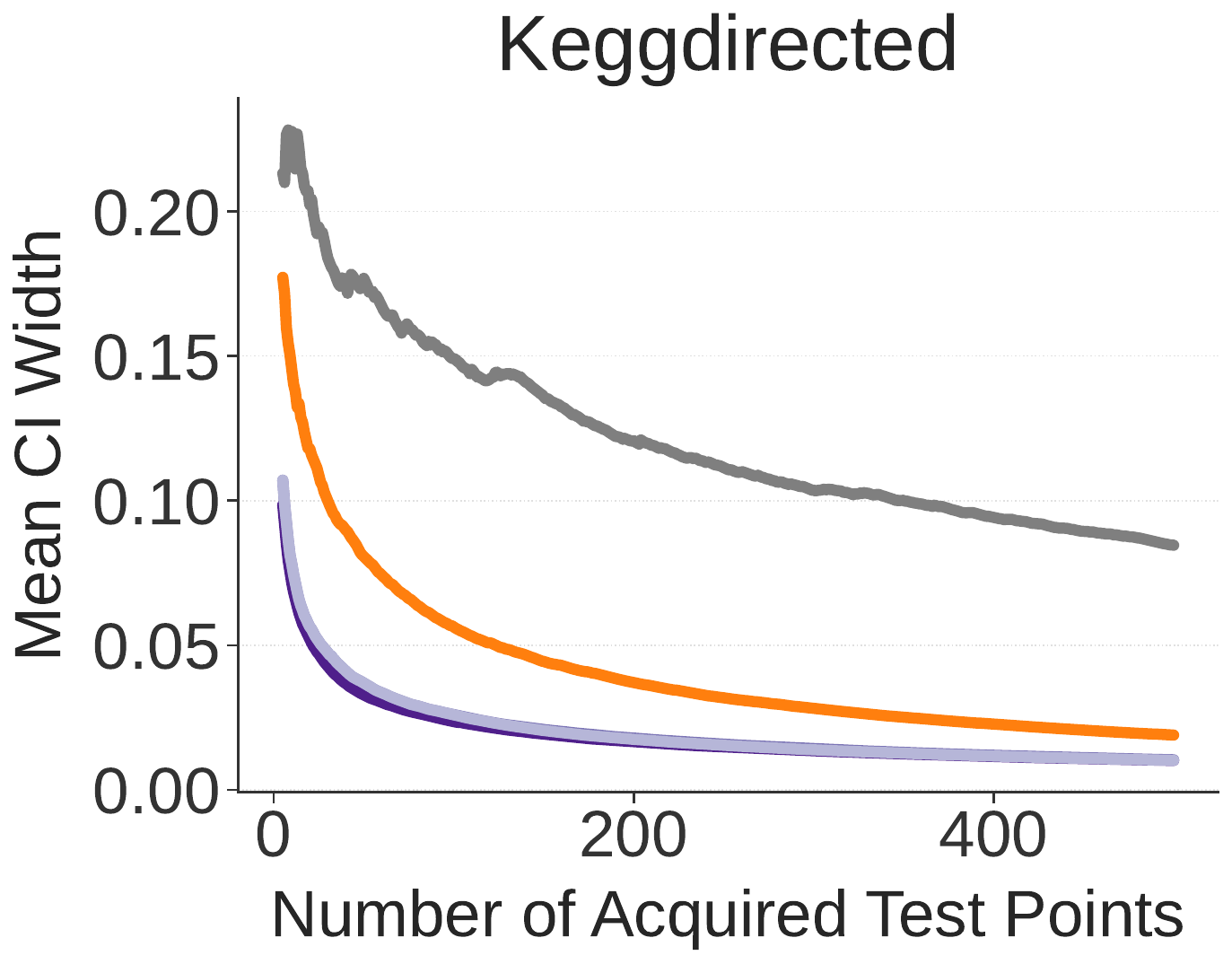}
        \label{fig:coverage_keggdir_width}
    \end{subfigure}
    \hspace{0.25cm}
    \begin{subfigure}[t]{0.22\textwidth}
        \centering
        \includegraphics[width=\linewidth]{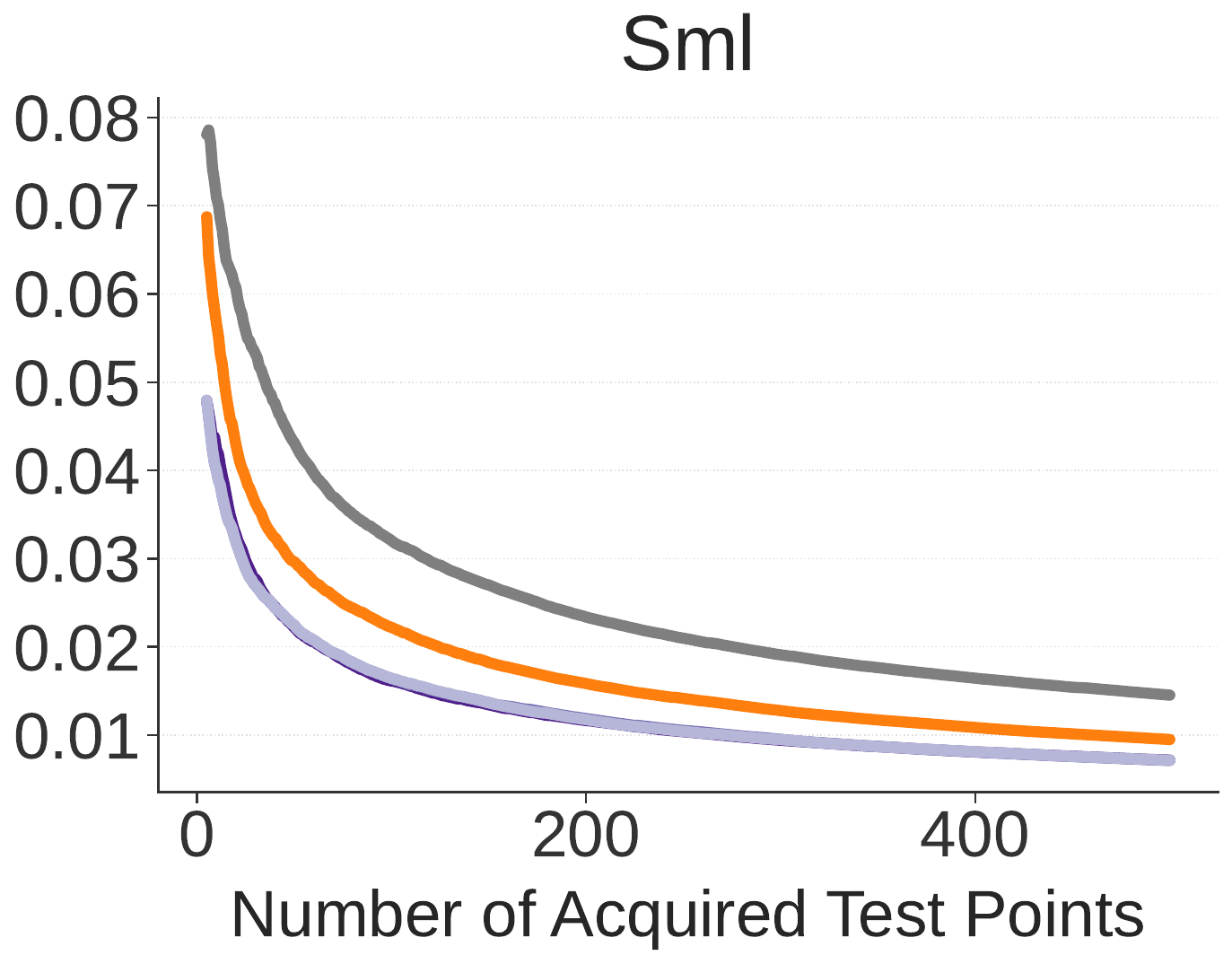}
        \label{fig:coverage_gas_width}
    \end{subfigure}
    \hspace{0.25cm}
    \begin{subfigure}[t]{0.22\textwidth}
        \centering        \includegraphics[width=\linewidth]{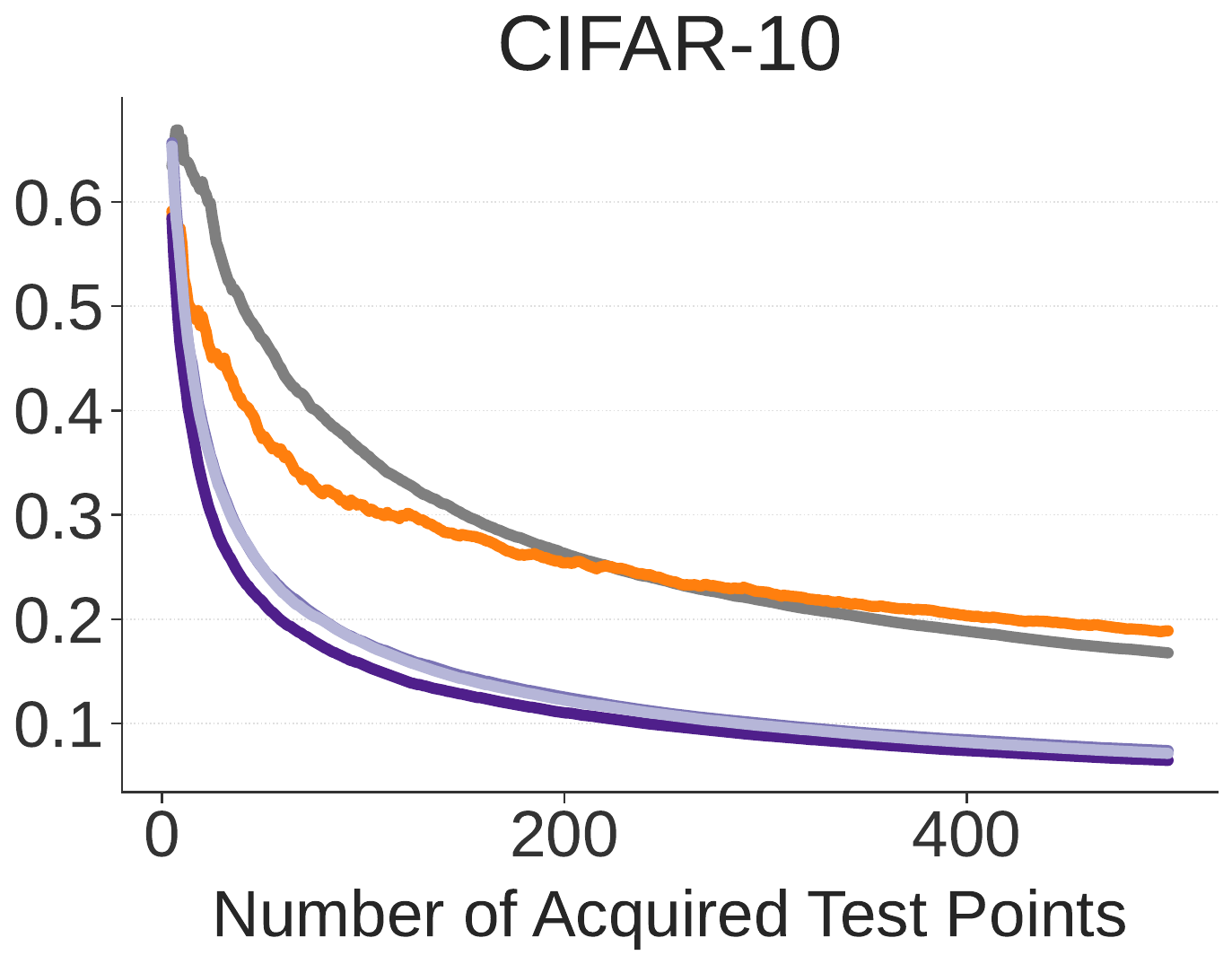}
        \label{fig:coverage_cifar10_width}
    \end{subfigure}
    \hspace{0.25cm}
    \begin{subfigure}[t]{0.22\textwidth}
        \centering
        \includegraphics[width=\linewidth]{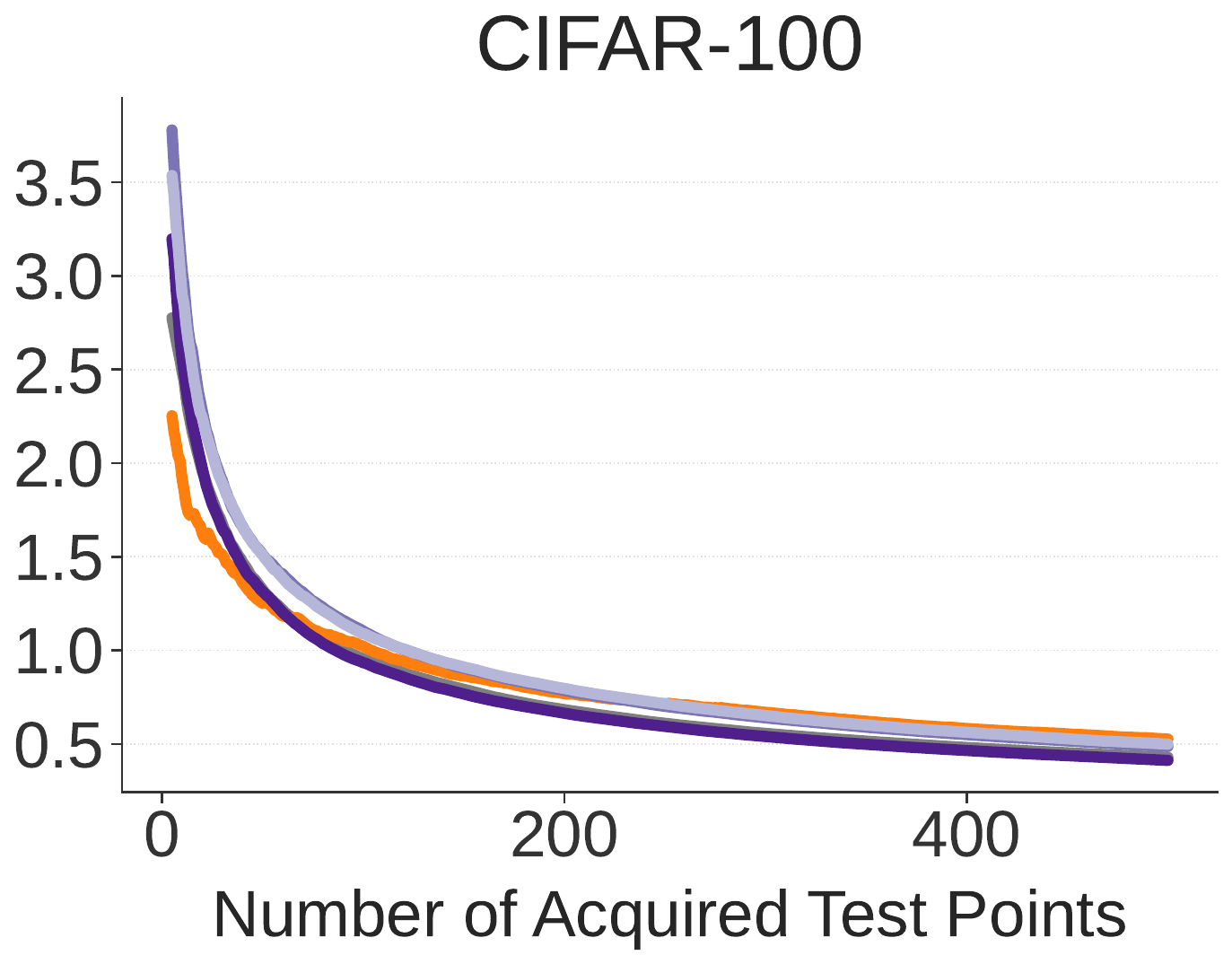}
        \label{fig:coverage_cifar100_width}
    \end{subfigure}

    \caption{Representative experiments on UCI and classification datasets for the mean widths of the confidence intervals of \textcolor{violet}{PPAT}, \textcolor{gray}{Random}, \textcolor{orange}{LURE}. Plots show the widths of asymptotic confidence intervals constructed with $\delta=0.1$ across 1000 trials. The \textcolor{NavyBlue}{ASE} baseline is excluded as it does \textbf{not} provide asymptotically valid confidence intervals.}
    \label{fig:ci_widths}
\end{figure}


\subsection{Additional Comparisons}
\label{app:add_comparisons}

Here, we present a broader set of comparative analyses against \colppat{}. 

\subsubsection{Further Comparisons with Prediction--Powered Inference}
\label{app:comparison_ppi}

Fig. \ref{fig:effect_of_aq_extended} extends Fig. \ref{fig:effect_of_aq} by additionally including the standard \random{} and \colure{} baselines, shown as dashed lines. 
Across all datasets, \textcolor{gray}{Random + PPI} outperforms both \random{} and \colure{}, further demonstrating that prediction--powered inference provides a strong baseline. Nevertheless, \colppat{} consistently achieves lower median squared error. This shows that the gains of \colppat{} do not come only from using a PPI--like estimator, but also from adaptively sampling points in a way that is tailored to reducing the variance of that estimator.

\begin{figure}[!h]
    \centering    \includegraphics[width=.65\textwidth]{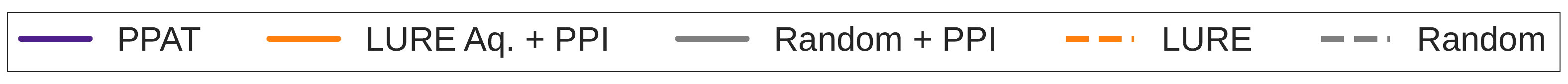}
    \vspace{0.75em}    
    
    \begin{subfigure}[t]{0.22\textwidth}
        \centering
        \includegraphics[width=\linewidth]{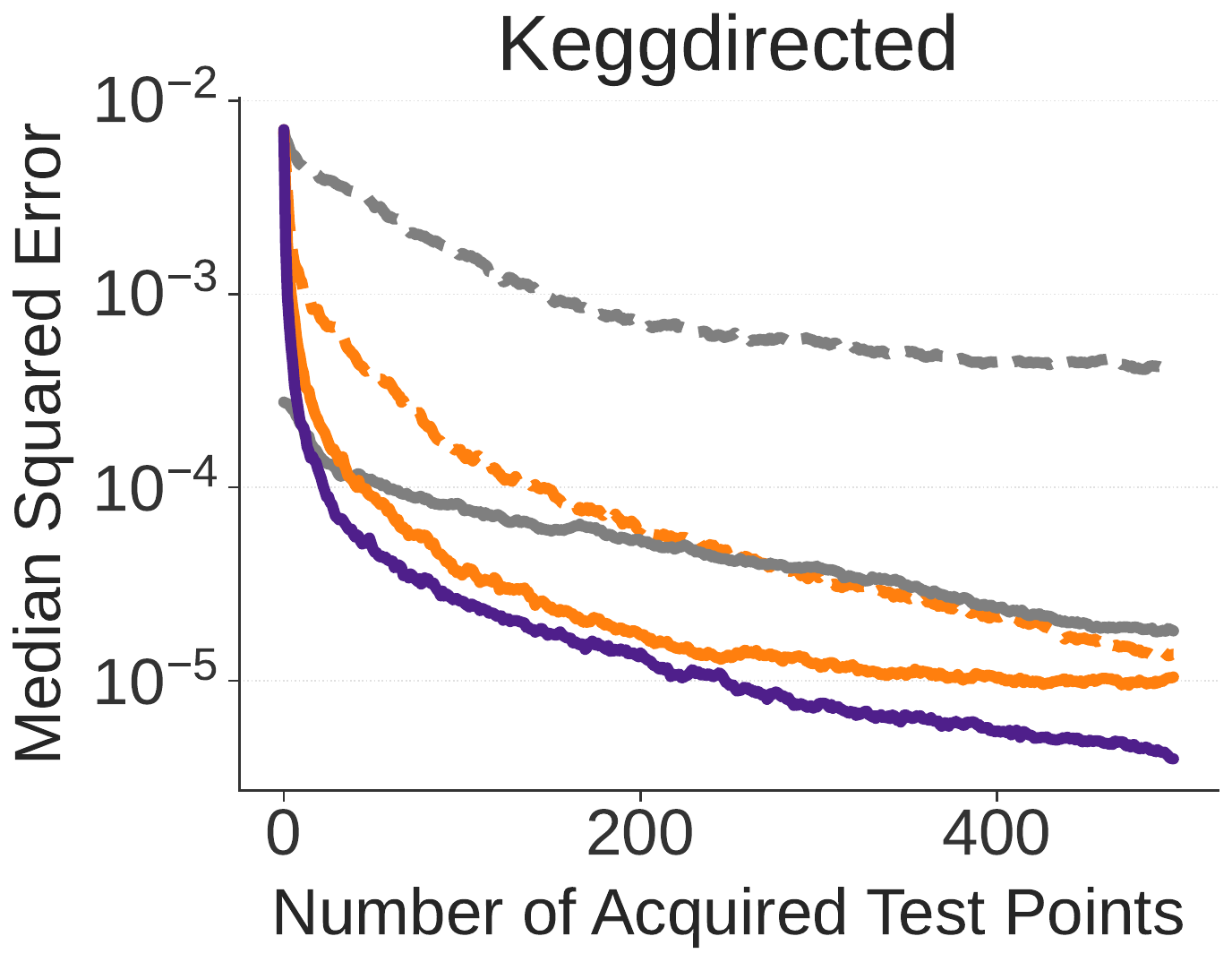}
        \label{fig:acq_ablations_keggdir_extended}
    \end{subfigure}
    \hspace{0.25cm}
    \begin{subfigure}[t]{0.22\textwidth}
        \centering
        \includegraphics[width=\linewidth]{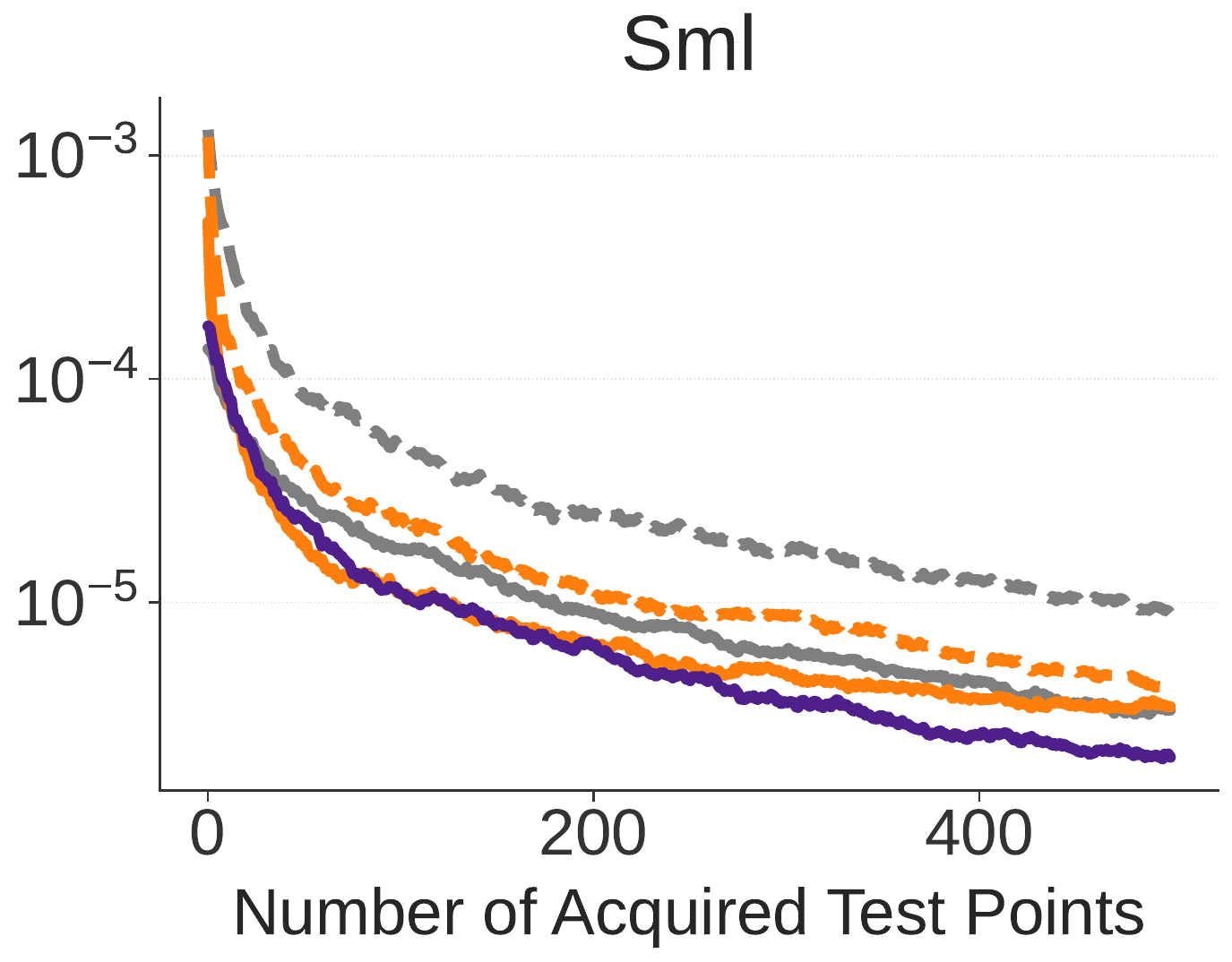}
        \label{fig:acq_ablations_gas_extended}
    \end{subfigure}
    \hspace{0.25cm}
    \begin{subfigure}[t]{0.22\textwidth}
        \centering        \includegraphics[width=\linewidth]{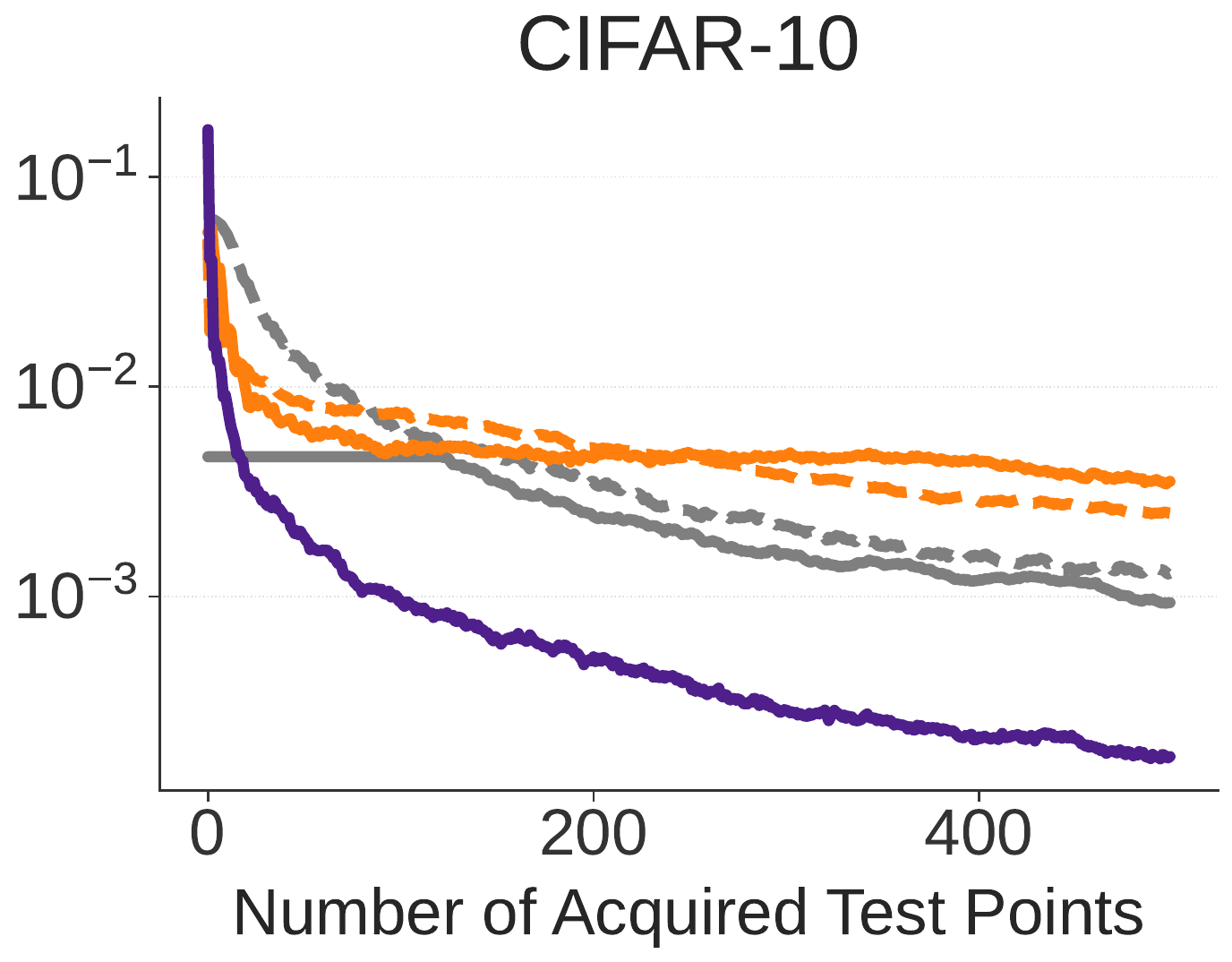}
        \label{fig:acq_ablations_cifar10_extended}
    \end{subfigure}
    \hspace{0.25cm}
    \begin{subfigure}[t]{0.22\textwidth}
        \centering
        \includegraphics[width=\linewidth]{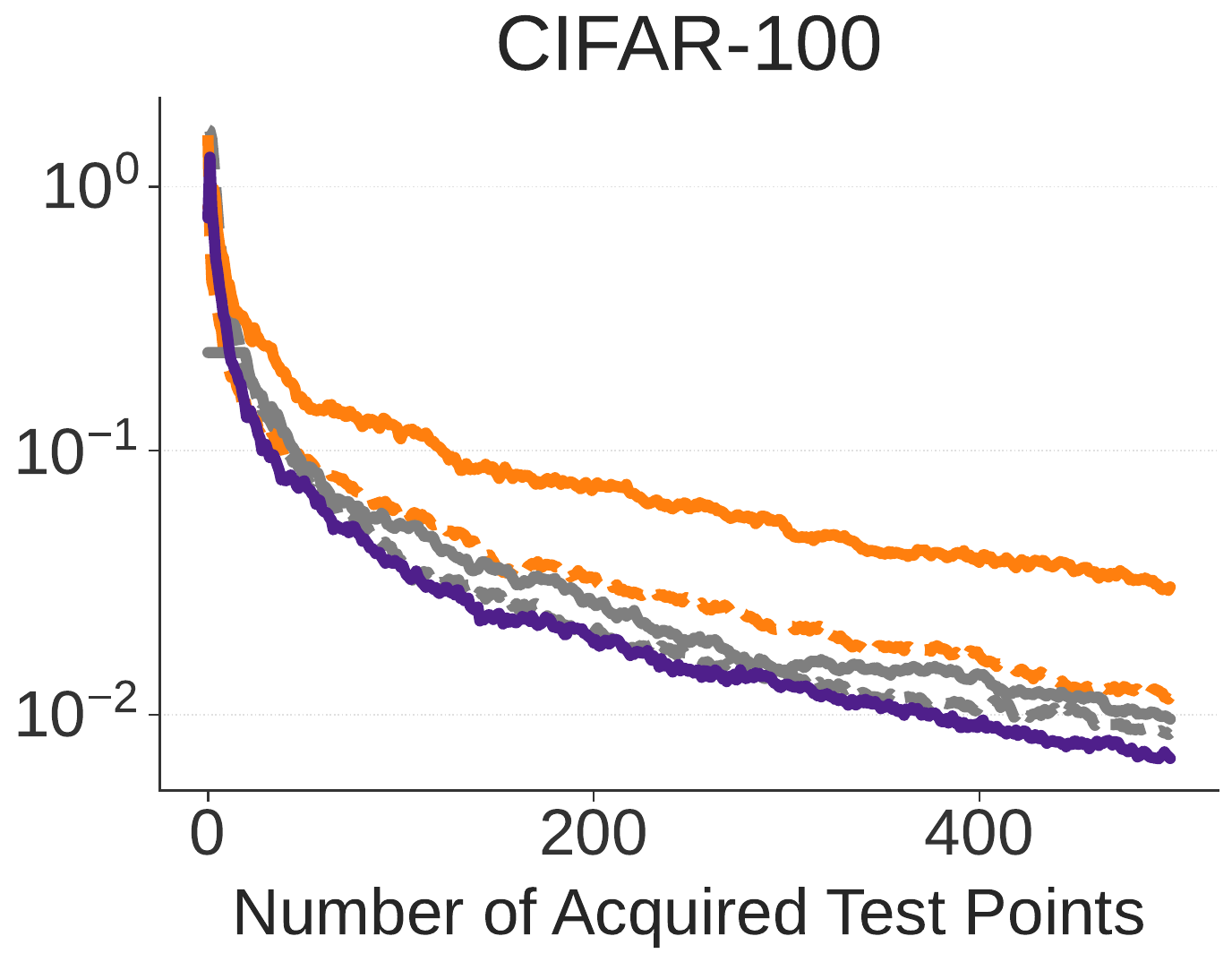}
        \label{fig:acq_ablations_cifar100_extended}
    \end{subfigure}

    \caption{Experiments on UCI and classification datasets studying the influence of our acquisition strategy. We compare \textcolor{violet}{PPAT} with $\lambda=1$ to using the PPI estimator with the LURE acquisition strategy (\textcolor{orange}{LURE Acq. + PPI}, \solidline{orange}) and random sampling (\textcolor{gray}{Random + PPI}, \solidline{gray}), as well as the standard \colure{} (\dashedline{orange}) and \random{} (\dashedline{gray}) baselines. Plots show the median squared error across 1000 trials. }
    \label{fig:effect_of_aq_extended}
\end{figure}

\subsubsection{Comparison with Active Surrogate Estimators Using the Proxy}
\label{app:comparison_ase_proxy}

In \S\ref{sec:experiments}, we compared \colppat{} with the original formulation of \ase{}. Here, we consider a related baseline that uses the proxy predictions directly. Specifically, since the proxy \(g\) is evaluated on the full test pool, we form the proxy losses \(\widetilde \ell_i\) and estimate the risk by their full-pool average. We refer to this baseline as \textcolor{red}{Proxy}. This tests whether the initial proxy model is already accurate enough to replace the unknown labels, rather than using the proxy only as a control variate as in \colppat{}.

This baseline is related to \ase{}, but is not directly comparable to the original \ase{} approach. In our setting, the proxy \(g\) can be a fixed black--box predictor: even if its predictions are available on the entire pool, we may not have access to posterior samples, ensemble members, or other measures of epistemic uncertainty needed to compute the XWED acquisition score. Like \ase{}, however, \textcolor{red}{Proxy} generally yields a biased estimate of the risk and does not inherit the unbiasedness guarantee of \colppat{}.

\begin{figure}[!h]
    \centering    \includegraphics[width=.75\textwidth]{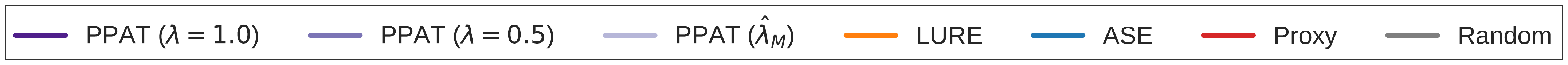}
    \vspace{0.75em}    
    
    \begin{subfigure}[t]{0.25\textwidth}
        \centering
        \includegraphics[width=\linewidth]{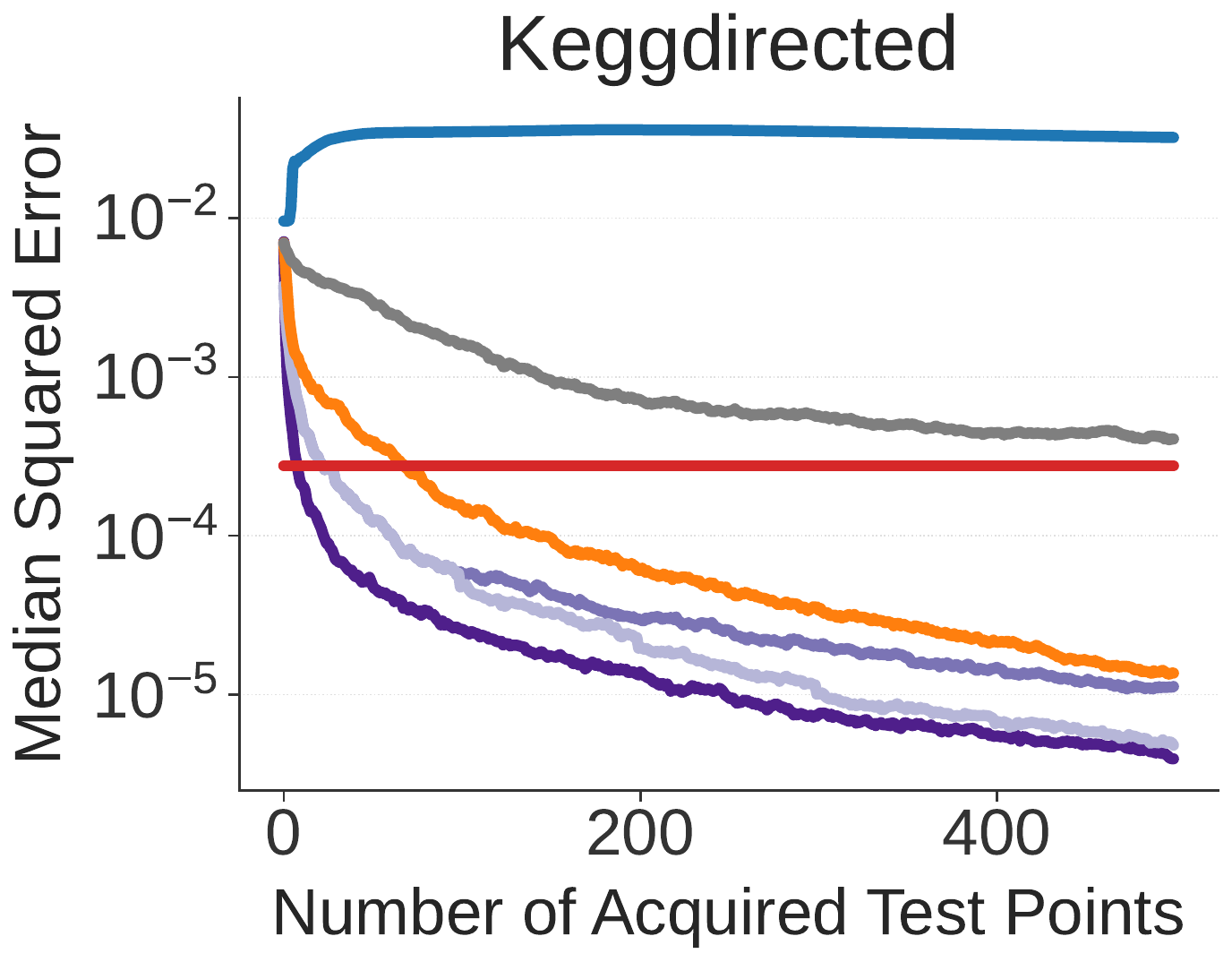}
        \label{fig:keggdir_with_proxy}
    \end{subfigure}
    \hspace{0.25cm}
    \begin{subfigure}[t]{0.25\textwidth}
        \centering
        \includegraphics[width=\linewidth]{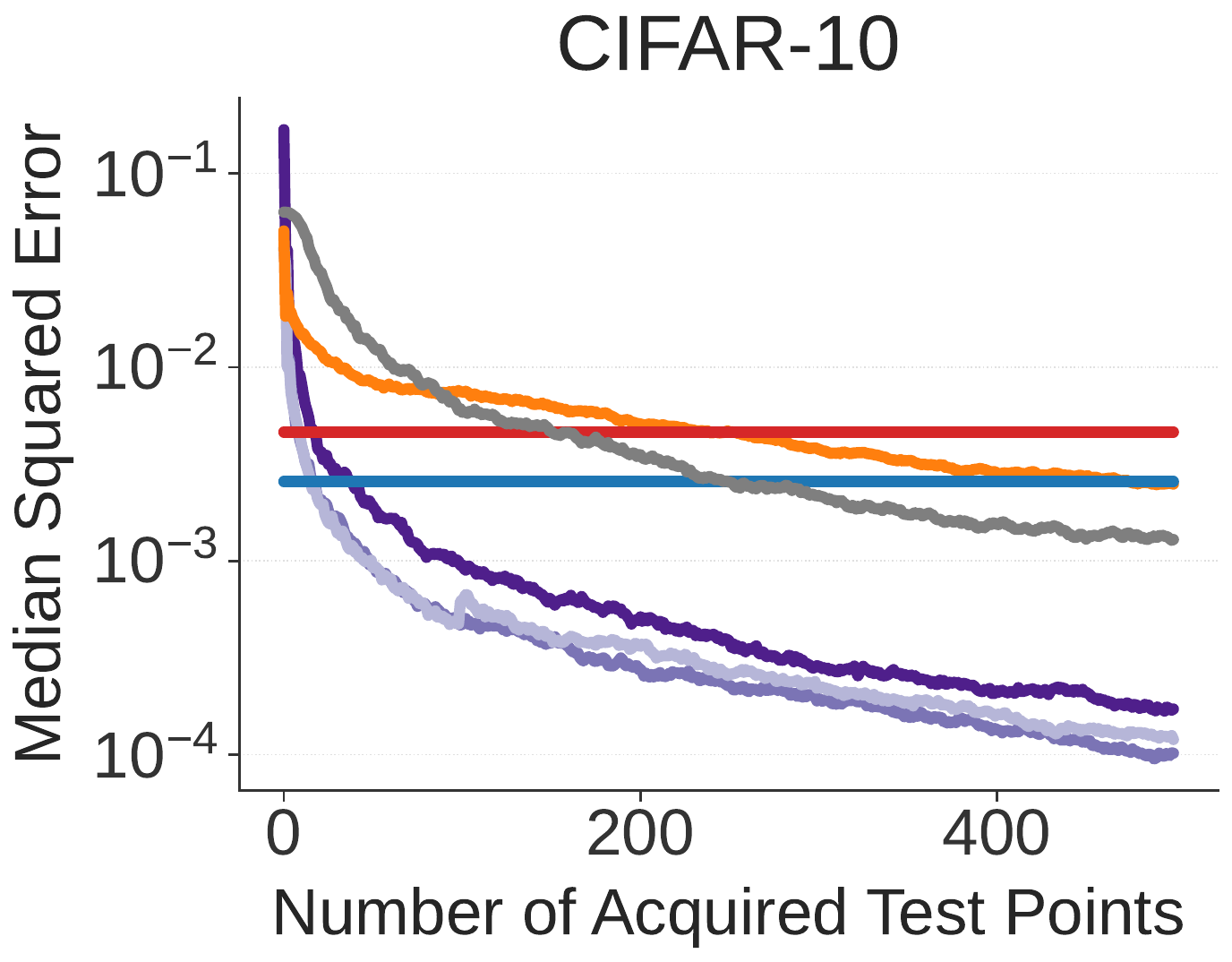}
        \label{fig:cifar10_with_proxy}
    \end{subfigure}
    \hspace{0.25cm}
    \caption{Experiments on \texttt{Keggdirected} and \texttt{CIFAR-10} comparing \textcolor{violet}{PPAT} with \textcolor{gray}{Random}, \textcolor{NavyBlue}{ASE}, \textcolor{red}{Proxy}, \textcolor{orange}{LURE}. Plots show the median squared error across 1000 trials. Note that \textcolor{NavyBlue}{ASE} and \textcolor{red}{Proxy} result in \textbf{biased} estimates of the risk.}
    \label{fig:experiments_with_proxy_losses}
\end{figure}

We report this comparison on \texttt{Keggdirected} and \texttt{CIFAR-10}, representing one regression dataset and one classification dataset. We use the same setup as in Figs. \ref{fig:main_uci} and \ref{fig:classification_datasets} which is described in detail in \S\ref{app:regression_exp-app_details} and \S\ref{app:classification_exp_details} respectively.

Fig. \ref{fig:experiments_with_proxy_losses} shows that \colppat{} continues to outperform the \textcolor{red}{Proxy} baseline for all choices of $\lambda$ on both datasets. This indicates that, although the proxy predictions contain useful information, directly treating them as labels is not sufficient to match the performance of our approach. Interestingly, the \textcolor{red}{Proxy} baseline performs better than \ase{} on the regression dataset, but worse than \ase{} on the classification dataset. A likely explanation is that, in the regression setting, the \texttt{TabPFN--2.5} proxy makes predictions using in--context learning with the task--specific training set, making its proxy losses relatively informative for the test pool. In contrast, in the classification setting, the CLIP proxy predictions are zero--shot and are not adapted to the specific downstream classification task, so the resulting proxy losses may be less aligned with the true losses.

\subsubsection{Comparison with Active Statistical Inference}
\label{app:comparison_with_asi}

Here, we compare \colppat{} with Active Statistical Inference (\textcolor{darkgreen}{ASI}, \cite{zrnic2024active}). Like before, we use one regression  (\texttt{Keggdirected}) and one classification dataset (\texttt{CIFAR-10}). For a fair comparison with \colppat{}, we use the ``batch'' setting of \textcolor{darkgreen}{ASI} for \texttt{CIFAR-10} and the ``active'' setting of \textcolor{darkgreen}{ASI} for \texttt{Keggdirected}, where we use the default hyperparameters from the paper for both\footnote{This includes the mixing hyperparameter $\tau$ and the finetuning batch size. We did not find any noticeable gains from further tuning of the hyperparameters.}. Note that for the ``active'' setting of \textcolor{darkgreen}{ASI}, the proxy model is updated  periodically as we collect more labels, whereas for us the proxy model is kept fixed. 

\begin{table}[!h]
\centering
\caption{Median squared error, mean confidence--interval width, and mean coverage on \texttt{Keggdirected} and \texttt{CIFAR--10} over 1000 trials, reported to two significant figures. $\pm$ indicates one standard error; standard errors below $1 \times 10^{-3}$ are reported as $0.00$. Note that \ase{} does not provide confidence intervals, so its width and coverage results are omitted.}
\label{tab:c43_summary}
\resizebox{\textwidth}{!}{%
\begin{tabular}{lccc|ccc}
\toprule
\multirow{2}{*}{Method} & \multicolumn{3}{c}{\texttt{Keggdirected}} & \multicolumn{3}{c}{\texttt{CIFAR-10}} \\
\cmidrule(lr){2-4} \cmidrule(lr){5-7}
& Med. squared error & Mean width & Mean coverage & Med. squared error & Mean width & Mean coverage \\
\midrule
\ase{} \citep{kossen2022active} & $3.2 \times 10^{-2} {\scriptstyle \pm 0.00}$ & --- & --- & $2.6 \times 10^{-3} {\scriptstyle \pm 0.00}$ & --- & --- \\
\colure{} \citep{kossen2021active} & $1.4 \times 10^{-5}{\scriptstyle \pm 0.00}$ & $1.9 \times 10^{-2}{\scriptstyle \pm 0.00}$ & $0.89{\scriptstyle \pm 0.0099}$ & $2.5 \times 10^{-3}{\scriptstyle \pm 0.00}$ & $1.9 \times 10^{-1}{\scriptstyle \pm 0.0047}$ & $0.68{\scriptstyle \pm 0.015}$ \\
\random{} & $4.1 \times 10^{-4}{\scriptstyle \pm 0.00}$ & $7.2 \times 10^{-1}{\scriptstyle \pm 0.013}$ & $0.084{\scriptstyle \pm 0.0018}$ & $1.3 \times 10^{-3} {\scriptstyle \pm 0.00}$ & $1.7 \times 10^{-1}{\scriptstyle \pm 0.0011}$ & $0.83{\scriptstyle \pm 0.010}$ \\
\textcolor{darkgreen}{ASI} \citep{zrnic2024active} & $4.5 \times 10^{-5}{\scriptstyle \pm 0.00}$ & $4.9 \times 10^{-2}{\scriptstyle \pm 0.00}$ & $0.99{\scriptstyle \pm 0.0024}$ & $9.8 \times 10^{-5}{\scriptstyle \pm 0.00}$ & $6.5 \times 10^{-2}{\scriptstyle \pm 0.0010}$ & $0.98{\scriptstyle \pm 0.0042}$ \\
\bottomrule
\bottomrule
\rowcolor{purple!10} \colppat{} ($\lambda=1$) & $4.0 \times 10^{-6}{\scriptstyle \pm 0.00}$ & $1.0 \times 10^{-2}{\scriptstyle \pm 0.00}$ & $0.90{\scriptstyle \pm 0.0096}$ & $1.7 \times 10^{-4}{\scriptstyle \pm 0.00}$ & $6.5 \times 10^{-2}{\scriptstyle \pm 0.00}$ & $0.90{\scriptstyle \pm 0.0097}$ \\
\rowcolor{purple!10} \colppat{} ($\lambda=0.5$) & $1.1 \times 10^{-5}{\scriptstyle \pm 0.00}$ & $1.0 \times 10^{-2}{\scriptstyle \pm 0.00}$ & $0.89{\scriptstyle \pm 0.010}$ & $1.0 \times 10^{-4}{\scriptstyle \pm 0.00}$ & $7.4 \times 10^{-2}{\scriptstyle \pm 0.00}$ & $0.90{\scriptstyle \pm 0.0097}$ \\
\rowcolor{purple!10} \colppat{} ($\lambda=\widehat \lambda_M$) & $4.8 \times 10^{-6}{\scriptstyle \pm 0.00}$ & $1.0 \times 10^{-2}{\scriptstyle \pm 0.00}$ & $0.88{\scriptstyle \pm 0.010}$ & $1.2 \times 10^{-4}{\scriptstyle \pm 0.00}$ & $7.1 \times 10^{-2}{\scriptstyle \pm 0.00}$ & $0.89{\scriptstyle \pm 0.010}$ \\
\bottomrule
\end{tabular}%
}
\end{table}

Following the original \textcolor{darkgreen}{ASI} paper, we report the final mean coverage and widths of the confidence intervals of both approaches, as well as the median squared error. From Table \ref{tab:c43_summary} we see that on \texttt{Keggdirected}, \colppat{} improves over \textcolor{darkgreen}{ASI} in both risk estimation and confidence–interval width while still attaining the desired coverage level: each of the three \colppat{} variants attains a lower median squared error and a substantially narrower mean interval width than \textcolor{darkgreen}{ASI}, whose intervals are moreover conservative relative to the nominal level. We believe this advantage stems from the pool–based nature of our acquisition strategy: because \colppat{} conditions on the full test pool, each acquisition is guided not only by the previously acquired labels but also by the remaining pool, allowing it to use more information when deciding which points to query. By contrast, the active variant of \textcolor{darkgreen}{ASI} used here acquires labels in a more streaming manner. On \texttt{CIFAR–10}, however, \textcolor{darkgreen}{ASI} performs comparably to, or slightly better than, \colppat{} in terms of median squared error and mean interval width, while \colppat{} attains coverage closer to the nominal level. This is unsurprising, since the batch variant of \textcolor{darkgreen}{ASI} used here is closely related to our approach: both are ultimately built on importance–weighted estimates of the risk but with different proposals.

%% file: appendix/additional_details/explanation_of_general_surrogate_score.tex
\subsection{On the Choice of the Surrogate Score in Section \ref{sub:active_testing_lure_estimator}}
\label{app:surrogate_score_choice}

Here, we discuss the choice of the surrogate score in \S\ref{sub:active_testing_lure_estimator}. This also explains our choice in \S\ref{sub:ppat_active_proposal}.

The myopic oracle proposal of \S\ref{sub:active_testing_lure_estimator} is given by
\begin{equation}
\label{eqn:oracle_proposal_general_app}
Q^\star_{m}(i)
\;\propto\;
\bigl|\ell_i\bigr|.
\end{equation}
Because the true label $y_i$ (and hence $\ell_i$) is unknown for $i \in S_m$,
$|\ell_i|$ must be replaced by a surrogate score computed from the surrogate
predictive distribution $\pi_m(\cdot \mid \bx_i)$. There are two natural candidates:
\begin{align}
    a^{\,\mathrm{abs}}_{m}(i)
    &:= \mathbb{E}_{Y \sim \pi_m(\cdot \mid \bx_i)}\!
        \bigl[\, \bigl| \mathcal{L}(f(\bx_i), Y) \bigr| \,\bigr],
        \label{eq:abs-surrogate-lure}
    \\[4pt]
    a_{m}(i)
    &:= \sqrt{\,
        \mathbb{E}_{Y \sim \pi_m(\cdot \mid \bx_i)}\!
        \bigl[\, \mathcal{L}(f(\bx_i), Y)^{2} \,\bigr]\,}.
        \label{eq:rms-surrogate-lure}
\end{align}
The score $a^{\,\mathrm{abs}}_{m}$ is obtained by applying $\pi_m$ directly to the oracle expression $|\ell_i|$; this is the score that one might expect on first reading. The score $a_{m}$ in \eqref{eq:rms-surrogate-lure} is the one we use. The two scores coincide only when
$\mathcal{L}(f(\bx_i), Y)$ is constant under $\pi_m(\cdot \mid \bx_i)$; otherwise, by Jensen's inequality applied to $|\cdot|$,
\begin{equation*}
    a^{\,\mathrm{abs}}_{m}(i) \;\leq\; a_{m}(i),
\end{equation*}
with strict inequality whenever the surrogate has nontrivial predictive uncertainty about the loss. Below, we discuss why  $a_{m}$ is the more principled surrogate of the oracle.

Firstly, note that the oracle proposal $Q^\star_{m}$ in \eqref{eqn:oracle_proposal_general_app} is derived by minimising the conditional sampling variance of LURE, which is upper-bounded by the second-moment objective

\begin{equation}
    \mathbb{E}\!\left[
        \left(\frac{\ell_{I_m}}{N\,Q_m(I_m)}\right)^{\!2}
        \;\Big|\; I_{1:m-1}
    \right]
    \;=\;
    \frac{1}{N^2}\sum_{i \in S_m} \frac{\ell_i^{2}}{Q_m(i)}.
    \label{eq:variance-objective-lure}
\end{equation}

See \S\ref{sec:app:myopic_oracle_proposal} for further details. Now, the objective on the right-hand side is \emph{quadratic} in the loss: the per--point contribution is $\ell_i^{2} / Q_m(i)$, not $|\ell_i|/Q_m(i)$. The closed--form oracle $Q^\star_{m}(i) \propto |\ell_i|$ arises only after solving this quadratic problem in closed form via Lagrange multipliers; the absolute value is a consequence of the optimisation, not of the objective.

When $\ell_i$ is unknown and a surrogate $\pi_m$ is introduced, the right
quantity to substitute into \eqref{eq:variance-objective-lure} is therefore the predictive expectation of $\ell_i^{2}$, i.e.\ the second moment, not a point estimate of $|\ell_i|$. Concretely, replacing $\ell_i^{2}$ by $\mathbb{E}_{\pi_m(\cdot | \mathbf{x_i})}[\mathcal{L}(f(\bx_i), Y)^2]$ in \eqref{eq:variance-objective-lure} and re-solving the same minimisation problem yields exactly
\begin{equation*}
    Q_{m}^{\mathrm{surr}}(i)
    \;\propto\;
    \sqrt{\mathbb{E}_{\pi_m(\cdot | \mathbf{x_i})}\!\bigl[\mathcal{L}(f(\bx_i), Y)^{2}\bigr]}
    \;=\;
    a_{m}(i).
\end{equation*}

In other words, $a_{m}$ is the surrogate score induced by the \emph{variance objective itself}, not by an after--the--fact substitution into the closed--form oracle.

%% file: appendix/additional_details/additional_details.tex

\subsection{Simplified Expressions for the Proposals of LURE and PPAT}
\label{sub:simplification_proposal_mse_loss}
Here, we describe the proposals used for LURE and PPAT in our experiments. We discuss the case for the MSE loss -- the case for the cross--entropy loss follows similarly.

\paragraph{LURE.} Let
$f_i := f(\mathbf{x}_i)$,
$\mu_{m,i} := \E_{\pi_m(\cdot\mid \mathbf{x}_i)}[y]$ and
$\sigma^2_{m,i} := \Var_{\pi_m(\cdot\mid \mathbf{x}_i)}(y).$
For the MSE loss, we follow \cite{kossen2021active} and use the proposal
\begin{equation*}
Q_m(i) \propto \E_{\pi_m(\cdot\mid \mathbf{x}_i)}\!\left[(f_i-y)^2\right] = \Var_{\pi_m(\cdot\mid \mathbf{x}_i)}(y)
+ \Bigl(\E_{\pi_m(\cdot\mid \mathbf{x}_i)}[y]-f_i\Bigr)^2 = \sigma^2_{m,i}+(\mu_{m,i}-f_i)^2.
\end{equation*}
Note that this is not the same as the proposal $Q_m(i) \propto \sqrt{\E_{\pi_m(\cdot \mid \mathbf{x}_i)}[(f_i - y)^4]}$ described in \S\ref{sub:active_testing_lure_estimator} and  \S\ref{app:surrogate_score_choice}; since we know that our loss is nonnegative, we can instead use the simpler $Q_m(i) \propto \E_{\pi_m(\cdot\mid \mathbf{x}_i)}\!\left[(f_i-y)^2\right]$.

From the above, we can see that under the MSE loss the LURE proposal has the standard bias--variance
decomposition form: it favours points for which either the surrogate is uncertain about the label, or the surrogate mean disagrees strongly with the prediction of the model being evaluated. 

\paragraph{PPAT.} Let  $c_i = \widetilde{\ell}_i - \widetilde{R}$, and $\widetilde{\ell}_i = (f_i-g(\bx_i))^2$. The proposal  $a_{m,\lambda}$ \eqref{eqn:ppat_score} admits the following decomposition
\begin{equation}
\label{eqn:ppat_proposal_bias_var}
\mathbb{E}_{\pi_m(\cdot \mid \bx_i)}
\!\left[
\bigl(\mathcal{L}(f(\bx_i),Y)-\lambda c_i\bigr)^2
\right]
=
\mathrm{Var}_{\pi_m(\cdot \mid \bx_i)}
\!\left(
\mathcal{L}(f(\bx_i),Y)
\right)
+
\left(
\mathbb{E}_{\pi_m(\cdot \mid x_i)}
\!\left[
L(f(x_i),Y)
\right]
-
\lambda c_i
\right)^2.    
\end{equation}
This shows that the PPAT proposal favours points for which either the surrogate is uncertain about the loss, or the surrogate mean loss is far from the proxy term $\lambda c_i$. 

Under the MSE loss,  \eqref{eqn:ppat_proposal_bias_var} becomes
\begin{equation*}
\E_{\pi_m(\cdot\mid \bx_i)}\!\left[\bigl((f_i-Y)^2-\lambda c_i\bigr)^2\right]
=
\Var_{\pi_m(\cdot\mid \bx_i)}\!\left((f_i-Y)^2\right)
+
\left(
\E_{\pi_m(\cdot\mid \bx_i)}[(f_i-Y)^2]-\lambda c_i
\right)^2.    
\end{equation*}

Now, using the standard bias--variance decomposition, the second term simplifies to
\begin{equation*}
\left(
\sigma^2_{m,i}+(\mu_{m,i}-f_i)^2-\lambda c_i
\right)^2.    
\end{equation*}
To simplify the first term, let $a_{m,i}=f_i-\mu_{m,i}$ and $z=Y-\mu_{m,i}$ so that $\E[z]=0$, $\Var(z)=\sigma^2_{m,i}$, and $(f_i-Y)^2=(a_{m,i}-z)^2$. Then we have:
\begin{equation*}
\Var\!\left((f_i-Y)^2\right)
=
\Var(z^2-2a_{m,i}z)
=
\bigl(\E[z^4]-\sigma_{m,i}^4\bigr)
+4a_{m,i}^2\sigma_{m,i}^2
-4a_{m,i}\E[z^3].    
\end{equation*}

Now, assuming the surrogate predictive distribution is Gaussian\footnote{This is the case for our experiments where we use a Bayesian linear regression model with Gaussian priors as the surrogate. It is also the case for a large family of models such as Gaussian process regression with a Gaussian likelihood and Bayesian linear regression with conjugate Gaussian priors.}, i.e. $Y\mid \bx_i \sim \mathcal{N}(\mu_{m,i},\sigma^2_{m,i})$,
then $\E[z^3]=0$ and $\E[z^4]=3\sigma_{m,i}^4$, giving
\begin{equation*}
\Var_{\pi_m(\cdot\mid \bx_i)}\!\left((f_i-Y)^2\right)
=
2\sigma_{m,i}^4 + 4\sigma_{m,i}^2(\mu_{m,i}-f_i)^2.    
\end{equation*}

Therefore, under the MSE loss and a Gaussian surrogate, the PPAT proposal simplifies to
\[
Q_{m,\lambda}^{\mathrm{PPAT}}(i)\propto
\sqrt{
2\sigma_{m,i}^4
+4\sigma_{m,i}^2(\mu_{m,i}-f_i)^2 +
\left(\sigma^2_{m,i}+(\mu_{m,i}-f_i)^2-\lambda c_i\right)^2
}.
\]

%% file: appendix/additional_details/ppi.tex
\subsection{Extended Background}
\label{app:extended_background}

%% file: appendix/additional_details/ase.tex
\subsubsection{Active Testing with Active Surrogate Estimators}
\label{app:extended_background_ase}

Here, we provide further details on the XWED acquisition function used in \cite{kossen2022active}.

At acquisition round $m$, let $\pi_m(\cdot \mid \mathbf{x})$ denote the surrogate predictive distribution used in the main text. When this surrogate is Bayesian, ensemble--based, or otherwise
represented as a mixture, we write $\phi_m$ for the fitted quantities that determine the surrogate at round $m$ -- for example, posterior hyperparameters, variational parameters, or the trained ensemble -- and let $\Theta$ denote the latent component of the surrogate predictive model. With this notation, the round-$m$ surrogate predictive distribution can be written as the posterior predictive mixture
\begin{equation*}
\pi_m(Y \mid \mathbf{x}) \equiv \pi_m(Y \mid \mathbf{x};\phi_m)
= \mathbb{E}_{\Theta \sim \pi_m(\cdot;\phi_m)}
\bigl[\pi_m(Y \mid \mathbf{x},\Theta;\phi_m)\bigr].
\end{equation*}
The XWED score is then given by
\begin{align*}
\mathrm{XWED}(\mathbf{x})
&=
\mathbb{E}_{Y \sim \pi_m(\cdot \mid \mathbf{x})}
\!\left[
-\mathcal{L}(f(\mathbf{x}), Y) \log \pi_m(Y \mid \mathbf{x})
\right]
 \\ &\quad -
\mathbb{E}_{\Theta \sim \pi_m(\cdot;\phi_m)}
\!\left[
\mathbb{E}_{Y \sim \pi_m(\cdot \mid \mathbf{x},\Theta; \phi_m)}
\!\left[
-\mathcal{L}(f(\mathbf{x}), Y) \log \pi_m(Y \mid \mathbf{x},\Theta;\phi_m)
\right]
\right].
\end{align*}
with the following acquisition rule at round $m$:
\begin{equation*}
i_m
=
\arg\max_{i \in S_m}
\mathrm{XWED}(\mathbf{x}_i).
\end{equation*}

%% file: appendix/mathematical_results/asymptoticresults.tex
\label{sec:app:asymptotics}

Here, we present the full statements and proofs for the results in the main body of our paper. Since the main body states results in a simplified form for readability, we first introduce the definitions needed to state these results formally in \S\ref{sec:app:defns}. All subsequent results in this section make use of these definitions.

We note that our finite-sample unbiasedness, variance, and \(L^2\)-consistency arguments (\S\ref{sec:app:luredef}--\ref{sec:app:lureconsistency}) closely follow the analysis of
\citet[Appendix B.4--B.7]{farquhar2021statistical}. The main difference is that we state the arguments conditionally on a fixed finite test pool and for general real-valued arrays \(\xi_{M,i}\), which generalises the results in \cite{farquhar2021statistical} to include both standard LURE (\(\xi_{M,i}=\ell_i\)) and PPAT (\(\xi_{M,i}=\ell_i-\lambda(\widetilde{\ell}_i-\tilde R_M)\)). In addition to this, \S\ref{sec:app:lureCLT} provides a new asymptotic normality result for the resulting LURE-style estimators.


\subsection{Definitions}
\label{sec:app:defns}
We first introduce the definition of an \emph{active proposal process} (APP), which formalises the adaptive sampling proposals considered throughout the paper.

\begin{definition}[Active proposal process, APP]
\label{def:admissible_active_proposal}
Let \(H_0\) denote all information available before querying the true labels.
This includes the inputs \(\bx_{1:N}\), model predictions, or
other side information, but it does not include the unqueried true labels. 

An admissible active proposal process (APP) \(Q_{1:M}=(Q_1,\ldots,Q_M)\) is generated
by a collection of measurable maps \(\eta_1,\ldots,\eta_M\) such that, at step
\(m\), the map \(\eta_m\) takes as input the observed history $H_0,\,
(I_1,y_{I_1}),\ldots,(I_{m-1},y_{I_{m-1}})$
and returns a probability distribution on \(S_m\) with full support. The
proposal used at step \(m\) is denoted by
\begin{equation*}
Q_m
=
\eta_m\!\left(
H_0,
(I_1,y_{I_1}),\ldots,(I_{m-1},y_{I_{m-1}})
\right),
\end{equation*}
and satisfies
\begin{equation*}
Q_m(i)=0 \quad \text{for } i\notin S_m,
\qquad
\sum_{i\in S_m}Q_m(i)=1,
\qquad
Q_m(i)>0 \quad \text{for } i\in S_m.    
\end{equation*}
The next queried index is sampled according to
\begin{equation*}
\Pr\!\left(
I_m=i
\mid
H_0,I_1,y_{I_1},\ldots,I_{m-1},y_{I_{m-1}}
\right)
=
Q_m(i),
\qquad i\in S_m.    
\end{equation*}
\end{definition}

We now formalise the definition of a surrogate model introduced in \S\ref{sub:active_testing_lure_estimator}.

\begin{definition}[Surrogate model]
Let  $\cF_m=\sigma(\{I_1,\ldots,I_m\})$. At round $m$, a surrogate model is an $\mathcal F_{m-1}$--measurable Markov kernel
$\pi_m(\cdot\mid x)$
on $\mathcal Y$. That is, for each input $x$, $\pi_m(\cdot\mid x)$ is a predictive distribution over the unknown label, constructed using only information available before round $m$.
\end{definition}

Typical examples include Bayesian neural networks, random forests and Gaussian processes.


\subsection{Generalising LURE with Triangular--Arrays}
\label{sec:app:luredef}
We begin by introducing a general triangular--array formulation that captures both the standard LURE estimator and the PPAT estimator as special cases, and we establish two algebraic identities -- a telescoping representation in terms of running averages and a normalisation identity for the finite--pool weights -- that underpin all subsequent results.

Let $(\bx_1,\bx_2,\ldots)$ be a sequence of inputs, with associated labels $(y_1,y_2,\ldots)$. Moreover, let $(N_1,N_2,\ldots)$ be a sequence of test pool sizes, with $N_M/M\to\alpha$ as $M\to\infty$ for some $\alpha>1$, and let $(\xi_{M,i})_{M\geq 1,i=1,\ldots,N_M}$ be a triangular array of reals, with $\xi_{M,i}=h_M(\bx_{1:N_M},y_i)$ for some function $h_M:\calX^{N_M}\times\calY\to \bbR$. Define $R_{M}$ as
$$R_{M}:=\frac{1}{N_M}\sum_{i=1}^{N_M} \xi_{M,i}.$$

\begin{rmrk}
\label{rmrk:array_special_cases}
In this paper, we consider either $\xi_{M,i}=\ell_i$ (standard LURE), where $\ell_i=\calL(f(\bx_i),y_i)$,  or $\xi_{M,i}=\ell_i-\lambda(\widetilde\ell_i-\widetilde R_M)$ (PPAT), where $\widetilde\ell_i=\calL(f(\bx_i),g(\bx_i))$ and \(\widetilde R_M=\frac{1}{N_M}\sum_{i=1}^{N_M} \widetilde\ell_i \). Note that in both cases, $R_M=\frac{1}{N_M}\sum_{i=1}^{N_M} \ell_i$. 
\end{rmrk}

Consider the LURE estimator of $R_M$ given by
$$
\widehat R_{M}=\frac{1}{M}\sum_{m=1}^{M} V_{M,m} \xi_{M,I_{M,m}},
$$
with $I_{M,m}\mid I_{M,1:m-1}\sim Q_{M,m}$, where, for each $M$, $Q_{M,1},\ldots,Q_{M,M}$ is a sequence of active proposals in the sense of Def. \ref{def:admissible_active_proposal}, and 
\begin{align}
V_{M,m} &= 1 + \frac{N_M-M}{N_M-m}\left(\frac{1}{(N_M-m+1)\,Q_{M,m}(I_{M,m})} - 1\right).
\end{align}

Throughout this section, for each \(M\), we treat the finite test pool $(\bx_1,y_1),\ldots,(\bx_{N_M},y_{N_M})$ as fixed. All probabilities, expectations, variances, and asymptotic statements
are with respect to the randomisation of the active sampling procedure only. Moreover, we write \(N\) for \(N_M\) whenever the dependence on \(M\) is clear and let $S_{M,m}=[N]\setminus\{I_{M,1},\ldots,I_{M,m-1}\}$
be the set of unqueried indices before round \(m\), and  $\mathcal F_{M,m}=\sigma(I_{M,1},\ldots,I_{M,m})$.

We first state a telescoping representation of LURE in terms of running averages.
\begin{lem}[Alternative representation of LURE]
\label{lem:lure_alt_rep}
For each \(M\), define
\[
\gamma_{M,m}
=
\frac{N(N-M)}{(N-m)(N-m+1)}
\]
and
\[
A_{M,m}
=
\frac1N
\left\{
\frac{\xi_{M,I_{M,m}}}{Q_{M,m}(I_{M,m})}
+
\sum_{t=1}^{m-1}\xi_{M,I_{M,t}}
\right\}.
\]
Then
\[
\widehat R_M
=
\frac1M\sum_{m=1}^M
\gamma_{M,m}A_{M,m}.
\]
\end{lem}

\begin{proof}
We expand
\[
\sum_{m=1}^M\gamma_{M,m}A_{M,m}
=
\sum_{m=1}^M
\gamma_{M,m}
\left[
\frac{\xi_{M,I_{M,m}}}{NQ_{M,m}(I_{M,m})}
+
\frac1N\sum_{t=1}^{m-1}\xi_{M,I_{M,t}}
\right].
\]
Collect the coefficient of a fixed sampled value
\(\xi_{M,I_{M,t}}\), where \(t\in\{1,\ldots,M\}\). This coefficient is
\[
\frac{\gamma_{M,t}}{NQ_{M,t}(I_{M,t})}
+
\sum_{m=t+1}^M\frac{\gamma_{M,m}}{N}.
\]
The first term is given by
\[
\frac{\gamma_{M,t}}{NQ_{M,t}(I_{M,t})}
=
\frac{N-M}{(N-t)(N-t+1)Q_{M,t}(I_{M,t})},
\]
and for the second term we have
\[
\sum_{m=t+1}^M\frac{\gamma_{M,m}}{N}
=
\sum_{m=t+1}^M
\frac{N-M}{(N-m)(N-m+1)}.
\]

Now, using
\[
\frac1{k(k+1)}=\frac1k-\frac1{k+1},
\]

and setting \(k=N-m\), this telescopes to:
\[
\sum_{m=t+1}^M
\frac{N-M}{(N-m)(N-m+1)}
=
(N-M)
\sum_{k=N-M}^{N-t-1}
\frac1{k(k+1)}
=
1-\frac{N-M}{N-t}.
\]

The coefficient of $\xi_{M,I_{M,t}}$ is therefore given by:
\begin{align*}
&1-\frac{N-M}{N-t}
+
\frac{N-M}{(N-t)(N-t+1)Q_{M,t}(I_{M,t})}. \\
&=1+
\frac{N-M}{N-t}
\left\{
\frac{1}{(N-t+1)Q_{M,t}(I_{M,t})}-1
\right\} \\
&=V_{M,t}.
\end{align*}

Finally, we have
\begin{equation*}
\sum_{m=1}^M\gamma_{M,m}A_{M,m}
=
\sum_{t=1}^M
V_{M,t}\xi_{M,I_{M,t}},    
\end{equation*}
and dividing by $M$ gives the desired result.
\end{proof}

\vspace{0.25cm}
We now state a normalisation identity for the finite-pool weights.
\begin{lem}[Telescoping identity]
\label{lem:gamma_sum}
For each \(M\), we have $\sum_{m=1}^M\gamma_{M,m}=M.$
\end{lem}

\begin{proof}
Using the definition of \(\gamma_{M,m}\),
\[
\sum_{m=1}^M\gamma_{M,m}
=
N(N-M)
\sum_{m=1}^M
\frac1{(N-m)(N-m+1)}.
\]

Setting \(k=N-m\), we have:
\[
\sum_{m=1}^M
\frac1{(N-m)(N-m+1)}
=
\sum_{k=N-M}^{N-1}
\frac1{k(k+1)}.
\]

Now, using the identity
\[
\frac1{k(k+1)}=\frac1k-\frac1{k+1},
\]
the sum telescopes:
\[
\sum_{k=N-M}^{N-1}
\frac1{k(k+1)}
=
\frac1{N-M}-\frac1N.
\]

Therefore
\[
\sum_{m=1}^M\gamma_{M,m}
=
N(N-M)
\left(
\frac1{N-M}-\frac1N
\right)
=
M.
\]
\end{proof}


\subsection{Finite--Sample Unbiasedness and Variance}
\label{sec:app:lureunbiased}
Using the running-average representation of Lemma \ref{lem:lure_alt_rep}, we now establish two finite--sample properties of our general formulation of LURE: that it is unbiased for the test--pool risk under any APP, and that its variance admits a clean decomposition into a sum of conditional variances of the importance--weighted increments. Using this, we then prove Prop. \ref{prop:cv-unbiased} from the main text
\begin{prop}[Unbiasedness and finite-sample variance]
\label{prop:lure_unbiased_variance}
Assume that we have an APP in the sense of Def. \ref{def:admissible_active_proposal}. Then, for each \(m=1,\ldots,M\), $\mathbb E[A_{M,m}\mid \mathcal F_{M,m-1}]=R_M$, and, consequently, $\mathbb E[A_{M,m}]=R_M$. Moreover, we have that
\begin{equation*}
\operatorname{Var}(A_{M,m})
=
\frac1{N^2}
\mathbb E\!\left[
\operatorname{Var}
\left(
\frac{\xi_{M,I_{M,m}}}{Q_{M,m}(I_{M,m})}
\;\middle|\;
\mathcal F_{M,m-1}
\right)
\right],    
\end{equation*}
and for $m\neq k$, we have $\operatorname{Cov}(A_{M,m},A_{M,k})=0.$ It therefore follows that
\begin{equation*}
\mathbb E[\widehat R_M]=R_M    
\end{equation*}
and
\[
\operatorname{Var}(\widehat R_M)
=
\frac1{M^2}
\sum_{m=1}^M
\gamma_{M,m}^2\operatorname{Var}(A_{M,m}).
\]
Equivalently,
\[
\operatorname{Var}(\widehat R_M)
=
\frac1{M^2N^2}
\sum_{m=1}^M
\gamma_{M,m}^2
\mathbb E\!\left[
\operatorname{Var}
\left(
\frac{\xi_{M,I_{M,m}}}{Q_{M,m}(I_{M,m})}
\;\middle|\;
\mathcal F_{M,m-1}
\right)
\right].
\]
\end{prop}

\begin{proof}
Fix \(m \in 1,\ldots,M\). Conditional on \(\mathcal F_{M,m-1}\), the set \(S_{M,m}\), the
proposal \(Q_{M,m}\), and the previously sampled values
\(\xi_{M,I_{M,1}},\ldots,\xi_{M,I_{M,m-1}}\) are fixed. 
Moreover, by admissibility we know that
\[
\mathbb P(I_{M,m}=i\mid\mathcal F_{M,m-1})
=
Q_{M,m}(i),
\qquad i\in S_{M,m}.
\]

Therefore, the conditional expectation of the importance-weighted sampled value is
\[
\begin{aligned}
\mathbb E\!\left[
\frac{\xi_{M,I_{M,m}}}{NQ_{M,m}(I_{M,m})}
\;\middle|\;
\mathcal F_{M,m-1}
\right]
&=
\sum_{i\in S_{M,m}}
Q_{M,m}(i)
\frac{\xi_{M,i}}{NQ_{M,m}(i)}
\\
&=
\frac1N\sum_{i\in S_{M,m}}\xi_{M,i}.
\end{aligned}
\]

Adding the contribution from the previously sampled indices gives
\[
\begin{aligned}
\mathbb E[A_{M,m}\mid \mathcal F_{M,m-1}]
&=
\frac1N\sum_{i\in S_{M,m}}\xi_{M,i}
+
\frac1N\sum_{t=1}^{m-1}\xi_{M,I_{M,t}}
\\
&=
\frac1N\sum_{i=1}^N\xi_{M,i}
\\
&=
R_M.
\end{aligned}
\]

Taking expectations then gives \(\mathbb E[A_{M,m}]=R_M\). Now, noting that 
\begin{equation*}
A_{M,m}
=
\frac1N
\frac{\xi_{M,I_{M,m}}}{Q_{M,m}(I_{M,m})}
+
\frac1N\sum_{t=1}^{m-1}\xi_{M,I_{M,t}},    
\end{equation*}
where the second term is \(\mathcal F_{M,m-1}\)--measurable, we have 
\[
\operatorname{Var}(A_{M,m}\mid \mathcal F_{M,m-1})
=
\frac1{N^2}
\operatorname{Var}
\left(
\frac{\xi_{M,I_{M,m}}}{Q_{M,m}(I_{M,m})}
\;\middle|\;
\mathcal F_{M,m-1}
\right).
\]

Since $\mathbb E[A_{M,m}\mid\mathcal F_{M,m-1}]=R_M$ is deterministic, the law of total variance then yields the desired variance formula:
\begin{align*}
\operatorname{Var}(A_{M,m})
 &=
\mathbb E[
\operatorname{Var}(A_{M,m}\mid \mathcal F_{M,m-1})] \\
&= \frac1{N^2} \mathbb E\left[
\operatorname{Var}\left(\frac{\xi_{M,I_{M,m}}}{Q_{M,m}(I_{M,m})}\;\middle|\;\mathcal F_{M,m-1}\right)\right]
\end{align*}

\vspace{0.25cm}
Suppose now that $m<k$. Since \(A_{M,m}\) is \(\mathcal F_{M,k-1}\)--measurable, we have
\[
\begin{aligned}
\operatorname{Cov}(A_{M,m},A_{M,k})
&=
\mathbb E[(A_{M,m}-R_M)(A_{M,k}-R_M)]
\\
&=
\mathbb E\!\left[
(A_{M,m}-R_M)
\mathbb E[A_{M,k}-R_M\mid \mathcal F_{M,k-1}]
\right]
\\
&=0.
\end{aligned}
\]
The case \(k<m\) follows by symmetry.

\vspace{0.25cm}
Finally, using Lemma~\ref{lem:lure_alt_rep} and \ref{lem:gamma_sum} gives:
\begin{align*}
\mathbb E[\widehat R_M] &= \frac1M\sum_{m=1}^M\gamma_{M,m}\mathbb E[A_{M,m}] \\ 
&= \frac{R_M}{M}\sum_{m=1}^M\gamma_{M,m} \\
&= R_M.
\end{align*}
The variance formula follows similarly by using Lemma~\ref{lem:lure_alt_rep} again and the fact that the $A_{M,m}$'s are pairwise uncorrelated.
\end{proof}

\vspace{0.25cm}
Prop. \ref{prop:cv-unbiased} follows as an immediate corollary to the above proposition. We restate the proposition more formally below.
\begin{cor}[Proposition~\ref{prop:cv-unbiased}]
\label{cor:ppat_unbiasedness} 
Consider the setup of \S\ref{sec:app:luredef} and fix any \(\lambda \in \mathbb R\). Then, for any APP,
\begin{equation*}
\mathbb E\!\left[\widehat R^{\mathrm{PPAT}}_M(\lambda)\right] = R_M. \end{equation*}
\end{cor}

\begin{proof}
This follows immediately from Prop.~\ref{prop:lure_unbiased_variance}. Indeed, fix \(\lambda \in \mathbb R\) and apply Prop.~\ref{prop:lure_unbiased_variance} with the triangular array
\begin{equation*}
\xi_{M,i}
=
\ell_i-\lambda c_{M,i}
=
\ell_i-\lambda(\widetilde\ell_i-\widetilde R_M).    
\end{equation*}

The corresponding LURE estimator is exactly
\begin{equation*}
\widehat R^{\mathrm{PPAT}}_M(\lambda)
= \frac{1}{M}\sum_{m=1}^M V_{M,m}\xi_{M,I_{M,m}}.    
\end{equation*}

Moreover, the finite-pool mean of this array is
\begin{equation*}
\frac{1}{N}\sum_{i=1}^N \xi_{M,i}
=
\frac{1}{N}\sum_{i=1}^N \ell_i
-
\lambda
\frac{1}{N}\sum_{i=1}^N(\widetilde\ell_i-\widetilde R_M)
=
R_M,    
\end{equation*}
since \(\frac{1}{N}\sum_{i=1}^N(\widetilde\ell_i-\widetilde R_M)=0\). Prop.~\ref{prop:lure_unbiased_variance} therefore gives
\begin{equation*}
\mathbb E\!\left[\widehat R^{\mathrm{PPAT}}_M(\lambda)\right]
= R_M,    
\end{equation*}
as claimed.
\end{proof}

We now more formally state and prove Prop. \ref{prop:var-quadratic} from the main text.
\begin{prop}[Proposition \ref{prop:var-quadratic}]
\label{prop:ppat-variance}
Consider the setup of \S\ref{sec:app:luredef} and fix an APP $Q_{M,1:M}$. Let $c_{M,i}=\widetilde\ell_i-\widetilde R_M$ and define
\begin{equation*}
\widehat R^{\mathrm{LURE}}_M=\frac1M\sum_{m=1}^M V_{M,m}\,\ell_{I_{M,m}},
\qquad
\widehat C^{\mathrm{LURE}}_M=\frac1M\sum_{m=1}^M V_{M,m}\,c_{M,I_{M,m}} .    
\end{equation*}
Then, for every $\lambda\in\mathbb R$,
\begin{equation*}
\operatorname{Var}\!\big(\widehat R^{\mathrm{PPAT}}_M(\lambda)\big)
=\operatorname{Var}\!\big(\widehat R^{\mathrm{LURE}}_M\big)
-2\lambda\,\operatorname{Cov}\!\big(\widehat R^{\mathrm{LURE}}_M,\widehat C^{\mathrm{LURE}}_M\big)
+\lambda^2\operatorname{Var}\!\big(\widehat C^{\mathrm{LURE}}_M\big).    
\end{equation*}

Thus, 
$\operatorname{Var}(\widehat R^{\mathrm{PPAT}}_M(\lambda))<\operatorname{Var}(\widehat R^{\mathrm{LURE}}_M)$
whenever $\lambda\in\big(\min\{0,2\lambda^\star\},\,\max\{0,2\lambda^\star\}\big)$ and $\operatorname{Var}(\widehat C^{\mathrm{LURE}}_M)>0$, where
\begin{equation*}
\lambda^\star=\frac{\operatorname{Cov}(\widehat R^{\mathrm{LURE}}_M,\widehat C^{\mathrm{LURE}}_M)}
{\operatorname{Var}(\widehat C^{\mathrm{LURE}}_M)}.
\end{equation*}

In particular, $\lambda^\star$ minimises the variance of $\widehat R^{\mathrm{PPAT}}_M(\lambda)$.
\end{prop}

\begin{proof}
Fix $\lambda\in\mathbb R$. By Remark \ref{rmrk:array_special_cases}, $\widehat R^{\mathrm{PPAT}}_M(\lambda)$ is the LURE
estimator associated with the residualised array
$\zeta_{M,i}(\lambda)=\ell_i-\lambda c_{M,i}$. That is, 
\begin{align*}
\widehat R^{\mathrm{PPAT}}_M(\lambda)&=\frac1M\sum_{m=1}^M V_{M,m}\,\zeta_{M,I_{M,m}}(\lambda) \\    
&=\frac1M\sum_{m=1}^M V_{M,m}\big(\ell_{I_{M,m}}-\lambda c_{M,I_{M,m}}\big) \\[0.75em]
&=\widehat R^{\mathrm{LURE}}_M-\lambda\,\widehat C^{\mathrm{LURE}}_M. 
\end{align*}

Now, by the properties of variance we have
\begin{equation*}
 \operatorname{Var}\!\big(\widehat R^{\mathrm{PPAT}}_M(\lambda)\big)
=\operatorname{Var}\!\big(\widehat R^{\mathrm{LURE}}_M\big)
-2\lambda\,\operatorname{Cov}\!\big(\widehat R^{\mathrm{LURE}}_M,\widehat C^{\mathrm{LURE}}_M\big)
+\lambda^2\operatorname{Var}\!\big(\widehat C^{\mathrm{LURE}}_M\big),   
\end{equation*}
as claimed.

\vspace{0.25cm}
For the second part, write $V_C:=\operatorname{Var}(\widehat C^{\mathrm{LURE}}_M)>0$ and
$b:=\operatorname{Cov}(\widehat R^{\mathrm{LURE}}_M,\widehat C^{\mathrm{LURE}}_M)$, and set
\begin{equation*}
\phi(\lambda):=\operatorname{Var}\!\big(\widehat R^{\mathrm{PPAT}}_M(\lambda)\big)
-\operatorname{Var}\!\big(\widehat R^{\mathrm{LURE}}_M\big)
=V_C\,\lambda^2-2b\,\lambda
=V_C\,\lambda\big(\lambda-2\lambda^\star\big),    
\end{equation*}
where $\lambda^\star=b/V_C$. This is a strictly
convex parabola with roots $0$ and $2\lambda^\star$, so $\phi(\lambda)<0$ if and only if
$\lambda$ lies strictly between these roots, i.e.
$\lambda\in(\min\{0,2\lambda^\star\},\max\{0,2\lambda^\star\})$; this is the stated variance--reduction condition. Examining the derivatives of $\phi$ then gives $\lambda^\star$ as the minimiser 
of $\operatorname{Var}\!\big(\widehat R^{\mathrm{PPAT}}_M(\lambda)\big)$, as required.
\end{proof}


\subsubsection{PPAT Variance Bound}
\label{app:sec:var_ppat_bound}
Here, we formally state and prove the variance bound of Prop.~\ref{prop:ppat-bound}.
Throughout we fix the budget $M$ and the finite test pool, and write $N$ for $N_M$. Recall from Remark~\ref{rmrk:array_special_cases} that, for a fixed $\lambda\in\mathbb{R}$, the PPAT estimator is exactly the LURE estimator associated with the triangular array $\zeta_{M,i}(\lambda)\;=\;\ell_i-\lambda\,c_{M,i}$.
That is, by the running--average representation of Lemma~\ref{lem:lure_alt_rep},
\begin{equation*}
  \widehat R^{\mathrm{PPAT}}_M(\lambda)
  \;=\;
  \frac1M\sum_{m=1}^M V_{M,m}\,\zeta_{M,I_{M,m}}(\lambda)
  \;=\;
  \frac1M\sum_{m=1}^M \gamma_{M,m}\,A_{M,m},    
\end{equation*}
where $\gamma_{M,m}$ and $A_{M,m}$ are as defined there for the array
$\zeta_{M,\cdot}(\lambda)$. 

\begin{prop}[Variance bound for PPAT]
\label{prop:ppat_var_bound}
Suppose we are given an APP satisfying a uniform overlap condition: there exists $\beta>0$ such that
\begin{equation*}
  Q_{M,m}(i)\;\ge\;\frac{\beta}{N},
  \qquad m=1,\dots,M,\ \ i\in S_{M,m}.    
\end{equation*}
Then, for every $\lambda\in\mathbb{R}$,
\begin{equation*}
  \operatorname{Var}\!\bigl(\widehat R^{\mathrm{PPAT}}_M(\lambda)\bigr)
  \;\le\;
  \frac{1}{\beta M}\,\frac{N}{N-M+1}
  \left(\frac1N\sum_{i=1}^N(\ell_i-\lambda c_{M,i})^2\right).    
\end{equation*}
Whenever $\tfrac1N\sum_{i=1}^N c_{M,i}^2>0$, the right--hand side is minimised over $\lambda$ by
\begin{equation*}
\lambda^{\dagger}
\;=\;
\frac{N^{-1}\sum_{i=1}^N \ell_i c_{M,i}}{N^{-1}\sum_{i=1}^N c_{M,i}^2}.    
\end{equation*}
\end{prop}

\begin{proof}
Fix $\lambda\in\mathbb{R}$, abbreviate $\zeta_i:=\zeta_{M,i}(\lambda)=\ell_i-\lambda c_{M,i}$,
and set $D(\lambda):=\tfrac1N\sum_{i=1}^N\zeta_i^2=\tfrac1N\sum_{i=1}^N(\ell_i-\lambda c_{M,i})^2$.

By Prop.~\ref{prop:lure_unbiased_variance}, the running averages
$A_{M,1},\dots,A_{M,M}$ are pairwise uncorrelated and satisfy
\begin{align*}
\mathbb{E}[A_{M,m}\mid\mathcal F_{M,m-1}]&=R_M,  \\
\operatorname{Var}\!\bigl(\widehat R^{\mathrm{PPAT}}_M(\lambda)\bigr)
&=\frac{1}{M^2}\sum_{m=1}^M \gamma_{M,m}^2\,\operatorname{Var}(A_{M,m}), \\
\operatorname{Var}(A_{M,m})
&=\frac{1}{N^2}\,
\mathbb{E}\!\left[\operatorname{Var}\!\left(
\frac{\zeta_{I_{M,m}}}{Q_{M,m}(I_{M,m})}\,\middle|\,\mathcal F_{M,m-1}\right)\right].
\end{align*}

We first bound $\operatorname{Var}(A_{M,m})$ uniformly in $m$. For this, fix an arbitrary $m$ and condition on $\mathcal F_{M,m-1}$.
Bounding the conditional variance by the conditional second moment and evaluating the latter gives
\begin{align*}
\operatorname{Var}\!\left(\frac{\zeta_{I_{M,m}}}{Q_{M,m}(I_{M,m})}\,\middle|\,\mathcal F_{M,m-1}\right)
\;&\le\;
\mathbb{E}\!\left[\frac{\zeta_{I_{M,m}}^2}{Q_{M,m}(I_{M,m})^2}\,\middle|\,\mathcal F_{M,m-1}\right] \\
\;&=\;
\sum_{i\in S_{M,m}} Q_{M,m}(i)\,\frac{\zeta_i^2}{Q_{M,m}(i)^2} \\
\;&=\;
\sum_{i\in S_{M,m}}\frac{\zeta_i^2}{Q_{M,m}(i)}.    
\end{align*}

Now, the overlap condition implies $1/Q_{M,m}(i)\le N/\beta$ for every $i\in S_{M,m}$. As the summands are nonnegative and $S_{M,m}\subseteq\{1,\dots,N\}$,
\[
  \sum_{i\in S_{M,m}}\frac{\zeta_i^2}{Q_{M,m}(i)}
  \;\le\;
  \frac{N}{\beta}\sum_{i\in S_{M,m}}\zeta_i^2
  \;\le\;
  \frac{N}{\beta}\sum_{i=1}^N\zeta_i^2 .
\]
The right--hand side is deterministic, so taking expectations and dividing by $N^2$ yields
\begin{equation*}
  \operatorname{Var}(A_{M,m})
  \;\le\;
  \frac{1}{N^2}\cdot\frac{N}{\beta}\sum_{i=1}^N\zeta_i^2
  \;=\;
  \frac1\beta\left(\frac1N\sum_{i=1}^N\zeta_i^2\right)
  \;=\;
  \frac{D(\lambda)}{\beta},    
\end{equation*}
a bound independent of $m$.

\vspace{0.25cm}
We now aggregate over the different rounds. The weights
$\gamma_{M,m}=\frac{N(N-M)}{(N-m)(N-m+1)}$ are positive, and since $(N-m)(N-m+1)$ is
decreasing in $m$ on $\{1,\dots,M\}$, $\gamma_{M,m}$ is increasing in $m$; its maximum is
therefore attained at $m=M$,
\begin{equation*}
  \max_{1\le m\le M}\gamma_{M,m}
  =\gamma_{M,M}
  =\frac{N(N-M)}{(N-M)(N-M+1)}
  =\frac{N}{N-M+1}.    
\end{equation*}

Using $\gamma_{M,m}^2\le\gamma_{M,m}\cdot\max_{1\le k\le M}\gamma_{M,k}$ together with the identity $\sum_{m=1}^M\gamma_{M,m}=M$ of Lemma~\ref{lem:gamma_sum}, gives
\begin{equation*}
\sum_{m=1}^M\gamma_{M,m}^2
\;\le\;
\Bigl(\max_{1\le k\le M}\gamma_{M,k}\Bigr)\sum_{m=1}^M\gamma_{M,m}
\;=\;
\frac{N}{N-M+1}\,M.    
\end{equation*}

Combining this with the variance decomposition above then gives
\begin{align*}
  \operatorname{Var}\!\bigl(\widehat R^{\mathrm{PPAT}}_M(\lambda)\bigr)
  =\frac{1}{M^2}\sum_{m=1}^M\gamma_{M,m}^2\,\operatorname{Var}(A_{M,m})
  \;&\le\;
  \frac{1}{M^2}\cdot\frac{N}{N-M+1}\,M\cdot\frac{D(\lambda)}{\beta} \\ 
  &=\frac{1}{\beta M}\,\frac{N}{N-M+1}\,D(\lambda),    
\end{align*}
as claimed.

\vspace{0.25cm}
Finally, only $D(\lambda)$ depends on $\lambda$. Expanding $D(\lambda)$ gives
\begin{equation*}
D(\lambda)
=\frac1N\sum_{i=1}^N\ell_i^2
-2\lambda\,\frac1N\sum_{i=1}^N\ell_i c_{M,i}
+\lambda^2\,\frac1N\sum_{i=1}^N c_{M,i}^2,    
\end{equation*}
a quadratic in $\lambda$ with leading coefficient $\tfrac1N\sum_i c_{M,i}^2$. When
$\tfrac1N\sum_i c_{M,i}^2>0$ it is strictly convex, and its derivative vanishes at the unique point
\begin{equation*}
\lambda^{\dagger}
=\frac{N^{-1}\sum_{i=1}^N\ell_i c_{M,i}}{N^{-1}\sum_{i=1}^N c_{M,i}^2}.    
\end{equation*}

As the constant multiplying $D(\lambda)$ is independent of $\lambda$, $\lambda^{\dagger}$ minimises the upper bound.
\end{proof}


\subsubsection{Tightness of the PPAT Variance Bound}
\label{app:sec:tightness_var_ppat_bound}
\input{appendix/mathematical_results/tightness_var_bound_v3}


\subsection{Myopic Oracle Proposals}
\label{sec:app:myopic_oracle_proposal}
Here, we consider choosing proposals that minimise the variance of our estimator. As minimising the full--horizon variance is intractable, we instead target a tractable \emph{myopic} proxy. We state our result for the general triangular array $\xi_{M,i}$ considered above; Prop.~\ref{prop:ppat_myopic_oracle} in the main text then follows immediately by taking $\xi_{M,i}=\ell_i-\lambda(\widetilde{\ell}_i-\widetilde R)$.

Recall from \S\ref{sec:app:luredef}-\ref{sec:app:lureunbiased} that, for a general triangular array $\xi_{M,i}$, 
the estimator $\widehat R_M=\tfrac1M\sum_{m=1}^M V_{M,m}\,\xi_{M,I_{M,m}}$ admits the running--average representation of Lemma~\ref{lem:lure_alt_rep} and the variance decomposition
\begin{equation*}
\operatorname{Var}(\widehat R_M)
=\frac1{M^2}\sum_{m=1}^M \gamma_{M,m}^2\,\operatorname{Var}(A_{M,m})
\end{equation*}
of Prop.~\ref{prop:lure_unbiased_variance}. Since the sum
$\sum_{t=1}^{m-1}\xi_{M,I_{M,t}}$ in $A_{M,m}$ is $\mathcal F_{M,m-1}$--measurable and
$\gamma_{M,m}$ is a deterministic constant, the only part of the round--$m$ contribution that depends on the proposal $Q_{M,m}$ is the conditional variance of the importance--weighted increment,
\begin{equation*}
\operatorname{Var}\!\big(A_{M,m}\mid \mathcal F_{M,m-1}\big)
=\frac1{N_M^2}\,
\operatorname{Var}\!\left(\frac{\xi_{M,I_{M,m}}}{Q_{M,m}(I_{M,m})}\,\Big|\,\mathcal F_{M,m-1}\right).
\end{equation*}
The full--horizon variance--minimising proposal is intractable, because the choice of $Q_{M,m}$ also shapes the laws of all later pools $S_{M,m+1},\dots$. Following \S\ref{sub:ppat_active_proposal},
we therefore adopt a \emph{myopic} objective: at each round $m$, conditionally on the history $\mathcal F_{M,m-1}$, we choose the proposal minimising this conditional variance. The next proposition solves this problem in closed form for an arbitrary array.

\begin{prop}[Myopic oracle proposal]
\label{prop:myopic_oracle_general}
Fix $M\ge 1$ and a round $m\in\{1,\dots,M\}$, and condition on $\mathcal F_{M,m-1}$, so that the remaining pool $S_{M,m}$ and the values $(\xi_{M,i})_{i\in S_{M,m}}$ are fixed. Assume $\sum_{j\in S_{M,m}}|\xi_{M,j}|>0$. Then, among all proposals  supported on $S_{M,m}$, the conditional variance
\begin{equation*}
\mathcal V(Q_{M,m})
:=\operatorname{Var}\!\left(
\frac{\xi_{M,I_{M,m}}}{N_M\,Q_{M,m}(I_{M,m})}\,\Big|\,\mathcal F_{M,m-1}\right)
\end{equation*}
is minimised by
\begin{equation*}
Q^\star_{M,m}(i)=\frac{|\xi_{M,i}|}{\sum_{j\in S_{M,m}}|\xi_{M,j}|},
\end{equation*}
with minimal value
\begin{equation*}
\mathcal V\big(Q^\star_{M,m}\big)
=\frac1{N_M^2}\left[
\Big(\sum_{i\in S_{M,m}}|\xi_{M,i}|\Big)^2
-\Big(\sum_{i\in S_{M,m}}\xi_{M,i}\Big)^2\right]\ \ge\ 0 .
\end{equation*}
\end{prop}

\begin{proof}
Condition throughout on $\mathcal F_{M,m-1}$. To lighten notation, let $S=S_{M,m}$, $\xi_i=\xi_{M,i}$, $Q=Q_{M,m}$, $I=I_{M,m}$ and $N=N_M$. Set $W:=\xi_I/(N\,Q(I))$.

Firstly, note that the conditional mean of $W$ does not depend on $Q$. Indeed, 
\begin{equation*}
\mathbb E[W\mid\mathcal F_{M,m-1}] = \sum_{i\in S}Q(i)\,\frac{\xi_i}{N\,Q(i)} =\frac1N\sum_{i\in S}\xi_i.
\end{equation*}

Now, because $\mathcal V(Q)=\mathbb E[W^2\mid\mathcal F_{M,m-1}] -\big(\mathbb E[W\mid\mathcal F_{M,m-1}]\big)^2$ and the second term is constant in $Q$, minimising $\mathcal V(Q)$ is equivalent to minimising the conditional second moment
\begin{equation*}
\mathbb E[W^2\mid\mathcal F_{M,m-1}]
=\sum_{i\in S}Q(i)\,\frac{\xi_i^2}{N^2\,Q(i)^2}
=\frac1{N^2}\sum_{i\in S}\frac{\xi_i^2}{Q(i)} .
\end{equation*}

Applying the Cauchy--Schwarz inequality with $a_i=|\xi_i|/\sqrt{Q(i)}$ and $b_i=\sqrt{Q(i)}$ then gives
\begin{equation*}
\Big(\sum_{i\in S}|\xi_i|\Big)^2
=\Big(\sum_{i\in S}a_ib_i\Big)^2
\le\Big(\sum_{i\in S}a_i^2\Big)\Big(\sum_{i\in S}b_i^2\Big)
=\Big(\sum_{i\in S}\frac{\xi_i^2}{Q(i)}\Big)
\underbrace{\sum_{i\in S}Q(i)}_{=1} .
\end{equation*}
Hence $\sum_{i\in S}\xi_i^2/Q(i)\ge\big(\sum_{i\in S}|\xi_i|\big)^2$. Since $b_i>0$, equality holds iff $a_i=c\,b_i$ for some constant $c$, i.e.\ $Q(i)=|\xi_i|/c$ for all $i\in S$. Noting that $\sum_{i\in S}Q(i)=1$, $c$ is forced to be $c=\sum_{j\in S}|\xi_j|$, giving
\begin{equation*}
Q^\star(i)=\frac{|\xi_i|}{\sum_{j\in S}|\xi_j|},\qquad i\in S,
\end{equation*}
which is a valid distribution because $\sum_{j\in S}|\xi_j|>0$. 

Now, substituting $Q^\star$ into the second moment yields $\frac1{N^2}\big(\sum_{i\in S}|\xi_i|\big)^2$; subtracting the constant squared mean $\frac1{N^2}\big(\sum_{i\in S}\xi_i\big)^2$ gives the stated minimal value of $\mathcal V$, which is non-negative by the triangle inequality $\big|\sum_i\xi_i\big|\le\sum_i|\xi_i|$.
\end{proof}

\vspace{0.25cm}
Prop. \ref{prop:ppat_myopic_oracle} follows as an immediate corollary to Prop. \ref{prop:myopic_oracle_general}.
\begin{cor}[Proposition~\ref{prop:ppat_myopic_oracle}]
\label{cor:ppat-myopic}
Consider the setup of \S\ref{sec:app:luredef} and fix $\lambda \in \mathbb{R}$. Take the PPAT array
\begin{equation*}
\xi_{M,i} = \ell_i - \lambda\bigl(\widetilde\ell_i - \tilde R_M\bigr)    
\end{equation*}
of Remark~\ref{rmrk:array_special_cases}, where $\widetilde\ell_i = L(f(x_i), g(x_i))$ and
$\widetilde R_M = \tfrac{1}{N_M}\sum_{i=1}^{N_M}\widetilde\ell_i$. Fix a round $m \in \{1,\dots,M\}$ and condition on $\mathcal F_{M,m-1}$ and assume $\sum_{j \in S_{M,m}}\bigl|\ell_j - \lambda(\widetilde\ell_j - \tilde R_M)\bigr| > 0$. Then, among all proposals supported on $S_{M,m}$, the myopic proxy
\begin{equation*}
\mathcal V(Q_{M,m})
:= \Var\!\left(
\frac{\ell_{I_{M,m}} - \lambda\bigl(\widetilde\ell_{I_{M,m}} - \widetilde R_M\bigr)}
{N_M\, Q_{M,m}(I_{M,m})}
\;\middle|\; \mathcal F_{M,m-1}\right)    
\end{equation*}
is minimised by
\begin{equation*}
Q^\star_{M,m,\lambda}(i)
= \frac{\bigl|\ell_i - \lambda(\widetilde\ell_i - \widetilde R_M)\bigr|}
{\sum_{j \in S_{M,m}}\bigl|\ell_j - \lambda(\widetilde\ell_j - \widetilde R_M)\bigr|},
\end{equation*}

This is exactly the myopic oracle proposal of Proposition ~\ref{prop:ppat_myopic_oracle}.
\end{cor}

\begin{proof}
Fix $\lambda \in \mathbb{R}$. This is an application of
Prop.~\ref{prop:myopic_oracle_general} to the triangular array
$\xi_{M,i} = \ell_i - \lambda(\tilde\ell_i - \tilde R_M)$ of
Remark~\ref{rmrk:array_special_cases}. With this array,
$\xi_{M,I_{M,m}} = \ell_{I_{M,m}} - \lambda(\tilde\ell_{I_{M,m}} - \tilde R_M)$, so the conditional
variance $\mathcal V(Q_{M,m})$ of Prop.~\ref{prop:myopic_oracle_general} is precisely the
residualised myopic proxy in the statement, and the condition
$\sum_{j \in S_{M,m}} |\xi_{M,j}| > 0$ holds by assumption. Prop.~\ref{prop:myopic_oracle_general} therefore yields that, among all proposals supported on $S_{M,m}$, the minimiser is
\begin{equation*}
Q^\star_{M,m}(i)
= \frac{|\xi_{M,i}|}{\sum_{j \in S_{M,m}} |\xi_{M,j}|}
= \frac{\bigl|\ell_i - \lambda(\tilde\ell_i - \tilde R_M)\bigr|}
{\sum_{j \in S_{M,m}}\bigl|\ell_j - \lambda(\tilde\ell_j - \tilde R_M)\bigr|}
=: Q^\star_{M,m,\lambda}(i),    
\end{equation*}
which is the stated proposal. 
\end{proof}

\vspace{0.25cm}
Note that setting $\lambda=0$ in Corr. \ref{cor:ppat-myopic} recovers the LURE oracle $Q^\star_m(i)\propto|\ell_i|$ of
\S\ref{sub:active_testing_lure_estimator}.


\subsection{Asymptotics}
\label{sec:app:general_asymptotics}

We now establish the asymptotic guarantees underlying the confidence intervals used in the main text.  We first work with our general triangular-array version of the LURE estimator, proving consistency, 
asymptotic normality, and the validity of a plug-in variance estimate for constructing confidence  intervals. We then specialise these results to PPAT with a fixed value of $\lambda$, before extending  the argument to the plug-in choice $\widehat{\lambda}_M$. We conclude by discussing the main assumptions required for these asymptotic results.

\subsubsection{Consistency of the General LURE Estimator}
\label{sec:app:lureconsistency}
Here we establish that, under a uniform overlap condition on our proposals and a mild second--moment--like condition on the finite--pool array, the general LURE estimator is $L^2$--consistent for the test--pool risk; the proof proceeds by bounding the conditional variance of each running average using the overlap condition and aggregating across rounds.

\begin{thm}[\(L^2\)-consistency]
\label{thm:lure_l2_consistency}
Assume that \(N_M/M\to\alpha>1\). Assume also that there exists \(\beta>0\)
such that, for all sufficiently large \(M\),
\[
Q_{M,m}(i)\ge \frac{\beta}{N_M},
\qquad
m=1,\ldots,M,\quad i\in S_{M,m}.
\]
Then, for every \(m=1,\ldots,M\),
\[
\operatorname{Var}(A_{M,m})
\le
\frac1{\beta N_M}
\sum_{i=1}^{N_M}\xi_{M,i}^2
=
\frac1\beta
\left(
\frac1{N_M}\sum_{i=1}^{N_M}\xi_{M,i}^2
\right).
\]
If, in addition $\frac1{N_M}\sum_{i=1}^{N_M}\xi_{M,i}^2=O(1)$, then
\begin{equation*}
\mathbb E[(\widehat R_M-R_M)^2]\to 0.    
\end{equation*}
More generally, the same conclusion holds if $\frac1{N_M}\sum_{i=1}^{N_M}\xi_{M,i}^2=o(M)$.
\end{thm}

\begin{proof}
To simplify our notation, we write \(N=N_M\). From the proof of
Prop.~\ref{prop:lure_unbiased_variance}, we have

\[
\operatorname{Var}(A_{M,m})
=
\mathbb E[
\operatorname{Var}(A_{M,m}\mid \mathcal F_{M,m-1})
].
\]

Moreover,
\[
\operatorname{Var}(A_{M,m}\mid \mathcal F_{M,m-1})
\le
\mathbb E\!\left[
\left.
\left(
\frac{\xi_{M,I_{M,m}}}{NQ_{M,m}(I_{M,m})}
\right)^2
\right|
\mathcal F_{M,m-1}
\right].
\]

This gives:
\[
\begin{aligned}
\operatorname{Var}(A_{M,m}\mid \mathcal F_{M,m-1})
&\le
\sum_{i\in S_{M,m}}
Q_{M,m}(i)
\frac{\xi_{M,i}^2}{N^2Q_{M,m}(i)^2}
\\
&=
\frac1{N^2}
\sum_{i\in S_{M,m}}
\frac{\xi_{M,i}^2}{Q_{M,m}(i)}.
\end{aligned}
\]

Now, using the overlap condition  we get $\frac1{Q_{M,m}(i)}\le \frac{N}{\beta}$ and consequently:
\begin{equation*}
\operatorname{Var}(A_{M,m}\mid \mathcal F_{M,m-1})
\le
\frac1{N^2}\frac{N}{\beta}
\sum_{i\in S_{M,m}}\xi_{M,i}^2
\le
\frac1{\beta N}
\sum_{i=1}^N\xi_{M,i}^2.    
\end{equation*}
Taking expectations then gives
\begin{equation*}
\operatorname{Var}(A_{M,m})
\le \frac1{\beta N} \sum_{i=1}^N\xi_{M,i}^2.    
\end{equation*}

Moreover, since \(N_M/M\to\alpha>1\), the sequence $\max_{1\le m\le M}\gamma_{M,m}$ is bounded. Indeed,
\[
\gamma_{M,m}
=
\frac{N(N-M)}{(N-m)(N-m+1)}
\le
\frac{N}{N-M+1},
\]
and
\[
\frac{N}{N-M+1}\to\frac{\alpha}{\alpha-1}<\infty.
\]

Finally, using Prop.~\ref{prop:lure_unbiased_variance}, there exists a constant
\(C_\gamma<\infty\) such that, for all sufficiently large \(M\),
\[
\begin{aligned}
\mathbb E[(\widehat R_M-R_M)^2]
&=
\operatorname{Var}(\widehat R_M)
\\
&=
\frac1{M^2}
\sum_{m=1}^M
\gamma_{M,m}^2\operatorname{Var}(A_{M,m})
\\
&\le
\frac1{M^2}
\sum_{m=1}^M
C_\gamma^2
\frac1{\beta N}
\sum_{i=1}^N\xi_{M,i}^2
\\
&=
\frac{C_\gamma^2}{\beta M}
\left(
\frac1N\sum_{i=1}^N\xi_{M,i}^2
\right).
\end{aligned}
\]

If \(N^{-1}\sum_i\xi_{M,i}^2=O(1)\), the right--hand side is \(O(M^{-1})\),
and hence tends to zero. More generally, it tends to zero whenever
\(N^{-1}\sum_i\xi_{M,i}^2=o(M)\).
\end{proof}


\subsubsection{Asymptotic Normality and Confidence Intervals for the General LURE Estimator}
\label{sec:app:lureCLT}
Building on the martingale structure exposed in \S\ref{sec:app:lureunbiased}, we now establish a new central limit theorem for $\sqrt{M}(\widehat{R}_M - R_M)$ under the uniform overlap condition and a $(2+\delta)$-moment condition, and we show that the asymptotic variance can be consistently estimated from the queried labels alone, yielding the asymptotically valid confidence intervals reported in \S\ref{sec:asymp_props_approx_cis}. Our result is obtained by applying a martingale central limit theorem for triangular arrays \citep[Theorem 35.12]{Billingsley1995}.

\begin{thm}[Martingale CLT for LURE]
\label{thm:lure_clt}
Assume \(N_M/M\to\alpha>1\). Assume that our APP satisfies a uniform overlap condition:
\[
Q_{M,m}(i)\ge \frac{\beta}{N_M},
\qquad
m=1,\ldots,M,\quad i\in S_{M,m},
\]
for all sufficiently large \(M\). For \(m=1,\ldots,M\), define $s_{M,m}^2=\operatorname{Var}(A_{M,m}\mid \mathcal F_{M,m-1})$ and $\sigma_M^2
= \frac1M \sum_{m=1}^M \gamma_{M,m}^2s_{M,m}^2.$

Assume that $\sigma_M^2\to_p \sigma^2$  for some \(0<\sigma^2<\infty\). Assume also that, for some \(\delta>0\),
\[
\frac1{N_M}\sum_{i=1}^{N_M}|\xi_{M,i}|^{2+\delta}=O(1).
\]
Then
\[
\sqrt M(\widehat R_M-R_M)
\Rightarrow
N(0,\sigma^2).
\]
\end{thm}

\begin{proof}
Firstly, define the martingale--difference array
$\Delta_{M,m}=\frac{\gamma_{M,m}}{\sqrt M}(A_{M,m}-R_M)$. By Proposition~\ref{prop:lure_unbiased_variance}, we have
\[
\mathbb E[A_{M,m}-R_M\mid \mathcal F_{M,m-1}]=0,
\]
and thus
\[
\mathbb E[\Delta_{M,m}\mid \mathcal F_{M,m-1}]=0.
\]

Now, using Lemma~\ref{lem:lure_alt_rep} and Lemma~\ref{lem:gamma_sum}, we get
\[
\begin{aligned}
\sqrt M(\widehat R_M-R_M)
&=
\sqrt M
\left[
\frac1M\sum_{m=1}^M\gamma_{M,m}A_{M,m}
-
R_M
\right]
\\
&=
\frac1{\sqrt M}
\sum_{m=1}^M
\gamma_{M,m}(A_{M,m}-R_M)
\\
&=
\sum_{m=1}^M\Delta_{M,m}.
\end{aligned}
\]

Here, the predictable quadratic variation is given by
\[
\sum_{m=1}^M
\mathbb E[\Delta_{M,m}^2\mid \mathcal F_{M,m-1}]
=
\frac1M
\sum_{m=1}^M
\gamma_{M,m}^2s_{M,m}^2
=
\sigma_M^2.
\]
Now, by assumption, we have $\sigma_M^2\to_p\sigma^2$. It remains to verify the conditional Lindeberg condition.

\vspace{0.5cm}
Let $p=2+\delta$. As in the proof of Thm.~\ref{thm:lure_l2_consistency}, the sequence
\(\max_{1\le m\le M}\gamma_{M,m}\) is bounded. Moreover,
\[
A_{M,m}-R_M
=
\frac{\xi_{M,I_{M,m}}}{NQ_{M,m}(I_{M,m})}
-
\frac1N\sum_{i\in S_{M,m}}\xi_{M,i}.
\]

Using \(|a-b|^p\le 2^{p-1}(|a|^p+|b|^p)\), we have
\[
\begin{aligned}
\mathbb E[
|A_{M,m}-R_M|^p
\mid \mathcal F_{M,m-1}]
&\le
2^{p-1}
\mathbb E\!\left[
\left.
\left|
\frac{\xi_{M,I_{M,m}}}{NQ_{M,m}(I_{M,m})}
\right|^p
\right|
\mathcal F_{M,m-1}
\right]
\\
&\quad+
2^{p-1}
\left|
\frac1N\sum_{i\in S_{M,m}}\xi_{M,i}
\right|^p.
\end{aligned}
\]

For the first term, our overlap condition gives
\[
\begin{aligned}
\mathbb E\!\left[
\left.
\left|
\frac{\xi_{M,I_{M,m}}}{NQ_{M,m}(I_{M,m})}
\right|^p
\right|
\mathcal F_{M,m-1}
\right]
&=
\frac1{N^p}
\sum_{i\in S_{M,m}}
\frac{|\xi_{M,i}|^p}{Q_{M,m}(i)^{p-1}}
\\
&\le
\frac1{N^p}
\left(\frac{N}{\beta}\right)^{p-1}
\sum_{i\in S_{M,m}}|\xi_{M,i}|^p
\\
&\le
\beta^{1-p}
\left(
\frac1N\sum_{i=1}^N|\xi_{M,i}|^p
\right).
\end{aligned}
\]

For the second term, we have
\[
\left|
\frac1N\sum_{i\in S_{M,m}}\xi_{M,i}
\right|^p
\le
\left(
\frac1N\sum_{i=1}^N|\xi_{M,i}|
\right)^p
\le
\frac1N\sum_{i=1}^N|\xi_{M,i}|^p.
\]

We therefore have that there exists a constant \(C_{p,\beta}<\infty\) such that
\[
\mathbb E[
|A_{M,m}-R_M|^p
\mid \mathcal F_{M,m-1}]
\le
C_{p,\beta}
\left(
\frac1N\sum_{i=1}^N|\xi_{M,i}|^p
\right).
\]

Consequently, for another finite constant \(C<\infty\), we have
\[
\begin{aligned}
\sum_{m=1}^M
\mathbb E[
|\Delta_{M,m}|^p
\mid \mathcal F_{M,m-1}]
&=
\sum_{m=1}^M
\frac{\gamma_{M,m}^p}{M^{p/2}}
\mathbb E[
|A_{M,m}-R_M|^p
\mid \mathcal F_{M,m-1}]
\\
&\le
C M^{1-p/2}
\left(
\frac1N\sum_{i=1}^N|\xi_{M,i}|^p
\right)
\\
&=
C M^{-\delta/2}
\left(
\frac1N\sum_{i=1}^N|\xi_{M,i}|^{2+\delta}
\right)
\\
&\to 0.
\end{aligned}
\]
Hence, for every \(\varepsilon>0\),
\[
\sum_{m=1}^M
\mathbb E[
\Delta_{M,m}^2
\mathbf 1\{|\Delta_{M,m}|>\varepsilon\}
\mid \mathcal F_{M,m-1}]
\le
\varepsilon^{-\delta}
\sum_{m=1}^M
\mathbb E[
|\Delta_{M,m}|^{2+\delta}
\mid \mathcal F_{M,m-1}]
\to 0.
\]
This is the conditional Lindeberg condition.

\vspace{0.25cm}
We can now apply the martingale central limit theorem for triangular arrays: if a martingale--difference array has predictable quadratic variation converging
in probability to \(\sigma^2\) and satisfies the conditional Lindeberg
condition, then its sum converges in distribution to \(N(0,\sigma^2)\). Concretely, we have
\[
\sum_{m=1}^M\Delta_{M,m}
\Rightarrow
N(0,\sigma^2).
\]
Using
$\sum_{m=1}^M\Delta_{M,m}
=
\sqrt M(\widehat R_M-R_M)$,
then gives our desired result.
\end{proof}

\vspace{0.25cm}
\begin{rmrk}[Explicit form of the predictable variance]
\label{rmrk:explicit_form_predictable_var}
The conditional variance \(s_{M,m}^2\) can be written explicitly as
\[
s_{M,m}^2
=
\operatorname{Var}
\left(
\frac{\xi_{M,I_{M,m}}}{NQ_{M,m}(I_{M,m})}
\;\middle|\;
\mathcal F_{M,m-1}
\right).
\]
Equivalently,
\[
s_{M,m}^2
=
\frac1{N^2}
\left[
\sum_{i\in S_{M,m}}
\frac{\xi_{M,i}^2}{Q_{M,m}(i)}
-
\left(
\sum_{i\in S_{M,m}}\xi_{M,i}
\right)^2
\right].
\]
Thus, the asymptotic variance in Thm.~\ref{thm:lure_clt} is the limit of
\[
\sigma_M^2
=
\frac1M
\sum_{m=1}^M
\gamma_{M,m}^2
\frac1{N^2}
\left[
\sum_{i\in S_{M,m}}
\frac{\xi_{M,i}^2}{Q_{M,m}(i)}
-
\left(
\sum_{i\in S_{M,m}}\xi_{M,i}
\right)^2
\right].
\]
\end{rmrk}

\vspace{0.25cm}
\begin{thm}[Consistent estimation of the CLT variance and studentized CLT]
\label{thm:lure_variance_estimator}
Assume the conditions of Thm.~\ref{thm:lure_clt}. In addition, assume $\frac1{N_M}\sum_{i=1}^{N_M}|\xi_{M,i}|^4=O(1)$.
Moreover, define
\[
\widehat\sigma_M^2
=
\frac1M
\sum_{m=1}^M
\gamma_{M,m}^2
(A_{M,m}-\widehat R_M)^2.
\]
Then $\widehat\sigma_M^2\to_p\sigma^2$. Consequently, we have
\[
\frac{\sqrt M(\widehat R_M-R_M)}{\widehat\sigma_M}
\Rightarrow
N(0,1).
\]
An approximate \((1-\eta)\)--level confidence interval for \(R_M\) is given by:
\begin{equation}
\label{eqn:asymptotic_valid_cis}
\widehat R_M
\pm
\Phi^{-1}(1-\eta/2)
\frac{\widehat\sigma_M}{\sqrt M},
\end{equation}
where $\Phi$ is the CDF of the standard normal distribution.
\end{thm}

\begin{proof}
Define the infeasible realised quadratic variation
\[
\widetilde\sigma_M^2
:=
\frac1M
\sum_{m=1}^M
\gamma_{M,m}^2
(A_{M,m}-R_M)^2.
\]
We first show that $\widetilde\sigma_M^2-\sigma_M^2\to_p0$. Let
\[
Y_{M,m}
=
\frac{\gamma_{M,m}^2}{M}
\left[
(A_{M,m}-R_M)^2-s_{M,m}^2
\right].
\]
Then $\mathbb E[Y_{M,m}\mid \mathcal F_{M,m-1}]=0$.
Thus \(\sum_{m=1}^M Y_{M,m}\) is a martingale. Since
\(\max_m\gamma_{M,m}\) is bounded and, by the same overlap argument used in Thm.~\ref{thm:lure_clt},
\[
\mathbb E[
|A_{M,m}-R_M|^4
\mid \mathcal F_{M,m-1}]
\le
C
\left(
\frac1N\sum_{i=1}^N|\xi_{M,i}|^4
\right),
\]
we have
\[
\mathbb E\left[
(\widetilde\sigma_M^2-\sigma_M^2)^2
\right]
\le
\frac{C}{M}
\left(
\frac1N\sum_{i=1}^N|\xi_{M,i}|^4
\right)
\to 0.
\]

Consequently, $\widetilde\sigma_M^2-\sigma_M^2\to_p0$.
Now, since \(\sigma_M^2\to_p\sigma^2\), it therefore follows that $\widetilde\sigma_M^2\to_p\sigma^2$. 

\vspace{0.25cm}
\looseness=-1
It now remains to replace \(R_M\) by \(\widehat R_M\). Let
$D_{M,m}=A_{M,m}-R_M$. Then $A_{M,m}-\widehat R_M = D_{M,m}-(\widehat R_M-R_M)$ and, consequently,
\[
\widehat\sigma_M^2-\widetilde\sigma_M^2
=
-2(\widehat R_M-R_M)
\frac1M\sum_{m=1}^M\gamma_{M,m}^2D_{M,m}
+
(\widehat R_M-R_M)^2
\frac1M\sum_{m=1}^M\gamma_{M,m}^2.
\]
By Thm.~\ref{thm:lure_l2_consistency}\footnote{This applies here as the fourth
moment condition implies the second moment condition.}, we have $\widehat R_M-R_M=o_p(1)$.
Note that we also have:
\begin{align}
\frac1M\sum_{m=1}^M\gamma_{M,m}^2&=O(1), \label{eqn:thmg8_o_1}\\
\frac1M\sum_{m=1}^M\gamma_{M,m}^2D_{M,m}
&=O_p(1),\label{eqn:thmg8_o_p1}
\end{align}

where \eqref{eqn:thmg8_o_1} follows from the boundedness of $\gamma_{M,m}$, and  \eqref{eqn:thmg8_o_p1} follows from \(\widetilde\sigma_M^2\to_p\sigma^2\) and an application of the Cauchy-Schwarz inequality:
\[
\left|
\frac1M\sum_{m=1}^M\gamma_{M,m}^2D_{M,m}
\right|
\le
\left(
\frac1M\sum_{m=1}^M\gamma_{M,m}^2
\right)^{1/2}
\left(
\frac1M\sum_{m=1}^M\gamma_{M,m}^2D_{M,m}^2
\right)^{1/2}.
\]
Consequently, we have $\widehat\sigma_M^2-\widetilde\sigma_M^2=o_p(1)$. Combining this with \(\widetilde\sigma_M^2\to_p\sigma^2\) then yields:
\[
\widehat\sigma_M^2\to_p\sigma^2.
\]

The studentised CLT follows from Thm.~\ref{thm:lure_clt} and Slutsky's
theorem.
\end{proof}



\subsubsection{Fixed--$\lambda$ PPAT}
\label{sec:app:fixedlambdaPPAT}
Here, we prove Thm. \ref{thm:lure_ppat_unified} for the fixed-$\lambda$ case; the plug-in case is treated in \S\ref{sec:app:pluginPPAT}. This follows as an immediate corollary to the theorems in \S\ref{sec:app:lureconsistency}-\ref{sec:app:lureCLT}. The key observation is that, for any fixed $\lambda\in\mathbb R$, PPAT is exactly the
general LURE estimator of \S\ref{sec:app:luredef} applied to the residualised triangular array
\[
  \zeta_{M,i}(\lambda)=\ell_i-\lambda\,c_{M,i},
  \qquad c_{M,i}=\widetilde\ell_i-\widetilde R_M,
\]
whose finite--pool mean is the test--pool risk $R_M$; the consistency, central--limit,
and variance--estimation results then transfer directly once the moment conditions on
this array are translated into separate conditions on the true and proxy losses.

Recall from Remark \ref{rmrk:explicit_form_predictable_var} that, for the array $\zeta_{M,\cdot}(\lambda)$, the predictable
quadratic variation $\sigma_M^2$ of Thm. \ref{thm:lure_clt} takes the explicit form
\begin{equation}
  \label{eq:sigma2M-fixed}
  \sigma_M^2(\lambda)
  =\frac{1}{M}\sum_{m=1}^{M}\gamma_{M,m}^2\,
    \frac{1}{N_M^2}
    \left[\sum_{i\in S_{M,m}}\frac{\zeta_{M,i}(\lambda)^2}{Q_{M,m}(i)}
    -\Bigl(\sum_{i\in S_{M,m}}\zeta_{M,i}(\lambda)\Bigr)^{\!2}\right].
\end{equation}

\vspace{0.25cm}
\begin{cor}[Fixed--$\lambda$ PPAT, Theorem \ref{thm:lure_ppat_unified}]
\label{cor:fixed-lambda-ppat}
Consider the setup of \S\ref{sec:app:luredef} and fix $\lambda\in\mathbb R$. Assume that
the pool ratio satisfies $N_M/M\to\alpha>1$ and that our APP satisfies a uniform overlap condition: there exists $\beta>0$ such that,
for all sufficiently large $M$,
\begin{equation*}
Q_{M,m}(i)\ge\frac{\beta}{N_M},\qquad m=1,\dots,M,\ \ i\in S_{M,m}.  
\end{equation*}
Then the following hold.
\begin{enumerate}
\item \textup{(Consistency.)} Assume
$\tfrac{1}{N_M}\sum_{i=1}^{N_M}\ell_i^2=O(1)$ and
$\tfrac{1}{N_M}\sum_{i=1}^{N_M}\widetilde\ell_i^{\,2}=O(1)$, then
\begin{equation*}
\widehat R^{\mathrm{PPAT}}_M(\lambda)-R_M \ \xrightarrow{\ p\ }\ 0.  
\end{equation*}
\item \textup{(Asymptotic normality and confidence intervals.)} Assume, additionally, 
$\tfrac{1}{N_M}\sum_{i=1}^{N_M}\ell_i^4=O(1)$ and
$\tfrac{1}{N_M}\sum_{i=1}^{N_M}\tilde\ell_i^{\,4}=O(1)$, and the predictable quadratic
variation \eqref{eq:sigma2M-fixed} satisfies
$\sigma_M^2(\lambda)\xrightarrow{p}\sigma^2_{\mathrm{PPAT}}(\lambda)$ for some $\sigma^2_{\mathrm{PPAT}}(\lambda)\in(0,\infty)$, then
\begin{equation*}
\sqrt{M}\,\bigl\{\widehat R^{\mathrm{PPAT}}_M(\lambda)-R_M\bigr\}
\ \Rightarrow\ \mathcal N\!\bigl(0,\sigma^2_{\mathrm{PPAT}}(\lambda)\bigr).    
\end{equation*}

Moreover, writing $A_{M,m}(\lambda)$ for the running average of Lemma \ref{lem:lure_alt_rep} applied to the
array $\zeta_{M,\cdot}(\lambda)$, i.e.
\begin{equation*}
  A_{M,m}(\lambda)
  =\frac{1}{N_M}\left(\frac{\zeta_{M,I_{M,m}}(\lambda)}{Q_{M,m}(I_{M,m})}
  +\sum_{t=1}^{m-1}\zeta_{M,I_{M,t}}(\lambda)\right),    
\end{equation*}
the plug--in estimator
\begin{equation*}
  \widehat\sigma_M^2(\lambda)
  =\frac{1}{M}\sum_{m=1}^{M}\gamma_{M,m}^2
   \bigl(A_{M,m}(\lambda)-\widehat R^{\mathrm{PPAT}}_M(\lambda)\bigr)^2    
\end{equation*}
satisfies $\widehat\sigma_M^2(\lambda)\xrightarrow{p}\sigma^2_{\mathrm{PPAT}}(\lambda)$, and hence
\begin{equation*}
\frac{\sqrt{M}\,\bigl\{\widehat R^{\mathrm{PPAT}}_M(\lambda)-R_M\bigr\}}
   {\widehat\sigma_M(\lambda)}
\ \Rightarrow\ \mathcal N(0,1).    
\end{equation*}

In particular,
\begin{equation*}
\widehat R^{\mathrm{PPAT}}_M(\lambda)\ \pm\
\Phi^{-1}\!\bigl(1-\tfrac{\eta}{2}\bigr)\,\frac{\widehat\sigma_M(\lambda)}{\sqrt{M}}    
\end{equation*}
is an asymptotic $(1-\eta)$--level confidence interval for $R_M$, where $\Phi$ is the CDF for the standard normal distribution.
\end{enumerate}
\end{cor}

\begin{proof}
By Remark \ref{rmrk:array_special_cases}, for any fixed $\lambda\in\mathbb R$ the PPAT estimator
$\widehat R^{\mathrm{PPAT}}_M(\lambda)$ is exactly the generic LURE estimator of
\S\ref{sec:app:luredef} applied to the array $\xi_{M,i}=\zeta_{M,i}(\lambda)=\ell_i-\lambda\,c_{M,i}$.
Since $c_{M,i}=\tilde\ell_i-\tilde R_M$ has zero pool average,
$\tfrac{1}{N_M}\sum_{i=1}^{N_M} c_{M,i}=0$, the finite--pool mean of this array is
\[
  \frac{1}{N_M}\sum_{i=1}^{N_M}\zeta_{M,i}(\lambda)
  =\frac{1}{N_M}\sum_{i=1}^{N_M}\ell_i
   -\lambda\,\frac{1}{N_M}\sum_{i=1}^{N_M} c_{M,i}
  =R_M .
\]
Hence the quantity ``$R_M$'' appearing in Thms. \ref{thm:lure_l2_consistency}, \ref{thm:lure_clt} and \ref{thm:lure_variance_estimator} for this array is precisely the test--pool risk, and it only remains to translate the moment conditions on $\xi_{M,i}$ into conditions on $\ell_i$ and $\tilde\ell_i$.

\medskip
It remains to translate the moment assumptions on $\xi_{M,i}$ into the corresponding assumptions on $\ell_i$ and $\tilde{\ell}_i$. Fix $p\ge 1$. Since $c_{M,i}=\tilde\ell_i-\tilde R_M$ and
$\tilde R_M$ is the pool average of the $\tilde\ell_i$, Jensen's inequality gives
$|\tilde R_M|^p\le \tfrac{1}{N_M}\sum_{j}|\tilde\ell_j|^p$, so that, using
$|a-b|^p\le 2^{p-1}(|a|^p+|b|^p)$,
\[
  \frac{1}{N_M}\sum_{i=1}^{N_M}|c_{M,i}|^p
  \le 2^{p-1}\!\left(\frac{1}{N_M}\sum_{i=1}^{N_M}|\tilde\ell_i|^p+|\tilde R_M|^p\right)
  \le C_p\,\frac{1}{N_M}\sum_{i=1}^{N_M}|\tilde\ell_i|^p ,
\]
for a constant $C_p<\infty$. As $\lambda$ is fixed, applying the same inequality to
$\zeta_{M,i}(\lambda)=\ell_i-\lambda c_{M,i}$ yields
\begin{equation}
  \label{eq:moment-translation}
  \frac{1}{N_M}\sum_{i=1}^{N_M}|\zeta_{M,i}(\lambda)|^p
  \le C_{p,\lambda}\left(\frac{1}{N_M}\sum_{i=1}^{N_M}|\ell_i|^p
  +\frac{1}{N_M}\sum_{i=1}^{N_M}|\tilde\ell_i|^p\right),
\end{equation}
for a constant $C_{p,\lambda}<\infty$.

\medskip
We first establish consistency. Taking $p=2$ in \eqref{eq:moment-translation}, the assumptions
$\tfrac{1}{N_M}\sum\ell_i^2=O(1)$ and $\tfrac{1}{N_M}\sum\tilde\ell_i^{\,2}=O(1)$ imply
$\tfrac{1}{N_M}\sum\xi_{M,i}^2=O(1)$. Thm. \ref{thm:lure_l2_consistency} then gives
$\mathbb E[(\widehat R^{\mathrm{PPAT}}_M(\lambda)-R_M)^2]\to 0$, and $L^2$ convergence implies $\widehat R^{\mathrm{PPAT}}_M(\lambda)-R_M\xrightarrow{p}0$.

\medskip
We now turn to the central limit theorem. Taking $p=4$ in \eqref{eq:moment-translation}, the
fourth--moment assumptions give $\tfrac{1}{N_M}\sum|\xi_{M,i}|^4=O(1)$, and in particular
the $(2+\delta)$--moment condition $\tfrac{1}{N_M}\sum|\xi_{M,i}|^{2+\delta}=O(1)$ with
$\delta=2$. The assumed stabilisation
$\sigma_M^2(\lambda)\xrightarrow{p}\sigma^2_{\mathrm{PPAT}}(\lambda)\in(0,\infty)$ is exactly
the predictable--quadratic--variation condition of Thm. \ref{thm:lure_clt} for the array
$\zeta_{M,\cdot}(\lambda)$. Thm. \ref{thm:lure_clt}
therefore yields
\begin{equation*}
  \sqrt{M}\bigl\{\widehat R^{\mathrm{PPAT}}_M(\lambda)-R_M\bigr\}
  \ \Rightarrow\ \mathcal N\!\bigl(0,\sigma^2_{\mathrm{PPAT}}(\lambda)\bigr).    
\end{equation*}

\medskip
Finally, we address estimation of the asymptotic variance and the asymptotic confidence interval. The bound
$\tfrac{1}{N_M}\sum|\xi_{M,i}|^4=O(1)$ established above is precisely the additional condition required by Thm. \ref{thm:lure_variance_estimator}. Applying that theorem to the array
$\zeta_{M,\cdot}(\lambda)$ shows that the plug--in estimator is consistent, $\widehat\sigma_M^2(\lambda)\xrightarrow{p}\sigma^2_{\mathrm{PPAT}}(\lambda)$, and, by Slutsky's theorem, gives both the studentised limit and the asymptotic $(1-\eta)$--level confidence interval for $R_M$.
\end{proof}

\vspace{0.25cm}
\begin{rmrk}
Setting $\lambda=0$ gives $\zeta_{M,i}(0)=\ell_i$ and
$\widehat R^{\mathrm{PPAT}}_M(0)=\widehat R^{\mathrm{LURE}}_M$, so
Corr.~\ref{cor:fixed-lambda-ppat} specialises to the corresponding consistency,
asymptotic normality, and confidence--interval statements for LURE.
\end{rmrk}


\subsubsection{Plug--in PPAT}
\label{sec:app:pluginPPAT}
Finally, we extend the results in \S \ref{sec:app:fixedlambdaPPAT} for the plug--in PPAT estimator, where the tuning 
parameter $\lambda$ is replaced by the data-driven estimate $\widehat{\lambda}_M$ in \S\ref{sub:lambda_estimation}. The proof compares the plug--in estimator with the PPAT
estimator that uses the finite-pool oracle coefficient $\lambda_M^\dagger$. We show
that replacing $\lambda_M^\dagger$ by $\widehat{\lambda}_M$ changes the estimator
only by $o_p(M^{-1/2})$. Thus, at the $\sqrt{M}$ rate, the plug--in estimator has the same asymptotic behaviour as this oracle estimator: it is consistent, satisfies the same central limit theorem, and yields asymptotically valid confidence intervals when the variance is estimated from the data.

To simplify our notation, for any given triangular array $a_{M,1:N_M}$ we will write
\begin{equation*}
\bar a_M = \frac1{N_M}\sum_{i=1}^{N_M}a_{M,i}, \qquad \widehat\mu_M(a) = \frac1M\sum_{m=1}^M V_{M,m}a_{M,I_{M,m}}.    
\end{equation*}
\looseness=-1
In this way, $\widehat\mu_M(a)$ is the LURE estimator applied to the array $a$. As before, define 
$$c_{M,i} = \widetilde\ell_i-\widetilde R_M, \qquad \widetilde R_M = \frac1{N_M}\sum_{i=1}^{N_M}\widetilde\ell_i,$$
and note that $\bar c_M=0$. Moreover, let 
$$D_M = \frac1{N_M}\sum_{i=1}^{N_M}c_{M,i}^2, \qquad G_M = \frac1{N_M}\sum_{i=1}^{N_M}\ell_i c_{M,i},$$ 
such that the oracle finite-pool coefficient in \S\ref{sub:lambda_estimation} can be written as $\lambda_M^\dagger = \frac{G_M}{D_M}$, whenever \(D_M>0\).

As discussed in \S\ref{sub:lambda_estimation}, $c_{M,i}$ is known for all test points so $D_M$ can be computed exactly. On the other hand, we estimate $G_M$ via LURE,
\begin{equation*}
\widehat G_M = \widehat\mu_M(\ell c)
= \frac1M\sum_{m=1}^M V_{M,m}\ell_{I_{M,m}}c_{M,I_{M,m}},    
\end{equation*}
and set $\widehat\lambda_M = \frac{\widehat G_M}{D_M}$. Using the notation defined earlier, the plug-in PPAT estimator is
\begin{equation*}
\widehat R_M^{\mathrm{PPAT}}(\widehat\lambda_M) = \widehat\mu_M(\ell) - \widehat\lambda_M\widehat\mu_M(c).    
\end{equation*}

We now formally state and prove the consistency and asymptotic normality of the plug-in PPAT estimator. We prove the confidence interval result as a corollary in Corr. \ref{cor:plugin_ppat_studentized}.

\begin{thm}[Consistency and asymptotic normality of plug-in PPAT, Theorem~\ref{thm:lure_ppat_unified}]
\label{thm:plugin_ppat_clt}
Assume that the pool ratio satisfies $N_M/M\to\alpha>1$ and that the APP satisfies a uniform overlap condition: there exists
$\beta>0$ such that, for all sufficiently large $M$,
\begin{equation*}
Q_{M,m}(i)\ge \frac{\beta}{N_M},
\qquad
m=1,\ldots,M,\quad i\in S_{M,m}.
\end{equation*}
Assume also that there exists $d>0$ such that, for all sufficiently large $M$,
\begin{equation*}
D_M=\frac1{N_M}\sum_{i=1}^{N_M}c_{M,i}^2\ge d.
\end{equation*}
Then the following hold.
\begin{enumerate}
\item \emph{(Consistency.)} Assume
\begin{equation*}
\frac1{N_M}\sum_{i=1}^{N_M}\ell_i^2=O(1),
\qquad
\frac1{N_M}\sum_{i=1}^{N_M}c_{M,i}^2=O(1),
\qquad
\frac1{N_M}\sum_{i=1}^{N_M}\ell_i^2c_{M,i}^2=O(1).
\end{equation*}
Then
\begin{equation*}
\widehat\lambda_M-\lambda_M^\dagger=O_p(M^{-1/2})
\qquad\text{and}\qquad
\widehat R_M^{\mathrm{PPAT}}(\widehat\lambda_M)-R_M=O_p(M^{-1/2}).
\end{equation*}
In particular, both quantities are $o_p(1)$, so
$\widehat R_M^{\mathrm{PPAT}}(\widehat\lambda_M)-R_M\to_p 0$ and $\widehat\lambda_M-\lambda_M^\dagger\to_p 0$.

\item \emph{(Asymptotic normality.)} Assume, in addition, that for some $\delta>0$,
\begin{equation*}
\frac1{N_M}\sum_{i=1}^{N_M}|\ell_i|^{2+\delta}=O(1),
\qquad
\frac1{N_M}\sum_{i=1}^{N_M}|c_{M,i}|^{2+\delta}=O(1),
\end{equation*}
and define the oracle residualised array $z_{M,i}^{\dagger}=\ell_i-\lambda_M^\dagger c_{M,i}$, with running averages and predictable variances
\begin{align*}
A_{M,m}(z^\dagger) &=\frac1{N_M}\left\{
\frac{z_{M,I_{M,m}}^\dagger}{Q_{M,m}(I_{M,m})} +\sum_{t=1}^{m-1}z_{M,I_{M,t}}^\dagger\right\}, \\[0.5em]
s_{M,m}^2(z^\dagger) &=\operatorname{Var}\!\left(A_{M,m}(z^\dagger)\mid\mathcal F_{M,m-1}\right).
\end{align*}
If the predictable quadratic variation stabilises,
\begin{equation*}
\sigma_{M,\dagger}^2
:=\frac1M\sum_{m=1}^M\gamma_{M,m}^2 s_{M,m}^2(z^\dagger)
\to_p\sigma_\dagger^2
\end{equation*}
for some $0<\sigma_\dagger^2<\infty$, then
\begin{equation*}
\sqrt M\left\{\widehat R_M^{\mathrm{PPAT}}(\widehat\lambda_M)-R_M\right\}
\Rightarrow N(0,\sigma_\dagger^2).
\end{equation*}
\end{enumerate}
\end{thm}

\begin{proof}
Throughout we apply the $L^2$-consistency result (Thm. \ref{thm:lure_l2_consistency}) and the martingale
central limit theorem for LURE (Thm. \ref{thm:lure_clt}) to fixed finite-pool arrays.

We first bound $\widehat\lambda_M-\lambda_M^\dagger$. Since
\begin{equation*}
\widehat\lambda_M-\lambda_M^\dagger=\frac{\widehat G_M-G_M}{D_M}
\end{equation*}
and $D_M\ge d>0$, it suffices to control $\widehat G_M-G_M$. Here $\widehat G_M=\widehat\mu_M(\ell c)$
is the LURE estimator of the finite-pool mean $G_M$ of the array $g_{M,i}:=\ell_i c_{M,i}$, which satisfies
\begin{equation*}
\frac1{N_M}\sum_{i=1}^{N_M}g_{M,i}^2 =\frac1{N_M}\sum_{i=1}^{N_M}\ell_i^2c_{M,i}^2=O(1).
\end{equation*}
Thm. \ref{thm:lure_l2_consistency} applied to $g_{M,i}$ therefore gives $\widehat G_M-G_M=O_p(M^{-1/2})$, and hence
\begin{equation*}
\widehat\lambda_M-\lambda_M^\dagger=O_p(M^{-1/2}).
\end{equation*}
Similarly, the centred proxy array $c_{M,i}$ has zero finite-pool mean, $\bar c_M=0$,
and satisfies $\frac1{N_M}\sum_i c_{M,i}^2=O(1)$, so Thm. \ref{thm:lure_l2_consistency} applied to $c_{M,i}$ gives
\begin{equation*}
\widehat\mu_M(c)=\widehat\mu_M(c)-\bar c_M=O_p(M^{-1/2}).
\end{equation*}

We next describe the properties of the oracle residualised array
$z_{M,i}^\dagger=\ell_i-\lambda_M^\dagger c_{M,i}$ needed for our results. First, $\lambda_M^\dagger$ is
bounded: by Cauchy--Schwarz,
\begin{equation*}
|G_M|
=\left|\frac1{N_M}\sum_{i=1}^{N_M}\ell_i c_{M,i}\right|
\le\left(\frac1{N_M}\sum_{i=1}^{N_M}\ell_i^2\right)^{1/2}
\left(\frac1{N_M}\sum_{i=1}^{N_M}c_{M,i}^2\right)^{1/2},
\end{equation*}
so that
\begin{equation*}
|\lambda_M^\dagger| =\frac{|G_M|}{D_M} \le\left(\frac{N_M^{-1}\sum_i\ell_i^2}{D_M}\right)^{1/2} =O(1),
\end{equation*}
using $D_M\ge d>0$ and $N_M^{-1}\sum_i\ell_i^2=O(1)$. Second, since $\bar c_M=0$, the array $z^\dagger$ has finite-pool mean
\begin{equation*}
\frac1{N_M}\sum_{i=1}^{N_M}z_{M,i}^\dagger =\frac1{N_M}\sum_{i=1}^{N_M}\ell_i -\lambda_M^\dagger\frac1{N_M}\sum_{i=1}^{N_M}c_{M,i} =R_M,
\end{equation*}
and, using $(a-b)^2\le 2a^2+2b^2$ together with the boundedness of $\lambda_M^\dagger$,
\begin{equation*}
\frac1{N_M}\sum_{i=1}^{N_M}(z_{M,i}^\dagger)^2
\le 2\frac1{N_M}\sum_{i=1}^{N_M}\ell_i^2
+2(\lambda_M^\dagger)^2\frac1{N_M}\sum_{i=1}^{N_M}c_{M,i}^2 =O(1).
\end{equation*}
Thm. \ref{thm:lure_l2_consistency} applied to the fixed array $z^\dagger$ therefore gives
\begin{equation*}
\widehat\mu_M(z^\dagger)-R_M=O_p(M^{-1/2}).
\end{equation*}

We now prove consistency. Writing $\widehat R_M^{\mathrm{PPAT}}(\widehat\lambda_M)=\widehat\mu_M(\ell)-\widehat\lambda_M\widehat\mu_M(c)$
and using the linearity of $\widehat\mu_M$,
\begin{equation*}
\begin{aligned}
\widehat R_M^{\mathrm{PPAT}}(\widehat\lambda_M)-R_M
&=\widehat\mu_M(\ell)-\widehat\lambda_M\widehat\mu_M(c)-R_M\\
&=\widehat\mu_M(\ell)-\lambda_M^\dagger\widehat\mu_M(c)-R_M
-(\widehat\lambda_M-\lambda_M^\dagger)\widehat\mu_M(c)\\
&=\bigl(\widehat\mu_M(z^\dagger)-R_M\bigr)
-(\widehat\lambda_M-\lambda_M^\dagger)\widehat\mu_M(c).
\end{aligned}
\end{equation*}
The first term is $O_p(M^{-1/2})$ by the previous step, while the second term is
\begin{equation*}
(\widehat\lambda_M-\lambda_M^\dagger)\widehat\mu_M(c)
=O_p(M^{-1/2})\,O_p(M^{-1/2})=O_p(M^{-1}).
\end{equation*}
Therefore
\begin{equation*}
\widehat R_M^{\mathrm{PPAT}}(\widehat\lambda_M)-R_M=O_p(M^{-1/2}),
\end{equation*}
which proves part~1.

It remains to prove asymptotic normality. Since $|\lambda_M^\dagger|=O(1)$, the additional
$(2+\delta)$-moment assumptions control the corresponding moment of $z^\dagger$: by
$|a-b|^{2+\delta}\le 2^{1+\delta}(|a|^{2+\delta}+|b|^{2+\delta})$,
\begin{equation*}
\frac1{N_M}\sum_{i=1}^{N_M}|z_{M,i}^\dagger|^{2+\delta}
\le C\left[
\frac1{N_M}\sum_{i=1}^{N_M}|\ell_i|^{2+\delta}
+|\lambda_M^\dagger|^{2+\delta}\frac1{N_M}\sum_{i=1}^{N_M}|c_{M,i}|^{2+\delta}
\right]=O(1)
\end{equation*}
for a finite constant $C$. The fixed array $z^\dagger$ thus satisfies the hypotheses of the
LURE central limit theorem (Thm. \ref{thm:lure_clt}), and, using the assumed variance stabilisation $\frac1M\sum_{m=1}^M\gamma_{M,m}^2 s_{M,m}^2(z^\dagger)\to_p\sigma_\dagger^2$, we obtain
\begin{equation*}
\sqrt M\{\widehat\mu_M(z^\dagger)-R_M\}\Rightarrow N(0,\sigma_\dagger^2).
\end{equation*}
Finally, we show that replacing $\lambda_M^\dagger$ by $\widehat\lambda_M$ is first-order
negligible. From the decomposition above,
\begin{equation*}
\sqrt M\left\{\widehat R_M^{\mathrm{PPAT}}(\widehat\lambda_M)-\widehat\mu_M(z^\dagger)\right\}
=-(\widehat\lambda_M-\lambda_M^\dagger)\sqrt M\,\widehat\mu_M(c).
\end{equation*}
We have shown $\widehat\lambda_M-\lambda_M^\dagger=O_p(M^{-1/2})=o_p(1)$, while
$\sqrt M\,\widehat\mu_M(c)=O_p(1)$; hence the right-hand side is $o_p(1)$, that is,
\begin{equation*}
\sqrt M\left\{\widehat R_M^{\mathrm{PPAT}}(\widehat\lambda_M)-\widehat\mu_M(z^\dagger)\right\}=o_p(1).
\end{equation*}
Combining this asymptotic equivalence with the central limit theorem for
$\widehat\mu_M(z^\dagger)$ and applying Slutsky's theorem yields
\begin{equation*}
\sqrt M\left\{\widehat R_M^{\mathrm{PPAT}}(\widehat\lambda_M)-R_M\right\}
\Rightarrow N(0,\sigma_\dagger^2),
\end{equation*}
which proves part~2.
\end{proof}

\vspace{0.25cm}
The preceding theorem gives a limiting normal distribution with generally unknown variance. We now show that this variance can be replaced by an estimate computed from the acquired labels, while still yielding asymptotically valid confidence intervals. This allows us to obtain approximate confidence intervals for our plug-in PPAT estimator.
\begin{cor}[Studentised plug-in PPAT CLT]
\label{cor:plugin_ppat_studentized}
Assume the conditions of Thm.~\ref{thm:plugin_ppat_clt}, and in addition that
\begin{equation*}
\frac1{N_M}\sum_{i=1}^{N_M}|\ell_i|^4=O(1),
\qquad \frac1{N_M}\sum_{i=1}^{N_M}|c_{M,i}|^4=O(1).
\end{equation*}

For $\lambda \in \mathbb{R}$, define the residualised array $z_{M,i}(\lambda)=\ell_i-\lambda c_{M,i}$ with running average
\begin{equation*}
A_{M,m}(\lambda)=\frac1{N_M}
\left\{\frac{z_{M,I_{M,m}}(\lambda)}{Q_{M,m}(I_{M,m})} + \sum_{t=1}^{m-1}z_{M,I_{M,t}}(\lambda)\right\},
\end{equation*}
and let
\begin{equation*}
\widehat\sigma_M^2(\widehat\lambda_M)=\frac1M\sum_{m=1}^M \gamma_{M,m}^2 \left\{A_{M,m}(\widehat\lambda_M) - \widehat R_M^{\mathrm{PPAT}}(\widehat\lambda_M)\right\}^2.    
\end{equation*}

Then 
\begin{equation*}
\widehat\sigma_M^2(\widehat\lambda_M)\to_p\sigma_\dagger^2,
\end{equation*}
and hence
\begin{equation*}
\frac{\sqrt M\left\{\widehat R_M^{\mathrm{PPAT}}(\widehat\lambda_M)-R_M\right\}}
{\widehat\sigma_M(\widehat\lambda_M)}
\Rightarrow
N(0,1).
\end{equation*}
\end{cor}

\begin{proof}
Recall from Thm.~\ref{thm:plugin_ppat_clt} the oracle residualised array
$z_{M,i}^\dagger=\ell_i-\lambda_M^\dagger c_{M,i}$, together with the facts
$|\lambda_M^\dagger|=O(1)$ and $\widehat\lambda_M-\lambda_M^\dagger=o_p(1)$. Evaluated at the
deterministic oracle coefficient, $\widehat\sigma_M^2(\lambda_M^\dagger)$ is precisely the LURE
variance estimator for the fixed array $z^\dagger$. The variance-estimation result of Thm. \ref{thm:lure_variance_estimator} therefore applies: its moment condition holds because $|\lambda_M^\dagger|=O(1)$ and the
fourth-moment assumptions give
\begin{equation*}
\frac1{N_M}\sum_{i=1}^{N_M}|z_{M,i}^\dagger|^4=O(1),
\end{equation*}
This gives $\widehat\sigma_M^2(\lambda_M^\dagger)\to_p\sigma_\dagger^2$. 

It remains to replace $\lambda_M^\dagger$ with $\widehat\lambda_M$.  Both $A_{M,m}(\lambda)$ and
$\widehat R_M^{\mathrm{PPAT}}(\lambda)=\widehat\mu_M(\ell)-\lambda\widehat\mu_M(c)$ are affine in
$\lambda$, so
\begin{equation*}
A_{M,m}(\widehat\lambda_M)-\widehat R_M^{\mathrm{PPAT}}(\widehat\lambda_M) = \left\{A_{M,m}(\lambda_M^\dagger)-\widehat R_M^{\mathrm{PPAT}}(\lambda_M^\dagger)\right\}
-(\widehat\lambda_M-\lambda_M^\dagger)\left\{A_{M,m}(c)-\widehat\mu_M(c)\right\},
\end{equation*}
where $A_{M,m}(c)$ is the running average of the array $c_{M,\cdot}$. Substituting this identity into
$\widehat\sigma_M^2(\widehat\lambda_M)$ and expanding, every term other than
$\widehat\sigma_M^2(\lambda_M^\dagger)$ carries a factor of $\widehat\lambda_M-\lambda_M^\dagger$.
Since $\widehat\lambda_M-\lambda_M^\dagger=o_p(1)$, while $\widehat\sigma_M^2(\lambda_M^\dagger)$ and
the empirical quadratic average
\begin{equation*}
\frac1M\sum_{m=1}^M\gamma_{M,m}^2\left\{A_{M,m}(c)-\widehat\mu_M(c)\right\}^2
\end{equation*}
are both $O_p(1)$ (the latter by the fourth-moment condition on $c$), these terms are $o_p(1)$. Hence
\begin{equation*}
\widehat\sigma_M^2(\widehat\lambda_M)-\widehat\sigma_M^2(\lambda_M^\dagger)=o_p(1),
\qquad\text{and therefore}\qquad
\widehat\sigma_M^2(\widehat\lambda_M)\to_p\sigma_\dagger^2.
\end{equation*}
The studentised CLT then follows from Thm.~\ref{thm:plugin_ppat_clt} and Slutsky's theorem.
\end{proof}


\subsubsection{Discussion of Assumptions}
\label{sec:app:assumption-discussion}
\input{appendix/mathematical_results/asymptotic_assumptions}


%% file: appendix/mathematical_results/tightness_var_bound_v3.tex
Here, we discuss when the bound in Prop. \ref{prop:ppat_var_bound} is tight and, consequently, when we can expect $\lambda^\dagger$ to serve as a suitable alternative to $\lambda^\star$. In particular, we show that the tightness of the bound depends critically on how small $M/N$ is and how close the proposal is to being uniform.

We begin by noting that the bound is derived by chaining four separate inequalities to relax the exact variance. By tracking when each of these relaxations holds as an equality, we can 
better understand its overall tightness. As before, write $\zeta_i := \zeta_{M,i}(\lambda) = \ell_i - \lambda c_{M,i}$ for the residualised loss, $R_M = N^{-1}\sum_i \zeta_i$ for its pool mean, and $\sigma_\zeta^2 := N^{-1}\sum_i(\zeta_i - R_M)^2$ for its pool variance, so that $D(\lambda) = \sigma_\zeta^2 + R_M^2$. The proof passes through the following four relaxations:

\begin{enumerate}
    \item[(i)] \emph{Second-moment relaxation.} We bound the conditional variance of the round-$m$
    increment by its conditional second moment, which discards the (nonnegative) squared mean of
    the \emph{remaining} residuals, $N^{-2}\big(\sum_{i \in S_{M,m}} \zeta_i\big)^2$. 

    \item[(ii)] \emph{Overlap relaxation.} We replace $1/Q_{M,m}(i)$ by its worst-case value $N/\beta$ based on the uniform overlap condition. This is exact only where the proposal sits exactly at the overlap floor, $Q_{M,m}(i) = \beta/N$, for every point with $\zeta_i \neq 0$. The more concentrated (non-uniform) the proposal is, the looser this step becomes.

    \item[(iii)] \emph{Pool-extension relaxation.} We enlarge the sum over the remaining pool
    $S_{M,m}$ to a sum over the full pool $\{1,\dots,N\}$, discarding the squared residuals of the
    $m-1$ points already queried. This discarded mass is a vanishing fraction of the total whenever
    the number of acquired labels is small relative to the pool, $M \ll N$.

    \item[(iv)] \emph{Round-weight aggregation.} Finally, we bound $\sum_m \gamma_{M,m}^2$ by
    $\big(\max_m \gamma_{M,m}\big) \sum_m \gamma_{M,m} = \big(\max_m \gamma_{M,m}\big) M$. Since $\gamma_{M,m}$ is \emph{strictly increasing} in $m$, this step is strict for any $M \geq 2$: the later rounds contribute disproportionately more to the bound than to a hypothetical ``equal-weight'' aggregation. The size of this gap is controlled by the spread of the weights
    across rounds,
    \begin{equation*}
        \frac{\gamma_{M,M}}{\gamma_{M,1}} = \frac{N(N-1)}{(N-M)(N-M+1)} \; \longrightarrow \; 1
        \quad \text{as } M/N \to 0,        
    \end{equation*}
    i.e., when only a small fraction of the pool is queried, the weights $\gamma_{M,1},\dots,\gamma_{M,M}$
    are all $\approx 1$, 
    and this relaxation becomes negligible.
\end{enumerate}

It's not difficult to see that the above relaxations are governed by two separate sets of conditions. Firstly, the relaxations (ii)-(iv) are all controlled by the same two knobs -- \emph{how small $M/N$ is} and
\emph{how close the proposal is to uniform} -- and become negligible together as $M/N \to 0$ under a near-uniform proposal. Secondly, relaxation (i) depends only on whether the residualised losses are ``centred'', i.e. on the ratio $R_M^2/\sigma_\zeta^2$ between the squared pool mean and the pool variance of $\zeta_i$. 

In the settings that motivate our approach, we expect relaxation (i) to be negligible.
Its cost is controlled solely by the ratio $R_M^2/\sigma_\zeta^2$, which is small precisely
when the residualised losses are highly variable across the pool relative to their average magnitude, $\sigma_\zeta^2 \gg R_M^2$. This is exactly the high-variance regime in which active testing is beneficial. Thus, we expect our bound to be tight when the budget is a small fraction of the pool ($M \ll N$) and our proposal is close to uniform. In particular, we note that $\lambda^\star$ coincides with $\lambda^\dagger$ when our proposal is uniform.

%% file: appendix/mathematical_results/asymptotic_assumptions.tex
Thms.~\ref{thm:lure_l2_consistency}, \ref{thm:lure_clt} and \ref{thm:lure_variance_estimator} rely on a number of assumptions, which we now examine one at a time. Broadly, they fall into four groups: the finite--population asymptotic regime, two conditions on the active proposal (admissibility and uniform overlap), a family of moment conditions on the loss
array, and the stabilisation of the predictable quadratic variation. As we argue below, the first three are either standard, directly enforceable, or mild conditions on the fixed test pool that involve no distributional assumptions; only the last is genuinely substantive, which we verify through our empirical results. Throughout, recall that these results are stated for the general triangular array $\xi_{M,i}$ and specialise to LURE ($\xi_{M,i}=\ell_i$) and PPAT ($\xi_{M,i}=\zeta_{M,i}(\lambda)=\ell_i-\lambda(\tilde\ell_i-\tilde R_M)$).

\paragraph{Finite--population regime:}
All three results are stated in the finite--population asymptotic regime, in which the pool size $N_M$ and the budget $M$ grow together with
$N_M/M\to\alpha$ for some $\alpha>1$. This is the standard regime for
without--replacement sampling from a finite population \citep{farquhar2021statistical}, and it simply
says that we label a fixed fraction $M/N_M\to 1/\alpha\in(0,1)$ of the pool. The requirement $\alpha>1$ rules out the degenerate case in which we label (essentially) the entire pool.

\paragraph{Uniform overlap.}
This is a positivity (overlap) condition requiring that, at every round, each remaining point retains at least a fixed fraction of uniform sampling mass. It is the finite--population analogue of the overlap/positivity assumption that is ubiquitous in importance sampling, inverse--probability weighting, and
off--policy evaluation, and it plays exactly the same role here: it keeps the importance weights bounded.  Crucially, this condition is not merely plausible but is enforced by construction in our experiments, as described in \S\ref{sub:ppat_active_proposal}.

\paragraph{Moment conditions on the loss array.}
The three results require control of, respectively, the second, $(2+\delta)$--th,
and fourth finite--pool moments of the array:
\begin{equation*}
\underbrace{\tfrac{1}{N_M}\sum_{i=1}^{N_M}\xi_{M,i}^2=O(1)}_{\text{consistency (Thm. \ref{thm:lure_l2_consistency})}},
\qquad
\underbrace{\tfrac{1}{N_M}\sum_{i=1}^{N_M}|\xi_{M,i}|^{2+\delta}=O(1)}_{\text{normality (Thm.~\ref{thm:lure_clt})}},
\qquad
\underbrace{\tfrac{1}{N_M}\sum_{i=1}^{N_M}|\xi_{M,i}|^{4}=O(1)}_{\text{variance/CIs (Thm.~\ref{thm:lure_variance_estimator})}}.
\end{equation*}
The single most important point about these conditions is that they are
statements about the empirical moments of a fixed, deterministic collection of
losses: because we work conditionally on the test pool, they involve no
distributional or i.i.d.\ assumption on the data whatsoever. They simply require
that the average magnitude of the (residualised) losses does not blow up as the
pool grows.

For PPAT, the array $\xi_{M,i}=\zeta_{M,i}(\lambda)$ is a linear combination of the true loss $\ell_i$ and the centred proxy loss $c_{M,i}=\widetilde\ell_i-\widetilde R_M$. Since centring does not increase moments beyond a constant factor, each of the above conditions reduces to the same bounded--moment requirement on the true losses $\ell_i$ and the proxy losses $\tilde\ell_i$ separately -- for instance, $N_M^{-1}\sum_i\ell_i^4=O(1)$ and $N_M^{-1}\sum_i\tilde\ell_i^4=O(1)$ for the fourth--moment condition. These are easy to satisfy. For any \emph{bounded} loss -- for example the $0$--$1$ loss, or any loss on a compact label
space -- all finite--pool moments are automatically bounded and the conditions hold trivially. For \emph{unbounded} losses such as squared error or cross--entropy, they hold provided the losses do not develop pool--growing heavy tails, i.e.\ provided no vanishingly small subset of points carries losses large enough to dominate the average $p$--th power; in practise, this is a weak requirement.

\paragraph{Stabilisation of the predictable quadratic variation.}
The asymptotic normality and variance--estimation results (Thms.~\ref{thm:lure_clt} and \ref{thm:lure_variance_estimator}) additionally require that the predictable quadratic variation $\sigma_M^2=\frac1M\sum_{m=1}^M\gamma_{M,m}^2\,s_{M,m}^2$ converge in probability to a finite, strictly positive constant $\sigma^2$. This is the one genuinely non--trivial assumption, which we now discuss below.

Firstly, we note that it is the standard stabilisation
condition for a martingale CLT \citep[Chapter 3]{Hall1980MartingaleLT}. The centred, rescaled error
$\sqrt M(\widehat R_M-R_M)$ is a sum of martingale differences, and the fluctuations of such a sum are governed by its accumulated conditional variance $\sigma_M^2$;
the CLT holds precisely when this accumulated variance stabilises. 

Secondly, the finiteness of $\sigma^2$ is essentially free: combining the overlap bound with the second--moment condition gives
$s_{M,m}^2\le\beta^{-1}\,N_M^{-1}\sum_i\xi_{M,i}^2$, and since $\gamma_{M,m}$ is uniformly bounded , $\sigma_M^2$ is uniformly bounded above. Moreover, the positivity of $\sigma^2>0$  asks that the estimator have genuinely non--degenerate fluctuations at the $\sqrt M$ scale. It fails only in degenerate situations, for example if the residualised losses became asymptotically constant across the pool. The convergence requirement, however, is more substantive and 
generally does not follow from pool--level moment bounds alone, except in simpler cases such as non-adaptive proposals or random sampling without replacement. While it is difficult to establish this convergence for arbitrary fully adaptive proposals, our coverage experiments in \S\ref{sub:coverage} provide indirect empirical support for it: the PPAT intervals attain the nominal coverage level, and do so at least as quickly as, and often faster than, the competing methods. Such behaviour would be difficult to achieve if $\sigma_M^2$ failed to stabilise. We also note that while the random and LURE baselines sometimes fall short of nominal coverage at the budgets considered, we expect this to improve with an increased labelling budget.